\pdfoutput=1
\documentclass[11pt]{article}
\usepackage[letterpaper,margin=1in]{geometry}

\usepackage{lmodern}

\usepackage{lipsum}

\usepackage{caption}
\usepackage{comment}

\usepackage[utf8]{inputenc} 
\usepackage[T1]{fontenc}    

\usepackage{url}            
\usepackage{booktabs}       
\usepackage{nicefrac}       
\usepackage{microtype}      
\usepackage{enumitem}
\usepackage{dsfont}
\usepackage{graphicx}
\usepackage{multicol}
\usepackage{stmaryrd,eucal,verbatim,mathtools,xcolor,tikz}
\usetikzlibrary{arrows.meta, automata, positioning}
\usepackage{amsthm,amsfonts,amsmath,amssymb,amscd,epsfig,color,float,graphicx,verbatim, enumitem, subfigure}
\usepackage{wasysym}
\usepackage{algorithm,algorithmicx}
\usepackage[ruled, vlined, algo2e]{algorithm2e}
\usepackage[noend]{algpseudocode}

\usepackage{mathtools}
\usepackage{bbm}

\usepackage{thmtools,thm-restate}
\usepackage{turnstile}

\usepackage{changepage}

\usepackage[colorlinks,citecolor=blue,linkcolor=magenta,bookmarks=true,hypertexnames=false]{hyperref}
\usepackage[nameinlink]{cleveref}

\crefname{equation}{equation}{equations}
\crefname{lemma}{lemma}{lemmata}
\crefname{claim}{claim}{claims}
\crefname{theorem}{theorem}{theorems}
\crefname{proposition}{proposition}{propositions}
\crefname{corollary}{corollary}{corollaries}
\crefname{claim}{claim}{claims}
\crefname{remark}{remark}{remarks}
\crefname{definition}{definition}{definitions}
\crefname{fact}{fact}{facts}
\crefname{question}{question}{questions}
\crefname{condition}{condition}{conditions}
\crefname{algorithm}{algorithm}{algorithms}
\crefname{assumption}{assumption}{assumptions}
\crefname{problem}{problem}{problems}

\newtheorem{theorem}{Theorem}[section]
\newtheorem{lemma}[theorem]{Lemma}
\newtheorem{proposition}[theorem]{Proposition}
\newtheorem{corollary}[theorem]{Corollary}
\newtheorem{claim}[theorem]{Claim}

\newtheorem{definition}[theorem]{Definition}
\newtheorem{observation}{Observation}[section]
\newtheorem{example}{Example}
\newtheorem{fact}[theorem]{Fact}

\theoremstyle{definition}

\newtheorem{remark}[theorem]{Remark}
\newtheorem{problem}[theorem]{Problem}

\newcommand{\cons}{\text{\cons}}

\newcommand{\eps}{\varepsilon}
\newcommand{\huber}{\text{huber}}

\newcommand{\poly}{\mathrm{poly}}
\newcommand{\polylog}{\mathrm{polylog}}

\def\D{\mathcal D}

\def\R{\mathbb R}

\def\N{\mathbb N}

\def\Z{\mathbb Z}

\newcommand{\citep}{\cite}
\newcommand{\citet}{\cite}

\newcommand{\cB}{\mathcal{B}}
\newcommand{\cC}{\mathcal{C}}
\newcommand{\cD}{\mathcal{D}}

\newcommand{\cN}{\mathcal{N}}

\newcommand{\cS}{\mathcal{S}}

\newcommand*{\pE}{\widetilde{E}}

\newcommand{\Paren}[1]{\left(#1\right)}

\newcommand{\abs}[1]{\left\vert#1\right\vert}

\newcommand{\norm}[1]{\lVert#1\rVert}
\newcommand{\Norm}[1]{\left\lVert#1\right\rVert}

\newcommand{\normt}[1]{\norm{#1}_2}

\newcommand{\normo}[1]{\norm{#1}_1}

\newcommand{\linfty}[1]{\left\vert{#1}\right\vert_\infty}

\newcommand{\iprod}[1]{\langle#1\rangle}

\newcommand{\hide}[1]{}
\DeclareMathOperator{\supp}{supp}

\DeclareMathOperator*{\E}{E}

\DeclarePairedDelimiter\ceil{\lceil}{\rceil}
\DeclarePairedDelimiter\floor{\lfloor}{\rfloor}

\newcommand{\vsos}{\mathrm{v}}

\newcommand{\dsos}{D}
\newcommand{\Dref}{\D_{ref}}
\newcommand{\Dpl}{\D_{pl}}

\newcommand{\totalsize}{\mathrm{total}}

\newcommand{\cu}{C_U}
\newcommand{\cl}{C_L}

\newcommand{\Iso}{I_{s}}

\newcommand{\intset}{\mathcal{P}_{\gamma_j,\ldots,\gamma_1,\tau,{\gamma'}^{\top}_1,\ldots,{\gamma'}^{\top}_j}}

\newcommand\intsetoflength[1]{\mathcal P_{\mathrm{length}={#1}}}

\newcommand{\im}{\text{Im}}

\renewcommand{\circle}{\mathrm{\Circle}}
\newcommand{\circled}[1]{\textup{\textcircled{$#1$}}}
\makeatletter
\newcommand{\squared}[1]{\text{\fboxsep=.2em\fbox{\m@th$\displaystyle#1$}}}
\makeatother

\newcommand{\Id}{\mathrm{Id}}

\newcommand{\pl}{\mathrm{pl}}

\newcommand{\diag}{\mathrm{diag}}

\newcommand{\Sum}{\sum\limits}
\newcommand{\Prod}{\prod\limits}

\newcommand{\Sminwt}{S_{min}}

\newcommand{\almost}[1]{\stackrel{#1}{\approx}}

\newcommand\restr[2]{{
  \left.\kern-\nulldelimiterspace 
  #1 
  \littletaller 
  \right|_{#2} 
  }}
  \newcommand{\littletaller}{\mathchoice{\vphantom{\big|}}{}{}{}}

\renewcommand{\ss}{\text{SS}}
\renewcommand{\SS}[1]{{#1}_{\ss}}
\newcommand{\ssd}{\text{SSD}}
\newcommand{\SSD}[1]{{#1}_{\ssd}}

\newcommand{\can}{\mathrm{can}}

\newcommand{\truncation}{D_{trunc}}

\newcommand{\extraintersect}{\mathrm{ext}}

\newcommand{\mset}[1]{\{\!\!\{#1\}\!\!\}}

\newcommand{\parity}{\mathrm{parity}}

\newcommand{\resultP}{{\tau_{\mathcal{P}}}}
\newcommand{\Anorm}{{A_{\mathrm{norm}}}}
\newcommand{\slack}{\mathrm{slack}}
\newcommand{\dside}{D_{\mathrm{side}}}
\newcommand{\wellbehaved}{{\mathrm{well-behaved}}}
\newcommand{\reduced}{{\mathrm{\hspace{1pt} red}}}
\newcommand{\defaultslack}{\mathrm{slack}^{\mathrm{default}}}
\newcommand{\extraslack}{\mathrm{slack}^{\mathrm{extra}}}
\newcommand{\assigned}{{\mathrm{assigned}}}

\newcommand{\SALD}{\mathrm{SA}_{\dsos}}
\newcommand{\SALDprod}{\star}
\newcommand{\reppre}[3]{\rhopre\left(#1,#2;#3\right)}

\newcommand{\sinv}{\star\text{-inv}}
\newcommand{\rhopre}{\rho_{\mathrm{pre}}}

\newcommand{\target}{\mathrm{target}}
\newcommand{\hn}{h^{\text{normal}}}

\newcommand{\wbp}{*_{\mathrm{wb}}}

\def\colorful{1}
\ifnum\colorful=1

\else

\fi

\newcommand{\Cuniv}{{C_{\mathrm{univ}}}}

\newcommand{\hw}{{\mathrm{hw}}}

\allowdisplaybreaks

\begin{document}

\title{Sum-of-Squares Lower Bounds \\ 
for Non-Gaussian Component Analysis}

\date{}

\author{
Ilias Diakonikolas\thanks{University of Wisconsin-Madison. Email: ilias@cs.wisc.edu. Supported by NSF Media Award \# 2107547.} \hspace{0.5cm}  
Sushrut Karmalkar\thanks{University of Wisconsin-Madison. Email: s.sushrut@gmail.com. Supported by NSF under Grant \#2127309 to the Computing Research Association for the CIFellows 2021 Project.} \hspace{0.5cm} 
Shuo Pang\thanks{University of Copenhagen. Email: shpa@di.ku.dk. Funded by the European Union MSCA Postdoctoral Fellowships 2023 project 101146273 NoShortProof. Views expressed are the authors' and do not reflect the European Union or the Research Executive Agency.} \hspace{0.5cm} 
Aaron Potechin\thanks{University of Chicago. Email: potechin@uchicago.edu.  Supported by NSF grant CCF-2008920.}
}
\maketitle
	
\begin{abstract}%
Non-Gaussian Component Analysis (NGCA) is the statistical task of finding a non-Gaussian direction in a high-dimensional dataset. Specifically, given i.i.d.\ samples from a distribution $P^A_{v}$ on $\mathbb{R}^n$ that behaves like a known distribution $A$ in a hidden direction $v$ and like a standard Gaussian in the orthogonal complement, the goal is to 
approximate the hidden direction. The standard formulation posits that 
the first $k-1$ moments of $A$ match those of the standard Gaussian 
and the $k$-th moment differs. Under mild assumptions, this problem has sample complexity $O(n)$. On the other hand, all known efficient algorithms require $\Omega(n^{k/2})$ samples. Prior work developed sharp Statistical Query and low-degree testing lower bounds suggesting an information-computation tradeoff for this problem. 

Here we study the complexity of NGCA in the Sum-of-Squares (SoS) framework. Our main contribution is the first super-constant degree SoS lower bound for NGCA. Specifically, we show that if the non-Gaussian distribution $A$ matches the first $(k-1)$ moments of $\cN(0, 1)$ and satisfies other mild conditions, then with fewer than $n^{(1 - \eps)k/2}$ many samples from the normal distribution, with high probability, degree $(\log n)^{{1\over 2}-o_n(1)}$ SoS fails to refute the existence of such a direction $v$. Our result significantly strengthens prior work by establishing a super-polynomial information-computation tradeoff against a broader family of algorithms. As corollaries, we obtain 
SoS lower bounds for several problems in robust statistics and the learning of mixture models.

Our SoS lower bound proof introduces a novel technique, that we believe may be of broader interest, and a number of refinements over existing methods. As in previous work, we use the framework of [Barak et al. FOCS 2016], where we express the moment matrix $M$ as a sum of graph matrices, find a factorization $M\approx LQL^T$ using minimum vertex separators, and show that with high probability $Q$ is positive semidefinite (PSD) while the errors are small. Our technical innovations are as follows. First, instead of the minimum weight separator used in prior work, we crucially make use of the minimum square separator. Second, proving that $Q$ is PSD poses significant challenges due to an intrinsic reason. In all prior work, the major part of $Q$ was always a constant term, meaning a matrix whose entries are constant functions of the input. Here, however, even after removing a small error term, $Q$ remains a nontrivial linear combination of non-constant, equally dominating terms. We develop an algebraic method to address this difficulty, which may have wider applications. Specifically, we model the multiplications between the ``important'' graph matrices by an $\mathbb R$-algebra, construct a representation of this algebra, and use it to analyze $Q$. Via this approach, we show that the PSDness of $Q$ boils down to the multiplicative identities of Hermite polynomials.
\end{abstract}

\thispagestyle{empty}
\newpage

\thispagestyle{empty}
\tableofcontents
\thispagestyle{empty}
\newpage

\setcounter{page}{1}


\section{Introduction}

Non-Gaussian Component Analysis (NGCA) is a statistical estimation task first considered in the signal processing literature~\cite{JMLR:blanchard06a} and subsequently extensively studied (see, e.g.,~Chapter 8 of~\cite{DK23-book} and references therein). As the name suggests, the objective of this task is to find a non-Gaussian direction  (or, more generally, low-dimensional subspace) in a high-dimensional dataset. Since its introduction, the NGCA problem has been studied in a range of works from an algorithmic standpoint; see~\cite{theis2011uniqueness, sugiyama2008approximating,diederichs2010sparse, diederichs2013sparse,sasaki2016non,virta2016projection,TanV18,GoyalS19,dudeja2024statistical,cao2023contrastive}. Here we explore this problem from a hardness perspective with a focus on Sum-of-Squares algorithms. 

The standard formulation of NGCA is the following. Fix a univariate distribution $A$. For a unit vector direction $v$, let $P_v^{A}$ be the distribution on $\R^n$ defined as follows: The projection of $P_v^{A}$ in the $v$-direction is equal to $A$, and its projection in the orthogonal complement is the standard Gaussian distribution. Observe that $P_v^{A}$ is a product distribution with respect to a non-standard coordinate system. It is further assumed that, for some parameter $k$, the first $k-1$ moments of the univariate distribution $A$ match those of the standard Gaussian $\cN(0, 1)$ and the $k$-th moment differs by a non-trivial amount. Given i.i.d.\ samples from a distribution $P_v^{A}$, for an unknown $v$, the goal is to estimate the hidden direction $v$. It is known that, under mild assumptions on the distribution $A$, this problem has sample complexity $O(n)$. Unfortunately, all known methods to achieve this sample upper bound run in time exponential in $n$, by essentially using brute-force over a cover of the unit sphere to identify the hidden direction. On the other hand, if we have $\gg n^{k/2}$ samples, a simple spectral algorithm (on the $k$-th moment tensor) solves the problem in sample-polynomial time (see e.g., \cite{dudeja2024statistical}). A natural question is whether more sample-efficient polynomial-time algorithms exist or if the observed gap is inherent---i.e., the problem exhibits a statistical-computational tradeoff. As our main result, we show (roughly speaking) that the gap is inherent for SoS algorithms of degree $o(\sqrt{\frac{\log n}{\log\log n}})$.

In addition to being interesting on its own merits, further concrete motivation to understand the hardness of NGCA comes from its applications to various well-studied learning problems. Specifically, the NGCA problem captures interesting (hard) instances of several statistical estimation problems that superficially appear very different. The idea is simple: Let $\Pi$ be a statistical estimation task. It suffices to find a univariate distribution $A_{\Pi}$ such that for any direction $v$ the high-dimensional distribution $P_v^{A_{\Pi}}$ is a {\em valid instance} of problem $\Pi$. Solving $\Pi$ then requires solving NGCA on these instances. We provide two illustrative examples below.

\paragraph{Example 1: Robust Mean Estimation}
Consider the following task, known as (outlier-)robust mean estimation: Given i.i.d.\ samples from a distribution $D$ on $\R^n$ such that $\mathrm{d}_{\mathrm{TV}}(D, \cN(\mu, I)) \leq \eps$ for some small $\eps>0$, the goal is to approximate the mean vector $\mu$ in $\ell_2$-norm. Suppose that $A$ is an $\eps$-corrupted one-dimensional Gaussian distribution in total variation distance, namely a distribution that satisfies $\mathrm{d}_{\mathrm{TV}}(A, G) \leq \eps$ where $G \sim \cN(\delta, 1)$ for some $\delta \in \R$. For any unit vector $v$, the distribution $P^A_v$ is an $\eps$-corrupted Gaussian on $\R^n$, i.e., $\mathrm{d}_{\mathrm{TV}}(P^A_v, \cN(\delta v, \Id)) \leq \eps$. It is then easy to see that the NGCA task on this family of $P^A_v$'s is an instance of robust mean estimation. Namely, approximating $v$ is equivalent to approximating the target mean vector (once we know $v$, the one-dimensional problem of estimating $\delta$ robustly is easy).

\paragraph{Example 2: Learning Mixtures of Gaussians}
A $k$-mixture of Gaussians (GMM) on $\R^n$ is a convex combination of Gaussians, i.e., a distribution of the form $F(x) = \sum_{i=1}^k w_i \cN(\mu_i, \Sigma_i)$ where $\sum_{i=1}^k w_i = 1$. The prototypical learning problem for GMMs is the following: Given i.i.d.\ samples from an unknown GMM, the goal is to learn the underlying distribution in total variation distance (or, more ambitiously, approximate its parameters). Suppose that $A$ is a $k$-mixture of {\em one-dimensional} Gaussians $\sum_{i=1}^k w_i \cN(t_i, \delta^2)$ and the $t_i$ are chosen to be sufficiently far apart such that each pair of components has small overlap. For any unit vector $v$, the distribution $P^A_v$ is a mixture of $k$ Gaussians on $\R^n$ of the form $\sum_{i=1}^k w_i \cN(t_i v, \Id - (1-\delta^2) vv^T)$. For small $\delta$, this can be thought of as $k$ ``parallel pancakes'', in which the means lie in the direction $v$. 

All $n - 1$ orthogonal directions to $v$ will have an eigenvalue of $1$, which is much larger than the smallest eigenvalue in this direction (which is $\delta$). In other words, for each unit vector $v$, the $k$-GMM $P^A_{v}$ will consist of $k$ ``skinny'' Gaussians whose mean vectors all lie in the direction of $v$. Once again,  the NGCA task on this family of $P^A_v$'s is an instance of learning GMMs: once the direction $v$ is identified, the corresponding problem collapses to the problem of learning a one-dimensional mixture, which is easy to solve. 

By leveraging the aforementioned connection, hardness of NGCA can be used to obtain similar hardness for a number of well-studied learning problems that superficially appear very different. These include learning mixture models~\cite{diakonikolas2017statistical, DiakonikolasKPZ23, DKS23}, robust mean/covariance estimation~\cite{diakonikolas2017statistical}, robust linear regression~\cite{DKS19}, learning halfspaces and other natural concepts with adversarial or semi-random label noise~\cite{DKZ20, GoelGK20, DK20-Massart-hard, DiakonikolasKPZ21, DKKTZ21-benign, tiegel2024improved}, list-decodable mean estimation and linear regression~\cite{DKS18-list, DKPPS21}, learning simple neural networks~\cite{DiakonikolasKKZ20, GoelGJKK20}, and even learning simple generative models~\cite{Chen0L22}. Concretely, to achieve this it suffices to find a distribution $A$ {\em of the required form} that matches as many moments with the standard Gaussian as possible.

\paragraph{Hypothesis testing version of NGCA}

Since we are focusing on establishing hardness, we will consider the natural hypothesis testing version of NGCA, noticing that the learning version of the problem typically reduces to the testing problem. Specifically, our goal is to distinguish between a standard multivariate Gaussian and the product distribution that is equal to a pre-specified univariate distribution $A$ in a hidden direction $v$ and is the standard Gaussian in the orthogonal complement.

\begin{problem}[Non-Gaussian Component Analysis, Testing Version]\label{def:ngca_distinguishing}

Let $A$ be a one-dimensional distribution that matches the 1st to the $(k-1)$ moments with $\cN(0, 1)$. Given $m$ i.i.d. samples $ \{ x_1, \dots, x_m  \} \subseteq \R^n$ drawn from one of the following two distributions, the goal is to determine which one generated them.
\begin{enumerate}
    \item (Reference distribution, $\Dref$) The true $n$-dimensional multivariate normal distribution $\cN(0, \Id_n)$.
    \item (Planted distribution, $\Dpl$) Choose $v \in \{ \pm 1/\sqrt{n}\}^n$ uniformly at random (called the hidden/planted direction) and draw $x'\sim\cN(0, \Id_n)$, then we take $x = x' - \langle x', v\rangle v + {a}v$, where $a\sim A$, to be the result. In other words, the distribution is $\cN(0,\Id_{n-1})_{{ v}^\perp}\times A_{ v}$ where $ v$ is chosen uniformly at random from $\{\pm 1/\sqrt{n}\}^n$.
\end{enumerate}
\end{problem}

We use Boolean planted directions $v \in \{ \pm 1/\sqrt{n}\}^n$ in Definition \ref{def:ngca_distinguishing} for technical convenience (cf. the calculation in \Cref{lem:calib}). Lower bounds in this setting imply, in particular, that NGCA is hard w.r.t. an adversarial distribution of $v$.
\paragraph{Prior Evidence of Hardness}

Prior work~\cite{diakonikolas2017statistical} established hardness of NGCA in a restricted computational model, known as the Statistical Query (SQ) model (see also \cite{diakonikolas2023sq} for a recent refinement). SQ algorithms are a class of algorithms that are allowed to query expectations of bounded functions on the underlying distribution through an SQ oracle rather than directly access samples. The model was introduced by Kearns~\cite{Kearns:98} as a natural restriction of the PAC model~\cite{Valiant:84} in the context of learning Boolean functions. Since then, the SQ model has been extensively studied in a range of settings, including unsupervised learning~\cite{FeldmanGRVX17}.

\cite{diakonikolas2017statistical} gave an SQ lower bound for NGCA 
under the moment-matching assumption and an additional regularity 
assumption about the distribution $A$ (which was removed in~\cite{diakonikolas2023sq}). Intuitively, the desired 
hardness result amounts to the following statistical-computational trade-off: Suppose that $A$ matches its first $k-1$ moments with the standard Gaussian. Then any SQ algorithm for the hypothesis testing version of NGCA requires either $n^{\Omega(k)}$ samples, where $n$ is the ambient dimension, or super-polynomial time in $n$.

Lower bounds on NGCA have also been shown for the low-degree testing framework where we use a low-degree polynomial to distinguish\footnote{More precisely, we want to find a low-degree polynomial which has large expected value under the planted distribution but mean $0$ and variance $1$ under the reference distribution} between the reference and planted distributions. Ghosh et al. \cite{ghosh2020sum} implicitly showed a low-degree testing lower bound for NGCA\footnote{As noted in their Remark 5.9, Attempt 1 for the proof of Lemma 5.7 there shows that w.h.p. $\tilde{E}[1] = 1 \pm{o(1)}$, which is equivalent to a low-degree polynomial lower bound} and Mao and Wein \cite{mao2022optimal} directly showed a low-degree testing lower bound for an essentially equivalent problem. Such a lower bound can also be deduced from the result of Brennan et al. \cite{brennan2020statistical} that under certain conditions (which are satisfied by NGCA), SQ algorithms are essentially equivalent to low-degree polynomial tests.

While the SQ and low-degree testing models are quite broad, they do not in general capture the class of algorithms obtained via convex relaxations. With this motivation, in this work we focus on establishing lower bounds for NGCA for the Sum-of-Squares hierarchy (SoS). We remark that SoS lower bounds that are proved via pseudo-calibration subsume low-degree testing lower bounds~\cite{hopkins2018statistical} so since we prove our SoS lower bound via pseudo-calibration, our SoS lower bound subsumes both low-degree testing lower bounds and SQ lower bounds (via the connection of \cite{brennan2020statistical}).

\paragraph{Informal Main Result}
A necessary condition for this testing problem to be computationally hard is that the univariate distribution $A$ matches its low-degree moments with the standard Gaussian. At a high-level, our main contribution is to prove that this condition is also {\em sufficient} in the SoS framework (subject to mild additional conditions). Here, we state an informal version of our main theorem, \Cref{thm:main-formal}.

\begin{theorem}[Main Theorem, Informal]\label{thm:main-informal}
Given $n$, suppose $2\leq k\leq (\log n)^{O(1)}$ and $A$ is a distribution on $\R$ such that: 
\begin{enumerate}
    \item  (Moment matching) The first $k-1$ moments of $A$ match those of $\cN(0, 1)$. 
    \item  (Moment bounds) $|\E_A[h_t(x)]| \!\leq\! (\log n)^{O(t)}$ for all Hermite polynomials of degree up to $2k (\log n)^2$, and $\frac{\E_A [p^2(x)]}{\E_{x\sim\cN(0,1)}[p^2(x)]} \!\geq\! (\log n)^{-O\big(\deg(p)\big)}$ for all nonzero polynomial $p(x)$ of degree up to $\sqrt{\log n}$.
\end{enumerate}
If $n$ is sufficiently large, then with high probability, given fewer than $n^{(1-\eps)k/2}$ many samples, Sum-of-Squares of degree $o(\sqrt{\frac{\eps\log n}{ \log\log n}})$ fails to distinguish between the random and planted distributions for the corresponding NGCA \Cref{def:ngca_distinguishing}.\footnote{As is usual for SoS lower bounds for average-case problems, technically what we show is that if we apply pseudo-calibration, the moment matrix is PSD with high probability.}
\end{theorem}

The SoS algorithms we consider are semi-definite programs whose variables are all degree $\leq D$ monomials in $\vsos$ which represents the unknown planted direction. The constraints are ``$\vsos_i^2=1/n$'' and that for any low-degree polynomial $p(\cdot)$, the average value of $p$ evaluated on the inner product between $\vsos$ and the samples should be reasonably close to $\E_{x\sim A}[p(x)]$. As a corollary of Theorem \ref{thm:main-informal}, any such algorithm requires either a large number of samples ($\geq n^{(1-\eps)k/2}$ many) or super-polynomial time to solve the corresponding NGCA. Note that the bound here is sharp, since with $O(n^{k/2})$ samples it is possible to efficiently solve the problem (see, e.g.,~\cite{dudeja2024statistical}).

\begin{remark}
We highlight that NGCA can be viewed as a ``meta-problem'', parameterized by the ``structure'' of the one-dimensional distribution $A$, which captures hard instances of a wide variety of learning problems. Our main contribution is to establish SoS-hardness of NGCA for {\em any} moment-matching distribution $A$ (under mild conditions). This is a powerful result for showing SoS-hardness for other learning problems via reductions. For certain special cases of $A$---specifically when $A$ is (essentially) a mixture of Gaussians---there exists reduction-based hardness for the problem under cryptographic assumptions (namely, the sub-exponential hardness of LWE) \cite{bruna2021continuous, gupte2022continuous}. However, these reductions are tailored to that specific choice of $A$. For other choices of $A$, no such reduction-based hardness is known, and it appears that LWE may not be the right starting point. For all other applications in this work, with the exception of learning GMMs, the only prior evidence of hardness was from the aforementioned SQ or low-degree testing lower bounds.
\end{remark}

\begin{remark} \label{rem:LLL}
It is worth noting that for the special case where the distribution 
$A$ is discrete, recent works~\cite{DK22LLL, ZSWB22LLL} showed polynomial-time algorithms for this version of the problem with sample complexity $O(n)$, regardless of the number of matched moments. Such a result is not surprising, as these algorithms are based on the LLL-method for lattice basis reduction which is not captured by the SoS framework. Importantly, these algorithms are extremely fragile and dramatically fail if we add a small amount of ``noise'' to $A$.
\end{remark}

\paragraph{Applications to Robust Statistics and Mixture Models}

Our main result (\Cref{thm:main-informal}) implies information-computation tradeoffs in the SoS framework for a range of fundamental problems in learning theory and robust statistics. (See \Cref{tab:sample_complexity} for a description of the problems we consider and the guarantees we obtain.) For the problems we consider, SQ and low-degree testing lower bounds were previously known. 

At a high level, for all our problems, our SoS lower bounds follow using the same template: we show that for \emph{specific} choices of the distribution $A$, the problem NGCA is an instance of a hypothesis testing problem known to be efficiently reducible to the learning problem in question. As long as the one-dimensional moment-matching distribution $A$ in question satisfies the hypotheses required for our main theorem to hold, we directly obtain an SoS lower bound for the corresponding hypothesis testing problem. 

As an illustrative example, we explain how to reduce the hypothesis testing version of special NGCA instances to the problem of learning a mixture of $k$ Gaussians in $n$ dimensions. Consider the distribution $A = \sum_{k=1}^n \frac 1 k~ \cN(\mu_i, \sigma_i^2)$, which is a mixture of Gaussians, and let the planted distribution be given by $\cN(0, \Id_{n-1})_{v^\perp} \times A_v$ where the hidden direction $v$ is from $\{\pm \frac{1}{\sqrt n} \}^n$. Expanding the expression, we see that the hypothesis testing problem is exactly to distinguish a true Gaussian from the mixture $\sum_{i=1}^k \frac 1 k ~\cN(\mu_i v, \Id - (1-\sigma_i^2)vv^T)$.  

\Cref{tab:sample_complexity} gives a list of problems where this work shows SoS lower bounds. 
It compares the information-theoretic sample complexity (the minimum sample size achievable by any algorithm) with the ``computational'' sample complexity implied by our SoS lower bound. 

\begin{table}[t]
\centering
\textbf{Sample Complexity versus Computational Sample Complexity}\\[6pt] 
\begin{tabular}{@{}l|c|c@{}}
\toprule
Statistical Estimation Task                      & Information-Theoretic & Degree-$O(\sqrt{\frac{\eps\log n}{\log\log n}})$ SoS
\\ 
\midrule
RME ($\Sigma \preceq \Id$) to $\ell_2$-error $\Omega(\sqrt \tau)$      & $O(n)$             & $\Omega(n^{2(1-\eps)})$ \\ \hline
RME ($\Sigma = \Id$) to $\ell_2$-error $\Omega(\frac{\tau \log(1/\tau)^{1/2}}{k^2})$       & $O(n)$             & $\Omega(n^{k(1-\eps)/2})$ \\ \hline
List-decodable Mean Estimation to error $O(\tau^{-1/k})$  & $O(n)$         & $\Omega(n^{k(1-\eps)/2})$ \\ \hline
RCE (multiplicative) to constant error             & $O(n)$             & $\Omega(n^{2(1-\eps)})$ \\ \hline
RCE (additive) to spectral error $O(\frac{\tau \log(1/\tau)}{k^4})$                   & $O(n)$             & $\Omega(n^{k(1-\eps)/2})$ \\ \hline
\raisebox{0pt}[2.5ex][1.5ex]{Estimating $k$-GMM}                 &   $\widetilde{O}(kn)$         & $\Omega(n^{k(1-\eps)})$ \\ \hline
Estimating 2-GMM (common unknown covariance)                 & $O(n)$         & $\Omega(n^{2(1-\eps)})$ \\ \bottomrule
\end{tabular}

\caption{A contrast between the information-theoretic sample complexity and the sample complexity required by degree-$O(\sqrt{\eps\log n/\log\log n})$-SoS for a range of natural tasks in robust statistics and learning mixture models. This includes robust mean estimation (RME), robust covariance estimation (RCE), learning Gaussian mixture models, and list-decodable mean estimation. The parameter $\tau$, when it appears, is related to the 
proportion of contamination; see \Cref{sec:concrete_apps}.}\label{tab:sample_complexity}
\end{table}

\subsection{Technical Overview of the Lower Bound}
In this section, we provide a brief high-level overview of our lower bound proof.

\paragraph{Pseudo-calibration technique and graph matrices.} 
We employ the general technique of pseudo-calibration as introduced in \cite{BHKKMP16} to produce a suitable candidate SoS solution $\pE$. For a given degree $D$, this solution can be described by a moment matrix indexed by sets $I, J \subseteq [n]$ where $|I|,|J| \leq \dsos$. The matrix entries are defined as $M_{\pE}(I, J) := \pE(\vsos^{I+J})$, with $\vsos^{I+J}$ representing the monomial $\prod_{i=1}^n \vsos_i^{I(i)+J(i)}$. These entries are functions of the input $x_1, \dots, x_m \in \R^n$ and are expressed in terms of Hermite polynomials (see \Cref{def:pseudoexpectation} and \Cref{eq:calib-calc}). 

As in most SoS lower bounds, the most challenging part is to show that $M_{\pE}$ is PSD. We provide an overview of the new ideas required for the proof in the proceeding discussion. Similar to prior works such as \cite{BHKKMP16,ghosh2020sum,potechin2020machinery,jones2021sum, pang2021sos, jones2023sum}, we expand $M_{\pE}$ as a linear combination of special matrices called graph matrices \cite{ahn2021graphmatricesnormbounds}, whose spectral norm we can bound in terms of combinatorial properties of the underlying {\it shapes} (\Cref{thm:norm_control})\footnote{A shape is, roughly speaking, a graph plus two distinguished vertex subsets called the left and right sides. The two sides are used to identify rows and columns of the associated matrix.}. Using graph matrices, we carefully factorize $M$ as $M = LQL^{\top} + \text{(error terms)}$ where $M$ is $M_{\pE}$ rescaled for technical convenience, thus reducing the task to showing that $Q \succ 0$. Here, the construction of matrices $L,Q$ in the factorization uses a recursive procedure like in previous works, where we repeatedly use minimum vertex separators of a shape to decompose the shape, and hence the associated graph matrix, in a canonical way; see  \Cref{def:recursivefactorization}. 

\paragraph{Minimum square separators.} 
The first technical novelty in this work is the introduction of the {\it minimum square vertex separators} in the factorization of $M$. Rather than using the minimum weight vertex separator or the sparse minimum vertex separator as in previous works, we define this new concept for bipartite graphs with two types of vertices—--circles and squares--—which naturally arise in our analysis of the NGCA problem (\Cref{def:separators}). 

Choosing the correct notion of vertex separators is a crucial first step in our analysis. This is because the combinatoroial properties of minimum square separators and minimum weight separators are key to controlling the norms of all the error terms generated in the resulting factorization $M\approx LQL^\top$ (see \Cref{sec:error_analysis}), which we will use throughout our analysis. Importantly, the use of minimum square separators leads to a characterization of the dominant terms in the expansion of $Q$, which we describe now.

\paragraph{The dominant family in $Q$ and well-behaved products.}
Recall that our goal is to show that with high probability, the matrix $Q$ from factorization $M\approx LQL^{\top}$ is positive-definite. We view $Q$ as a linear combination, where each term is a graph matrix multiplied by its coefficient. In all prior works that utilize the factorization approach \cite{BHKKMP16,ghosh2020sum,potechin2020machinery,jones2021sum, pang2021sos, jones2023sum}, the dominant term in $Q$ was a constant term, i.e., a matrix whose entries are numbers independent of the input. 
Here, however, we encounter a new and intrinsic difficulty: $Q$ contains an entire family of non-constant terms that are almost equally dominant. 

Using the refined tools developed in our error analysis (\Cref{sec:error_analysis}), we are able to characterize the shapes of the dominant terms, which we refer to as {\it simple spider disjoint unions (SSD)} (\Cref{def:disj}). A related but less restrictive family of shapes, called ``spiders'', was introduced in \cite{ghosh2020sum} in the context of the Sherrington-Kirkpatrick problem, which can be seen as a special case of NGCA where the unknown distribution $A$ is the uniform distribution on $\{ \pm 1\}$. Their technique of using the null-space to annihilate all spiders relies on $A$ being a discrete distribution, which does not apply to our setting. Additionally, we note that their work establishes a sample complexity lower bound of $n^{3/2}$, 
in contrast to the $O(n^2)$ upper bound \cite{dudeja2024statistical}. 
To achieve an almost optimal lower bound of $n^{(1-\eps)k/2}$ (see the second paragraph of the introduction), we need to study of the `rigid' structure of these shapes and their linear combinations.

As discussed earlier, the dominant terms in $Q$ are simple spider disjoint union graph matrices. To prove that their sum in $Q$ is positive-definite with high probability, we begin by examining the recursive factorization procedure that generates $Q$. Roughly speaking, $Q$ is a sum of numerous matrix products derived during the factorization. Among these products, we identify those that significantly impact $Q$---referred as the well-behaved intersection configurations (\Cref{def:wellbehavedconfig})---and use the error analysis in \Cref{sec:error_analysis} to show that the remaining other terms altogether contribute minimally. This leads to an expression of a dominating part of $Q$ denoted by $[Q]_{\wellbehaved}$, or $\SSD{Q}$ for short (\Cref{def:Qwell-behaved}, \Cref{def:QSS}), along with a characterizing equation $\SSD{L}\wbp\SSD{Q}\wbp\SSD{L}^\top=\SSD{M}$ (\Cref{lem:Q_wellbehaved_characterization}). Here, $\SSD{L},\SSD{M}$ denotes the SSD part of $L,M$ respectively, and $\wbp$ denotes what we call the well-behaved product between graph matrices.

Before giving an overview of the proof of the positive-definiteness of $\SSD{Q}$, we make two important remarks. First, the coefficients of the graph matrices in $L$ and $M$ are delicate. For instance, in $L$, the coefficient of a simple spider is $n^{-|E|/2}\cdot\E_A[h_j]$, where $j$ denotes the degree of the unique circle vertex of the spider, and $h_j$ is a Probabilist's Hermite polynomial; for disjoint union shapes, the coefficient is the product of those of its components. When matrices multiply, the coefficients multiply as well, and at several places we need to handle them in an exact way rather than doing mere magnitude estimates. Second, and more subtly, we will not analyze the matrix $\SSD{L}$ or $\SSD{M}$ in the same way we analyze $\SSD{Q}$, as both of them contain terms with larger norms. Instead, we focus our analysis on $\SSD{Q}$. 

\paragraph{PSDness via representation.} 
To show that $\SSD{Q}$ is positive-definite, we start with simple spiders. We use an algebraic method to study their multiplicative structure. The multiplication of general graph matrices is very complicated, but for simple spiders, an algebraic study turns out to be feasible. It goes as follows. 

First, we show that $\SS{Q}$---a further restricted matrix that collects all simple spider terms in $\SSD{Q}$---is positive-definite in a non-standard sense. Specifically, we model the multiplications of simple spiders as an associative $\R$-algebra, which we call $\SALD$ ({\it simple-spider algebra of degree $\dsos$}), given by \Cref{def:SALD}. The multiplication in $\SALD$ is the well-behaved product restricted by taking only simple spiders in the result (\Cref{def:wbp}). This is a non-commutative algebra, and it approximates the major terms in the multiplication of simple spiders graph matrices in special cases, although not always. To understand the structure of $\SALD$, we construct essentially all its irreducible representations in \Cref{subsec:rep} and obtain a concrete Artin-Wedderburn decomposition as a direct sum of matrix algebras (\Cref{lem:directsum}). This decomposition greatly simplifies the objects under study: it maps elements of $\SALD$, which represent graph matrices of dimension $n^{\Theta(\dsos)}$, to real matrices of dimension at most $\dsos+1$, while preserving algebra operations and matrix transposes. Using this decomposition, we prove in \Cref{lem:QSSPSD} that the matrix $\SS{L}\SS{Q}\SS{L}^{\top}$ is ``positive-definite'', and hence so is $\SS{Q}$. Here, $\SS{L},\SS{M}$ are the simple spider part of $L,M$ respectively. The proof of this fact somewhat surprisingly boils down to the multiplicative identities of Hermite polynomials. 

The quotation mark around ``positive-definite'' means that we obtain a sum-of-squares expression of $\SS{L}\SS{Q}\SS{L}^{\top}$ in the approximation algebra, $\SALD$. In reality, since the well-behaved product only approximates real matrix products of certain simple spiders but not all, to extend the ``positive-definiteness in $\SALD$'' to the positive-definiteness of the matrix $\SS{Q}$, we need to make sure that $\SS{Q}$ contains only special simple spiders where this approximation works well. This requires additional analytic arguments which are given in \Cref{subsec:PDQhat}. 

Our second step is to extend the positive-definiteness from $\SS{Q}$ to $\SSD{Q}$. Recall that $\SSD{Q}$ is the dominant part of $Q$, and it is a linear combination of simple spider disjoint unions (SSD). This time, we do not have to model the multiplication of SSD shapes algebraically (as we did for simple spiders), but instead, given the sum-of-squares expression of $\SS{Q}$ obtained from the above, we can directly construct a square root of $\SSD{Q}$ via an operation which we call $D$-combination (\Cref{def:Dcombination}). The intuition is that given a linear combination $\alpha$ of simple spiders, its $\dsos$-combination linearly combines all possible disjoint union of shapes in $\alpha$ with their coefficients multiplied together.\footnote{Technically, we require $\alpha$ to satisfy a certain consistency condition (\Cref{def:consistent}).} This construction is combinatorial rather than algebraic, but it turns out that a useful algebraic property holds: $\dsos$-combination commutes, in a sense, with well-behaved products (\Cref{lem:disj}). Using this property, we prove that if $X\cdot X^{\top}\approx \SS{Q}$ then $[X]^{\dsos}\cdot([X]^{\dsos})^{\top}\approx\SSD{Q}$, where $[X]^{\dsos}$ means the $\dsos$-combination of $X$. The positive-definiteness of $\SSD{Q}$ follows as \Cref{lem:PDQSSD}. Again, additional analytic arguments are required in the actual proof, where we also need to show that $[X]^{\dsos}$ is not too close to being singular so that we can use $[X]^{\dsos}\cdot([X]^{\dsos})^{\top}$ to compensate for the error terms. This is done in \Cref{lem:gammanorm}. 
\bigskip

To summarize, from the error analysis we have $\norm{Q-\SSD{Q}}\leq n^{-\Omega(\eps)}$. 
By the above steps, we show that $\SSD{Q}\succ n^{-o(\eps)}\Id$ 
assuming that $A$ satisfies some mild conditions besides matching $(k-1)$ moments. Together, we get the positive-definiteness of $Q$. From here, it is not hard to show that $M\approx LQL^{\top}$ is PSD, as is done in \Cref{sec:puttogether}. 
We now turn to handling the error terms. 

\paragraph{Handling error terms via configurations.}
We handle the error terms in \Cref{sec:error_analysis}, which is largely independent of \Cref{sec:psdness-qSS} and \Cref{sec:psdness-qSSD}. As described more precisely in the proof overview in \Cref{subsec:overview}, there are two main sources of error terms:
\begin{enumerate}
\item The error $Q - \SSD{Q}$ in the approximation of $Q$ by the sum of well-behaved configurations. 
\item The truncation error $M - LQL^{\top}$.
\end{enumerate}
We also need to analyze the error in our PSD approximation $\SSD{Q} \approx [X]^{\dsos}\cdot([X]^{\dsos})^{\top}$ of $\SSD{Q}$.

Our framework for handling these error terms is as follows. We formalize the way an error term can be generated as a configuration (\Cref{def:configuration}). The goal is then to show that the following number is small: the product of the coefficients from all shapes in the configuration, multiplied with the norm of any graph matrix that results from the configuration. To estimate this number, we use a charging argument that assigns edge factors to vertices. 
The idea is that to calculate, for example, the exponent over $n$ in the expression $n^{-|E_\alpha|/2}$ times the norm bound on a graph matrix $M_\alpha$, we take $\log_{\sqrt{n}}(\cdot)$ of the expression. We imagine that each edge in shape $\alpha$ has an additive factor of 1, and we assign the edge factors to its endpoints so that each vertex receives a sufficient amount of factors. 

The main result we show is a dichotomy: either the configuration is a well-behaved SSD product and has approximate norm $1$, or the configuration has norm $o(1)$ (see \Cref{thm:erroranalysis}). This allows us to show that the errors $Q - \SSD{Q}$ and $\SSD{Q} - [X]^{\dsos}\cdot([X]^{\dsos})^{\top}$ have norm $n^{-\Omega(\eps)}$. 
The design and analysis of the edge factors assignment scheme relies on properties of the minimum square separators and the minimum weight separators, which might be of independent interest.

To handle the truncation error, we observe that the truncation error only contains configurations which are very large and all such configurations have norm $n^{-\Omega({\eps}\truncation)}$, where $\truncation$ is a ``total size'' threshold on shapes that we set in pseudo-calibration.

\subsection{Organization of the Paper}
The paper is organized as follows. 
In \Cref{sec:prelim}, after preparing some general preliminaries, we state the NGCA problem in the context of sum-of-squares algorithms in  \Cref{sec:prob_statement}, and we derive the pseudo-calibration expression in \Cref{sec:pseudocalib_prelims}. We state our main result formally as \Cref{thm:main-formal}. 
In \Cref{lem:lowdegtest}, we show that pseudo-calibration passes a natural family of low-degree tests. In particular, with high probability, for every low-degree Hermite polynomial, its empirical average value on the inner product of the samples with the SoS solution (i.e., the hidden direction $\vsos$) is close to its expectation value under $A$.

In \Cref{sec:graph}, after recalling the definitions and properties about graph matrices, we give the factorization of the moment matrix based on a new kind of minimum vertex separators, namely minimum square separators. Here, we introduce the family of simple spiders and their disjoint unions, which are the important family of graph matrices in our analysis. In \Cref{sec:well-behaved} and \Cref{sec:spider_product}, we introduce the notion of well-behaved intersection configurations and well-behaved products, which are the major products that affect the PSDness of the matrix $Q$ in $M\approx LQL^{\top}$. In \Cref{sec:psdness-qSS}, we prove the positive-definiteness of the (approximate) simple spider part of the matrix $Q$, and we extend the result to the (approximate) simple spider disjoint union part of $Q$ in \Cref{sec:psdness-qSSD}. Here, we use the idea of algebra representations to study the multiplicative structure of simple spiders, and we use the $D$-combination construction to study their disjoint unions. In \Cref{sec:error_analysis}, we analyze the error terms in our analysis using a careful charging argument, which might be of independent interest. In \Cref{sec:puttogether}, we combine all of these pieces to prove \Cref{thm:main-formal}. Finally, in \Cref{sec:applications}, we give the applications of our SoS lower bounds to various classical problems in statistics and learning theory. This includes robust mean estimation (RME), list-decodable mean estimation, robust covariance estimation (RCE), and learning Gaussian mixtures, which are summarized in \Cref{tab:sample_complexity}.

\section{Technical Preliminaries and Formal Statement of Main Result}\label{sec:prelim}

In this section, we formally define the problem statement and set up notation that we will use throughout the paper. In \Cref{prelim:notation}, we set up basic notation. In \Cref{sec:prob_statement}, we recall the notions of pseudo-expectation values and the moment matrix, formally define the NGCA problem, and formulate it in the sum-of-squares framework. In \Cref{sec:pseudocalib_prelims}, we recall the technique of pseudo-calibration introduced by \cite{BHKKMP16} and identify the pseudo-calibrated moment matrix whose PSDness we want to prove. In \Cref{subsec:pEproperties} we show that our pseudo-expectation values satisfy the desired constraints except for PSDness of the moment matrix. 
Finally, in \Cref{sec:algebra}, we recall some notions from abstract algebra.
 
\subsection{Notation} \label{prelim:notation}
\paragraph{Basic Notation:}
$\N,\Z,\R$ denotes the set of natural numbers, integers and real numbers, respectively. For $t \in \Z_+$, $[t] := \{1,\ldots,t\}$. We will use $n$ to denote the dimension of the input data, and $m$ for the number of samples. The given $m$ samples are denoted by $x_1, \dots, x_m$, where each $x_u=(x_{u1},\ldots,x_{un})\in\R^n$. We will use index symbols $u\in[m]$, $i\in[n]$. The SoS degree is $\dsos$. 

For an integer vector $a\in\N^n$, $\normo{a}=\Sum_{i=1}^n a(i)$, $a!:=\Prod_{i=1}^n a(i)!$. For a sequence of $m$ such vectors, accordingly, $a=(a_1,\ldots,a_m)\in(\N^n)^m$, $\normo{a}=\Sum_u\normo{a_u}=\Sum_{u,i}a_u(i)$ and $a!:=\Prod_u a_u!=\Prod_{u,i}a_u(i)!$. It might be helpful to think of $a\in(\N^n)^m$ as an edge-weighted and undirected bipartite graph on vertex sets $[n],[m]$ where $a_u(i)$ is the weight of the edge $\{i,u\}$. 

For matrices $M,N$ and a number $C>0$, we use $M\almost{C}N$ to denote that $\norm{M-N}\leq C$ where $\norm{\cdot}$ on matrices always means the operator norm. If $M,N$ are square matrices, $M \succcurlyeq N$ denotes that $M-N$ is positive-semidefinite (PSD). We use $\poly(\cdot)$ to indicate a quantity that is polynomially upper-bounded in its arguments. Similarly, $\polylog(\cdot)$ denotes a quantity that is polynomially upper-bounded in the logarithm of its arguments. By writing $\log(\cdot)$, we mean $\log_2(\cdot)$.

\paragraph{Probability Notation:} 
For a random variable $X$, we write $\E[X]$ for its expectation. $\cN(\mu,\sigma^2)$ denotes the 1-dimensional Gaussian distribution with mean $\mu$ and variance $\sigma^2$. When $\D$ is a distribution, we use $X \sim \D$ to denote that the random variable $X$ is distributed according to $\D$. When $S$ is a set, we let $\E_{X \sim S}[\cdot]$ denote the expectation under the uniform distribution over $S$. 

\paragraph{Hermite Polynomials:} 
The probabilist's Hermite polynomial $He_i(x) = \sum_{j=0}^{\lfloor{i}\rfloor}{(-1)^{j}\binom{i}{2j}\frac{(2j)!}{(2^j)j!}x^{i-2j}}$ will be denoted by $h_{i}(x)$. 
The $n$-dimensional Hermite polynomials are $h_{a}=\Prod_{i=1}^nh_{a(i)}$ for $a\in\N^n$. 
Recall that $h_a/\sqrt{a!}$ ($a\in\N^n$) form an orthonormal basis of polynomials under the inner product 
\[\iprod{f,g} := \E_{ x\sim \cN(0,\Id_n)}\big[f( x)g( x)\big]\]
where $\cN(0,\Id_n)$ is the $n$-dimensional multivariate normal distribution with mean $0$ and covariance matrix $\Id_n$. For $a\in(\N^n)^m$, we let $h_{a}=\Prod_{u=1}^mh_{a_u}$. 

\subsection{Problem Statement and Sum-of-Squares Solutions}
\label{sec:prob_statement} 

We now define the SoS formulation for NGCA that we use. The inputs to our SoS program are i.i.d. samples $x_1,\dots,x_m$ drawn from the reference distribution. We denote the SoS variables by $\vsos=(\vsos_1,\dots,\vsos_n)$. 

A true solution to the NGCA problem would assign a real value to each $\vsos_i$ such that the following constraints are satisfied:
\begin{enumerate}
    \item (Booleanity) For all $i \in [n]$, $\vsos_i^2-{1\over n}=0$.
    \item (Soft NGCA constraints) 
    $\left|\frac{1}{m}\Sum_{i=1}^{m}\left(h_j(x \cdot \vsos)\right) - \E\limits_{a \sim A}[h_j(a)]\right|=\widetilde{O}(\frac{1}{\sqrt{m}})$ where $x_1,\dots,x_m$ are the input samples.
\end{enumerate}

Degree-$\dsos$ SoS gives a relaxation of the problem where instead of assigning a real value to each $\vsos_i$, we have an $\R$-linear map $\widetilde{E}$ called {\it pseudo-expectation values} which assigns a real value $\widetilde{E}[p]$ to each polynomial $p(\vsos_1,\dots,\vsos_n)$ of degree at most $\dsos$. We can think of $\widetilde{E}[p]$ as the estimate given by degree-$\dsos$ SoS for the expected value of $p$ over a (possibly fictitious) distribution of solutions.

\begin{definition}
We define $\R^{\leq d}(\vsos)$ to be the set of all polynomials of degree at most $d$ in the variables $\vsos_1,\dots,\vsos_n$.
\end{definition}

\begin{definition}[Pseudo-expectation Operator for NGCA]
\label{def:pseudoexpectation}
Given input samples $x_1,\dots,x_m$ and a target distribution $A$, we say that an $\R$-linear map $\pE: \R^{\leq D}(\vsos_1, \dots, \vsos_n) \rightarrow \R$ is a degree $D$ pseudo-expectation operator for NGCA if it satisfies the following conditions.
	
\begin{enumerate}
    \item (Oneness) $\pE(1) = 1$;
    \item (Booleanity) $\pE\Big(f(\vsos)\cdot(\vsos_i^2-{1\over n})\Big)=0$ for all $i\in[n]$ and all $f\in\R^{\leq D-2}(\vsos)$;
    \item (Soft NGCA constraints) 
    $\left|\frac{1}{m}\Sum_{u=1}^{m}\pE\left(h_j(x_u \cdot \vsos)\right) - \E\limits_{a \sim A}[h_j(a)]\right|=\widetilde{O}(\frac{1}{\sqrt{m}})$ for all $j\leq \dsos$;
    \label{item:softaxiom}
    \item (Positivity) $\pE[p^2] \geq 0$ for all $p \in \R^{\leq \dsos}(\vsos)$. 
    \label{item:psdness} 
\end{enumerate}
\end{definition}

If this relaxation is infeasible then degree-$\dsos$ SoS can prove that there is no vector $\vsos=(\vsos_1,\dots,\vsos_n)$ such that the input has distribution $A$ in the direction $\vsos$. For our degree-$\dsos$ SoS lower bounds, we show that w.h.p. (with high probability) this does not happen. To do this, we design a candidate pseudo-expectation operator $\widetilde{E}$ and show that w.h.p. it is a degree $\dsos$ pseudo-expectation operator for NGCA.
\begin{remark}
We use Roman text font for the SoS variables $\vsos_1,\dots,\vsos_n$ to distinguish them other expressions such as the input variables which take fixed real values once the input is given.
\end{remark}

\begin{remark}[Soft NGCA constraints]
\label{rmk:sosprogram}

The ``soft'' NGCA constraints (\Cref{item:softaxiom}) indicate that we study generic SoS lower bounds, i.e., there is no strict polynomial identity constraint other than booleanity of the variables $\vsos$. This is more or less an inevitable feature of the algorithms dealing with NGCA for general distributions $A$ in Definition \ref{def:ngca_distinguishing}, as opposed to cases where $A$ is more restricted such as being discrete. 
\end{remark}

The positivity condition on $\widetilde{E}$ can be expressed using the following (pseudo-)moment matrix.

\begin{definition}
We define $\vsos^I$ to be the monomial $\Prod_{i=1}^n \vsos_i^{I(i)}$ where $I\in\N^n$. We define the degree of $\vsos^I$ to be $\normo{I}=\Sum_i I(i)$. 
\end{definition}

\begin{definition}[Moment matrix]
\label{def:moment_matrix}
Given a linear map $\pE:\R^{\leq 2\dsos}[\vsos]\rightarrow \R$, its degree $\dsos$ pseudo-moment matrix, or moment matrix for short, is an $\binom{[n]}{\leq \dsos} \times \binom{[n]}{\leq \dsos}$ matrix $M_{\pE}$ whose rows and columns indexed by subsects of $I\subseteq [n]$ of size at most $\dsos$, and the entries are: 
\[M_{\pE}(I, J) := \pE(\vsos^{I+J}),\] 
where $I,J$ are viewed as the indicator functions from $[n]$ to $\{0,1\}$. 
\end{definition}

The verification of the following fact is straightforward.
\begin{fact}
    Suppose $M_{\pE}$ satisfies the Booleanity condition. 
    Then the positivity condition, \Cref{item:psdness} in \Cref{def:pseudoexpectation}, is equivalent to the condition that $M_{\pE}\succcurlyeq 0$. 
\end{fact}

In the next section, we describe the standard pseudo-calibration technique used to prove SoS lower bounds.

\subsection{Pseudo-calibration Technique for NGCA and Main Result}
\label{sec:pseudocalib_prelims} 
Pseudo-calibration, introduced in \cite{BHKKMP16}, is a method to construct a candidate pseudo-expectation operator $\pE$ for an input $x$ drawn from the problem distribution (reference distribution, $\Dref$). The idea is to show that there is another distribution (planted distribution, $\Dpl$) supported on feasible instances and solutions $(x,v)$, such that it cannot be distinguished from the problem distribution via any low-degree test. Once we have this planted distribution, we will choose the candidate pseudo-expectation values $\pE(\vsos^I)$ so that 
\begin{align}\label{eqn:pseudocalib1}
\mathop{E}\limits_{x\sim \Dref} \left[t( x) \pE(\vsos^I)\right] = \mathop{E}\limits_{(x,v)\sim\Dpl} \Big[t( x) { v}^I\Big]
\end{align}
for all low-degree polynomials $t(x)$, where $v^I$ is the monomial $\Prod_{i=1}^n v_i^{I(i)}$, $I\in\N^n$. Moreover, we impose the condition:
\begin{equation}\label{eqn:pseudocalib2}
 \text{Each $\pE(\vsos^I)$ itself is a low-degree polynomial in $x$.}
\end{equation}

For our problem, the input data $(x_1, \dots, x_m)\in (\R^n)^m$ are i.i.d. samples drawn from the standard Gaussian distribution, and $\vsos = (\vsos_1, \dots, \vsos_n)$ are the SoS variables representing the unknown direction whose solution existence SoS wants to refute. 

In light of the NGCA problem in \Cref{def:ngca_distinguishing}, our planted distribution $D_{pl}$ is the following: first choose a planted vector $v\sim \{\pm{1\over\sqrt{n}}\}^n$ uniformly at random, then choose i.i.d. samples $x_1,\dots,x_m$ so that
\[x_i=\big((x_i)_{v^\perp},\ (x_i)_v\big)\sim \cN(0,\Id_{n-1})_{v^\perp}\times A_v,\]
where $A_v$ is the one-dimensional distribution of interest in direction $v$ matching $k-1$ moments with $\cN(0, 1)$.
Concretely, conditions \Cref{eqn:pseudocalib1} and \Cref{eqn:pseudocalib2} enforce the pseudo-calibration to have the following form: 
\begin{equation}\label{eq:calib}
	\pE(\vsos^I) :=
 \Sum_{\substack
{a\in(\N^n)^m:\ \text{``total size'' of $a$ is upper bounded by $\truncation$}}}
 \ \mathop{E}_{(x,v)\sim\Dpl} \left[ v^I {h_{a}\over\sqrt{a!}} \right]\cdot {h_{a}\over\sqrt{a!}},
\end{equation}
where $\truncation$ is a parameter deciding the meaning of ``low-degree'' in \Cref{eqn:pseudocalib1} and \Cref{eqn:pseudocalib2}. 
We will choose $\truncation$ based on the technical analysis. It turns out that we can choose $\truncation$ to be any value between $C_1\log n$ and $n^{C_2}$ for some constants $C_1, C_2$ depending on $\dsos$ and $\eps$. 

For any fixed $I\in \N^{n}$, we let the \emph{$I$-total size} of $a=(a_1,\dots,a_u)\in(\N^n)^m$ be $\totalsize^I(a):=$ 
\[\normo{a}+\left\vert \{i\in[n]:\ I(i)>0\text{ or }(\exists u\in[m])\ a_u(i)>0\}\right\vert + \left\vert\{u\in[m]:\ (\exists i\in[n])\ a_u(i)>0\}\right\vert.\] 
We use this to measure the ``total size'' of $a$ in the above equation (its meaning will become clear in \Cref{sec:graph}). The calculation of \eqref{eq:calib} is similar to the one in \cite{ghosh2020sum} and is reproduced in~\Cref{app:pseud-calib}, giving the following expression.
\begin{lemma}[Pseudo-calibration]\label{lem:calib1}
	For any $I\in\N^n$, 
	\begin{equation}\label{eq:calib-calc}
		\begin{aligned}		    \pE(\vsos^I)=\Sum_{\substack
            {a\in(\N^n)^m:\ \totalsize^I(a)\leq \truncation,\vspace{2pt}\\ 
            \text{and }(\forall i\in[n])\ I(i)+\sum_{u}a_u(i)\text{ is even}}}
            n^{-{\normo{I}+\normo{a}\over2}}\Prod_{u=1}^m \E_A \Big[h_{|a_u|}\Big] {h_{a_u}\over a_u!}
		\end{aligned}
    \end{equation}
\end{lemma}

\begin{remark}[Only moments matter]\label{rmk:moments}
By \Cref{eq:calib-calc}, the pseudo-expectation values are determined
by the moments of $A$ up to the truncation threshold $\truncation$. 
In particular, if $\truncation$ is smaller than the number of matched moments (i.e., $k-1$), then \Cref{eq:calib-calc} will be a constant function and the resulting matrix will be diagonal and trivially PSD.
\end{remark}

Our results depend on the following measures of the distribution $A$.
\begin{definition}[$U_A$, $L_A$]\label{def:UALA}
Given a distribution $A$ and integer $t\geq 0$, we let 
\begin{align}
U_A(t)&:=\max\limits_{0\leq i\leq t}\left\vert\E\limits_A\ [h_i(x)]\right\vert \label{cond:A<}\\
L_A(t)&:=\min\limits_{\substack{p(x):\ \text{polynomial of degree $\leq t$}\\ \text{where }\E_{\cN(0,1)}[p^2(x)]=1}}\ \E\limits_A\ [p^2(x)] \label{cond:A>}
\end{align}
\noindent Note that this minimum exists\footnote{The reason is that $\E_A[\cdot]$ is a continuous function on the compact set $\{p(x):\ \deg(p)\leq t, \E_{\cN(0,1)} [p^2(x)]=1\}$. This set is compact because the map $p(x)\mapsto\E_{\cN(0,1)}[p(x)^2]$ is a non-degenerate quadratic form on $\{p(x):\ \deg(p)\leq t\}$.} and $U_A(0)=L_A(0)=1$. 
\end{definition}

In the lower bound proof, however, a slightly different pair of measures are more handy to use, as they can simplify many expressions of estimates.

\begin{definition}[$\cu,\cl$]\label{def:cucl}
Given $n$, $\dsos$, $\truncation$ and distribution $A$, we let $\cu, \cl$ be the minimum values such that $\cu,\cl\geq 1$ and 
\begin{enumerate}
\item For all $t \leq 3\truncation$, ${\abs{\E\limits_A\ \left[h_{t}(x)\right]}} \leq \cu^t$. 
\item For all polynomials $p(x)$ of degree at most $\dsos$ such that $\int_{\cN(0,1)}{p^2(x)}=1$, $\E\limits_A\ [p^2(x)] \geq \cl^{-\deg(p)}$.
\end{enumerate}
\end{definition}

We can now state the main theorem formally. 

\begin{restatable}[Main Theorem]{theorem}{mainformal}\label{thm:main-formal}
   There is a universal constant $\Cuniv\geq 1$ such that for any $\delta\in(0,1)$, if $n$ is sufficiently large then the following holds. 
   
    Suppose $\eps\in(0,1)$, $k\geq 2$, $A$ is a 1-dimensional distribution, and $\dsos$ and $\truncation$ are integer parameters (where $\eps$, $k$, $A$, $\dsos$, and $\truncation$ may all depend on $n$) such that: 
    \begin{align}
    &\text{$A$ matches the first $k-1$ moments with $\cN(0, 1)$.}\label{eq:main_moment_match}\\ 
    &\truncation\geq\max\{50{\dsos}^2, \frac{500}{\eps}\dsos, 2k\log n\},\text{ and } (5\cu)^{20\dsos^2} \cl^{2\dsos} (10\truncation)^{256\Cuniv\dsos^2} < n^{{\eps \over 30}}.\label{eq:main_condition} 
    \end{align}
    Then if we draw $m < n^{(1-\eps)k/2}$ i.i.d. samples from $\cN(0,\Id_n)$, with probability greater than $1-\delta$, 
    the moment matrix of the degree-$\dsos$ pseudo-calibration with truncation threshold $\truncation$ is positive-definite. 
\end{restatable}

As an illustration, we can set $k\leq(\log n)^{O(1)}$, $\dsos=o(\sqrt{\frac{\eps\log n}{\log\log n}})$, and $\truncation=2k \dsos \log n$. If the distribution $A$ matches $k-1$ moments with $\cN(0,1)$ and satisfies $\cu,\cl \leq (\log n)^{O(1)}$, then \Cref{thm:main-formal} provides an almost optimal $n^{\frac{(1-\eps)k}{2}}$ sample lower bound for the corresponding NGCA problem in degree-$\dsos$ SoS.

\subsection{Properties of Pseudo-calibration}\label{subsec:pEproperties}
Before proving \Cref{thm:main-formal}, let us first show that the pseudo-calibration ``solution'' of $v$ from \Cref{eq:calib-calc} satisfies the Booleanity constraints and the soft NGCA constraints.

It is not hard to check directly that $\widetilde{E}$ satisfies the Booleanity constraints as multiplying $\vsos^{I}$ by $v_i^2$ increases the number $I(i)$ by $2$. This is also a special case of the following more general fact (Cf. \cite{BHKKMP16,ghosh2020sum}), whose proof is a simple expansion of the definition \Cref{eq:calib}.

\begin{lemma}[Pseudo-expectation preserves zero]\label{lem:pseudo-calib}
If $f(\vsos)$ is a polynomial only in the SoS variables $\vsos$ such that $\deg(f)\leq \dsos$ and $f(v)=0$ for all $v$ in the planted distribution, then $\pE(f)=0$ independent of the input $x$.
\end{lemma}

As for the soft NGCA constraints, we will show that for any low-degree Hermite polynomial evaluated on the inner product $x\cdot v$, with high probability the pseudo-expectation value is close to the expectation under $A$ (\Cref{lem:lowdegtest}). Below, recall that $h_j(x)$ denotes $He_j(x)$. We will show the following lemmas about the pseudo-expectation of $h_j(x \cdot \vsos)$. We first prove a lemma giving us the expansion of $h_j(x \cdot \vsos)$. 

\begin{fact}
For all vectors $v \in \{-\frac{1}{\sqrt{n}},\frac{1}{\sqrt{n}}\}^{n}$, all vectors $x \in \mathbb{R}^n$ and all $j \in \mathbb{N}$, 
\[
h_j(x \cdot v) = j!\sum_{J\in\N^n:\ \normo{J} = j}\ 
{\Prod_{i\in[n]}{\left(\frac{h_{J(i)}(x_i)}{{J(i)}!n^{\lfloor {J(i)\over 2}\rfloor}}{v_i}^{J(i)\text{ mod 2}}\right)}}.
\]
\end{fact}

\begin{proof}
The Hermite generating function is $\exp\left(xt-{t^2\over 2}\right)=\Sum_{j}h_j(x){t^j\over j!}$. 
We have the identity $(x\cdot v)t-{t^2\over 2}=\Sum_{i=1}^n\left(x_i(v_it)-{(v_it)^2\over 2}\right)$ if $v_1^2+\ldots+v_n^2=1$, so by taking the exponential, 
\[\Sum_{j=0}^\infty h_j\left(x\cdot v\right){t^j\over j!}=\Prod_{i=1}^n\left(\Sum_{j=0}^\infty h_j(x_i){(v_i t)^j\over j!}\right).\]
By comparing the $j$th power of $t$ we see that $h_j(x\cdot v)=j!\Sum_{J\in\N^n:\ J(1)+\dots+J(n)=j}\ \Prod_{i=1}^n\left({h_{J(i)}(x_i)\over J(i)!}v_i^{J(i)}\right)$. Finally, notice that $v_i^2={1\over n}$ for each $i$.
\end{proof}

As a corollary to this, we obtain the expansion of $\frac 1 m \sum_{u=1}^m \pE \left(h_j(x_u \cdot \vsos) \right)$. 

\begin{corollary}\label{cor:hermitepseudoexpectation}
For $J\in\N^n$, let $J_{\text{mod 2}}\in\{0,1\}^n$ be such that $J_{\text{mod 2}}(i)=J(i)\text{ mod 2}$. Then 
\[
\frac{1}{m}\sum_{u=1}^{m}\pE\Big(h_j(x_u \cdot \vsos)\Big) 
=\frac{1}{m}\sum_{u=1}^{m}{j!\sum_{J:\ J\in\N^n,\ \normo{J} = j}{\left(\prod_{i=1}^n {\frac{h_{J(i)}\big((x_u)_i\big)}{J(i)!n^{\lfloor{J(i)\over 2}\rfloor}}}\cdot\pE\left( \vsos^{J_{\text{mod 2}}}\right)\right)}}, \ \ \forall x_1,\dots,x_m\in\R^n.
\]
\end{corollary}

We can now sketch the proof of the fact that the pseudo-expectation values ``fool'' Hermite polynomial tests up to degree $D$. 

\begin{lemma}[Hermite Tests in the Hidden Direction]\label{lem:lowdegtest}
For all Hermite polynomials $h_j$ of degree at most $\dsos$,
with high probability, 
$\left|\frac{1}{m}\Sum_{i=1}^{m}\pE\left(h_j(x \cdot \vsos)\right) - \E\limits_{a \sim A}[h_j(a)]\right|$ is $o(1)$. More precisely,
\[
\left|\frac{1}{m}\sum_{i=1}^{m}\pE\Big(h_j(x_i \cdot \vsos)\Big) - \E\limits_{a \sim A}[h_j(a)]\cdot\pE(1)\right| \leq \frac{\polylog(m)}{\sqrt{m}}.
\]
\end{lemma}
\begin{proof}[Proof sketch] 
In this proof we assume the definitions and properties of graph matrices from \Cref{sec:graph}. 
We view the polynomials as ribbons with $U=V=\emptyset$. 
We will compare the ribbons, with coefficients, that appear in $\frac{1}{m}\sum_{i=1}^{m}\pE\left(h_j(x \cdot v)\right)$ and $\E_{a \sim A}[h_j(a)]\pE(1)$.

By Corollary \ref{cor:hermitepseudoexpectation}, we can obtain the (perhaps improper) ribbons appearing in $\pE\Big(h_j(x_i \cdot \vsos)\Big)$ by starting with the ribbons appearing in $\pE\left(\vsos^{J_{mod\ 2}}\right)$ and adding a circle with label $u$ and edges with label $J(i)$ from this circle to the square vertices with label $i$. We multiply some additional factors to the ribbon coefficient too: a factor of $\frac{1}{m}$ which we think of as associated with the circle vertex with label $u$, a factor of $\frac{1}{n^{\floor{J(i)/2}}}$ associated with each of the added edges, and a factor of $\frac{j!}{\prod_{i=1}^n J(i)!}$.

Consider the set of edges incident to the circle vertex with label $u$ in a ribbon for  $\pE\left(\vsos^{J_{mod\ 2}}\right)$. There are a few cases:
\begin{enumerate}[leftmargin=*]
    \item These edges do not match the added edges. In this case, we can use an edge factor assignment scheme to show that the resulting term has negligible norm upper bound from \Cref{thm:norm_control}. For the edges incident to the circle vertex with label $u$, we assign all weight to the square vertex. For other edges, we split the weight evenly between the two endpoints. It is not hard to check that in this case, the resulting norm is $\widetilde{O}(\frac{1}{\sqrt{m}})$ because the circle vertex with label $u$ is not isolated and because $U=V=\emptyset$, i.e., each square vertex has degree at least $2$ in the starting ribbon in $\pE\left(\vsos^{J_{mod\ 2}}\right)$.

    \item These edges exactly match the added edges. Note that for each matched edge $e$, if its label is $l$, then right after adding edges there is a factor $h_{l}(x_e)\cdot h_{l}(x_e)$ (one newly added and one from the ribbon in $\pE\left(\vsos^{J_{mod\ 2}}\right)$). When we expand this multi-edge into single edges using the Hermite expansion of $h_l\cdot h_l$, the coefficient of $h_0$ is precisely $l!$. Other $\{h_j(x)\mid j>0\}$ from the expansion make the shape norm $\widetilde{O}({1\over \sqrt{m}})$ like in case 1, and they go to the error terms. If we look at the resulting shape $R'$ from this process that removes a fixed circle vertex $u$ and the matched edges, its additional factor besides the coefficient from $\pE(1)$ is: 
    \[\Sum_{J\in\N^n:\ \normo{J}=j}\left(\frac{j!}{\prod_{i\in[n]}\Big(J(i)!\cdot n^{\floor{J(i)/2}}\Big)}\cdot\frac{\E_{x\sim A}[h_j(x)]}{\sqrt{n}^{|J_{mod\ 2}|+\normo{J}}\prod_{i\in[n]}{J(i)!}}\cdot\ {\prod_{i\in[n]}{J(i)!}}\right).\] 
    Here, notice that from different ribbons in $\pE(\vsos^{J_{\text{mod 2}}})$ we can get the same $R'$ in the described process, by using different $J$ to add the edges. Since $n^j=(1+\dots+1)^j=\Sum_{J\in\N^n:\ \normo{J}=j}\frac{j!}{J(1)!\cdots J(n)!}$, the expression above simplifies to $\E_{A}[h_j]$. Summing over $u\in [m]$ and dividing by $m$, this gives $\E_A[h_j]\cdot\frac{m-|V_{\circle}(R)|}{m}$ times the coefficient in $\pE(1)$. By a similar argument to case 1 we can show that the difference part, collected from all $R'$, has norm at most $\widetilde{O}({1\over\sqrt{m}})$. 
\end{enumerate} 
This shows that for when we sum over the ribbons for $\pE\left(\vsos^{J_{mod\ 2}}\right)$, the contributions from the edges incident to the circle vertex labeled $u$ gives $\E_{a \sim A}[h_j(a)]$ plus an error of magnitude $\widetilde{O}(\frac{1}{\sqrt{m}})$. Note that these ribbons may also have edges incident to other circle vertices but such terms are matched by the corresponding terms in $\E_{a \sim A}[h_j(a)]\pE(1)$. More precisely, if we consider the ribbons for $\pE\left(\vsos^{J_{mod\ 2}}\right)$ which contain a given set $E$ of edges which are incident to other circle vertices, this matches (up to error $\widetilde{O}(\frac{1}{\sqrt{m}})$) the term in $\E_{a \sim A}[h_j(a)]\pE(1)$ where the ribbon for $\pE(1)$ has edges $E$ and has no edges incident to the circle vertex with label $u$. 
Finally, we note that the sum of the terms in $\pE(1)$ where there is an edge incident to the circle vertex with label $u$ is $\widetilde{O}(\frac{1}{\sqrt{m}})$.
\end{proof}

\subsection{Algebra Preliminaries}
\label{sec:algebra}

We close this section with some algebraic notions that will be used in \Cref{sec:psdness-qSS}.

\begin{definition}[Algebra Definitions]
\label{def:algebra} 
A vector space $A$ over a field $F$ together with a binary operation ``$\cdot$'':  $A \times A \to A$ is a \emph{unital associative algebra over $F$}, or an \emph{$F$-algebra} in short, if the following conditions are satisfied for all $x, y, z \in A$, all $a, b \in F$, and some special element $1_A\in A$:
\begin{enumerate}[wide, labelindent=18pt]
    \item[(Unit)] $1_A\cdot x=x=x\cdot 1_A$;
    \item[(Linearity)] $(ax) \cdot (by) = (ab) (x \cdot y)$;
    \item[(Distributivity)] $(x+y)\cdot z= x\cdot z + y\cdot z$, and $z\cdot (x+y)=z\cdot x + z\cdot y$;
    \item[(Associativity)] $(x\cdot y)\cdot z = x\cdot(y\cdot z)$.
\end{enumerate}
We call the binary operation ``$\cdot$'' the \emph{multiplication} in $A$ and ``$1_A$'' the \emph{unit} in $A$.

A \emph{homomorphism} between two $F$-algebras $A$ and $B$ is an $F$-linear map $\rho:A\to B$ that preserves unit and multiplication, i.e., $\rho(1_A)=1_B$ and $\rho\left(x\cdot_A y\right)=\rho(x)\cdot_B \rho(y)$ for all $x,y\in A$. 
An \emph{isomorphism} is a bijective homomorphism. 
We say $A,B$ are isomorphic, denoted by $A\simeq B$, if there is an isomorphism between them.

The \emph{direct sum} of two $F$-algebras $A$ and $B$, denoted by $A\oplus B$, is the direct sum of vector spaces with unit $(1_A,1_B)$ and multiplication $(x,y)\cdot(x',y'):=(x\cdot_A x', y\cdot_B y')$.

\end{definition} 

\begin{definition}[Representation]\label{def:representation}
A representation of an $F$-algebra $A$ is a homomorphism $\rho: A \rightarrow End_F(V)$, where $V$ is an $F$-vector space, $End_F(V)$ is the $F$-algebra consisting of all $F$-linear maps from $V$ to $V$ where the unit is the identity map and the multiplication is defined by composition. The dimension of $\rho$ is the dimension of $V$. Note that if $\dim(V)=n$ then $End_F(V)\simeq \mathbb{M}_n(F)$ where the latter is the algebra of all $n$-by-$n$ matrices over $F$. 

A representation $\rho$ is \emph{irreducible} if there is no proper nonzero subspace $W$ of $V$ such that $\rho(A)(w)\in W$ for all $w\in W$. 
Two representations $(\rho_1,V_1)$, $(\rho_2,V_2)$ are isomorphic if there is a bijective map $f:V_1\to V_2$ such that $\big(\rho_2(x)\circ f\big) (v)=\big(f\circ\rho_1(x)\big) (v)$ for all $x\in A$ and all $v\in V_1$. 
\end{definition}

The rest of the paper is devoted to proving \Cref{thm:main-formal}.

        \section{Graph Matrices and Factorization of the Moment Matrix}
\label{sec:graph}
Graph matrices, defined in \cite{ahn2021graphmatricesnormbounds}, are large-size random matrices whose entries are symmetric functions of the input variables. Their utility is due to the fact that the spectral norm of these matrices can be upper bounded in terms of properties of the graph that indexes them. 

In this section, we start by recalling standard definitions about graph matrices and introducing some new notions. We then introduce simple spiders and simple spider disjoint unions, which will be the main objects of our analysis. Then, in \Cref{sec:tools}, we recall some standard tools for the PSDness analysis and introduce a three-way decomposition of shapes (and graph matrices) based on minimum square separators. In \Cref{sec:recursive}, we use this decomposition recursively to obtain a matrix factorization $M\approx LQL^{\top}$ for some $L$ and $Q$ where  $M$ is $M_{\pE}$ rescaled for technical convenience. 

\begin{remark}
This section is similar to previous sum of squares analyses for average-case problems, except for \Cref{def:SS}, \Cref{def:disj}, 
and the use of minimum square separators (\Cref{def:separators}).
\end{remark}

\begin{definition}[Graph Matrix Definitions]
\label{def:graph_matrix}
The matrices we will consider are indexed by bipartite graphs between two fixed vertex sets: $\cC_m$ and $\cS_n$, where $\cC_m$ denotes \emph{circle vertices} whose elements are denoted by $\circled{u}$ where $u\in[m]$ (representing a sample), and $\cS_n$ denotes \emph{square vertices} whose elements are denoted by $\squared{i}$ where $i\in [n]$ (representing a coordinate). 
\begin{enumerate}
    \item (Matrix indices) 
    A matrix index is a subset of $\cC_m \cup \cS_n$, representing a row or a column of the matrices we consider. In this paper, we will only use matrix indices that are subsets of $\cS_n$.
        
    \item (Ribbon) A ribbon is an undirected, edge-labeled graph $R = (V(R), E(R), U_R, V_R)$ where $V(R) \subseteq \cC_m \cup \cS_n$, each edge $e \in E(R)$ goes between a square and a circle vertex and has a weight label $l(e)\in\N$. $U_R,V_R\subseteq V(R)$ are viewed as two matrix indices\footnote{A matrix index here is a subset of vertices. In general, a matrix index can contain additional information; see e.g., \Cref{sec:parity}.}, we call them the left- and right- index of the ribbon respectively, which may or may not be disjoint.

    \item (Ribbon matrix) Given a ribbon $R = (V(R), E(R), U_R, V_R)$, the ribbon matrix $M_R$ has rows and columns indexed by matrix indices, and it has a single nonzero entry: 
	\begin{equation}\label{eq:ribbonmatrix}
            M_R(U_R, V_R) = \begin{cases*}
			\displaystyle\prod_{\substack{e = \left\{\squared{i}, \circled{u}\right\} \in E(R)}} h_{l(e)}(x_{u,i}), &  if $I = U_R$ and $J = V_R$;\\
			0, & \text{otherwise.}
		\end{cases*}
        \end{equation}
        
    \item (Shape vertices)
    Fix a set of fresh symbols of two sorts, one sort ``circle'' and the other ``square'', both having infinite supplies. A shape vertex is such a symbol, not to be confused with $\cC_m\cup\cS_n$.
        
    \item (Shapes) A shape is an undirected, edge-labeled graph $\alpha = (V(\alpha), E(\alpha), U_\alpha, V_\alpha)$ where $V(\alpha)$ is a set of shape vertices and $U_\alpha, V_\alpha \subseteq V(\alpha)$. We call $U_\alpha$, $V_\alpha$ the left- and right- index of $\alpha$ respectively. Note that $U_\alpha$ and $V_\alpha$ may have elements in common. We denote the set of square and circle vertices of $\alpha$ by $V_\square(\alpha)$ and $V_{\bigcirc}(\alpha)$, respectively.
  
    In this paper, we will only use shapes $\alpha$ where $U_\alpha \cup V_\alpha$ contains only square vertices and each edge $e \in E(\alpha)$ is between a square and a circle vertex in $V(\alpha)$ and has an edge-weight $l(e) \in \N$. 
    \item (Shape of a ribbon) We say that a ribbon $R = (V(R), E(R), U_R, V_R)$ has shape $\alpha$ if there is a bijection $\pi:V(\alpha)\to V(R)$ such that $\pi$ maps square shape vertices to $\cS_n$ and circle shape vertices to $\cC_m$, $\pi(I) = U_R$, $\pi(J) = V_R$, $\pi(E(\alpha)) = E(R)$ where all edge-weights match.
		
    \item (Middle vertices, isolated vertices, and being improper) 
    We define the middle vertices of a ribbon $R$ to be $W_R:=V(R)\backslash(U_R\cup V_R)$. We denote by $I_s(R)\subseteq W_R$ the set of middle vertices that are isolated in the underlying graph of $R$. We say $R$ is \emph{proper} if $I_s(R) =\emptyset$ and $R$ has no multi-edges.

    We make analgous definitions for shapes.
         
    \item (Trivial shapes) We say that a shape $\alpha = (V(\alpha), E(\alpha), U_\alpha, V_\alpha)$ is trivial if $V(\alpha) = U_\alpha = V_\alpha$ and $E(\alpha) = \emptyset$.
    
    \item (Transpose) The transpose of a shape $\alpha = (V(\alpha), E(\alpha), U_\alpha, V_\alpha)$ is $\alpha^{\top} = (V(\alpha), E(\alpha), V_\alpha, U_\alpha)$. 
        
    \item (Graph matrices) Given a shape $\alpha = (V(\alpha), E(\alpha), U_\alpha, V_\alpha)$, the graph matrix $M_\alpha$ is 
	\begin{equation}\label{eq:graphmatrix}
        M_\alpha = \sum_{R:\text{ ribbon of shape $\alpha$}} M_R 
    \end{equation}
    \end{enumerate}
\end{definition}

\begin{definition}[Parameters of a shape]\label{def:shape_param} 
We define the following parameters for a shape. 
\begin{enumerate}
    \item (Size) The size $|E|$ of an edge set $E$ is the set cardinality. The size $|V|$ of a vertex set $V$ is the set cardinality.
    
    \item (Weight) The weight $w(E)$ is the sum of the edge-labels. The weight of a square vertex is 1. The weight of a circle vertex is $\log_n m$. The weight $w(V)$ of a vertex set $V$ is the sum of the weights of its members.
    
    \item (Total size) The total size of a ribbon $R$ is $\totalsize(R):=|V(R)|+w(E(R))$. The total size of a shape $\alpha$ is  $|V(\alpha)|+w(E(\alpha))$. 
    Note that each vertex contributes $1$ to the total size regardless of whether it is a square vertex or a circle vertex.
    
    \item (Degree) The degree $\deg_R(v)$ of a vertex $v$ in a ribbon $R$ is the sum of all weights of the edges incident to $v$. The total degree of $v$ is $\deg(v)+1_{v\text{\emph{ is in }} U_R}+1_{v\text{\emph{ is in }}V_R}$, where ``$1_{P}$'' is the indicator function of predicate $P$, which is 1 if $P$ is true and 0 otherwise. The definitions for shapes are similar.
\end{enumerate}
\end{definition}

Note that if we view the edge set of a ribbon $R$ as a vector $a\in (\N^n)^m$ via $a_u(i):=l\left(\{\circled{u},\squared{i}\}\right)$ where $u\in[m]$, $i\in[n]$, then the weight $w(E(R))$ is precisely $\normo{a}$.

Recall that the moment matrix $M_{\pE}$ is obtained by letting its $(I,J)$th entry be $\pE(\vsos^{I+J})$, whose expression in the following corollary follows from inspecting \Cref{lem:calib1} using the Booleanity property in \Cref{lem:pseudo-calib}.

\begin{definition}[Scaling coefficients, Hermite coefficients]\label{def:coeff}
Given a shape $\alpha$, we define the scaling coefficient $\lambda_{\alpha} = {n^{-{w(E(\alpha))\over2}}}$ and the Hermite coefficient $\eta_\alpha = \left(\Prod_{v \in V_{\bigcirc}(\alpha)
}\E_A[h_{\deg_\alpha(v)}]\right)/\left(\Prod_{e\in E(\alpha)}l(e)!\right)$.

To each shape $\alpha$, we associate the scaled graph matrix $\lambda_{\alpha}M_{\alpha}$. Instead of always writing $\lambda_{\alpha}M_{\alpha}$ in a sum or product of matrices, we will sometimes just write $\alpha$ when it is clear that we mean the scaled graph matrix (see \Cref{remark:shorthand}).
\end{definition}

\begin{remark}
We will often have large linear combinations of graph matrices. When we refer to a term in such a linear combination, we mean a graph matrix times its coefficient.
\end{remark}

\begin{corollary}[Moment matrix $M_{\pE}$]\label{cor:graph_matrix_expansion}
The moment matrix from pseudo-calibration \eqref{eq:calib} is 
\begin{equation}\label{eq:Mscale}
M_{\pE}=\diag(n^{-{|I|\over 2}})\cdot M\cdot \diag(n^{-{|I|\over 2}}),
\end{equation}
where 
\begin{equation}\label{eq:ME}
M = \Sum_{\substack{
\text{$\alpha$: proper shape where all square vertices have even total degree}\\
|U_\alpha|\leq \dsos,\ |V_\alpha|\leq\dsos,\ {\totalsize(\alpha)\leq \truncation}}}
{\eta_{\alpha}(\lambda_\alpha M_\alpha)}.
\end{equation}
\end{corollary}

Our main task is to prove the PSDness of the matrix $M$, a rescaled version of the moment matrix which is slightly more convenient to work with. 
A special family of shapes will play a key role in the analysis of $M$, which we call \textit{simple spiders} and their disjoint unions.

\begin{definition}[Simple spiders]\label{def:SS}
A \emph{simple spider} (SS) is either a trivial shape, or a shape $\alpha=(V(\alpha),E(\alpha),U_\alpha, V_\alpha)$ that contains exactly one circle vertex, $V_\square(\alpha) = U_\alpha \cup V_\alpha$, every square vertex in $(U_\alpha\cup V_\alpha)\backslash (U_\alpha\cap V_\alpha)$ has a label-1 edge to the circle vertex, and there are no other edges.

We use $S(i,j;u)$ to denote the simple spider shape $\alpha$ multiplied by its scaling coefficient $\lambda_{\alpha}$, where there are $i$ left-legs, $j$ right-legs, and $u$ intersections. That is, $i = |U_\alpha\setminus V_\alpha|, j = |V_{\alpha}\setminus U_\alpha|, u = |U_\alpha \cap V_\alpha|$. 
\end{definition}

\begin{definition}[Simple spider disjoint union]\label{def:disj}
A {\it simple spider disjoint union} (SSD) is a disjoint set of simple spiders, whose left and right index is the union of the left indices and the right indices of the simple spiders, respectively. 

We use a multiset of simple spiders $\mset{S(a_i,b_i;u_i)\mid i\in[t]}$ to represent an SSD shape. Two multisets $\mset{S(a_i,b_i;u_i)\mid i\in[t]}$ and $\mset{S(a'_i,b'_i;u'_i)\mid i\in[t']}$ give the same SSD if $\Sum_{i=1}^t u_i=\Sum_{i=1}^{t'} u'_i$ and $\mset{(a_i,b_i)\mid a_i+b_i>0}=\mset{(a'_i,b'_i)\mid a'_i+b'_i>0}$.

\end{definition}

Note that a simple spider disjoint union (including a simple spider) is always proper. 
Also, note that the scaling coefficient of an SSD shape is the product of those of its components. 

\begin{definition}[Left and right spiders]
A simple spider $S(i,j;u)$ is {\bf left} if $i\geq j$, and {\bf right} if $i\leq j$. An SSD is left or right if all its components are left or right, respectively. An SS-combination is left or right if all the summands are left or right, respectively, and similarly for SSD-combinations.
\end{definition}

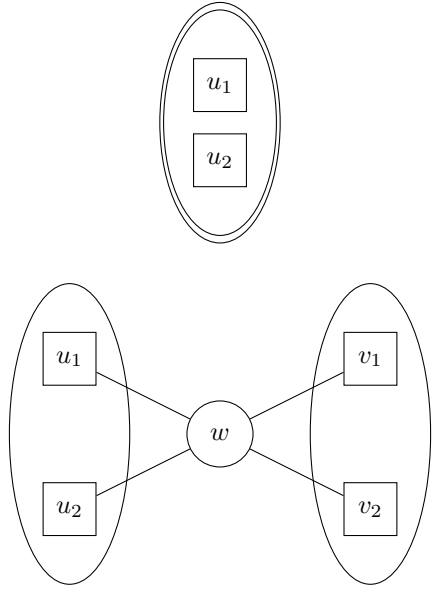
\begin{figure}
    \centering
        \begin{tikzpicture}[
      mycircle/.style={
         circle,
         draw=black,
         fill=white,
         fill opacity = 1,
         text opacity=1,
         inner sep=0pt,
         minimum size=25pt,
         font=\small},
      mysquare/.style={
         rectangle,
         draw=black,
         fill=white,
         fill opacity = 1,
         text opacity=1,
         inner sep=0pt,
         minimum height=20pt, 
         minimum width=20pt,
         font=\small},
      myarrow/.style={-Stealth},
      node distance=0.6cm and 1.2cm
      ]
      \draw (-2,0) ellipse (.8cm and 1.6cm);
      \draw (-2,0) ellipse (.75cm and 1.5cm);
      \node[mysquare]  at (-2, 0.5) (u1) {$u_1$};
      \node[mysquare]  at (-2, -0.5) (u2) {$u_2$};
      \end{tikzpicture}
      
\vspace{14pt}

    \begin{tikzpicture}[
      mycircle/.style={
         circle,
         draw=black,
         fill=white,
         fill opacity = 1,
         text opacity=1,
         inner sep=0pt,
         minimum size=25pt,
         font=\small},
      mysquare/.style={
         rectangle,
         draw=black,
         fill=white,
         fill opacity = 1,
         text opacity=1,
         inner sep=0pt,
         minimum height=20pt, 
         minimum width=20pt,
         font=\small},
      myarrow/.style={-Stealth},
      node distance=0.6cm and 1.2cm
      ]
      \draw (-2,0) ellipse (.8cm and 2cm);
      \draw (2,0) ellipse (.8cm and 2cm);

      \node[mysquare]  at (-2, 1) (u1) {$u_1$};
      \node[mysquare]  at (-2, -1) (u2) {$u_2$};
      \node[mycircle]  at (0, 0) (w) {$w$};
      \node[mysquare]  at (2, 1) (v1) {$v_1$};
      \node[mysquare]  at (2, -1) (v2) {$v_2$};
      \draw[-] (u1) -- (w);
      \draw[-] (u2) -- (w);
      \draw[-] (w) -- (v1);
      \draw[-] (w) -- (v2);
      \end{tikzpicture}
    \caption{Examples of simple spiders where the left and right indices have size two. If $k\geq 3$, the simple spider with one circle vertex and one label-1 edge to each side has coefficient 0 in $M$, so it is not drawn.}
    \label{fig:SS,1}
\end{figure}

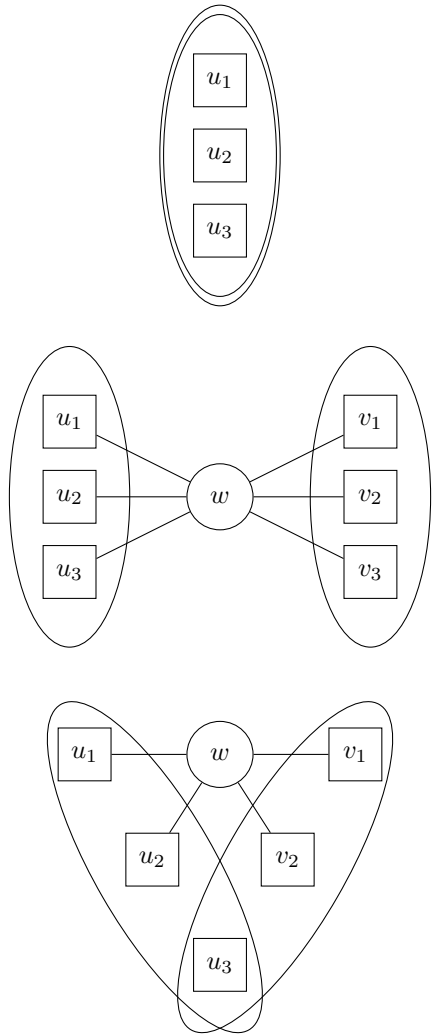
\begin{figure}
    \centering
        \begin{tikzpicture}[
      mycircle/.style={
         circle,
         draw=black,
         fill=white,
         fill opacity = 1,
         text opacity=1,
         inner sep=0pt,
         minimum size=25pt,
         font=\small},
      mysquare/.style={
         rectangle,
         draw=black,
         fill=white,
         fill opacity = 1,
         text opacity=1,
         inner sep=0pt,
         minimum height=20pt, 
         minimum width=20pt,
         font=\small},
      myarrow/.style={-Stealth},
      node distance=0.6cm and 1.2cm
      ]
      \draw (-2,0) ellipse (.8cm and 2cm);
      \draw (-2,0) ellipse (.75cm and 1.875cm);
      \node[mysquare]  at (-2, 1) (u1) {$u_1$};
      \node[mysquare]  at (-2, 0) (u2) {$u_2$};
      \node[mysquare]  at (-2, -1) (u3) {$u_3$};
      \end{tikzpicture}

\vspace{14pt}

    \begin{tikzpicture}[
      mycircle/.style={
         circle,
         draw=black,
         fill=white,
         fill opacity = 1,
         text opacity=1,
         inner sep=0pt,
         minimum size=25pt,
         font=\small},
      mysquare/.style={
         rectangle,
         draw=black,
         fill=white,
         fill opacity = 1,
         text opacity=1,
         inner sep=0pt,
         minimum height=20pt, 
         minimum width=20pt,
         font=\small},
      myarrow/.style={-Stealth},
      node distance=0.6cm and 1.2cm
      ]
      \draw (-2,0) ellipse (.8cm and 2cm);
      \draw (2,0) ellipse (.8cm and 2cm);
      \node[mysquare]  at (-2, 1) (u1) {$u_1$};
      \node[mysquare]  at (-2, 0) (u2) {$u_2$};
      \node[mysquare]  at (-2, -1) (u3) {$u_3$};
      \node[mycircle]  at (0, 0) (w) {$w$};
      \node[mysquare]  at (2, 1) (v1) {$v_1$};
      \node[mysquare]  at (2, 0) (v2) {$v_2$};
      \node[mysquare]  at (2, -1) (v3) {$v_3$};
      \draw[-] (u1) -- (w);
      \draw[-] (u2) -- (w);
      \draw[-] (u3) -- (w);
      \draw[-] (w) -- (v1);
      \draw[-] (w) -- (v2);
      \draw[-] (w) -- (v3);
    \end{tikzpicture}

\vspace{14pt}

    \begin{tikzpicture}[
      mycircle/.style={
         circle,
         draw=black,
         fill=white,
         fill opacity = 1,
         text opacity=1,
         inner sep=0pt,
         minimum size=25pt,
         font=\small},
      mysquare/.style={
         rectangle,
         draw=black,
         fill=white,
         fill opacity = 1,
         text opacity=1,
         inner sep=0pt,
         minimum height=20pt, 
         minimum width=20pt,
         font=\small},
      myarrow/.style={-Stealth},
      node distance=0.6cm and 1.2cm
      ]
      \draw[rotate = 30] (-1,0) ellipse (.8cm and 2.5cm);
      \draw[rotate = -30] (1,0) ellipse (.8cm and 2.5cm);
      \node[mysquare]  at (-1.8, 1) (u1) {$u_1$};
      \node[mysquare]  at (-.9, -.4) (u2) {$u_2$};
      \node[mysquare]  at (0, -1.8) (u3) {$u_3$};
      \node[mycircle]  at (0, 1) (w) {$w$};
      \node[mysquare]  at (1.8, 1) (v1) {$v_1$};
      \node[mysquare]  at (.9, -.4) (v2) {$v_2$};
      \draw[-] (u1) -- (w);
      \draw[-] (u2) -- (w);
      \draw[-] (w) -- (v1);
      \draw[-] (w) -- (v2);
    \end{tikzpicture}
    \caption{Examples of simple spiders where the left and right indices have size three.}
    \label{fig:SS,2}
\end{figure}

\subsection{Tools for Analyzing Graph Matrices}
\label{sec:tools}
In this section, we will describe some tools that we use to analyze graph matrices. 

\subsubsection{Separators and Norm Bounds}
Unless otherwise specified, paths in this paper include length-0 (degenerate) ones.
\begin{definition}[Vertex separators]
Let $\alpha = (V(\alpha), E(\alpha), U_\alpha, V_\alpha)$ be a shape. We say that $S \subseteq V(\alpha)$ is a vertex separator of $\alpha$ if every path from $I$ to $J$ contains at least one vertex in $S$.
\end{definition}

As observed in previous work~\cite{ahn2021graphmatricesnormbounds,ghosh2020sum}, the norms of graph matrices are determined by their minimum weight vertex separator. That said, to give an approximate PSD decomposition $M \approx LQL^{\top}$ of our moment matrix, we will need a new kind of minimum vertex separator, namely minimum square separators.

\begin{definition}[Minimum weight and minimum square separators]\label{def:separators}
Given a shape $\alpha$, we say that a vertex separator $S$ is a \emph{minimum weight vertex separator} if $w(S)$ is minimized, i.e., there is no vertex separator $S'$ such that $w(S') < w(S)$.

Given a shape $\alpha$ such that $I \cup J$ only contains square vertices, we say that a vertex separator $S$ is a \emph{square separator} if it contains only square vertices, and it is a \emph{minimum square separator} if furthermore $|S|$ is minimized, i.e., there is no square separator $S'$ such that $|S'| < |S|$.
\end{definition}

By a standard argument, both the set of minimum weight separators and the set of  minimum square separators form a lattice. The proof of the following two lemmas is in \Cref{sec:error_analysis}.

\begin{lemma}[Menger's theorem for bipartite graphs]\label{lem:menger-sec3}
For any shape $\alpha$, the maximum number of square disjoint paths is equal to the size of the minimum square separator between $U_{\alpha}$ and $V_{\alpha}$. 
\end{lemma}

\begin{definition}
Given a shape $\alpha$ and vertex separators $S,T$, we say that $S$ is to the left of $T$ if $S$ separates $U_\alpha$ and $T$ and we say that $S$ is to the right of $T$ if $S$ separates $V_\alpha$ and $T$.
\end{definition}
\begin{lemma}[Existence of leftmost and rightmost minimum vertex separators]\label{lem:leftmostseparator-sec3}
For all shapes $\alpha$, there exist unique leftmost and rightmost minimum weight separators of $\alpha$. Similarly, for all shapes $\alpha$, there exist unique leftmost and rightmost minimum square separators of $\alpha$.
\end{lemma}

\begin{theorem}[Norm bounds; see e.g. Lemma A.3 of \cite{ghosh2020sum}]\label{thm:norm_control}
There is a universal constant $C$ such that with probability $1-o_n(1)$, 
the following holds simultaneously for all shapes $\alpha$ which have no multi-edges and have total size at most $n$. 
\begin{equation}\label{eq:norm_control}
\norm{M_\alpha} \leq\Big[\Big(1 + w(E(\alpha))\Big)\cdot|V_\alpha|\cdot\log(mn)\Big]^{C\big(\big\vert V(\alpha) \setminus (U_\alpha \cap V_\alpha)\big\vert + w(E(\alpha))\big)} n^{\frac{w(V_\alpha) + w(\Iso(\alpha)) - w(\Sminwt)}{2}},
\end{equation}
where $\Sminwt$ is any choice of the minimum weight vertex separator of $\alpha$.
\end{theorem}

\subsubsection{Linearizing Ribbons and Shapes with Multi-Edges}
For our analysis, we will need to consider ribbons and shapes with multi-edges. To handle ribbons with multi-edges, we reduce each multi-edge to a linear combination of labeled edges.
\begin{definition}\label{def:ribbonexpansion}
Given a ribbon $R$ with multi-edges, we obtain its linearization as follows. For a multi-edge $e$ with labels $l_1,\ldots,l_m$, we replace it with a linear combination of edges with labels $0,\ldots,\sum_{j=1}^{m}{l_m}$ where the coefficient of the edge with label $j$ is the coefficient of $h_j$ in the product $\prod_{j=1}^{m}{h_{l_m}}$. We treat an edge with label 0 as a non-edge.
\end{definition}
\begin{example}
If we have a double edge then this becomes a linear combination of an edge with label $2$ plus an edge with label $0$ (which vanishes) as $x^2 = (x^2 - 1) + 1 = h_2(x) + h_0$.
\end{example}

\begin{example}
If we have a multi-edge consisting of edges with labels $1$, $2$, and $3$ then we would replace this multi-edge with a linear combination of four edges with labels $6$, $4$, $2$, and $0$ (the edge with label $0$ vanishes) and coefficients $1$, $11$, $24$, and $6$ respectively as 
\begin{align*}
(x^3 - 3x)(x^2 - 1)x &= x^6 - 4x^4 + 3x^2 \\
&= (x^6 - 15x^4 + 45x^2 - 15) + 11(x^4 - 6x^2 + 3) + 24(x^2 - 1) + 6
\end{align*}
\end{example}

To linearize the multi-edges of a graph matrix $M_{\alpha}$ where $\alpha$ has multi-edges, we use \Cref{def:ribbonexpansion} repeatedly to replace each ribbon with multi-edges by its linearizations. For a ribbon $R$ of shape $\alpha$, this creates a DAG-like process where $R$ is the top node, and each non-leaf node is a ribbon $R'$ which still has multi-edges and continues to expand to its children. The edge from node $R'$ to a child records the coefficient of that child in this step of expanding $R'$, and the leafs are ribbons with no multi-edges. The ribbon $R$ will be replaced by the sum over all sinks, each with a coefficient equal to the sum over all paths from the root to this sink where each path contributes the product of all coefficients on it. 

Note that the final expansion $\sum_i {c_i}R_i$ depends only on $R$ but not the order in the expansion process. Also, if we sum over all $R$ of shape $\alpha$, the result is invariant under vertex permutations. Thus, each shape $\alpha$ uniquely expands to a sum of shapes having no multi-edges. 

The expansion of $\alpha$ can also be described step-wise on the shape level. This requires a bit more care in adjusting the coefficient of the resulting shapes, as the number of ways of realizing a shape can change after a multi-edge-replacement step, due to the change in the size of the automorphism group for the shape. In any case, we have the following.

\begin{proposition}[Expanding multi-edges, cf. Prop. 5.29--5.31 of \cite{ghosh2020sum}]\label{prop:expansion}
Let $\alpha$ be an improper shape, and let $P$ be the set of proper shapes that can be obtained by expanding $\alpha$. Then there are coefficients $|c_{\gamma}| \leq C_{\text{Fourier}} \cdot C_{\text{Aut}}$ such that
\begin{equation} \label{eq:multiedge_expansion}
M_{\alpha} = \sum_{\gamma \in P} c_{\gamma}M_{\gamma} 
\end{equation}
where $C_{\text{Fourier}}$ (``Hermite'' in our case) is an upper bound on the largest absolute value of the coefficients of Hermite polynomials in the expansion of a product of any Hermite polynomials in the process described above, and $C_{\text{Aut}} = \max_{\gamma \in P} \frac{|Aut(\gamma)|}{|Aut(\alpha)|}$.
\end{proposition}

The magnitudes of $C_{\text{Aut}}$ and $C_{\text{Fourier}}$ can be bounded as follows.

\begin{lemma}[Bounds on $C_{\text{Aut}}$ and $C_{\text{Fourier}}$, cf. Propositions. 5.30 and 5.31 of \cite{ghosh2020sum}]\label{lem:expansionbounds} \ 
\begin{enumerate}
    \item If $l_1 \leq \cdots \leq l_k \in \mathbb{N}$ and $l: = l_1 + \cdots + l_k > 1$, then in the Hermite expansion of $h_{l_1}(z) \cdots h_{l_k}(z)$, the maximum magnitude of a coefficient is no more than $(2l)^{l-l_k}$.
    
    \item If $\alpha'$ is obtained from $\alpha$ by adding an edge, deleting an edge, or changing the label of an edge then $\frac{|Aut(\alpha')|}{|Aut(\alpha)|} \leq {|V(\alpha)|}^{2}$.
\end{enumerate}
\end{lemma}

Using \Cref{lem:expansionbounds} and \Cref{prop:expansion}, we get the following corollary of \Cref{thm:norm_control}.

\begin{corollary}\label{cor:expansionbounds}
If the constant $C$ in \Cref{thm:norm_control} is chosen to be sufficiently large, then for any improper shape $\alpha$ with multi-edges, the same upper bound there also upper bounds $\Sum_{\gamma\in P} \norm{c_\gamma M_{\gamma}}$ from the expansion \eqref{eq:multiedge_expansion}. In particular, the norm bound in \Cref{thm:norm_control} also holds for improper shapes with multi-edges.
\end{corollary}

For convenience, we choose the universal constant $C$ to be large enough so that \Cref{cor:expansionbounds} applies. We denote this constant by $\Cuniv\geq 1$ in the rest of the paper.

\subsubsection{Decomposition of Shapes}\label{sec:canonical_decomposition}
In our analysis, we will decompose each shape into three parts based on the leftmost and rightmost minimum square separators.

\begin{definition}[Left and right shapes]
We say that $\sigma = (V(\sigma),E(\sigma),U_\sigma, V_\sigma)$ is a left shape if $V_\sigma$ is the unique minimum square separator of $\sigma$ and every vertex in $V(\sigma) \setminus V_{\sigma}$ is reachable by a path from $U_\sigma$ without passing through any vertices in $V_\sigma$. We say that $\sigma$ is a right shape if $\sigma^{\top}$ is a left shape. 
\end{definition}
\begin{definition}[Middle shapes]
We say that $\tau = (V(\tau),E(\tau),U_\tau, V_\tau)$ is a middle shape if $U_\tau$ and $V_\tau$ are  minimum square separators of $\tau$. A \emph{proper middle shape} is a middle shape that is also proper. 
\end{definition}
We define left, right, and middle ribbons according to their shapes.

\begin{definition}[Square canonical decomposition]\label{def:canonicaldecomposition}
Given a shape $\alpha = (V(\alpha),E(\alpha),U_\alpha, V_\alpha)$ we decompose it into a triple of shapes $(\alpha_l,\alpha_m,\alpha_r)$ as follows. Let $S_l$, $S_r$ be the leftmost and rightmost minimum square separators of $\alpha$ respectively.
\begin{enumerate}
    \item The left shape of $\alpha$, $\alpha_l$, is the part of $\alpha$ between $U_\alpha$ and $S_l$. That is, $\alpha_l = (V(\alpha_l),E(\alpha_l),U_\alpha,S_l)$ where $V(\alpha_l)$ and $E(\alpha_l)$ are the vertices and edges of $\alpha$ which are reachable from $U_{\alpha}$ without passing through a vertex in $S_l$. Note that $V(\alpha_l)$ includes the vertices in $S_l$.
    \item The right shape of $\alpha$, $\alpha_r$, is the part of $\alpha$ between $S_r$ and $V_\alpha$. That is, $\alpha_l = (V(\alpha_r),E(\alpha_r),S_r,V_\alpha)$ where $V(\alpha_r)$ and $E(\alpha_r)$ are the vertices and edges of $\alpha$ which are reachable from $V_\alpha$ without passing through a vertex in $S_r$. Note that $V(\alpha_r)$ includes the vertices in $S_r$.
    \item The middle shape of $\alpha$, $\alpha_m$, is the remaining part of $\alpha$ together with the vertices in $S_l$ and $S_r$. That is, $\alpha_m = (V(\alpha_m),E(\alpha_m),S_l,S_r)$ where $V(\alpha_m)$ and $E(\alpha_m)$ are the vertices and edges of $\alpha$ which cannot be reached from $U_\alpha$ without touching a vertex in $S_l$ and which cannot be reached from $V_\alpha$ without touching a vertex in $S_r$.
\end{enumerate}
We call $(\alpha_l,\alpha_m,\alpha_r)$ the square canonical decomposition of $\alpha$.
\end{definition}
\begin{proposition}
If $\alpha$ is a proper shape then $\alpha_l$ is a proper left shape, $\alpha_m$ is a proper middle shape, and $\alpha_r$ is a proper right shape.
\end{proposition}

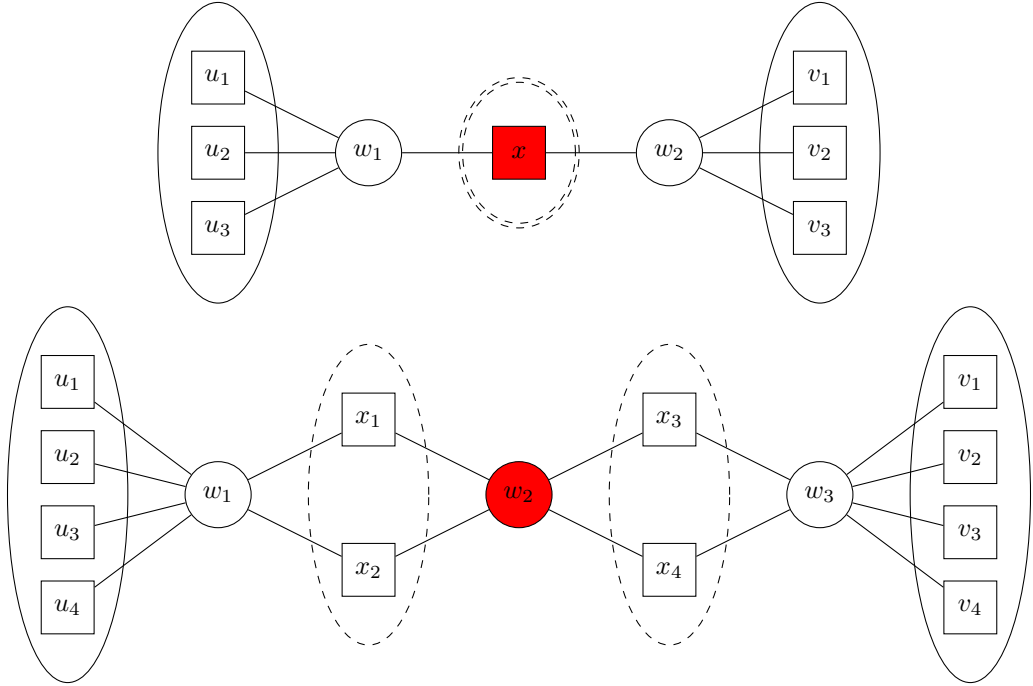
\begin{figure}
    \centering
    \begin{tikzpicture}
    [
      mycircle/.style={
         circle,
         draw=black,
         fill=white,
         fill opacity = 1,
         text opacity=1,
         inner sep=0pt,
         minimum size=25pt,
         font=\small},
      mysquare/.style={
         rectangle,
         draw=black,
         fill=white,
         fill opacity = 1,
         text opacity=1,
         inner sep=0pt,
         minimum height=20pt, 
         minimum width=20pt,
         font=\small},
      myarrow/.style={-Stealth},
      node distance=0.6cm and 1.2cm
      ]
      \draw (-4,0) ellipse (.8cm and 2cm);
      \draw (4,0) ellipse (.8cm and 2cm);
      \draw[dashed] (0,0) ellipse (.8cm and 1cm);
      \draw[dashed] (0,0) ellipse (.75cm and .9375cm);
      \node[mysquare]  at (-4, 1) (u1) {$u_1$};
      \node[mysquare]  at (-4, 0) (u2) {$u_2$};
      \node[mysquare]  at (-4, -1) (u3) {$u_3$};
      \node[mycircle]  at (-2, 0) (w1) {$w_1$};
      \node[mysquare, fill=red]  at (0, 0) (x) {$x$};
      \node[mycircle]  at (2, 0) (w2) {$w_2$};
      \node[mysquare]  at (4, 1) (v1) {$v_1$};
      \node[mysquare]  at (4, 0) (v2) {$v_2$};
      \node[mysquare]  at (4, -1) (v3) {$v_3$};
      \draw[-] (u1) -- (w1);
      \draw[-] (u2) -- (w1);
      \draw[-] (u3) -- (w1);
      \draw[-] (w1) -- (x);
      \draw[-] (x) -- (w2);
      \draw[-] (w2) -- (v1);
      \draw[-] (w2) -- (v2);
      \draw[-] (w2) -- (v3);
    \end{tikzpicture}
    \begin{tikzpicture}
    [
      mycircle/.style={
         circle,
         draw=black,
         fill=white,
         fill opacity = 1,
         text opacity=1,
         inner sep=0pt,
         minimum size=25pt,
         font=\small},
      mysquare/.style={
         rectangle,
         draw=black,
         fill=white,
         fill opacity = 1,
         text opacity=1,
         inner sep=0pt,
         minimum height=20pt, 
         minimum width=20pt,
         font=\small},
      myarrow/.style={-Stealth},
      node distance=0.6cm and 1.2cm
      ]
              \draw (-6,0) ellipse (.8cm and 2.5cm);
      \draw (6,0) ellipse (.8cm and 2.5cm);
      \draw[dashed] (-2,0) ellipse (.8cm and 2cm);
      \draw[dashed] (2,0) ellipse (.8cm and 2cm);
      \node[mysquare]  at (-6, 1.5) (u1) {$u_1$};
      \node[mysquare]  at (-6, 0.5) (u2) {$u_2$};
      \node[mysquare]  at (-6, -0.5) (u3) {$u_3$};
      \node[mysquare]  at (-6, -1.5) (u4) {$u_4$};
      \node[mycircle]  at (-4, 0) (w1) {$w_1$};
      \node[mysquare]  at (-2, 1) (x1) {$x_1$};
      \node[mysquare]  at (-2, -1) (x2) {$x_2$};
      \node[mycircle, fill = red]  at (0, 0) (w2) {$w_2$};
      \node[mysquare]  at (2, 1) (x3) {$x_3$};
      \node[mysquare]  at (2, -1) (x4) {$x_4$};
      \node[mycircle]  at (4, 0) (w3) {$w_3$};
      \node[mysquare]  at (6, 1.5) (v1) {$v_1$};
      \node[mysquare]  at (6, 0.5) (v2) {$v_2$};
      \node[mysquare]  at (6, -0.5) (v3) {$v_3$};
      \node[mysquare]  at (6, -1.5) (v4) {$v_4$};
      \draw[-] (u1) -- (w1);
      \draw[-] (u2) -- (w1);
      \draw[-] (u3) -- (w1);
      \draw[-] (u4) -- (w1);
      \draw[-] (w1) -- (x1);
      \draw[-] (w1) -- (x2);
      \draw[-] (x1) -- (w2);
      \draw[-] (x2) -- (w2);
      \draw[-] (w2) -- (x3);
      \draw[-] (w2) -- (x4);
      \draw[-] (x3) -- (w3);
      \draw[-] (x4) -- (w3);
      \draw[-] (w3) -- (v1);
      \draw[-] (w3) -- (v2);
      \draw[-] (w3) -- (v3);
      \draw[-] (w3) -- (v4);
    \end{tikzpicture}
    \caption{Examples of the decomposition into left, middle, and right parts. The sets $S_l$ and $S_r$ are shown with dotted ovals, and assuming $n \leq m \leq n^2$, a minimum weight vertex separator is shown in red.}
    \label{fig:canonical1}
\end{figure}
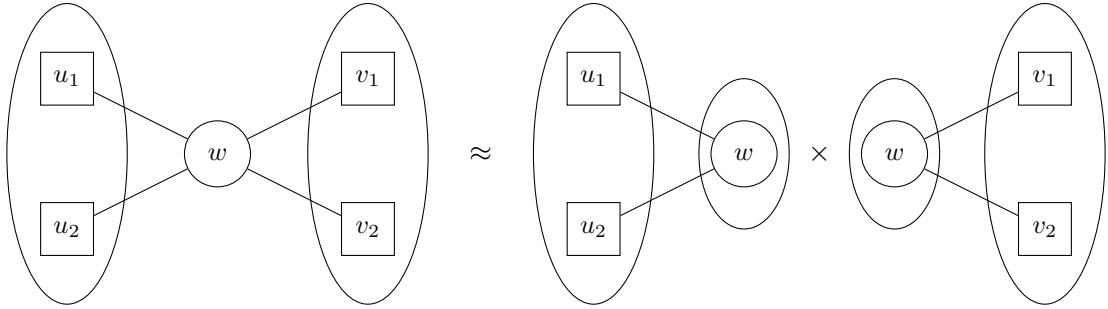
\begin{figure}[!hbt]
    \centering
    \begin{tikzpicture}[
      mycircle/.style={
         circle,
         draw=black,
         fill=white,
         fill opacity = 1,
         text opacity=1,
         inner sep=0pt,
         minimum size=25pt,
         font=\small},
      mysquare/.style={
         rectangle,
         draw=black,
         fill=white,
         fill opacity = 1,
         text opacity=1,
         inner sep=0pt,
         minimum height=20pt, 
         minimum width=20pt,
         font=\small},
      myarrow/.style={-Stealth},
      node distance=0.6cm and 1.2cm
      ]
      \draw (-6,0) ellipse (.8cm and 2cm);
      \draw (-2,0) ellipse (.8cm and 2cm);

      \node[mysquare]  at (-6, 1) (u1) {$u_1$};
      \node[mysquare]  at (-6, -1) (u2) {$u_2$};
      \node[mycircle]  at (-4, 0) (w) {$w$};
      \node[mysquare]  at (-2, 1) (v1) {$v_1$};
      \node[mysquare]  at (-2, -1) (v2) {$v_2$};
      \draw[-] (u1) -- (w);
      \draw[-] (u2) -- (w);
      \draw[-] (w) -- (v1);
      \draw[-] (w) -- (v2);
      \node at (-0.5,0) {$\approx$};
      \draw (1,0) ellipse (.8cm and 2cm);
      \draw (3,0) ellipse (.6cm and 1cm);
      \node[mysquare]  at (1, 1) (up1) {$u_1$};
      \node[mysquare]  at (1, -1) (up2){$u_2$};
      \node[mycircle]  at (3, 0) (wp) {$w$};
      \draw[-] (up1) -- (wp);
      \draw[-] (up2) -- (wp);
      \node at (4,0) {$\times$};
      \draw (5,0) ellipse (.6cm and 1cm);
      \draw (7,0) ellipse (.8cm and 2cm);
      \node[mycircle]  at (5, 0) (wpp) {$w$};
      \node[mysquare]  at (7, 1) (vp1) {$v_1$};
      \node[mysquare]  at (7, -1) (vp2){$v_2$};
      \draw[-] (wpp) -- (vp1);
      \draw[-] (wpp) -- (vp2);
      \end{tikzpicture}
    \caption{This figure shows the approximate decomposition of $S(2,2;0)$ using minimum weight vertex separators (assuming that $m \leq n^2$). While natural, this decomposition {\color{red}does not work} for our analysis.}
    \label{fig:failureofweightseparator}
\end{figure}

\begin{remark}
A natural attempt is to use minimum weight vertex separators for the decomposition of shapes into left, middle, and right parts. However, this decomposition is not the right one for our PSDness analysis. In particular, one property we want for our PSDness analysis is that square terms (i.e., terms which are composed of a left shape and its transpose) have a positive coefficient. However, the coefficient of $S(2,2;0)$ (which is a square term with this decomposition) is $l_4$ which can be negative. For an illustration, see \Cref{fig:failureofweightseparator}.
\end{remark}

\subsubsection{Indices With Parity Labels}
\label{sec:parity}
One complication for the matrix $M$ is the parity condition on the total degree of a square vertex in \eqref{eq:ME}.  
This means that if we factor $M$ as $XY$ where $X$, $Y$ are linear combinations of graph matrices, then in the multiplication $M_\alpha M_\beta$ where $M_A\in X$ and $M_B\in Y$, we had better match the parities of $\deg_{\alpha}(v)$ and $\deg_{\beta}(v)$ for all $v\in V_\alpha=U_\beta$. For this reason, we refine Definition \ref{def:graph_matrix} with a parity labeling as follows.

\begin{definition}[Modified graph matrix definitions with parities] \label{def:graphmatrixparity} \ 
\begin{enumerate}
    \item\label{item:pindices} (Matrix indices) A matrix index is a pair $(I,a)$ where $I\subseteq \cS$ and $a: I\to\{odd,even\}$ is a function we call the parity labeling on $I$.
    \item\label{item:pribbonsshapes} (Ribbons, shapes) A ribbon $R$ is defined as before except the two indices are now $(U_R,a_R),(V_R,b_R)$ where $a_R:U_R\to\{odd,even\}$ and $b_R:V_R\to\{odd,even\}$ are parity labels such that 
    \begin{enumerate}
        \item For each vertex $u \in U_R \setminus V_R$, $a_R$ maps $u$ to the parity of its {\bf total degree} in $R$. Similarly, each vertex $v \in V_R \setminus U_R$, $b_R$ maps $v$ to the parity of its {\bf total degree} in $R$.
        \item For each vertex $u \in U_R \cap V_R$, if the {\bf total degree} of $u$ is even then $a_R$ and $b_R$ map $u$ to the same parity and if the {\bf total degree} of $u$ is odd then $a_R$ and $b_R$ map $u$ to opposite parities.
    \end{enumerate}
    In other words, for each vertex $u \in U_R \cup V_R$, the sum of the parity label(s) for $u$ is the parity of the {\bf total degree} of $u$.

    We modify the definition of a shape similarly. 
    \item\label{item:pshapeofribbon} (Shape of a ribbon) We say a ribbon $R$ has shape $\alpha$ if there is a bijection $\pi:V(\alpha)\to V(R)$ that satisfies the previous conditions plus $a_\alpha=a_R\circ\pi$ and $b_\alpha=b_R\circ\pi$. 

    \item\label{item:pgraphmatrix} (Ribbon matrices, graph matrices) A ribbon matrix $M_R$ is defined by \eqref{eq:ribbonmatrix} except now the only nonzero element is at position $\big((I,a),(J,b)\big)=\big((U_R,a_R),(V_R,b_R)\big)$. 
    The graph matrix $M_\alpha$ is defined to be $\Sum_{R:\text{ has shape }\alpha} M_R$. 
    For each shape $\alpha$, we associate the scaled graph matrix $n^{-{w(E(\alpha))\over 2}}M_\alpha$ to it.
    \item\label{item:pcanonicaldecomp}(Canonical decomposition) When decomposing a shape $\alpha$ into $(\alpha_l,\alpha_m,\alpha_r)$, the parity labelings on indices of $\alpha_l$, $\alpha_m$, $\alpha_r$ are decided by the shapes themselves via \Cref{item:pribbonsshapes}. Note the parity labelings on $U_\alpha$ are the same in $\alpha_l$ and $\alpha$, and similarly for $V_\alpha$ in $\alpha_r$ and $\alpha$. 

    \item\label{item:ptranspose}(Transpose) For $\alpha = \big(V(\alpha), E(\alpha), (U_\alpha,a_\alpha), (V_\alpha,b_\alpha)\big)$, $\alpha^{\top}:= \big(V(\alpha), E(\alpha), (V_\alpha,b_\alpha), (U_\alpha,a_\alpha)\big)$. 
\end{enumerate}
Other concepts about ribbons and shapes (size, weight, middle vertices, isolated vertices, total size, trivial shapes, left/right/middle shapes) are defined in the same way as before. 
\end{definition}

While adding the parities is important for describing the correct factorization, its impact on the analysis turns out to be superficial. For the remainder of the paper, we use indices with parities but sweep the parities under the rug. That said, we give explanations at necessary places. The first one is the following.

\begin{remark}[Norm bounds for shapes with parities]
The norm bounds in \Cref{thm:norm_control} and \Cref{cor:expansionbounds} also hold in the setting with parities. 
For clarity, in this remark, we distinguish the previous and current objects by the superscript ``parity''. 
For fixed $\alpha$, \Cref{item:pribbonsshapes} implies that ribbons $R$ of shape $\alpha$ are in 1-1 correspondence to ribbons $R^{\parity}$ of shape $\alpha^{\parity}$. The nonzero Fourier characters in $M_\alpha$ are also in 1-1 correspondence to those in $M_{\alpha}^{\parity}$, although their relative positions can be different: it can happen that some $R_1,R_2$ have the same matrix indices while $R_1^{\parity},R_2^{\parity}$ have different ones. The trace power estimate for proving \Cref{thm:norm_control} applies to $M_\alpha^{\parity}$, since for all $t$, any product $R^{\parity}_1\cdots R^{\parity}_{2t}$ in $\left( M_{\alpha}^{\parity}\cdot\big( M_{\alpha}^{\parity}\big)^{\top}\right)^t$ is a product in $\left( M_{\alpha}\cdot M_{\alpha}^{\top}\right)^t$ by ``forgetting'' the parity, with the same expression in the resulting entry. Note the converse is not true: some products in $\left( M_{\alpha}\cdot M_{\alpha}^{\top}\right)^t$ may be non-multiplicable in $\left( M_{\alpha}^{\parity}\cdot\big( M_{\alpha}^{\parity}\big)^{\top}\right)^t$ because of unmatched matrix indices. Given \Cref{thm:norm_control}, \Cref{cor:expansionbounds} follows by applying \Cref{prop:expansion} and \Cref{lem:expansionbounds} as before.
\end{remark}

\subsubsection{Approximate Decomposition of $M$}
\begin{definition}[Composable]
We say that ribbons $R_1$ and $R_2$ are composable if $U_{R_2} = V_{R_1}$. More generally, we say that a ribbon sequence $R_1,\ldots,R_j$ is composable if $U_{R_{i+1}} = V_{R_i}$ for all $i\in [j-1]$. We make an analogous definition for shapes. Note that being composable is not a commutative property. 
\end{definition}

\begin{definition}[Ribbon composition]\label{def:composition}
    We call a composable ribbon sequence $R_1,\ldots,R_j$ a \emph{ribbon composition} or \emph{ribbon product}. The \emph{result} $R$ of the ribbon composition is the graph union, which is viewed as a ribbon that can have multi-edges and has $U_{R}=U_{R_1}$, $V_{R}=V_{R_j}$. 
\end{definition}

\begin{definition}[Properly composable; extra intersections]\label{def:extraintersections}
We say that ribbons $R_1$ and $R_2$ are properly composable, if they are composable and $V(R_1) \cap V(R_2) = V_{R_1} = U_{R_2}$. 
More generally, we say that ribbons $R_1,\ldots,R_j$ are properly composable, if they are composable and the \emph{pairwise extra intersections}, which we sometimes also refer to as the \emph{pairwise non-trivial intersections}, $V_{\extraintersect}(R_{i_1},R_{i_2}):=\Big(V(R_{i_1}) \cap V(R_{i_2})\Big) \backslash\Big(\mathop{\bigcap}\limits_{i' \in [i_1+1,i_2]}{U_{R_{i'}}}\Big)$ are empty for all $i_1 < i_2$. 
\end{definition}
\begin{definition}[Canonical product]  
Given shapes $\alpha_1,\alpha_2,\ldots,\alpha_j$, we define the canonical product 
\[
[M_{\alpha_1},M_{\alpha_2},\ldots, M_{\alpha_j}]_{\can} = \sum_{\substack{(R_1,\ldots,R_j):\\ R_1,\ldots,R_j \text{ are properly composable and have shapes }\alpha_1,\ldots,\alpha_j}}{M_{R_1}M_{R_2}\ldots{M_{R_j}}}
\]
which is $0$ if the shapes are not composable. We extend $[\ldots]_{\can}$ to linear combinations of graph matrices by making it a multi-linear operator.
\end{definition}

\begin{definition}[Matrix $L$]\label{def:L}
We define 
\[
L = \mathop{\sum}\limits_{\substack{\sigma: \\
\sigma \text{ is a proper left shape, all square vertices in $V(\sigma) \setminus V_{\sigma}$ have even total degree;}\\
\vert U_\sigma\vert\leq \dsos,\ \vert V_\sigma\vert\leq \dsos,\ \vert \totalsize(\sigma)\vert\leq\truncation}}
{\eta_\sigma(\lambda_{\sigma}M_{\sigma})}.
\]
\end{definition}
\begin{definition}[Matrix $Q_0$]\label{def:Q0}
We define 
\[Q_0 = \mathop{\sum}\limits_{\substack{\tau: 
\\ \tau \text{ is a proper middle shape}, \text{ all square vertices in $V(\tau) \setminus (U_{\tau} \cup V_{\tau})$ have even total degree};
\\ \vert U_\tau\vert\leq \dsos,\ \vert V_\tau\vert\leq \dsos,\ \vert \totalsize(\sigma)\vert\leq\truncation}} 
{\eta_\tau(\lambda_{\tau}M_{\tau})}.\]
\end{definition}

\begin{remark}[Parity labels for $L$ and $Q_0$]
For left shapes $\sigma$ in $L$, we define the parity labeling $a_{\sigma}$ on $U_{\sigma}$ to be all even; note that this automatically determines $b_{\sigma}$ on $V_{\sigma}$. 
Similarly, for right shapes $\sigma^\top$ in $L^{\top}$, we define the parity labeling $b_{\sigma^\top}$ on $V_{\sigma^\top}$ to be all even. 

For middle shapes $\tau\in Q_0$, there are two possibilities for the parity labels of each vertex $u \in U_{\sigma} \cap V_{\sigma}$ and we take both of these possibilities. Technically, this gives $2^{|U_{\tau} \cap V_{\tau}|}$ different terms for $\tau$ but we group (sum) these terms together and think of them as one term. 
The matrix norm bounds \Cref{thm:norm_control} and \Cref{cor:expansionbounds} apply to this sum as well. To see this, note that both the matrix row set and matrix column set can be partitioned into $2^{|U_{\tau} \cap V_{\tau}|}$ subsets such that each term is supported on a distinct row subset and a distinct column subset, so the norm of the sum is at most the maximum norm of the individual terms.
\end{remark}

We observe that with the parity labels, $M = [L,Q_0,L^{\top}]_{\can}$ up to a truncation error. 
In particular, the only terms appearing in $[L,Q_0,L^{\top}]_{\can} - M$ are terms where $|\totalsize(\alpha_l)| \leq \truncation$, $|\totalsize(\alpha_m)| \leq \truncation$, and $|\totalsize(\alpha_r)| \leq \truncation$ but $|\totalsize(\alpha)| > \truncation$. In \Cref{sec:error_analysis}, we show that $\norm{M - [L,Q_0,L^{\top}]_{\can}}\leq n^{-\Omega(\eps\truncation)}$ (\Cref{lem:Mcan}).

\subsection{Intersection Configurations and Finding $Q$}\label{sec:recursive}
To find a $Q$ such that $M = LQL^{\top}$, we use a recursive procedure as in \cite{BHKKMP16,pang2021sos,jones2021sum, jones2023sum, KPW24}. 
We start with $M_{target} = M$ and $i = 0$, then do the following:
\begin{definition}[Recursive factorization]\label{def:recursivefactorization}
The recursive factorization of the moment matrix $M$ (see \Cref{def:moment_matrix}) goes as follows. 
In round $i=0,1,\dots,2\dsos$, 
\begin{enumerate} 
\item\label{recursive1} We choose $Q_i$ so that $[L,Q_i,L^{\top}]_{\can} = M_{\target}$ up to a truncation error;
\item\label{recursive2} 
We update $M_{\target}\leftarrow M_{\target}-L{Q_i}L^{\top}$, $i\leftarrow i+1$.
\end{enumerate}

We repeat the procedure and show that it is guaranteed to terminate in $2\dsos$ steps i.e., $M_{\target}$ becomes 0. Finally, we take $Q = \Sum_{i=0}^{2\dsos}{Q_i}$. 
\end{definition}

Each $Q_i$ can be understood as follows. For simplicity, we will ignore truncation errors for now by writing ``$\approx$'' until \Cref{def:L<=t}. First, note that $[L,Q_0,L^\top] \approx M$. Then after round $i$, $M_{\target} \approx [L,Q_{i},L^{\top}]_{\can} - LQ_{i}L^{\top}$. That is, $M_{\target}$ is minus one times the sum of all non-canonical products (with coefficients) in $L{Q_{i}}L^{\top}$ up to a truncation error.

To find a $Q_i$ that fulfills item 1 in \Cref{def:recursivefactorization}, for each non-canonical product $(R_1,R_2,R_3)$ in $L{Q_{i-1}}L^{\top}$, we construct a tuple $(R'_1,R'_2,R'_3)$ such that $M_{R_1}M_{R_2}M_{R_3} = M_{R'_1}M_{R'_2}M_{R'_3}$, the shape of $R'_1$ is a left shape, the shape of $R'_3$ is a right shape, and $R'_1,R'_2,R'_3$ are properly composable. We find the tuple $(R'_1,R'_2,R'_3)$ as follows:

\begin{definition}[Separating decomposition]\label{def:find_Qi}\ 
\begin{enumerate}
\item Let $A'$ be the leftmost minimum {\bf square} separator of $A_{R_1}$ and $B_{R_1} \cup V_{\extraintersect}(R_1,R_2)\cup V_{\extraintersect}(R_1,R_3)$, where $V_{\extraintersect}(R_i,R_j)$ is the set of vertices involved in extra intersections between $R_i$ and $R_j$. Similarly, let $B'$ be the rightmost minimum {\bf square} separator of $A_{R_3} \cup V_{\extraintersect}(R_1,R_3) \cup V_{\extraintersect}(R_2,R_3)$ and $B_{R_3}$ in $R_3$ 

\item Let $R'_1$ be the part of $R_1$ between $A_{R_1}$ and $A'$ (cf. \Cref{def:canonicaldecomposition}), and let $R''_1$ be the part of $R_1$ between $A'$ and $B_{R_1}$. Similarly, let $R'_3$ be the part of $R_3$ between $B'$ and $B_{R_3}$, and let $R''_3$ be the part of $R_3$ between $A_{R_3}$ and $B'$.

\item We define $R'_2$ to be the composition of $R''_1$, $R_2$, and $R''_3$. Note that $R'_2$ might be improper.

\item We choose the parities for vertices in $A' \cap B'$ in $R'_2$ to match the parities of their total degrees in $R'_1$ and $R'_3$.
\end{enumerate}
\end{definition}

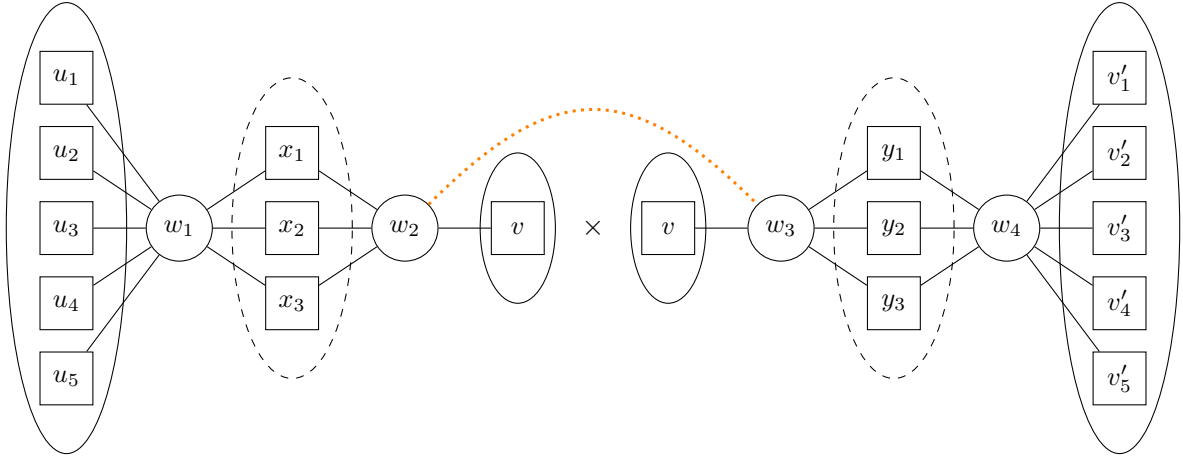
\begin{figure}
    \centering
    \begin{tikzpicture}[
      mycircle/.style={
         circle,
         draw=black,
         fill=white,
         fill opacity = 1,
         text opacity=1,
         inner sep=0pt,
         minimum size=25pt,
         font=\small},
      mysquare/.style={
         rectangle,
         draw=black,
         fill=white,
         fill opacity = 1,
         text opacity=1,
         inner sep=0pt,
         minimum height=20pt, 
         minimum width=20pt,
         font=\small},
      myarrow/.style={-Stealth},
      node distance=0.6cm and 1.2cm
      ]
      \draw (-7,0) ellipse (.8cm and 3cm);
      \draw[dashed] (-4,0) ellipse (.8cm and 2cm);
      \draw (-1,0) ellipse (.5cm and 1cm);
      \node at (0,0) {$\times$};
      \draw (1,0) ellipse (.5cm and 1cm);
      \draw[dashed] (4,0) ellipse (.8cm and 2cm);
      \draw (7,0) ellipse (.8cm and 3cm);
      \node[mysquare]  at (-7, 2) (u1) {$u_1$};
      \node[mysquare]  at (-7, 1) (u2) {$u_2$};
      \node[mysquare]  at (-7, 0) (u3) {$u_3$};
      \node[mysquare]  at (-7, -1) (u4) {$u_4$};
      \node[mysquare]  at (-7, -2) (u5) {$u_5$};
      \node[mycircle]  at (-5.5, 0) (w1) {$w_1$};
      \node[mysquare]  at (-4, 1) (x1) {$x_1$};
      \node[mysquare]  at (-4, 0) (x2) {$x_2$};
      \node[mysquare]  at (-4, -1) (x3) {$x_3$};
      \node[mycircle]  at (-2.5, 0) (w2) {$w_2$};
      \node[mysquare]  at (-1, 0) (v) {$v$};
      \node[mysquare]  at (1, 0) (up) {$v$};
      \node[mycircle]  at (2.5, 0) (w3) {$w_3$};
      \node[mysquare]  at (4, 1) (y1) {$y_1$};
      \node[mysquare]  at (4, 0) (y2) {$y_2$};
      \node[mysquare]  at (4, -1) (y3) {$y_3$};
      \node[mycircle]  at (5.5, 0) (w4) {$w_4$};
      \node[mysquare]  at (7, 2) (vp1) {$v'_1$};
      \node[mysquare]  at (7, 1) (vp2) {$v'_2$};
      \node[mysquare]  at (7, 0) (vp3) {$v'_3$};
      \node[mysquare]  at (7, -1) (vp4) {$v'_4$};
      \node[mysquare]  at (7, -2) (vp5) {$v'_5$};
      \draw[-] (u1) -- (w1);
      \draw[-] (u2) -- (w1);
      \draw[-] (u3) -- (w1);
      \draw[-] (u4) -- (w1);
      \draw[-] (u5) -- (w1);
      \draw[-] (w1) -- (x1);
      \draw[-] (w1) -- (x2);
      \draw[-] (w1) -- (x3);
      \draw[-] (x1) -- (w2);
      \draw[-] (x2) -- (w2);
      \draw[-] (x3) -- (w2);
      \draw[-] (w2) -- (v);
      \draw[-] (up) -- (w3);
      \draw[-] (w3) -- (y1);
      \draw[-] (w3) -- (y2);
      \draw[-] (w3) -- (y3);
      \draw[-] (y1) -- (w4);
      \draw[-] (y2) -- (w4);
      \draw[-] (y3) -- (w4);
      \draw[-] (w4) -- (vp1);
      \draw[-] (w4) -- (vp2);
      \draw[-] (w4) -- (vp3);
      \draw[-] (w4) -- (vp4);
      \draw[-] (w4) -- (vp5);
      \draw[color=orange, dotted, line width=0.4mm] (w2) .. controls (-0.5,2) and (0.5,2) .. (w3);
      \end{tikzpicture}
    \caption{An intersection configuration for the product $M_{\sigma}M_{\sigma^{\top}}$ where $\sigma$ is the left shape shown on the left. The orange arch between shape vertices $w_2$ and $w_3$ means that the two vertices coincide in the ribbon products being represented. The dotted set of vertices in $\sigma$ is the leftmost minimum square separator between $U_{\sigma}$ and the union of the set of intersected vertices and $V_{\sigma}$. Similarly the dotted set of vertices in $\sigma^{\top}$ is the rightmost minimum vertex separator between the union of $U_{\sigma^{\top}}$ and the set of the intersected vertices and $V_{\sigma^{\top}}$.}
    \label{fig:intersection_term}
\end{figure}

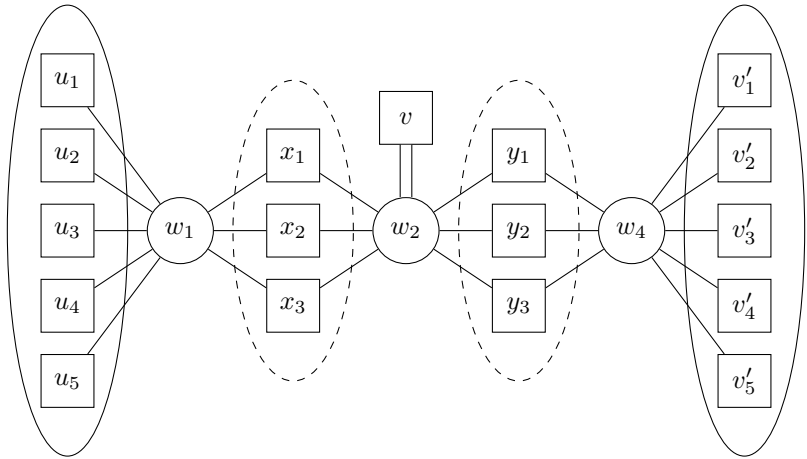
\begin{figure}
    \centering
    \begin{tikzpicture}[
      mycircle/.style={
         circle,
         draw=black,
         fill=white,
         fill opacity = 1,
         text opacity=1,
         inner sep=0pt,
         minimum size=25pt,
         font=\small},
      mysquare/.style={
         rectangle,
         draw=black,
         fill=white,
         fill opacity = 1,
         text opacity=1,
         inner sep=0pt,
         minimum height=20pt, 
         minimum width=20pt,
         font=\small},
      myarrow/.style={-Stealth},
      node distance=0.6cm and 1.2cm
      ]
      \draw (-4.5,0) ellipse (.8cm and 3cm);
      \draw[dashed] (-1.5,0) ellipse (.8cm and 2cm);
      \draw[dashed] (1.5,0) ellipse (.8cm and 2cm);
      \draw (4.5,0) ellipse (.8cm and 3cm);
      \node[mysquare]  at (-4.5, 2) (u1) {$u_1$};
      \node[mysquare]  at (-4.5, 1) (u2) {$u_2$};
      \node[mysquare]  at (-4.5, 0) (u3) {$u_3$};
      \node[mysquare]  at (-4.5, -1) (u4) {$u_4$};
      \node[mysquare]  at (-4.5, -2) (u5) {$u_5$};
      \node[mycircle]  at (-3, 0) (w1) {$w_1$};
      \node[mysquare]  at (-1.5, 1) (x1) {$x_1$};
      \node[mysquare]  at (-1.5, 0) (x2) {$x_2$};
      \node[mysquare]  at (-1.5, -1) (x3) {$x_3$};
      \node[mycircle]  at (0, 0) (w2) {$w_2$};
      \node[mysquare]  at (0, 1.5) (v) {$v$};
      \node[mysquare]  at (1.5, 1) (y1) {$y_1$};
      \node[mysquare]  at (1.5, 0) (y2) {$y_2$};
      \node[mysquare]  at (1.5, -1) (y3) {$y_3$};
      \node[mycircle]  at (3, 0) (w4) {$w_4$};
      \node[mysquare]  at (4.5, 2) (vp1) {$v'_1$};
      \node[mysquare]  at (4.5, 1) (vp2) {$v'_2$};
      \node[mysquare]  at (4.5, 0) (vp3) {$v'_3$};
      \node[mysquare]  at (4.5, -1) (vp4) {$v'_4$};
      \node[mysquare]  at (4.5, -2) (vp5) {$v'_5$};
      \draw[-] (u1) -- (w1);
      \draw[-] (u2) -- (w1);
      \draw[-] (u3) -- (w1);
      \draw[-] (u4) -- (w1);
      \draw[-] (u5) -- (w1);
      \draw[-] (w1) -- (x1);
      \draw[-] (w1) -- (x2);
      \draw[-] (w1) -- (x3);
      \draw[-] (x1) -- (w2);
      \draw[-] (x2) -- (w2);
      \draw[-] (x3) -- (w2);
      \draw[double, double distance = 4.0] (w2) -- (v);
      \draw[-] (w2) -- (y1);
      \draw[-] (w2) -- (y2);
      \draw[-] (w2) -- (y3);
      \draw[-] (y1) -- (w4);
      \draw[-] (y2) -- (w4);
      \draw[-] (y3) -- (w4);
      \draw[-] (w4) -- (vp1);
      \draw[-] (w4) -- (vp2);
      \draw[-] (w4) -- (vp3);
      \draw[-] (w4) -- (vp4);
      \draw[-] (w4) -- (vp5);
      \end{tikzpicture}
    \caption{Result of  \Cref{fig:intersection_term} represented by a multi-edged shape. All edges have label 1 and there are two edges between $v,w_2$.}
    \label{fig:intersection_term_result}
\end{figure}

We take $Q_i$ to be the sum of the ribbons $R'_2$ with the appropriate coefficients.
\begin{remark}
Note that the ribbon $R'_2$ and its coefficient is independent of $R'_1$ and $R'_3$, i.e., it only depends on $R''_1$, $R_2$, and $R''_3$. As in previous SoS lower bounds for average case problems, this property is crucial for the factorization of $M$.
\end{remark}

We now write down an explicit expression for $Q_j$ using intersection configurations. In particular, for all $j\geq 1$ we will have that 
\[
Q_j \approx (-1)^{j}\Sum_{\substack{\mathcal{P}:\ \text{an intersection configuration on shapes}\\ \gamma_j,\ldots,\gamma_1,\tau,{\gamma'}^{\top}_1,\ldots,{\gamma'}^{\top}_j}} 
{\eta_{\mathcal{P}}{\lambda_{\mathcal{P}}}N(\mathcal{P})M_{\tau_{\mathcal{P}}}}
\]
where recall that $\approx$ means ignoring a truncation error, and: 
\begin{enumerate}
\item $\gamma_j,\ldots,\gamma_1,\tau,{\gamma'}^{\top}_1,\ldots,{\gamma'}^{\top}_j$ are the shapes of composable ribbons that non-trivially intersect. 
Here, each shape has total size at most $\truncation$, each index set has size at most $\dsos$, each square vertex has even total degree, each $\gamma_i$ is a proper left shape, each $\gamma_i'^{\top}$ is a proper right shape, and $\tau$ is a proper middle shape. 
We will sometimes abbreviate these conditions as ``each $\gamma_i$ is in $L$, $\tau$ is in $Q_0$, and each $\gamma'^{\top}_i$ is in $L^{\top}$.''
\item The {\it intersection configuration} $\mathcal{P}$ (\Cref{def:intersconfig}) describes the intersections between the shapes and what happens to the resulting multi-edges (if any). $\tau_{\mathcal{P}}$ is the shape of the composition.
\item $\lambda_{\mathcal{P}}$ is the product of the scaling coefficients of all shapes in $\mathcal{P}$.
\item $\eta_{\mathcal{P}}$ is the part of coefficient of $\tau_{\mathcal{P}}$ coming from the Hermite coefficients of $\gamma_j,\ldots,\gamma_1,\tau,{\gamma'}^{\top}_1,\ldots,{\gamma'}^{\top}_j$ and the linearization of the multi-edges.
\item $N(\mathcal{P})$ is a combinatorial factor that counts the number of different ways to obtain a ribbon $R$ of shape $\alpha_{P,\reduced}$ from all possible ribbon realizations of an intersection configuration $\mathcal{P}$. 
\end{enumerate}

The following definitions can be quite technical, please see \Cref{fig:intersection_term,fig:intersection_term_result} for an example. 

\begin{definition}[Intersection configurations]\label{def:intersconfig}
Given $j \in \mathbb{N}$ together with proper left shapes $\gamma_j,\ldots,\gamma_1$, 
a proper middle shape $\tau$, and proper right shapes ${\gamma'}^{\top}_1,\ldots,{\gamma'}^{\top}_j$ such that $\gamma_j,\ldots,\gamma_1,\tau,{\gamma'}^{\top}_1,\ldots,{\gamma'}^{\top}_j$ are composable and for each $i \in [j]$, $\gamma_i$ and ${\gamma'}^{\top}_i$ are not both trivial, we define an intersection configuration $\mathcal{P}$ on $\gamma_j,\ldots,\gamma_1,\tau,{\gamma'}^{\top}_1,\ldots,{\gamma'}^{\top}_j$ to consist of the following data. 

First, $\mathcal{P}$ specifies intersection configurations $(\mathcal{P}_1,\ldots,\mathcal{P}_j)$ such that: 

\begin{enumerate}
    \item $\mathcal{P}_1$ specifies the matching between the indices in $V_{\gamma_1}$ and $U_{\tau}$ and the matching between the indices in $V_{\tau}$ and $U_{{\gamma'}^{\top}_1}$. $\mathcal{P}_1$ also specifies the non-trivial intersections between $\gamma_1$, $\tau$, and ${\gamma'}^{\top}_1$. 

    \item For each $i \in [2,j]$, $\mathcal{P}_i$ specifies the matching between the indices in $V_{\gamma_i}$ and $U_{\gamma_{i-1}}$ and the matching between the indices in $V_{{\gamma'}^{\top}_{i-1}}$ and $U_{{\gamma'}^{\top}_i}$. $\mathcal{P}_i$ also specifies the non-trivial intersections between $\gamma_i$, $\tau_{\mathcal{P}_{i-1}}$, and ${\gamma'}^{\top}_i$ where $\tau_{\mathcal{P}_{i-1}}$ is the result of applying the intersection configurations $\mathcal{P}_1,\ldots,\mathcal{P}_{i-1}$ to $\gamma_{i-1},\ldots,\gamma_1,\tau,{\gamma'}^{\top}_1,\ldots,{\gamma'}^{\top}_{i-1}$. In other words, $\tau_{\mathcal{P}_{i-1}}$ is the shape which results from taking the composition of $\gamma_{i-1},\ldots,\gamma_1,\tau,{\gamma'}^{\top}_1,\ldots,{\gamma'}^{\top}_{i-1}$ given by the matchings in $\mathcal{P}_1,\ldots,\mathcal{P}_{i-1}$ and merging all pairs of vertices which are intersected by $\mathcal{P}_1,\ldots,\mathcal{P}_{i-1}$. Note that $\tau_{\mathcal{P}_{i-1}}$ may be an improper shape as it may have multi-edges.

    \item (Intersection-separating condition) \label{cond:inters-sep}
    For all $i \in [j]$, $U_{\gamma_i}$ is the minimum square vertex separator in $\gamma_i$ between $U_{\gamma_i}$ and $\{\text{non-trivial intersections specified by } \mathcal{P}_i\} \cup V_{\gamma_i}$. Similarly, $V_{\gamma_i}$ is the minimum square vertex separator in $\gamma'^{\top}_i$ between $U_{\gamma_i} \cup \{\text{non-trivial intersections specified by } \mathcal{P}_i\}$ and $V_{\gamma_i}$.
\end{enumerate}
Second, letting $\tau_{\mathcal{P}_j}$ be the shape which results from applying the intersection configurations $\mathcal{P}_1,\ldots,\mathcal{P}_{j}$ to $\gamma_{j},\ldots,\gamma_1,\tau,{\gamma'}^{\top}_1,\ldots,{\gamma'}^{\top}_{j}$, for each multi-edge $e \in E(\tau_{\mathcal{P}_j})$, $\mathcal{P}$ specifies a label $l_e \in \mathbb{N} \cup \{0\}$ describing a term which results from linearizing this multi-edge. 

We define the shape $\tau_{\mathcal{P}}$ to be the result when we replace each multi-edge $e \in E(\tau_{\mathcal{P}_j})$ with a single edge with label $l_e$. Note that $\tau_{\mathcal{P}}$ may be an improper shape as some of the labels $l_e$ may be $0$ so $\tau_{\mathcal{P}}$ may contain isolated vertices.

We call $j$ the length of $\mathcal{P}$, and $\gamma_j,\ldots,\tau,\ldots,\gamma'^{\top}_j$ the shapes in $\mathcal{P}$. The vertex set of $\mathcal{P}$ is $V(P):=V(\gamma_j)\cup\dots\cup V(\tau)\cup\ldots\cup V({\gamma'}^{\top}_j)$, where vertices in indexed sets that are matched by $\mathcal{P}$ will be treated as the same but there is no other vertex identifications regardless of the non-trivial intersections specified by $\mathcal{P}$. Note that $V(\tau_{\mathcal{P}})$ is different from $V(P)$. We view two intersection configurations $\mathcal{P}$ and $\mathcal{P}'$ as the same if there are bijections between their vertices that maps $\mathcal{P}$ to $\mathcal{P}'$, including the edge labels.
\end{definition}

For the remainder of the paper, whenever we write $\gamma_j,\ldots,\gamma_1,\tau,{\gamma'}^{\top}_1,\ldots,{\gamma'}^{\top}_j$, we assume that 
\begin{enumerate}
\item $\gamma_j,\ldots,\gamma_1$ are proper left shapes that are in $L$.
\item $\tau$, is a proper middle shape that is in $Q_0$.
\item ${\gamma'}^{\top}_1,\ldots,{\gamma'}^{\top}_j$ are proper right shapes that are in $L^\top$.
\item $\gamma_j,\ldots,\gamma_1,\tau,{\gamma'}^{\top}_1,\ldots,{\gamma'}^{\top}_j$ are composable.
\item For each $i \in [j]$, $\gamma_i$ and ${\gamma'}^{\top}_i$ are not both trivial.
\end{enumerate}

\begin{definition}[Configuration coefficients]\label{def:configcoeff}
Given $\gamma_j,\ldots,\gamma_1,\tau,{\gamma'}^{\top}_1,\ldots,{\gamma'}^{\top}_j$, we let $\intset$ denote the set of all different intersection configurations on $\gamma_j,\ldots,\gamma_1,\tau,{\gamma'}^{\top}_1,\ldots,{\gamma'}^{\top}_j$. For $\mathcal{P} \in \intset$, 

\begin{enumerate}
\item We define the scaling coefficient $\lambda_{\mathcal{P}}$ to be 
$
\lambda_{\mathcal{P}} = \left(\prod_{i=1}^{j}{\lambda_{\gamma_i}}\right)\lambda_{\tau}\left(\prod_{i=1}^{j}{\lambda_{{\gamma'}^{\top}_i}}\right)$. 
\item We define the Hermite coefficient $\eta_{\mathcal{P}}$ for $\tau_{\mathcal{P}}$ to be 
\[
\eta_{\mathcal{P}} = \left(\prod_{i=1}^{j}{\eta_{\gamma_i}}\right)\eta_{\tau}\left(\prod_{i=1}^{j}{\eta_{{\gamma'}^{\top}_i}}\right)\prod_{\text{multi-edges } e \in E(\tau_{\mathcal{P}_j})}{(\text{coefficient of } l_e \text{ when linearizing } e)}.
\] 
\item There may be many different ways to obtain a ribbon realization of an intersection configuration given the resulting ribbon. To handle this, given a ribbon $R$ of shape $\tau_{\mathcal{P}}$, we define $N(\mathcal{P})$ to be the number of ways to specify the following data so that the induced intersection configuration is equivalent to $\mathcal{P}$ and the resulting ribbon is $R$.
\begin{enumerate}
\item[1.] We have a tuple $(R_{-j},\ldots,R_{-1},R_0,R_1,\ldots,R_j)$ having shapes $\gamma_j,\ldots,\gamma_1,\tau,{\gamma'}^{\top}_1,\ldots,{\gamma'}^{\top}_j$.
\item[2.] Letting $R'$ be the ribbon of shape $\tau_{\mathcal{P}_j}$ 
such that $\prod_{i=-j}^{j}{M_{R_i}} = M_{R'}$, we have a map from the multi-edges $e \in E(R')$ to $\mathbb{N} \cup \{0\}$ specifying which term to take when we linearize $e$.
\end{enumerate}
\end{enumerate}
\end{definition}
\begin{remark}
Computing $N(\mathcal{P})$ can be tricky. 
We only describe $N(\mathcal{P})$ for well-behaved intersection configurations involving simple spiders in \Cref{sec:well-behaved}. For an upper bound on $N(\mathcal{P})$ in general, see \Cref{lem:NP}.
\end{remark}

In order to avoid repeatedly writing out $\gamma_j,\ldots,\gamma_1,\tau,{\gamma'}^{\top}_1,\ldots,{\gamma'}^{\top}_j$ in our sums, we make the following definition.

\begin{definition}
We define $\intsetoflength{j}$ to be the set of all intersection configurations of length $j$. Note that when we sum over $\mathcal{P} \in \intsetoflength{j}$, we are implicitly summing over the possible $\gamma_j,\ldots,\gamma_1,\tau,{\gamma'}^{\top}_1,\ldots,{\gamma'}^{\top}_j$.
\end{definition}

With these definitions, up to truncation error, we can describe $Q_j$ $(j\geq 1)$ by:
\[
Q_j \approx (-1)^{j}\sum_{\mathcal{P} \in \intsetoflength{j}} 
{\eta_{\mathcal{P}}{\lambda_{\mathcal{P}}}N(\mathcal{P})M_{\tau_{\mathcal{P}}}}.
\]

We now describe and handle the truncation errors in a precise manner. 
The next two definitions are technical and are used to make sure that \eqref{eq:round_j} holds. 

\begin{definition}[$L_{\leq t}$]\label{def:L<=t}
We define $L_{\leq t} = 
\Sum_{\sigma\text{ in }L:\ \totalsize(\sigma)\leq t}{\eta_{\sigma}\lambda_{\sigma}M_{\sigma}}$. 
Note that $L = L_{\leq \truncation}$.
\end{definition}

\begin{definition}[$l_{\mathcal{P}}$, $r_{\mathcal{P}}$]\label{def:lPrP}
Given an intersection configuration $\mathcal{P} \in \intset$, we let  
$l_{\mathcal{P}}:= \truncation - \Sum_{i=1}^{j}{(\totalsize(\gamma_i) - |U_{\gamma_i}|)}$. Similarly, we let $r_{\mathcal{P}} := \truncation - \Sum_{i=1}^{j}{(\totalsize({\gamma'}^{\top}_i) - |V_{{\gamma'}^{\top}_i}|)}$.
\end{definition}

With these definitions, we have the following equation.

\begin{equation}\label{eq:round_j}
[L,Q_0,L^{\top}]_{\can} = L{Q_0}L^{\top} + 
\sum_{j=1}^{2D}(-1)^{j}
\sum_{\mathcal{P} \in \intsetoflength{j}} {\eta_{\mathcal{P}}{\lambda_{\mathcal{P}}}N(\mathcal{P})L_{\leq l_{\mathcal{P}}}M_{\tau_{\mathcal{P}}}L^{\top}_{\leq r_{\mathcal{P}}}}
\end{equation}
Based on this, we choose:

\begin{align}
Q_j &= (-1)^{j}\sum_{\substack{\mathcal{P} \in \intsetoflength{j},\ l_{\mathcal{P}} \geq |U_{\tau_{\mathcal{P}}}|,\ r_{\mathcal{P}} \geq |V_{\tau_{\mathcal{P}}}|}}
{\lambda_{\mathcal{P}}}N(\mathcal{P})M_{\tau_{\mathcal{P}}}, \label{eq:Qj}\\
Q &= \sum_{j=0}^{2D}{Q_j}.\label{eq:Q}
\end{align}
We then have the following expression of the truncation error:
\begin{equation}\label{eq:truncationerror}
\begin{aligned}
 &[L,Q_0,L^{\top}]_{\can} - LQL^{\top} \\
=&\Sum_{j=1}^{2D}(-1)^{j}
\Sum_{\substack{{\mathcal{P}} \in \intsetoflength{j},\ l_{\mathcal{P}} \geq |U_{\tau_{\mathcal{P}}}|,\ r_{\mathcal{P}} \geq |V_{\tau_{\mathcal{P}}}|}}
{\eta_{\mathcal{P}}{\lambda_{\mathcal{P}}}N(\mathcal{P})\left(L_{\leq l_{\mathcal{P}}}M_{\tau_{\mathcal{P}}}L^{\top}_{\leq r_{\mathcal{P}}} - {L}M_{\tau_{\mathcal{P}}}L^{\top}\right)}.
\end{aligned}
\end{equation}
\begin{remark}
The condition $l_{\mathcal{P}} \geq |U_{\tau_{\mathcal{P}}}|$ is equivalent to $\totalsize(\gamma_j \circ \ldots \circ \gamma_1) \leq \truncation$, 
where $\gamma_j \circ \ldots \circ \gamma_1$ denotes the shape of the proper composition of ribbons $(R_j,\ldots,R_1)$ where $R_i$ has shape $\gamma_i$ for all $i\in\{1,\ldots,j\}$. Similarly, the condition $r_{\mathcal{P}} \geq |V_{\tau_{\mathcal{P}}}|$ is equivalent to $\totalsize({\gamma'}^{\top}_1 \circ \ldots \circ {\gamma'}^{\top}_j) \leq \truncation$.
\end{remark}
In \Cref{sec:error_analysis}, we show that the truncation error \eqref{eq:truncationerror} has norm at most $n^{-\Omega(\eps\truncation)}$; see \Cref{lem:LQLT_truncation}.

        \section{Well-Behaved Intersection Configurations and Spider Products}
\label{sec:wellbehaved-and-spider}

In this section, we analyze the terms in the matrix $Q$ from the recursive factorization. In \Cref{sec:well-behaved}, we describe which terms of $Q$ are the most important for our analysis, namely disjoint unions of simple spiders coming from well-behaved intersection configurations. In \Cref{sec:spider_product}, we analyze products of simple spiders and disjoint unions of simple spiders. In \Cref{sec:QSS_formula}, we derive a formula determining $\SS{Q}$. Finally, in \Cref{sec:degree46}, we give direct analyses for degree 4 and degree 6 SoS which are interesting in their own right and illustrate the challenges which we overcome.

\subsection{Well-Behaved Intersection Configurations}
\label{sec:well-behaved}
It turns out that the dominant terms of $Q$ in \eqref{eq:Q} are the terms coming from intersection configurations satisfying certain properties which we call well-behaved intersection configurations.

\begin{definition}[Well-behaved intersection configuration]\label{def:wellbehavedconfig}
Given $\gamma_j,\ldots,\gamma_1,\tau,{\gamma'}^{\top}_1,\ldots,{\gamma'}^{\top}_j$ and ${\mathcal{P}} \in \intset$, we say that $\mathcal{P}$ is well-behaved if the following conditions are satisfied:
\begin{enumerate}
\item\label{well-behaved-cond1:SSDshapes} $\gamma_j,\ldots,\gamma_1,\tau,{\gamma'}^{\top}_1,\ldots,{\gamma'}^{\top}_j$ are all simple spider disjoint unions (SSD).
\item\label{well-behaved-cond2:no-square-intersection} $\mathcal{P}$ has no non-trivial intersections between square vertices.
\item \label{well-behaved-cond3:middle-square-isolated} Whenever there is a square vertex $v$ which is not in $U_{\tau_{\mathcal{P}}} \cup V_{\tau_{\mathcal{P}}}$, there is an intersection between the two neighbors of $v$ which results in a double edge in $E(\tau_{\mathcal{P}_j})$. For each such double edge $e \in E(\tau_{\mathcal{P}_j})$, we have that $l_e = 0$ (i.e., the double edge vanishes).
\end{enumerate}
\end{definition}

\begin{remark}[Results of well-behaved configurations]\label{rmk:wellbehaved_and_SSD}
Using the definition it can be checked that every well-behaved configuration results in an SSD shape together with some additional isolated vertices. For a formal proof, see item \ref{item:well-behaved-result-1} in the proof of \Cref{thm:erroranalysis}.
\end{remark}

In \Cref{sec:error_analysis}, we show that the contribution to $Q$ from intersection configurations which are not well-behaved is small. In particular, we show that

\begin{equation}\label{eq:nonwell-behavedissmall}
\sum_{j=1}^{2D}(-1)^{j}\sum_{{\mathcal{P}} \in \intsetoflength{j}:\ {\mathcal{P}} \text{ is not well-behaved, $l_{\mathcal{P}} \geq |U_{\tau_{\mathcal{P}}}|$, $r_{\mathcal{P}} \geq |V_{\tau_{\mathcal{P}}}|$}}{\eta_{\mathcal{P}}{\lambda_{\mathcal{P}}}N(\mathcal{P})M_{\tau_{\mathcal{P}}}}
\end{equation}
is small (\Cref{lem:error_QSSD}).

For each well-behaved intersection configuration $\mathcal{P}$, we obtain a cleaner expression by deleting the isolated vertices from $\tau_{\mathcal{P}}$ and taking the leading order term of the resulting coefficient.

\begin{definition}[Reduced coefficient]\label{def:lambdatauPreduced}
Given $\gamma_j,\ldots,\gamma_1,\tau,{\gamma'}^{\top}_1,\ldots,{\gamma'}^{\top}_j$ and well-behaved intersection configuration ${\mathcal{P}} \in \intset$, we define the reduced shape $\tau_{{\mathcal{P}},\reduced}$ to be the shape obtained by deleting all isolated middle vertices $\Iso(\tau_{\mathcal{P}})$ from $\tau_{\mathcal{P}}$. We set $\lambda_{\mathcal P, \reduced} = n^{-\frac{|E(\tau_{{\mathcal{P}},\reduced})|}{2}} = n^{|\Iso(\tau_{\mathcal{P}})|}\lambda_{\mathcal{P}}$. 
\end{definition} 

We observe that $\eta_{\mathcal{P}}{\lambda_{\mathcal{P}}}N(\mathcal{P})M_{\tau_{\mathcal{P}}}=(1\pm \O(n^{-1}))\eta_{\mathcal{P}}{\lambda_{\mathcal P, \reduced}}N(\mathcal{P})\frac{M_{\tau_{{\mathcal{P}},\reduced}}}{|\Iso(\tau_{\mathcal{P}})|!}$. That is,  
\begin{proposition}\label{prop:reduction}
Given $\gamma_j,\ldots,\gamma_1,\tau,{\gamma'}^{\top}_1,\ldots,{\gamma'}^{\top}_j$ and a well-behaved ${\mathcal{P}} \in \intset$, 
\[
\eta_{\mathcal{P}}{\lambda_{\mathcal{P}}}N(\mathcal{P})M_{\tau_{\mathcal{P}}} = \eta_{\mathcal{P}}{\lambda_{\mathcal P, \reduced}}N(\mathcal{P})\left(\frac{\Prod_{i=0}^{|\Iso(\tau_{\mathcal{P}})|-1}{(n - |V_{\square}(\tau_{{\mathcal{P}},\reduced})| - i)}}{n^{|\Iso(\tau_{\mathcal{P}})|}}\right)\frac{M_{\tau_{{\mathcal{P}},\reduced}}}{|\Iso(\tau_{\mathcal{P}})|!}.
\]
\end{proposition}

\begin{proof}
Notice that $\Iso(\tau_{\mathcal{P}})$ consists of square vertices since $\mathcal{P}$ is well-behaved. 
The proposition then follows immediately from the facts that $\lambda_{\mathcal P, \reduced}  = n^{|\Iso(\tau_{\mathcal{P}})|}\lambda_{\mathcal{P}}$ and $M_{\tau_{\mathcal{P}}} = \binom{n - |V_{\square}(\tau_{{\mathcal{P}},\reduced})|}{|\Iso(\tau_{\mathcal{P}})|}M_{\tau_{{\mathcal{P}},\reduced}}$.
\end{proof}

Based on this, we define $[Q]_{\wellbehaved}$ to consist of all the terms in $Q$ coming from well-behaved intersection configurations where $\eta_{\mathcal{P}}{\lambda_{\mathcal{P}}}N(\mathcal{P})M_{\tau_{\mathcal{P}}}$ is replaced by $\eta_{\mathcal{P}}{\lambda_{\mathcal P, \reduced}}N(\mathcal{P})\frac{M_{\tau_{{\mathcal{P}},\reduced}}}{|\Iso(\tau_{\mathcal{P}})|!}$.

\begin{definition}[{$[Q]_{\wellbehaved}$}]\label{def:Qwell-behaved}
We define 
\begin{equation}\label{eq:Q_wellbehaved}
[Q]_{\wellbehaved} = \sum_{j=1}^{2D}{(-1)^{j}\sum_{{\mathcal{P}} \in \intsetoflength{j}:\textrm{ $\mathcal{P}$ is well-behaved, $l_{\mathcal{P}} \geq |U_{\tau_{\mathcal{P}}}|$, $r_{\mathcal{P}} \geq |V_{\tau_{\mathcal{P}}}|$}}{\eta_{\mathcal{P}}{\lambda_{\mathcal P, \reduced}}N(\mathcal{P})\frac{M_{\tau_{{\mathcal{P}},\reduced}}}{|\Iso(\tau_{\mathcal{P}})|!}}}
\end{equation}
\end{definition}
In \Cref{sec:error_analysis}, we confirm that 
\begin{equation}\label{eq:differencewithQwellbehavedissmall}
\norm{\sum_{j=1}^{2D}{(-1)^{j}
\sum_{{\mathcal{P}} \in \intsetoflength{j}:\text{ $\mathcal{P}$ is well-behaved, $l_{\mathcal{P}} \geq |U_{\tau_{\mathcal{P}}}|$, $r_{\mathcal{P}} \geq |V_{\tau_{\mathcal{P}}}|$}}{\eta_{\mathcal{P}}N(\mathcal{P})\left({\lambda_{\mathcal{P}}}M_{\tau_{\mathcal{P}}} - {\lambda_{\mathcal P, \reduced}}\frac{M_{\tau_{{\mathcal{P}},\reduced}}}{|\Iso(\tau_{\mathcal{P}})|!}\right)}}}
\end{equation}
is small (see Lemma \ref{lem:error_QSSD}). Combining the two statements \eqref{eq:nonwell-behavedissmall}, \eqref{eq:differencewithQwellbehavedissmall}, we have that $\norm{Q - [Q]_{\wellbehaved}}$ is small.

Finally, we can use the following remark to get rid of the condition on $(l_{\mathcal{P}}, r_{\mathcal{P}})$ in $[Q]_{\wellbehaved}$.

\begin{remark}[{No truncation error from well-behaved configurations in $Q$}]\label{rmk:wellbehaved_no_truncation_err}
For any well-behaved $\mathcal{P}$, each shape in $\mathcal{P}$ is an SSD and thus has total size at most $6\dsos$. Since at most $2\dsos$ of $\gamma_j,\ldots,\gamma_1,{\gamma'}^{\top}_1,\ldots,{\gamma'}^{\top}_j$ can be non-trivial (as each non-trivial $\gamma_i$ increases the size of $U_{\tau_{\mathcal{P}}}$ by at least $1$ and each non-trivial ${\gamma'}^{\top}_i$ increases the size of $V_{\tau_{\mathcal{P}}}$ by at least $1$), $\totalsize(\mathcal P)\leq (2\dsos+1)6\dsos$. Therefore, as long as $\truncation\geq 20\dsos^2$, the conditions $l_{\mathcal{P}} \geq |U_{\tau_{\mathcal{P}}}|,\ r_{\mathcal{P}} \geq |V_{\tau_{\mathcal{P}}}|$ in the summation \eqref{eq:Q_wellbehaved} can be dropped since they automatically hold for all well-behaved $P$ there. 
\end{remark}

\begin{corollary}[{Simplified expression of $[Q]_{\wellbehaved}$}] \label{cor:Q_wellbehaved}
Let
\begin{equation}\label{eq:Qj_wellbehaved}
[Q_j]_{\wellbehaved}:=(-1)^{j}\sum_{{\mathcal{P}} \in \intsetoflength{j}:\textrm{ $\mathcal{P}$ is well-behaved}}{\eta_{\mathcal{P}}{\lambda_{\mathcal P, \reduced}}N(\mathcal{P})\frac{M_{\tau_{{\mathcal{P}},\reduced}}}{|\Iso(\tau_{\mathcal{P}})|!}}\ \ \text{ for }j=0,\ldots,2\dsos,
\end{equation}
then we have $[Q]_{\wellbehaved} = \Sum_{j=1}^{2D}{[Q_j]_{\wellbehaved}}$.
\end{corollary}

\subsection{Spider Products}\label{sec:spider_product}
In order to analyze $[Q]_{\wellbehaved}$, it is crucial to understand products of spiders and disjoint unions of simple spiders.

We start by considering the products of simple spiders, where we only take the products in which the circle vertices intersect and there are no non-trivial intersections for square vertices.

\begin{lemma}
\label{lem:Smultiplication}
For $u\leq\min\{k_1,k_2,k_3\}$, if we take the terms from the scaled graph matrix product $S(k_1-u,k_2-u;u)\cdot S(k_2-v,k_3-v;v)$ where the circle vertices intersect and there are no non-trivial intersections between the square vertices, the result is approximately
\[
    \sum_{i=\max\{0,u+v-k_2\}}^{\min\{u,v\}}\binom{k_1-i}{k_1-u}\binom{k_3-i}{k_3-v}/(k_2+i-u-v)!\cdot S(k_1-i,k_3-i;i)
\]
\end{lemma}
\begin{proof}
The terms of the product $S(k_1-u,k_2-u;u)\cdot S(k_2-v,k_3-v;v)$ where the circle vertices intersect and there are no non-trivial intersections between the square vertices are given by pairs of ribbons $(R_1,R_2)$ of shapes $S(k_1-u,k_2-u;u)$ and $S(k_2-v,k_3-v;v)$ such that $R_1$ and $R_2$ have the same circle vertex $w$ and the only square vertices which $R_1$ and $R_2$ have in common are $B_{R_1} = A_{R_2}$. For each such pair of ribbons $(R_1,R_2)$, letting $i = |A_{R_1} \cap B_{R_1} \cap B_{R_2}|$,
\begin{enumerate}
\item There are $u-i$ square vertices in $(A_{R_1} \cap B_{R_1}) \setminus B_{R_2}$ and $k_1 - u$ vertices in $A_{R_1} \setminus B_{R_1}$. These vertices are adjacent to $w$ in $R_2$ but not $R_1$.
\item There are $v-i$ vertices in $(A_{R_2} \cap B_{R_2}) \setminus A_{R_1}$ and $k_3 - v$ vertices in $B_{R_2} \setminus A_{R_2}$. These vertices are adjacent to $w$ in $R_1$ but not $R_2$.
\item There are $k_2 - (u-i) - (v-i) = k_2 + i - u - v$ vertices in $B_{R_1} \setminus (A_{R_1} \cup B_{R_2})$. These vertices are adjacent to $w$ in both $R_1$ and $R_2$.
\end{enumerate}
Observe that $M_{R_1}M_{R_2} = M_R$ where $R$ is the ribbon such that
\begin{enumerate}
\item $R$ has the circle vertex $w$.
\item $A_R \cap B_R = A_{R_1} \cap B_{R_1} \cap B_{R_2}$. Note that $|A_{R_1} \cap B_{R_1} \cap B_{R_2}| = i$ and none of these vertices are adjacent to $w$.
\item $A_R \setminus B_R = (A_{R_1} \setminus B_{R_1}) \cup (A_{R_1} \cap B_{R_1} \setminus B_{R_2})$. Note that $|A_R \setminus B_R| = k_1 - i$ and all vertices in $A_R \setminus B_R$ are adjacent to $w$.
\item $B_R \setminus A_R = (B_{R_2} \setminus A_{R_2}) \cup (A_{R_2} \cap B_{R_2} \setminus A_{R_1})$. Note that $|B_R \setminus A_R| = k_3 - i$ and all vertices in $A_R \setminus B_R$ are adjacent to $w_3$.
\item $V(R) \setminus (A_R \cup B_R) = B_{R_1} \setminus (A_{R_1} \cup B_{R_2})$. Note that $|V(R) \setminus (A_R \cup B_R)| = k_2 + i - u - v$ and each of the square vertices in $V(R) \setminus (A_R \cup B_R)$ has a double edge to $w$.
\end{enumerate}
For each such ribbon $R$, we observe that there are $\binom{k_1-i}{k_1-u}\binom{k_3-i}{k_3-v}$ pairs of ribbons $(R_1,R_2)$ which result in the ribbon $R$ as there are $\binom{k_1-i}{k_1-u}$ ways to choose which of the $k_1 - i$ indices in $A_R \setminus B_R$ are in $A_{R_1} \setminus B_{R_1}$ and there are $\binom{k_3-i}{k_3-u}$ ways to choose which of the $k_3 - i$ indices in $B_R \setminus A_R$ are in $B_{R_2} \setminus A_{R_2}$.

Finally, we note that following similar logic as the proof of Proposition \ref{prop:reduction}, deleting the $k_2 + i - u - v$ vertices in $V(R) \setminus (A_R \cup B_R)$ together with the double edges between these vertices and $w$ and then shifting to the corresponding shape $S(k_1-i,k_3-i;i)$ gives a factor of roughly $\frac{1}{|V(R) \setminus (A_R \cup B_R)|!} = \frac{1}{(k_2 + i - u - v)!}$.
\end{proof}

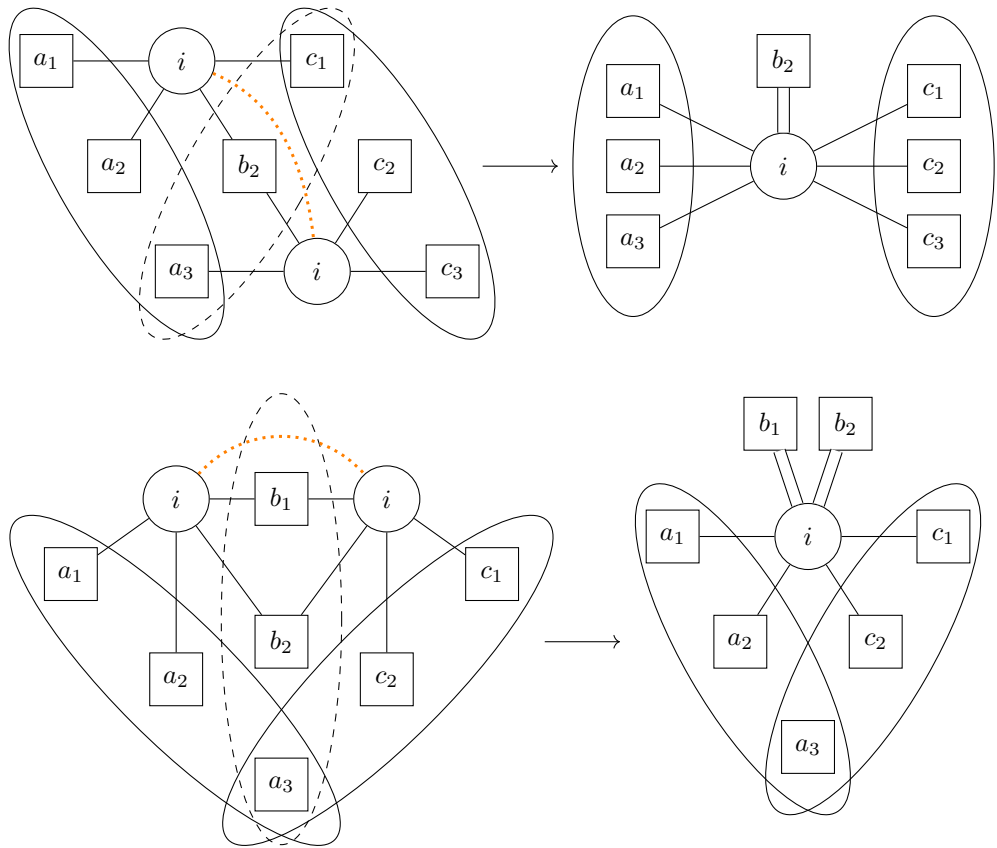
\begin{figure}
    \centering
    \begin{tikzpicture}[
      mycircle/.style={
         circle,
         draw=black,
         fill=white,
         fill opacity = 1,
         text opacity=1,
         inner sep=0pt,
         minimum size=25pt,
         font=\small},
      mysquare/.style={
         rectangle,
         draw=black,
         fill=white,
         fill opacity = 1,
         text opacity=1,
         inner sep=0pt,
         minimum height=20pt, 
         minimum width=20pt,
         font=\small},
      myarrow/.style={-Stealth},
      node distance=0.6cm and 1.2cm
      ]
      \draw[rotate = 30] (-1,0) ellipse (.8cm and 2.5cm);
      \draw[dashed, rotate = -30] (1,0) ellipse (.8cm and 2.5cm);
      \draw[rotate = 30] (2.1177,-1.8) ellipse (.8cm and 2.5cm);
      \node[mysquare]  at (-1.8, 1) (a1) {$a_1$};
      \node[mysquare]  at (-.9, -.4) (a2) {$a_2$};
      \node[mysquare]  at (0, -1.8) (a3) {$a_3$};
      \node[mycircle]  at (0, 1) (i1) {$i$};
      \node[mysquare]  at (1.8, 1) (c1) {$c_1$};
      \node[mysquare]  at (.9, -.4) (b2) {$b_2$};
      \node[mycircle]  at (1.8, -1.8) (i2) {$i$};
      \node[mysquare]  at (2.7, -.4) (c2) {$c_2$};
      \node[mysquare]  at (3.6, -1.8) (c3) {$c_3$};
      \draw[-] (a1) -- (i1);
      \draw[-] (a2) -- (i1);
      \draw[-] (i1) -- (c1);
      \draw[-] (i1) -- (b2);
      \draw[-] (b2) -- (i2);
      \draw[-] (a3) -- (i2);
      \draw[-] (i2) -- (c2);
      \draw[-] (i2) -- (c3);
      \draw[color=orange, dotted, line width=0.4mm] (i1) .. controls (1,.6) and (1.6,0) .. (i2);
      
      \draw (6,-0.4) ellipse (.8cm and 2cm);
      \draw (10,-0.4) ellipse (.8cm and 2cm);
      \node[mysquare]  at (6, .6) (a4) {$a_1$};
      \node[mysquare]  at (6, -0.4) (a5) {$a_2$};
      \node[mysquare]  at (6, -1.4) (a6) {$a_3$};
      \node[mycircle]  at (8, -0.4) (i3) {$i$};
      \node[mysquare]  at (8, 1) (b4) {$b_2$};
      \node[mysquare]  at (10, .6) (c4) {$c_1$};
      \node[mysquare]  at (10, -0.4) (c5) {$c_2$};
      \node[mysquare]  at (10, -1.4) (c6) {$c_3$};
      \draw[-] (a4) -- (i3);
      \draw[-] (a5) -- (i3);
      \draw[-] (a6) -- (i3);
      \draw[double, double distance = 4.0] (i3) -- (b4);
      \draw[-] (i3) -- (c4);
      \draw[-] (i3) -- (c5);
      \draw[-] (i3) -- (c6);
      
      \draw[->] (4,-0.4) -- (5,-0.4);
      \end{tikzpicture}

\vspace{14pt}

    \begin{tikzpicture}[
      mycircle/.style={
         circle,
         draw=black,
         fill=white,
         fill opacity = 1,
         text opacity=1,
         inner sep=0pt,
         minimum size=25pt,
         font=\small},
      mysquare/.style={
         rectangle,
         draw=black,
         fill=white,
         fill opacity = 1,
         text opacity=1,
         inner sep=0pt,
         minimum height=20pt, 
         minimum width=20pt,
         font=\small},
      myarrow/.style={-Stealth},
      node distance=0.6cm and 1.2cm
      ]
      \draw[rotate = 45] (-2,0) ellipse (.8cm and 3cm);
      \draw[dashed] (0,-0.6) ellipse (.8cm and 3cm);
      \draw[rotate = -45] (2,0) ellipse (.8cm and 3cm);
      \node[mysquare]  at (-2.8, 0) (a1) {$a_1$};
      \node[mysquare]  at (-1.4, -1.4) (a2) {$a_2$};
      \node[mysquare]  at (0, -2.8) (a3) {$a_3$};
      \node[mycircle]  at (-1.4, 1) (i1) {$i$};
      \node[mysquare]  at (0, 1) (b1) {$b_1$};
      \node[mysquare]  at (0, -0.9) (b2) {$b_2$};
      \node[mycircle]  at (1.4, 1) (i2) {$i$};
      \node[mysquare]  at (2.8, 0) (c1) {$c_1$};
      \node[mysquare]  at (1.4, -1.4) (c2) {$c_2$};
      \draw[-] (a1) -- (i1);
      \draw[-] (a2) -- (i1);
      \draw[-] (i1) -- (b1);
      \draw[-] (i1) -- (b2);
      \draw[-] (b1) -- (i2);
      \draw[-] (b2) -- (i2);
      \draw[-] (i2) -- (c1);
      \draw[-] (i2) -- (c2);
      \draw[color=orange, dotted, line width=0.4mm] (i1) .. controls (-0.5,2) and (0.5,2) .. (i2);
      
      \draw[rotate = 30] (4.812,-3.933) ellipse (.8cm and 2.5cm);
      \draw[rotate = -30] (7.312,3.067) ellipse (.8cm and 2.5cm);
      \node[mysquare]  at (5.2, .5) (a4) {$a_1$};
      \node[mysquare]  at (6.1, -.9) (a5) {$a_2$};
      \node[mysquare]  at (7, -2.3) (a6) {$a_3$};
      \node[mycircle]  at (7, .5) (i3) {$i$};
      \node[mysquare]  at (6.5, 2) (b4) {$b_1$};
      \node[mysquare]  at (7.5, 2) (b5) {$b_2$};      
      \node[mysquare]  at (8.8, .5) (c4) {$c_1$};
      \node[mysquare]  at (7.9, -.9) (c5) {$c_2$};
      \draw[-] (a4) -- (i3);
      \draw[-] (a5) -- (i3);
      \draw[-] (i3) -- (c4);
      \draw[-] (i3) -- (c5);
      \draw[double, double distance = 4.0] (i3) -- (b4);
      \draw[double, double distance = 4.0] (i3) -- (b5);
      
      \draw[->] (3.5,-0.9) -- (4.5,-0.9);
      \end{tikzpicture}
    \caption{Dominant terms of the product $S(2,2;1)\cdot S(2,2;1)$.}
    \label{fig:SS_product}
\end{figure}
Based on this, we define the following algebra.

\begin{definition}[Simple Spider Algebra of degree $D$, $\SALD$]\label{def:SALD} 
For fixed $D$, the $\R$-algebra $\SALD$ has its underlying $\R$-vector space spanned by a symbol set $\{S(i,j;u)\mid \max\{i+u, j+u\}\leq D\}$, and its product, denoted by $\SALDprod$, is defined as 
\begin{equation}\label{eq:SALDprod}
\begin{aligned}
&S(k_1-u,k_2-u;u)\SALDprod S(k_2-v,k_3-v;v)\\
:=&\sum_{i=\max\{0,u+v-k_2\}}^{\min\{u,v\}}\binom{k_1-i}{k_1-u}\binom{k_3-i}{k_3-v}/(k_2+i-u-v)!\cdot S(k_1-i,k_3-i;i),
\end{aligned}
\end{equation}
extended linearly to all elements. 
\end{definition}
\begin{remark}[Parity labels for simple spiders]
For the simple spider $\alpha$ with $i$ legs to the left, $j$ legs to the right, and $u$ vertices in $U_{\alpha} \cap V_{\alpha}$, there are $2^{u}$ possible parity labels for the vertices in $U_{\alpha} \cap V_{\alpha}$. Technically, this gives $2^{u}$ separate terms (where each term also incorporates the scaling coefficient $\lambda_{\alpha}$) but we group these terms together and use $S(i,j,u)$ to denote the sum of these terms. It is not hard to check that this does not affect the argument in \Cref{lem:Smultiplication}.
\end{remark}
We now define a product $\wbp$ (``well-behaved products'') on disjoint unions of simple spiders. This product generalizes the $\SALDprod$-product, and is analogous to $[Q]_{\wellbehaved}$ in the sense that it only takes well-behaved product configurations (\Cref{def:wb-product-config}) and then cleans up the coefficients.

In the following, we will abbreviate simple spider disjoint unions as SSD at places for spatial reasons.

\begin{definition}[Product configuration]\label{def:SSD-product-config}
Given composable shapes $\alpha_1$ and $\alpha_2$ which are disjoint unions of simple spiders, a product configuration $\mathcal{P}$ on $\alpha_1,\alpha_2$ consists of the following data. 
\begin{enumerate}
\item $\mathcal{P}$ specifies the matching between the indices in $V_{\alpha_1}$ and $U_{\alpha_2}$.
\item $\mathcal{P}$ specifies the non-trivial intersections between $\alpha_1$ and $\alpha_2$.
\item Letting $\alpha'$ be the shape which results from composing $\alpha_1$ and $\alpha_2$ according to this matching and then merging the vertices which are intersected, for each multi-edge $e \in E(\tau_{{\mathcal{P}_j}})$, $\mathcal{P}$ specifies a label $l_e \in \mathbb{N} \cup \{0\}$ describing a term which results from linearizing this multi-edge.
\end{enumerate}
We define the shape $\alpha_{\mathcal{P}}$ to be the result when we replace each multi-edge $e \in E(\alpha')$ with a single edge with label $l_e$.

We say that two product configurations $\mathcal{P}$ and $\mathcal{P}'$ on $\alpha_1$ and $\alpha_2$ are equivalent if there are permutations of $V(\alpha_1)$ and $V(\alpha_2)$ which map $\mathcal{P}$ to $\mathcal{P}'$.
\end{definition}

\begin{definition}[Product configuration coefficients]\label{def:SSD-product-config-coeff}
Given composable $\alpha_1,\alpha_2$ which are disjoint unions of simple spiders, we define $\mathcal{P}_{\alpha_1,\alpha_2}$ to be the set of non-equivalent possible product configurations on $\alpha_1$ and $\alpha_2$. 

Given ${\mathcal{P}} \in \mathcal{P}_{\alpha_1,\alpha_2}$, we define coefficients for $\alpha_{\mathcal{P}}$ as follows:
\begin{enumerate}
\item The scaling coefficient $\lambda_{\mathcal{P}}$ for $\alpha_{\mathcal{P}}$ is $\lambda_{\mathcal{P}} = \lambda_{\alpha_1}\lambda_{\alpha_2}$.
\item We define the Hermite coefficient $\eta_{\mathcal{P}}$ for $\alpha_{\mathcal{P}}$ to be 
\[
\eta_{\mathcal{P}} = \prod_{\text{multi-edges } e \in E(\alpha_{\mathcal{P}})}{(\text{coefficient of } l_e \text{ when linearizing } e)}.
\]
\item There can be many different ways to obtain a ribbon $R$ of shape $\alpha_{\mathcal{P}}$ from a product configuration $\mathcal{P}$. 
For this reason, we define $N(\mathcal{P})$ to be the number of ways to specify the following data so that the induced product configuration is equivalent to $\mathcal{P}$ and the resulting ribbon is $R$.
\begin{enumerate}[label=(\arabic*)]
\item We have a tuple of $(R_1,R_2)$ of ribbons with shapes $\alpha_1,\alpha_2$.
\item Letting $R'$ be the ribbon of shape $\alpha'$ such that $M_{R_1}M_{R_2} = M_{R'}$, we have a map from the multi-edges $e \in E(R')$ to $\mathbb{N} \cup \{0\}$ specifying which term to take when we linearize $e$.
\end{enumerate}
\end{enumerate}
\end{definition}

Now we define well-behaved product configurations. Note that they refer to products of simple spider disjoint unions only. 

\begin{definition}[Well-behaved product configurations]\label{def:wb-product-config}
Given composable shapes $\alpha_1,\alpha_2$ which are disjoint unions of simple spiders and a product configuration $\mathcal{P} \in \mathcal{P}_{\alpha_1,\alpha_2}$, we say that $\mathcal{P}$ is well-behaved if the following conditions are satisfied:
\begin{enumerate}
\item $\mathcal{P}$ has no non-trivial intersections between square vertices.
\item Whenever there is a square vertex $v$ outside of $U_{\alpha_{\mathcal{P}}} \cup V_{\alpha_{\mathcal{P}}}$, there is an intersection between the two neighbors of $v$ which results in a double edge in $E(\alpha')$. For each such double edge $e \in E(\alpha)$, we have that $l_e = 0$ (i.e., the double edge vanishes).
\end{enumerate}
\end{definition}

Different vertices in a shape never get identified in a product, so the conditions in \Cref{def:wb-product-config} imply that when two circles $u_1\in\alpha_1, u_2\in\alpha_2$ nontrivially intersect, all edges from $u_1$ to $V_{\alpha_1}\backslash (U_{\alpha_1} \cup V_{\alpha_2})= U_{\alpha_2}\backslash (U_{\alpha_1} \cup V_{\alpha_2})$ match all edges from $u_2$ to the same set. We define the reduced shape $\alpha_{P,\reduced}$ to be $\alpha_{\mathcal{P}}$ after removing these edges and the isolated vertices, with a coefficient $\lambda_{\mathcal{P},\reduced}$ that is the leading order term, as follows.

\begin{definition}[Reduced shapes and coefficients]
\label{def:reduced}
For each well-behaved product configuration $\mathcal{P} \in \mathcal{P}_{\alpha_1,\alpha_2}$, we define the reduced shape $\alpha_{P,\reduced}$ to be the shape obtained by deleting all isolated vertices from $\alpha_{\mathcal{P}}$. We set $\lambda_{\alpha_{P,\reduced}} = n^{-\frac{|E(\alpha_{P,\reduced})|}{2}} = n^{|\Iso(\alpha_{\mathcal{P}})|}\lambda_{\alpha_{\mathcal{P}}}$.
\end{definition}
As before, we observe that the dominant term of ${\lambda_{\alpha_{\mathcal{P}}}}N(\mathcal{P})M_{\alpha_{\mathcal{P}}}$ is ${\lambda_{\alpha_{P,\reduced}}}N(\mathcal{P})\frac{M_{\alpha_{P,\reduced}}}{|\Iso(\alpha_{\mathcal{P}})|!}$. Based on this, we make the following definition.

\begin{definition}[Well-behaved products]\label{def:wbp}
Given composable $\alpha_1,\alpha_2$ which are disjoint unions of simple spiders, we define
\[
(\lambda_{\alpha_1}M_{\alpha_1}) \wbp (\lambda_{\alpha_2}M_{\alpha_2}) 
:= 
\sum_{\mathcal{P} \in \mathcal{P}_{\alpha_1,\alpha_2}:\ \mathcal{P} \text{ is well-behaved}}{{\lambda_{\alpha_{P,\reduced}}}N(\mathcal{P})\frac{M_{\alpha_{P,\reduced}}}{|\Iso(\alpha_{\mathcal{P}})|!}}
\]
We extend this product multi-linearly to all linear combinations of disjoint unions of simple spiders.
\end{definition}

\begin{remark}\label{remark:shorthand}
As noted earlier, instead of always writing $\lambda_{\alpha}M_{\alpha}$ for well-behaved products, we will sometimes just write $\alpha$ for convenience.
\end{remark}

In \Cref{sec:error_analysis}, we show that if all non-trivial shaped simple spiders in $\alpha_1$ and $\alpha_2$ are good, i.e., they have at least $\frac{k}{2}$ legs to the left and right (see \Cref{def:good}), 
then 
\[\norm{\lambda_{\alpha_1}\lambda_{\alpha_2}M_{\alpha_1}M_{\alpha_2} - (\lambda_{\alpha_1}M_{\alpha_1}) \wbp (\lambda_{\alpha_2}M_{\alpha_2})}\] 
is small.

\begin{remark}[Range of approximation]\label{rmk:approx}
Note that in general, $\wbp$ does not always give the dominant terms. For example, $S(3,1,0) \wbp S(1,3,0) = S(3,3,0)$ but there are other shapes in $S(3,1,0)S(1,3,0)$ with norm $\widetilde{O}(\frac{m}{n})\gg 1$, where we assume $k=4$, $\eps$ is small, and $m=n^{(1-\eps)k/2}$. 

Also, it is possible that the well-behaved product of two simple spiders $\alpha\wbp\beta$ contains more shapes than $\alpha\star\beta$. For example, $S(4,0;0)\wbp S(0,4;4)$ contains their canonical product which is an SSD but not a simple spider.
\end{remark}

As a sanity check, we observe that when we restrict our attention to simple spiders, $\wbp$ reduces to the product $\SALDprod$ for the simple spider algebra.

\begin{lemma}
	\begin{align*}
		&\SS{[S(k_1-u,k_2-u;u) \wbp S(k_2-v,k_3-v;v)]} = S(k_1-u,k_2-u;u) \SALDprod S(k_2-v,k_3-v;v)\\
		=&\sum_{i=\max\{0,u+v-k_2\}}^{\min\{u,v\}}\binom{k_1-i}{k_1-u}\binom{k_3-i}{k_3-v}/(k_2+i-u-v)!\cdot S(k_1-i,k_3-i;i)
	\end{align*}
\end{lemma}
\begin{proof}[Proof sketch]
Assume $\mathcal{P}$ is a product configuration on $S(k_1-u,k_2-u;u)$ and $S(k_2-v,k_3-v;v)$ where there is an intersection between the circle vertices, $i$ of the vertices in $U_{S(k_1-u,k_2-u;u)} \cap V_{S(k_1-u,k_2-u;u)}$ are matched to vertices in $U_{S(k_2-v,k_3-v;v)} \cap V_{S(k_2-v,k_3-v;v)}$, and all double edges vanish. 
We have that $\alpha_{\mathcal{P}}$ is $S(k_1-i,k_3-i;i)$ plus $k_2+i-u-v$ isolated vertices, and $\alpha_{P,\reduced} = S(k_1-i,k_3-i;i)$. By the same logic as in \Cref{lem:Smultiplication}, we have that $N(\mathcal{P}) = \binom{k_1-i}{k_1-u}\binom{k_3-i}{k_3-v}$ so 
$
\frac{N(\mathcal{P})}{|\Iso(\alpha_{\mathcal{P}})|!} = \frac{\binom{k_1-i}{k_1-u}\binom{k_3-i}{k_3-v}}{(k_2+i-u-v)!}$.
\end{proof}

In fact, $\wbp$ is also associative. 
\begin{lemma}\label{lem:wbp-associative}
The product $\wbp$ is associative.
\end{lemma}

\begin{proof}[Proof sketch] 
When we compute $(\alpha_1 \wbp \alpha_2) \wbp \alpha_3$, we consider the products of ribbons $R_1,R_2,R_3$ whose product configurations are well-behaved and then take the leading order terms. Since ribbon products are associative, one would expect $\wbp$ to be associative as well. 

That said, we need to check that restricting to well-behaved product configurations and taking the leading order terms does not affect the associativity of ribbon products. To do this, we observe that when we have a sequence of ribbons whose product configuration is well-behaved, we perform the following types of operations in some order.
\begin{enumerate}
\item Multiply ribbons together by composing them and merging intersected vertices.
\item Delete a doubled edge incident to a square vertex, making this square vertex isolated. 
\item Delete isolated vertices and adjust the coefficients accordingly.
\end{enumerate}
It is not hard to see that for a given sequence of ribbons, the order in which we delete doubled edges does not matter. The following lemma implies that we can delay deleting isolated vertices until the end and the result will stay the same.

\begin{lemma}\label{lem:wbp_iso}
Let $\alpha_1, \alpha_2$ be composable shapes which are disjoint unions of simple spiders. Let $\alpha'_1$ be $\alpha_1$ together with $x$ isolated vertices and let $\alpha'_2$ be $\alpha_2$ together with $y$ isolated vertices. We take $\lambda_{\alpha'_1} = \frac{1}{n^x}\lambda_{\alpha_1}$ and $\lambda_{\alpha'_2} = \frac{1}{n^y}\lambda_{\alpha_2}$

If we extend the definition of $\wbp$ to simple spider disjoint unions together with isolated square vertices (with no non-trivial intersection among the isolated vertices), then we have that 
\[
x!{\alpha}'_{1} \wbp y!{\alpha}'_{2} = \alpha_1 \wbp \alpha_2
\]
\end{lemma}

\begin{proof}
Let $\mathcal{P}$ be a well-behaved product configuration for $\alpha_1$ and $\alpha_2$ and let $\mathcal{P}'$ be the corresponding well-behaved product configuration for $\alpha'_1$ and $\alpha'_2$. 

We claim that $N(\mathcal{P}') = \frac{(x+y+|\Iso(\alpha_{\mathcal{P}})|)!}{x!y!|\Iso(\alpha_{\mathcal{P}})|!}N(\mathcal{P})$. To see this, observe that given a ribbon $R'$ of shape $\alpha_{\mathcal{P}'}$, we can find ribbons $R'_1$ and $R'_2$ which have shapes $\alpha'_1$ and $\alpha'_2$, have product configuration $\mathcal{P}'$, and result in $R'$ as follows.
\begin{enumerate}
\item Choose $x$ out of the $(x+y+|\Iso(\alpha_{\mathcal{P}})|)!$ isolated vertices in $R'$ to be in $R'_1$.
\item Choose $y$ out of the remaining $(y+|\Iso(\alpha_{\mathcal{P}})|)!$ vertices 
\item After removing these isolated vertices, we are left with a ribbon $R$ of shape $\alpha_{\mathcal{P}}$. There are now $N(\mathcal{P})$ ways to choose $R_1$ and $R_2$ to have product configuration $\mathcal{P}$ and result in $R$. Adding the previously chosen isolated vertices to $R_1$ and $R_2$ gives $R'_1$ and $R'_2$.
\end{enumerate}
We further observe that $\lambda_{\alpha_{\mathcal{P}',\reduced}} = n^{|\Iso(\mathcal{P}')|}\lambda_{\alpha_{\mathcal{P}'}} = \frac{n^{|\Iso(\mathcal{P})| + x + y}}{n^{x}n^{y}}\lambda_{\alpha_{\mathcal{P}}} = \lambda_{\alpha_{\mathcal{P},\reduced}}$. Thus, we have that 
\[
x!y!\lambda_{\alpha_{\mathcal{P}',\reduced}}N(\mathcal{P}')\frac{M_{\alpha_{\mathcal{P}',\reduced}}}{|\Iso(\mathcal{P}')|!} = \lambda_{\alpha_{\mathcal{P},\reduced}}N(\mathcal{P})\frac{M_{\alpha_{P,\reduced}}}{|\Iso(P)|!}. \qedhere
\]
\end{proof}
Since we can delay deleting isolated vertices until the end and the result will stay the same, we can compute the terms of $(\alpha_1 \wbp \alpha_2) \wbp \alpha_3$ and $\alpha_1 \wbp (\alpha_2 \wbp \alpha_3)$ as follows:
\begin{enumerate}
\item Multiply the ribbons together by composing them and merging intersected vertices.
\item For each square vertex incident to a doubled edge, delete this doubled edge, making this square vertex isolated. 
\item Delete the isolated vertices and adjust the coefficients accordingly.
\end{enumerate}
Here, the operations are done in the given order. The first step gives the same results for $(\alpha_1 \wbp \alpha_2) \wbp \alpha_3$ and $\alpha_1 \wbp (\alpha_2 \wbp \alpha_3)$ so we have that $(\alpha_1 \wbp \alpha_2) \wbp \alpha_3 = \alpha_1 \wbp (\alpha_2 \wbp \alpha_3)$, as needed.
\end{proof}

From now on, by ``well-behaved products'' we always refer to well-behaved products, which includes well-behaved SS products by definition.

To conclude this sub-section, we summarize some relations of the definitions as follows. The shapes below are all simple spider disjoint unions.
\begin{enumerate}
\item $\alpha_1 \wbp \alpha_2$ is the sum of the resulting term of all well-behaved product configurations in $(\lambda_{\alpha_1}M_{\alpha_1})(\lambda_{\alpha_2}M_{\alpha_2})$. 
\item When we restrict our attention to simple spiders, $\wbp$ reduces to $\SALDprod$.
\item To obtain the results of the well-behaved intersection configurations for  $(\lambda_{\gamma}M_{\gamma})(\lambda_{\tau}M_{\tau})(\lambda_{\gamma'^{\top}}M_{\gamma'^{\top}})$, we take the terms from  $\gamma \wbp \tau \wbp {\gamma'}^{\top}$ where every non-trivial simple spider in $\gamma$ and ${\gamma'}^{\top}$ has a non-trivial intersection.
\end{enumerate}

In \Cref{sec:error_analysis}, we will show that the dominant terms in $Q_i$ are disjoint unions of simple spiders. For now and the upcoming sections, we focus on these shapes in $Q$. We end this subsection with a general notation. 
\begin{definition}[Restrictions]\label{def:restriction}
Except when $X = Q$ or $X = Q_i$, $i=0,\ldots,2\dsos$ (see equations \eqref{eq:Qj} and \eqref{eq:Q}), if $X = \sum_\alpha \lambda_\alpha M_\alpha$ is a linear combination of graph matrices and $G$ is a set of shapes, we use $X_G$ to denote $\sum_{\alpha \in G}\lambda_\alpha X_\alpha$. In particular, except for $X = Q$ and $X = Q_i$, $\SS{X}$ will denote the sum of the simple spider terms from $X$ and $\SSD{X}$ will denote the sum of the simple spider disjoint unions terms from $X$.

For $X=Q$ or $X=Q_i$, we first take the well-behaved part of $X$---i.e., the terms from well-behaved intersection configurations in \eqref{eq:Qj}, \eqref{eq:Q}---and then apply this restriction. See also \Cref{def:QSS}.
\end{definition}

\subsection{A Closed Formula for $[Q]_{\wellbehaved}$}\label{sec:QSS_formula}
In this section, we derive closed-form formulas that characterize $[Q]_{\wellbehaved}$ and its simple spider part.

We first set up some notation. By \Cref{cor:Q_wellbehaved}, 
\begin{align*}
    [Q_j]_{\wellbehaved} &= 
{(-1)^{j}\sum_{{\mathcal{P}} \in \intsetoflength{j}:\textrm{ $\mathcal{P}$ is well-behaved}}{\eta_{\mathcal{P}}{\lambda_{\mathcal P, \reduced}}N(\mathcal{P})\frac{M_{\tau_{{\mathcal{P}},\reduced}}}{|\Iso(\tau_{\mathcal{P}})|!}}},\\
    [Q]_{\wellbehaved} &= \Sum_{j=0}^{2\dsos} [Q_j]_{\wellbehaved}.
\end{align*}
Recall that each $[Q]_{\wellbehaved}$ contains only simple spider disjoint union shapes since we have removed all isolated vertices from each $\tau_{\mathcal{P},\reduced}$. 
Based on this, we make the following definition.

\begin{definition}[$\SS{Q}$ and $\SSD{Q}$]\label{def:QSS} 
Viewing $[Q]_{\wellbehaved}$ as a linear combination of graph matrices, we define $\SS{Q}$ to be the simple spider part of $[Q]_{\wellbehaved}$ and $\SSD{Q}$ to be the simple spider disjoint union part of $[Q]_{\wellbehaved}$. Note that $\SSD{Q}=[Q]_{\wellbehaved}$.

Similarly, for $j=0,\ldots,\dsos$, we define $\SS{(Q_j)}$ and $\SSD{(Q_j)}$ to be the simple spider part and the simple spider disjoint union part of $[Q_j]_{\wellbehaved}$, respectively. Note that $\SSD{(Q_j)}=[Q_j]_{\wellbehaved}$.
\end{definition}

\begin{lemma}[Closed formula for $\SS{Q}$]\label{lem:LQLSSbehavior}
$L_{SS} \SALDprod Q_{SS} \SALDprod L_{SS}^\top = L_{SS} + {(Q_0)}_{SS} + L^\top_{SS} - 2\Id.$
\end{lemma}

\begin{proof}
We make the following observations:
\begin{enumerate}
\item $(Q_1)_{SS} = (L_{SS} + (Q_0)_{SS} + L^\top_{SS} - 2\Id) -L_{SS} \SALDprod {(Q_0)}_{SS} \SALDprod L^\top_{SS}$ as the terms in $L_{SS} \SALDprod {(Q_0)}_{SS} \SALDprod L^\top_{SS}$ which do not give a well-behaved intersection configuration are $(L_{SS} + (Q_0)_{SS} + L^\top_{SS} - 2\Id)$.
\item For all $i > 1$, $(Q_i)_{SS} = (Q_{i-1})_{SS} - L_{SS} \SALDprod {(Q_{i-1})}_{SS} \SALDprod L^\top_{SS}$ as the terms in $L_{SS} \SALDprod {(Q_{i-1})}_{SS} \SALDprod L^\top_{SS}$ which do not give a well-behaved intersection configuration are $(Q_{i-1})_{SS}$.
\end{enumerate}
Using these observations, we can write
\begin{align*}
L_{SS} \SALDprod Q_{SS} \SALDprod L_{SS}^\top &= \sum_{j=0}^{2\dsos}{L_{SS} \SALDprod (Q_j)_{SS} \SALDprod L_{SS}^\top} \\
&= (L_{SS} + (Q_0)_{SS} + L^\top_{SS} - 2\Id) - (Q_1)_{SS} + \sum_{j=1}^{2\dsos}{((Q_j)_{SS} - (Q_{j+1})_{SS})} \\
&= L_{SS} + (Q_0)_{SS} + L^\top_{SS} - 2\Id. \qedhere
\end{align*}
\end{proof}

There is a similar equation for $[Q]_{\wellbehaved}$ but it is more intricate.

\begin{definition} \ 
\begin{enumerate}
\item We define $[L_{SSD} \wbp \SSD{(Q_i)} \wbp L^{\top}_{SSD}]_{\text{no intersections}}$ to be the part of $L_{SSD} \wbp \SSD{(Q_i)} \wbp L^{\top}_{SSD}$ where there are no intersections between $L_{SSD}$, $\SSD{(Q_i)}$, and $L^{\top}_{SSD}$.
\item We define $[L_{SSD} \wbp \SSD{(Q_i)} \wbp L^{\top}_{SSD}]_{\geq 1\text{ intersections}}$ to be the part of $L_{SSD} \wbp \SSD{(Q_i)} \wbp L^{\top}_{SSD}$ where there is at least one intersection between $L_{SSD}$, $\SSD{(Q_i)}$, and $L^{\top}_{SSD}$.
\end{enumerate}
\end{definition}

\begin{lemma}
\label{lem:Q_wellbehaved_characterization}
\begin{equation}\label{eq:Q_wellbehaved_characterization}
L_{SSD} \wbp \SSD{Q} \wbp L^{\top}_{SSD} = [L_{SSD} \wbp \SSD{(Q_0)} \wbp L^{\top}_{SSD}]_{\text{no intersections}}.
\end{equation}
\end{lemma}
\begin{proof}
We make the following observations:
\begin{enumerate}
\item\label{statement1} For all $i \in [2\dsos] \cup \{0\}$, 
\begin{align*}
L_{SSD} \wbp \SSD{(Q_i)} \wbp L^{\top}_{SSD} &= [L_{SSD} \wbp \SSD{(Q_i)} \wbp L^{\top}_{SSD}]_{\text{no intersections}} \\
&+ [L_{SSD} \wbp \SSD{(Q_i)} \wbp L^{\top}_{SSD}]_{\geq 1\text{ intersections}}
\end{align*}
\item\label{statement2} For all $i \in [2\dsos]$, 
\[
[L_{SSD} \wbp \SSD{(Q_i)} \wbp L^{\top}_{SSD}]_{\geq 1\text{ intersections}} = -[L_{SSD} \wbp \SSD{(Q_{i+1})} \wbp L^{\top}_{SSD}]_{\text{no intersections}}
\]
\item\label{statement3} $[L_{SSD} \wbp \SSD{(Q_{2\dsos})} \wbp L^{\top}_{SSD}]_{\geq 1\text{ intersections}} = 0$
\end{enumerate}
The first statement follows from the definitions. The third statement follows from the second statement and the fact that $Q_{2\dsos + 1} = 0$ (i.e., the recursive factorization terminates after $2\dsos$ steps). For the second statement, we observe that the terms of $L_{SSD} \wbp [Q_i]_{\wellbehaved} \wbp L^\top_{SSD}$ which involve at least one intersection exactly cancel out the terms of $L_{SSD} \wbp [Q_{i+1}]_{\wellbehaved} \wbp L^\top_{SSD}$ which do not involve any intersections. More precisely, we prove statement \ref{statement2} by the following bijection between ribbons in $[L_{SSD} \wbp \SSD{(Q_i)} \wbp L^{\top}_{SSD}]_{\geq 1\text{ intersections}}$ and ribbons in $[L_{SSD} \wbp \SSD{(Q_{i+1})} \wbp L^{\top}_{SSD}]_{\text{no intersections}}$. 

Given a triple of ribbons $R_1$, $R_2$, and $R_3$ contributing to $[L_{SSD} \wbp \SSD{(Q_i)} \wbp L^{\top}_{SSD}]_{\geq 1\text{ intersections}}$, we obtain a triple of ribbons $R'_1$, $R'_2$, and $R'_3$ contributing to $[L_{SSD} \wbp \SSD{(Q_{i+1})} \wbp L^{\top}_{SSD}]_{\text{no intersections}}$ as follows. For this argument, we do not delete double edges from the ribbons so $R_2$ and $R'_2$ may have double edges.

We split the ribbon $R_1$ into ribbons $R'_1$ and $R''_1$ as follows. 
\begin{enumerate}[label=(\arabic*)]
\item Let $C_{1,\text{int}}$ be the set of circle vertices of $R_1$ which also appear in $R_2$ or $R_3$ (i.e., the set of circle vertices in $R_1$ which are intersected with another circle vertex). Let $A$ be the set of vertices in $V(R_1)$ which are either in $C_{1,\text{int}}$ or are adjacent to a vertex in $C_{1,\text{int}}$. Let $B = V(R_1) \setminus A$.
\item We take $R''_1$ to be the ribbon such that $V(R''_1) = A \cup (B \cap V_{R_1})$, $U_{R''_1} = (A \cap U_{R_1}) \cup (B \cap V_{R_1})$, $V_{R''_1} = (A \cap V_{R_1}) \cup (B \cap V_{R_1})$, and $E(R''_1) = \{\{u,v\} \in E(R_1): u,v \in A\}$. Intuitively, $R''_1$ is the part of $R_1$ which is involved in the intersection(s).
\item Similarly, we take $R'_1$ to be the ribbon such that $V(R'_1) = B \cup (A \cap U_{R_1})$, $U_{R'_1} = (B \cap U_{R_1}) \cup (A \cap U_{R_1})$, $V_{R'_1} = (B \cap V_{R_1}) \cup (A \cap U_{R_1})$, and $E(R'_1) = \{\{u,v\} \in E(R_1): u,v \in B\}$. Intuitively, $R'_1$ is the part of $R_1$ which is not involved in the intersection(s).
\end{enumerate}
Observe that $M_{R_1} = M_{R'_1}M_{R''_1}$. Similarly, we split $R_3$ into 
$R''_3$ and $R'_3$ so that $M_{R_3} = M_{R''_3}M_{R'_3}$, and we let $R'_2$ be the improper ribbon such that $M_{R'_2} = M_{R''_1}M_{R_2}M_{R''_3}$. Then the triple $(R'_1, R'_2, R'_3)$ contributes to $[L_{SSD} \wbp \SSD{(Q_{i+1})} \wbp L^{\top}_{SSD}]_{\text{no intersections}}$, $M_{R_1}M_{R_2}M_{R_3} = M_{R'_1}M_{R'_2}M_{R'_3}$, and the product of  coefficients of $M_{R_1},M_{R_2},M_{R_3}$ in $[L_{SSD} \wbp \SSD{(Q_i)} \wbp L^{\top}_{SSD}]_{\geq 1\text{ intersections}}$ is $-1$ times the product of the coefficients of $M_{R'_1},M_{R'_2},M_{R'_3}$ in $[L_{SSD} \wbp \SSD{(Q_{i+1})} \wbp L^{\top}_{SSD}]_{\text{no intersections}}$. Thus, these terms cancel.

It is not hard to check that the map $(R_1,R_2,R_3) \to (R'_1,R'_2,R'_3)$ is invertible, and consequently, it induces a bijection between ribbons in $[L_{SSD} \wbp \SSD{(Q_i)} \wbp L^{\top}_{SSD}]_{\geq 1\text{ intersections}}$ and ribbons in $[L_{SSD} \wbp \SSD{(Q_{i+1})} \wbp L^{\top}_{SSD}]_{\text{no intersections}}$. Using these observations, we have that $L_{SSD} \wbp \SSD{Q} \wbp L^{\top}_{SSD}$ is equal to 
\begin{align*}
&\sum_{i=0}^{2\dsos}{[L_{SSD} \wbp \SSD{(Q_i)} \wbp L^{\top}_{SSD}]_{\geq 1\text{ intersections}}} + \sum_{i=0}^{2\dsos}{[L_{SSD} \wbp \SSD{(Q_i)} \wbp L^{\top}_{SSD}]_{\text{no intersections}}}\\
=&\sum_{i=0}^{2\dsos-1}{\left([L_{SSD} \wbp \SSD{(Q_i)} \wbp L^{\top}_{SSD}]_{\geq 1\text{ intersections}} + [L_{SSD} \wbp \SSD{(Q_{i+1})} \wbp L^{\top}_{SSD}]_{\text{no intersections}}\right)}\\
 \hspace{8pt} &+[L_{SSD} \wbp \SSD{(Q_{2\dsos})} \wbp L^{\top}_{SSD}]_{\geq 1\text{ intersections}}+[L_{SSD} \wbp \SSD{(Q_0)} \wbp L^{\top}_{SSD}]_{\text{no intersections}}\\
=&[L_{SSD} \wbp \SSD{(Q_0)} \wbp L^{\top}_{SSD}]_{\text{no intersections}}. \qedhere
\end{align*}
\end{proof}

It is not hard to see that $[L_{SSD} \wbp \SSD{(Q_0)} \wbp L^{\top}_{SSD}]_{\text{no intersections}}$ is equal to the SSD part of their canonical product, i.e., $\SSD{\left([\SSD{L},\SSD{(Q_0)},\SSD{L}^\top]_{\can}\right)}$. The latter is equal to $\SSD{\left([L,Q_0,L^\top]_{\can}\right)}$, which is further equal to $\SSD{M}$ as no well-behaved configuration of length 1 contributes to $[L,Q_0,L^\top]_{\can} - M$ when $\truncation\geq 18\dsos$ (Cf. \Cref{rmk:wellbehaved_no_truncation_err}).
Combining this with \Cref{lem:Q_wellbehaved_characterization}, we get:
\begin{corollary}[{Closed formula for $[Q]_{\wellbehaved}$}]\label{cor:Q_wellbehaved_characterization} If $\truncation \geq 20\dsos^2$, then 
\[L_{SSD} \wbp \SSD{Q} \wbp L^{\top}_{SSD} = \SSD{M}.\]
\end{corollary}

\subsection{Alternative Analyses for Degree 4 and Degree 6 SoS}\label{sec:degree46}
Before settling on our final approach for analyzing $Q$, we found two alternative approaches which work well for low degrees. We present these approaches because they are interesting in their own right and because they demonstrate the technical and conceptual challenges which we overcome. For simplicity, we assume that $A$ is an even distribution (i.e., $m_A(I)=m_A(-I)$ for all intervals $I\subseteq \R$). These approaches are as follows: 
\begin{enumerate}
\item Write $[Q]_{well-behaved} = Z \wbp Z^{\top}$ for some matrix $Z$ which is a linear combination of disjoint unions of simple spiders, find equations for the coefficients of $Z$, and show that these equations are solvable for all distributions $A$.
\item Show that we can write $[Q]_{well-behaved} = \E_{a \sim A}[Z(a) \wbp Z(a)^{\top}]$.
\end{enumerate}

We demonstrate how the two approaches can be used to obtain degree $4$ and degree $6$ SoS lower bounds when the distribution $A$ is even. The dominating part of $Q$ in these small degrees consists of simple spiders only. For simplicity, we take this for granted and ignore other terms. Thus, we can write $Q_{SS}$ rather than $[Q]_{well-behaved}$.
\smallskip

For degree 4, the final $Q_{SS}$ we obtain is $Q_{SS} = \Id + {l_4}S(2,2;0)$. 

For the first approach, taking $Z = \Id + xS(2,2;0)$ and recalling that $S(2,2;0) \SALDprod S(2,2;0) = \frac{1}{2}S(2,2;0)$, we have
\[
Z \SALDprod Z^{\top} = \Id + (2x + \frac{x^2}{2})S(2,2;0).
\]
To solve this, we need $l_4 = \E_{a \sim A}[h_4(a)] = 2x + \frac{x^2}{2}$. Since $\frac{x^2}{2} + 2x = \frac{1}{2}(x + 2)^2 - 2$, it can take values in the range $[2,\infty)$, so there is a solution as long as $l_4 \geq -2$. But this is true for any distribution $A$ with $\E_{a \sim A}[a^2] = 1$, as by Cauchy-Schwarz, $1 = \left(\E_{a \sim A}[a^2 \cdot 1]\right)^2 \leq \E_{a \sim A}[(a^2)^2] \E_{a \sim A}[1] = \E_{a \sim A}[a^4]$ so
\[
l_4 = \E_{a \sim A}[h_4(a)] = \E_{a \sim A}[a^4 - 6a^2 + 3] \geq -2.
\]

For the second approach, taking $Z = \Id + x(a)S(2,2;0)$ where $x$ is an undetermined function in the sample $a$ of $A$, we want: 
\[
l_4 = \E_{a \sim A}[h_4(a)] = \E_{a \sim A}\left[2x(a) + \frac{x(a)^2}{2}\right].
\]
We let $x(a) := {b_2}(a^2-1) + {b_1}a + b_0$ for constants $b_0,b_1,b_2$ and observe that since $\E_{a \sim A}[a^2] = 1$,
\[
E_{a \sim A}\left[(a^2 - 1)^2\right] = E\left[a^4 - 2a^2 + 1\right] = E\left[(a^4 - 6a^2 + 3) + 4(a^2 - 1) + 2\right] = l_4 + 2.
\]
This means
\[
E_{a \sim A}\left[\frac{x(a)^2}{2} + 2x(a)\right] = 
\frac{b_2^2}{2}(l_4 + 2) + \frac{b_1^2}{2} + \frac{b_0^2}{2} + 2b_0.
\]
In order for this to equal $l_4$, we need that $\frac{b_2^2}{2} = 1$ and $b_2^2 + \frac{b_1^2}{2} + \frac{b_0^2}{2} + 2b_0 = 0$. The real solutions are $b_2 = \pm{\sqrt{2}}$, $b_1 = 0$, and $b_0 = -2$. Thus, we can take $Z = \Id + (\sqrt{2}(a^2 - 1) - 2)S(2,2;0)$.
\smallskip

For degree $6$, the final $Q_{SS}$ is 
\[
Q_{SS} = \Id + {l_4}S(2,2;1) + (l_6 - l_4^2)S(3,3;0).
\]
For the first approach, we take $Z = \Id + xS(2,2;1) + yS(3,3;0)$ and expand $Z \SALDprod Z^{\top}$ using the following facts which are not hard to check:
\begin{enumerate}
    \item $S(3,3;0) \SALDprod S(3,3;0) = \frac{1}{6}S(3,3;0)$
    \item $S(3,3;0) \SALDprod S(2,2;1) = S(2,2;1) \SALDprod S(3,3;0) = \frac{3}{2}S(3,3;0)$
    \item $S(2,2;1) \SALDprod S(2,2;1) = \frac{1}{2}S(2,2;1) + 9S(3,3;0)$
\end{enumerate}
Letting $Z \SALDprod Z^{\top} = Q_{SS}$, we need to determine the values of $x,y$ so that two equations hold:
\begin{enumerate}
    \item $\frac{x^2}{2} + 2x = l_4$
    \item $\frac{y^2}{6} + 9x^2 + 3xy + 2y = l_6 - l_4$
\end{enumerate}
For a given $x$, the minimum value of $\frac{y^2}{6} + 9x^2 + 3xy + 2y$ is $-\frac{9}{2}x^2 - 18x - 6$ since  
\[
\frac{y^2}{6} + 9x^2 + 3xy + 2y = \frac{1}{6}(y + 9x + 6)^2 - \frac{9}{2}x^2 - 18x - 6.
\]
By the first equation, $-\frac{9}{2}x^2 - 18x - 36 = -9l_4 - 6$ so the two equations are feasible as long as $l_4 \geq -2$ and $l_6 - l_4^2 \geq -9l_4 - 6$. To show that these two inequalities hold for any even distribution $A$ with $\E_{a \sim A}[a^2] = 1$, we make the following observations:
\begin{enumerate}
\item $l_4 = \E_{a \sim A}[a^4 - 6a^2 + 3] = \E_{a \sim A}[a^4] - 3$
\item $l_6 = \E_{a \sim A}[a^6 - 15a^4 + 45a^2 - 15] = \E_{a \sim A}[a^6] - 15E_{a \sim A}[a^4] + 30$
\item $l_4^2 = \left(\E_{a \sim A}[a^4 - 6a^2 + 3]\right)^2 = (\E_{a \sim A}[a^4] - 3)^2 = (\E_{a \sim A}[a^4])^2 - 6E_{a \sim A}[a^4] + 9$
\item By Cauchy-Schwarz, $\left(\E_{a \sim A}[a^4]\right)^2 = \left(\E_{a \sim A}[a^3 \cdot a]\right)^2 \leq \E_{a \sim A}[a^6]E_{a \sim A}[a^2] = \E_{a \sim A}[a^6]$.
\end{enumerate}
Putting these observations together,
\[
l_6 - l_4^2 \geq -9E_{a \sim A}[a^4] + 21 = -9(l_4 + 3) + 21 = -9l_4 - 6,
\]
so we can find $x,y$ in the first approach so that $Z \SALDprod Z^{\top} = Q_{SS}$. 
\smallskip

For the second approach, we take $Z = \Id + x(a)S(2,2;1) + y(a)S(3,3;0)$. Similar to before,  we need to have that
\begin{enumerate}
\item $\E_{a \sim A}\left[\frac{x(a)^2}{2} + 2x(a)\right] = l_4$,
\item $\E_{a \sim A}\left[\frac{y(a)^2}{2} + 9x(a)^2 + 3x(a)y(a) + 2y(a)\right] = l_6 - l_4^2$.
\end{enumerate}
As before, we will take $x(a) = \sqrt{2}(a^2 - 1) - 2$. As for $y(a)$, we let  
\[
y(a) = {b_3}(a^3 - 3a) + {b_2}(a^2 - 1) + ({b'_1}l_4 + b_1)a + ({b'_0}l_4 + b_0)
\]
and determine the values of $b_3,b_2,b_1,b'_0,b_0$. The following computations are useful:
\begin{enumerate}
\item As we computed before, $\E_{a \sim A}[(a^2 - 1)(a^2 - 1)] = l_4 + 2$.
\item $\E_{a \sim A}[(a^3 - a)a] = \E_{a \sim A}[a^4 - 3a^2] = \E_{a \sim A}[(a^4 - 6a^2 + 3) + 3(a^2 - 1)] = l_4$
\item 
$\E_{a \sim A}[(a^3 - 3a)^2] = \E_{a \sim A}[a^6 - 6a^4 + 9a^2]= l_6 + 9l_4 + 6$
\end{enumerate}
We now compute the expected values of the products involving $x$ and/or $y$:
\begin{enumerate}
\item $\E_{a \sim A}[y(a)^2] = {b_3^2}(l_6 + 9l_4 + 6) + {b_2^2}(l_4 + 2) + 2b_3({b'_1}l_4 + b_1)l_4 + ({b'_1}l_4 + b_1)^2 + ({b'_0}l_4 + b_0)^2$
\item $\E_{a \sim A}[x(a)y(a)] = \sqrt{2}b_2(l_4 + 2) - 2({b'_0}l_4 + b_0)$
\item $\E_{a \sim A}[x(a)^2] = 2(l_4 + 2) + 4 = 2l_4 + 8$
\item $\E_{a \sim A}[y(a)] = ({b'_0}l_4 + b_0)$
\end{enumerate}
Recall that we want $\E_{a \sim A}\left[\frac{y(a)^2}{6} + 9x(a)^2 + 3x(a)y(a) + 2y(a)\right] = l_6 - l_4^2$. For this we observe: 
\begin{enumerate}
\item The coefficient of $l_6$ in $\E_{a \sim A}\left[\frac{y(a)^2}{6} + 9x(a)^2 + 3x(a)y(a) + 2y(a)\right]$ is $\frac{b_3^2}{6}$, so we let $b_3 = \pm{\sqrt{6}}$;
\item The coefficient of $l_4^2$ in $\E_{a \sim A}\left[\frac{y(a)^2}{6} + 9x(a)^2 + 3x(a)y(a) + 2y(a)\right]$ is $\frac{{b'_1}^2}{6} + \frac{{b_3}{b'_1}}{3} + \frac{{b'_0}^2}{6} = \frac{1}{6}(b'_1 + b_3)^2 + \frac{{b'_0}^2}{6} - 1$, which can only be $-1$ if $b'_1 = -\sqrt{6}$ and $b'_0 = 0$;
\item The constant coefficient in $\E_{a \sim A}\left[\frac{y(a)^2}{6} + 9x(a)^2 + 3x(a)y(a) + 2y(a)\right]$ is
\begin{align*}
b_3^2 + \frac{b_2^2}{3} + \frac{b_1^2}{6} + \frac{b_0^2}{6} + 72 + 6\sqrt{2}b_2 - 6b_0 + 2b_0 &= 6 + \frac{1}{3}(b_2 + 9\sqrt{2})^2 - 54 + \frac{1}{6}(b_0 - 12)^2 - 24 + 72 + \frac{b_1^2}{6}\\
&= \frac{1}{3}(b_2 + 9\sqrt{2})^2 + \frac{1}{6}(b_0 - 12)^2 + \frac{b_1^2}{6},
\end{align*}
which is only $0$ if $b_2 = -9\sqrt{2}$, $b_0 = 12$, and $b_1 = 0$.
\end{enumerate}
Thus, we let 
\[
y(a) = {\pm}\sqrt{6}(a^3 - 3a) -9\sqrt{2}(a^2 - 1) -\sqrt{6}{l_4}a + 12.
\]
With this $x(a)$ and $y(a)$, the coefficient of $l_4$ in $\E_{a \sim A}\left[\frac{y(a)^2}{6} + 9x(a)^2 + 3x(a)y(a) + 2y(a)\right]$ is
\[
9 + 27 - 54 + 18 = 0,
\]
as needed.

\begin{remark}\label{rmk:commutative}
In both approaches mentioned above, the ``solvability'' of the target object relies on the commutativity of products such as $S(3,3;0) \SALDprod S(2,2;1) =  S(2,2;1) \SALDprod S(3,3;0)$, which holds for simple spiders having an equal number of left and right legs. 
For general simple spiders, multiplication is non-commutative.
\end{remark}

\noindent{\bf Important reminder.} In the rest of the lower bound proof, namely from \Cref{sec:psdness-qSS} to \Cref{sec:error_analysis}, we always assume that the samples $x_1,\dots,x_m$ from $\R^n$ satisfies the event in \Cref{thm:norm_control}, which happens with probability $1-o_n(1)$. 
The remaining arguments involve no more probability.
	\newcommand{\rep}[3]{\rho\left(#1,#2;#3\right)}

\section{PSDness via Representation}
\label{sec:psdness-qSS}

The main result of this section is the following lemma. 
Recall the matrix $\SS{Q}$ from \Cref{def:QSS} and the quantities $\cu,\cl$ from \Cref{def:cucl}.

\begin{lemma}[Positive-definiteness of $\SS{Q}$ in $\SALD$]\label{lem:QSSPSD_informal}
Within the algebra $\SALD$, 
\begin{equation}\label{eq:QSSPSD}
  \SS{Q} \succeq (6\cu)^{-2\dsos}\cl^{-\dsos} \Id.
\end{equation} 
\end{lemma}

\Cref{lem:QSSPSD} is a more precise statement containing more technical conclusions needed later. 
\smallskip

The rest of this section is devoted to the proof of \Cref{lem:QSSPSD}, i.e., under mild conditions on the distribution $A$, the matrix $\SS{Q}$ is positive-definite in the simple spider algebra. In \Cref{subsec:overview} we give the high-level plan of the proof. In \Cref{subsec:rep} through \Cref{sec:rhoQhatPSD}, we define an algebra representation $\rho$ for the simple spider algebra and show the PSDness of $\rho(\SS{Q})$. Finally in \Cref{subsec:PDQhat}, we prove the positive-definiteness of $\rho(\SS{Q})$.

\subsection{Proof Overview for showing $M \succeq 0$}\label{subsec:overview}

The overall analysis of the PSDness of $M$ proceeds as follows. We approximate $M$ by $L{[Q]_{\wellbehaved}}L^T$ and show that the error terms are small. More precisely, we have the following equality
\begin{align*}
M &= \left(M - [L,Q_0,L^{\top}]_{\can}\right) + \left([L,Q_0,L^{\top}]_{\can} - LQL^{\top}\right) \\
&+L\left(Q - [Q]_{\wellbehaved}\right)L^{\top} + L[Q]_{\wellbehaved}L^{\top}
\end{align*}
\noindent To prove our lower bound, we need to show the following:

\begin{enumerate}
\item $[Q]_{\wellbehaved}$ is positive-definite. More precisely,  
$[Q]_{\wellbehaved} \succeq {\delta}\Id$ for some $\delta > 0$ (which will depend on $A$ and $D$ but not on $n$).
\item $\norm{Q - [Q]_{\wellbehaved}}$ is $n^{-\Omega(\eps/\dsos)}$, 
\item $\norm{M - [L,Q_0,L^{\top}]_{\can}}$ and $\norm{[L,Q_0,L^{\top}]_{\can} - LQL^{\top}}$ are both $n^{-\Omega(\eps\truncation)}$. 
\item $L$ is somewhat well-conditioned. More precisely, the smallest singular value of $L$ is at least $(5n)^{-D} \Id$ (\Cref{lem:LLT}).
\end{enumerate}

\noindent In this section we show that in the simple spider algebra, $Q_{SS}$ is positive-definite. We show this as follows:

\begin{enumerate}
\item We use the fact that $L_{SS} \SALDprod Q_{SS} \SALDprod L^{\top}_{SS} = L_{SS} + (Q_{0})_{SS} + L^{\top}_{SS} - 2\Id$ (\Cref{lem:LQLSSbehavior}). \label{step:explicitQSS}
\item Using a representation of the simple spider algebra, we show that $L_{SS} + (Q_{0})_{SS} + L^{\top}_{SS} - 2\Id$ is positive-definite in the simple spider algebra. \label{step:rep}
\end{enumerate}

In the next section, we use additional ideas to deal with disjoint unions of simple spiders in order to show that $[Q]_{\wellbehaved}$ is positive-definite.

\subsection{Useful Notions in the Study of Simple Spiders}

Recall that $S(i,j;u)$ denotes the scaled simple spider shape $\alpha$ with $i$ weight-1 edges to the left, $j$ weight-1 edges to the right, and $u$ intersections. 

\begin{definition}[Default expansion]
    Given $\alpha\in\SALD$, we use $c_{\alpha}(i,j;u)$ to denote the coefficient of $S(i,j;u)$ in $\alpha$. We will sometimes write $\alpha\!=\!\Sum_S c_\alpha(S)S$ or simply $\Sum_{i}c_iS_i$ when there is no confusion. 
\end{definition}

\begin{definition}[Good simple spiders and disjoint unions]\label{def:good}
    We call $S(i,j;u)$ {\bf good} if it either has a trivial shape, i.e., $i=j=0$, or its unique circle vertex has at least $\ceil{k/2}$ edges to both sides, i.e., $i,j \geq \ceil{k/2}$. 
    We call a simple spider disjoint union good if all of its components are good simple spiders. An SS- or SSD-linear combination is good if all nonzero summands are good.
\end{definition}

\begin{definition}[Consistent simple spiders]\label{def:consistent}
A simple spider linear combination $\alpha=\sum_{i} c_i S_i$ in $\SALD$ is {\bf consistent} if $c_i=c_j$ whenever $S_i,S_j$ differ only by their intersection sizes and the coefficient of $S(0,0;0)$ is 1.
\end{definition}

Note that if $\alpha\in\SALD$ is consistent, then every trivial shape with at most $\dsos$ vertices has coefficient 1 in $\alpha$.

\begin{definition}[$\linfty{\cdot}$-norm on $\SALD$ and SSD combinations]
\label{def:linfty}
For $\beta=\sum_{S} c_\beta(S) S\in\SALD$, we let $\linfty{\beta}:=\max_i\{|c_i|\}$. Similarly, if $\beta=\sum_{S} c_\beta(S) S$ where each $S$ is a scaled SSD shape (i.e., $S=\lambda_\alpha M_{\alpha}$ for some SSD shape $\alpha$), we let $\linfty{\beta}:=\max_i\{|c_i|\}$.
\end{definition}

When viewing an element of $\SALD$ as a linear combination of scaled SS graph matrices, the $\linfty{\cdot}$ norm corresponds to the maximal coefficient of the scaled SS shapes.

\begin{remark}[Algebra $\SALD$ and its dimension]\label{rmk:SALD}
Recall that the $\SALDprod$-product is associative by \Cref{lem:wbp-associative}, and there is a unit element $\sum_{i=0}^D S(i,i;0)$ in $\SALD$, so $\SALD$ is indeed an $\R$-algebra. (Another way to see this is via the representation \Cref{lem:iso} below.) The dimension $\dim_{\R}(\SALD)$ is the number of different simple spiders with indices size at most $D$. To count them, note that for fixed $|U_\alpha\cap V_\alpha|=i\in[0,D]$, both the left and right degree of the circle vertex in a simple spider can range in $[0,D-i]$, where if both degrees are $0$ it means that the shape is trivial. Thus, there are $\sum_{i=0}^D (D-i+1)^2=\sum_{i=1}^{D+1}i^2$ many simple spiders.
\end{remark}

\begin{remark}(Good simple spiders)\label{rmk:good}
    The algebra $\SALD$ approximates the simple spider part in the result of matrix multiplication of \emph{good} simple spiders. This follows from \Cref{lem:Smultiplication} together with \Cref{lemma:reductioneffect} and \Cref{lem:alphazeroerror} which we will show later.
\end{remark}

Observe that in the canonical decomposition of $M$, both $\SS{(Q_0)}$ and $\SS{L}$ are consistent. Moreover, $\SS{(Q_0)}$ is left, right, and good. The goodness comes from the assumptions that $\E_A[h_i]=0$ for $i=1,\dots,(k-1)$ and that all shapes in $Q_0$ are middle shapes. Note that $\SS{L}$ is left and $\SS{L}^{\top}$ is right, but they are not good.

\subsection{Representation of the Simple Spider Algebra}\label{subsec:rep}
In this subsection, we construct our representation $\rho$ of the simple spider algebra. First, we define a linear map $\rhopre$ from $S(k_1-u,k_2-u;u)$ to $\left(\sum_{i=0}^D(i+1)\right)$-by-$\left(\sum_{i=0}^D(i+1)\right)$ real matrices and show that it gives an algebra isomorphism. 
\begin{equation}\label{eq:rep}
\text{\bf Notation: we write }\reppre{k_1}{k_2}{u}\text{\bf for }\rhopre\big( S(k_1-u,k_2-u;u) \big).
\end{equation}
This notation is convenient for us since the image of $\rhopre(-)$ are block matrices and the block $(k_1,k_2)$ is where $\rhopre(S(k_1-u,k_2-u;u))$ is supported. The subscript means `preparatory', for reasons that will be clear soon.

\begin{definition}[Representation $\rhopre$]\label{def:rhopre}
For $k_1,k_2\in[0,D]$ and $0\leq u\leq\min\{k_1,k_2\}$, we define $\reppre{k_1}{k_2}{u}=\rho(S(k_1-u,k_2-u;u))$ to be a matrix on blocks $\{B(a,b)|0\leq a,b\leq D\}$ where $B(a,b)$ has dimension $[0,a]\times [0,b]$. The matrix $\reppre{k_1}{k_2}{u}$ is 0 outside $B(k_1,k_2)$, and within $B(k_1,k_2)$ it is as follows. For each $(i,j)\in[0,k_1]\times[0,k_2]$,
\begin{equation}\label{eq:rhopre}
\reppre{k_1}{k_2}{u}\upharpoonright_{B(k_1,k_2)}(i,j):=
\begin{cases*}
0 &if $j-i\neq k_2-k_1$ or $i<k_1-u$;\\
\binom{i}{k_1-u}/(k_2-u)!& o.w.
\end{cases*}
\end{equation}
In other words, $\reppre{k_1}{k_2}{u}$ is a matrix supported on block $B(k_1,k_2)$ with all the nonzero entries placed `diagonally up from the bottom-right'. We let $\rhopre$ be defined over $\SALD$ by linear extension.
\end{definition}

The key point is that $\rhopre$ gives an $\R$-algebra isomorphism, as shown in \Cref{app:psdness-qSS}.

\begin{lemma}\label{lem:iso}
For any fixed $D$, $\rhopre:\ \SALD\to\rhopre(\SALD)$ is an $\R$-algebra isomorphism.
\end{lemma}

Note that in general, $\rhopre(S^{\top})\neq\rhopre(S)^{\top}$. They in fact differ by a conjugation by a diagonal matrix, which we denote by $C$:

\begin{definition} [Rescaling matrix] 
Let the diagonal matrix $C$ be    
\begin{equation}\label{eq:rescaleC}
C\upharpoonright_{B(i,i)}(a,a)=a!\ \text{if $0\leq a\leq i\leq D$ and $0$ otherwise.}
\end{equation}
Let $\sqrt{C}$ be the diagonal matrix that takes the positive square roots of $C$'s entries. 
\end{definition}

\begin{proposition}\label{prop:transpose}
\[\rhopre(S)^{\top}=C^{-1}\rhopre(S^{\top})C\]
\end{proposition}
\begin{proof}
Suppose $S=S(i,j;u)$. By definition \eqref{eq:rhopre}, $\rhopre(S)\upharpoonright_{B(i,j)}(a,a+j-i)=\binom{a}{i}/j!$ so 
\[\rhopre(S)^{\top}\upharpoonright_{B(j,i)}(a+j-i,a)=\binom{a}{i}/j!.\] 
On the other hand, $\rhopre(S^{\top})\upharpoonright_{B(j,i)}(a+j-i,a)=\binom{a+j-i}{j}/i!$ and $C$ is diagonal, so 
\[C^{-1}\rhopre(S^{\top})C\upharpoonright_{B(j,i)}(a+j-i,a)=1/(a+j-i)!\cdot\binom{a+j-i}{j}/i!\cdot a!=\binom{a}{i}j!,\]
assuming the binomials are defined (otherwise they are zero). This proves the proposition.
\end{proof}

We now give our actual representation of the simple spider algebra, which is $\rhopre$ composed with a conjugation of $\sqrt{C}$.

\begin{definition}[Representation $\rho$]\label{def:rho} 
The $\R$-algebra isomorphism $\rho: \SALD\to\im(\rho)\subseteq \mathbb M_{D(D+1)/2}(\R)$ is the composition of $\text{$\sqrt{C}$-conjugation}$ and $\rhopre$, i.e., $\rho(X)=\sqrt{C}^{-1}\rhopre(X)\sqrt{C}$. Entry-wise, we again denote $\rho\left(k_1,k_2;u\right):=\rho\left(S(k_1-u,k_2-u;u)\right)$, then
\begin{equation}\label{eq:rho}
\rho\Big(k_1,k_2;u\Big)\upharpoonright_{B(k_1,k_2)}(i,j):=
\begin{cases*}
0 &if $j-i\neq k_2-k_1$ or $i<k_1-u$;\\
{\sqrt{i!j!}\over (k_1-u)!(k_2-u)!\big(u-(k_1-i)\big)!}& o.w.
\end{cases*}
\end{equation}
\end{definition}

By \Cref{prop:transpose} we have that $\rho$ preserves the transpose:
\begin{equation}\label{eq:Cstar}
\rho(S)^{\top}=\rho(S^{\top})\quad\forall S\in\SALD.
\end{equation}
In other words, we have:
\begin{lemma}[$\rho$ preserves sum-of-squares]\label{lem:transfer}
If $\rho(Y)=\sum_i Y_iY_i^{\top}$ such that $Y_i\in\im(\rho)$ for each $i$, then we have $Y=\sum_i Z_i\star Z_i^{\top}$ where $Z_i:=\rho^{-1}(Y_i)$.
\end{lemma}

Therefore, we can translate sums-of-squares expressions in $\im(\rho)$ to ones in $\SALD$. 
For this reason, we will be working with $\rho$ and only call $\rhopre$ to assist with computations.  

\subsection{Structure of the Representation $\rho$}
\label{subsec:repproperties}
In this subsection, we give an explicit decomposition of the algebra $\SALD\simeq\rho(\SALD)$ into matrix algebras which will be very useful. Recall that $\mathbb M_i(\R)$ denotes the algebra of all $i$-by-$i$ real matrices. 

\begin{lemma}[Structure of $\im(\rho)$]\label{lem:directsum}
We have an algebra isomorphism
\begin{equation}\label{eq:directsumabstract}
\rho(\SALD)\simeq\bigoplus_{i=1}^{D+1} \mathbb M_{i}(\R)
\end{equation}
constructed as follows. 
Recall that $\rho(\SALD)$ consists of block matrices on blocks $\{B(i,j)\mid 0\leq i,j\leq D\}$ where $B(i,j)$ has dimension $(i+1)\times(j+1)$. Let 
\begin{equation}\label{eq:Ti}
e_i\ (i=0,1,\dots,D) 
\end{equation}
be the diagonal $\{0,1\}$-matrix with 1s on $\{(b-i,b-i)\mid b\geq i\}$ in $B(b,b)$ for all $b$ in $[0,D]$. Then 
\begin{equation}\label{eq:projector}
e_i(-)e_i:\ \im(\rho)\to\mathbb M_{D-i+1}(\R)
\end{equation}
is the projector to the subalgebra $\mathbb M_{D-i+1}(\R)$ in \eqref{eq:directsumabstract}.
\end{lemma} 
\begin{proof}
Given $X\in\im(\rho)$, note that the operator $e_i(-)e_i$ takes the submatrix supported on 
\begin{equation}\label{eq:suppi}
\supp_i:=\{(a-i,b-i)\text{ of }B(a,b)\mid a,b\in[i,D]\}
\end{equation}
and pads 0s elsewhere. Thus $e_i(-)e_i$ is idempotent and linear, and it follows by inspection that $e_i(X\cdot Y)e_i=e_i(X)e_i\cdot e_i(Y)e_i$ for all $X,Y\in\im(\rho)$. In other words, $e_i(-)e_i$ an algebra projection. For $i\neq j$, $\supp_i$ and $\supp_j$ are disjoint so $e_je_i(-)e_ie_j=0$, so these projectors are orthogonal to each other. Now if $X\in\im(\rho)$, by \eqref{eq:rho}, 
\begin{equation}\label{eq:rhosupport}
\begin{aligned}
\supp(X)&\subseteq\ \{(i,j)\text{ of block }B(a,b)\mid 0\leq i\leq a\leq D, 0\leq j\leq b\leq D, j-i=b-a\}\\
&=\coprod_{i=0}^D \{(a-i,b-i)\text{ of }B(a,b)\mid a,b\in[i,D]\}=\coprod_{i=0}^D \supp_i,
\end{aligned}
\end{equation}
so $X=e_0Xe_0+\dots + e_DXe_D$, i.e., $\Sum_{i=0}^D e_i(-)e_i$ is the identity map. 
Together, this means that 
\begin{equation}\label{eq:directsum}
\im(\rho)=\bigoplus_{i=0}^De_i(\im(\rho))e_i\ \text{as $\R$-algebras.}
\end{equation}
Finally, counting the dimension, 
\[\Sum_{i=0}^D\dim_{\R}\left(e_i(\im(\rho))e_i^{\top}\right)\stackrel{\text{\Cref{lem:iso}}}{=}\dim_{\R}(\SALD)=1^2+\dots+(D+1)^2\] 
which equals $\Sum_{i=0}^D\dim_{\R}\mathbb M_i(\R)$. Therefore, each $e_i(\im(\rho))e_i$ is the full matrix algebra $\mathbb M_{D-i+1}(\R)$.
\end{proof}

\begin{remark}[Why call it a representation]\label{rmk:rep}
The decomposition \eqref{eq:directsum} can be read as ``$\rho$ contains the information of all irreducible representations of the algebra'', for the following reason. First, each $\pi_i(-):=e_i\rho(-)e_i$ is an irreducible representation of $\SALD$ as the image is $\mathbb M_i(\R)$. Second, given \eqref{eq:directsum}, it is not hard to show that every finite-dimensional representation $\psi$ of $\SALD$ can be decomposed into $\psi_1\oplus\ldots\oplus\psi_t$, where each $\psi_i:=\psi\circ\left(\rho^{-1}(e_i)\star(-)\star\rho^{-1}(e_i)\right)$ is either the 0 map or a direct sum of representations that are isomorphic to $\pi_i$.
\end{remark}

\begin{remark}\label{rmk:W-A}
By the Wedderburn-Artin theorem, every finite-dimensional semisimple algebra is isomorphic to a direct sum of matrix algebras.\footnote{Strictly speaking, the semisimplicity of $\SALD$ is not immediately clear a priori, and the theorem does not completely specify the dimensions of the matrix subalgebras nor their underlying division algebras ($\R$, $\mathbb C$, or the quaternions).} Here, however, we emphasize that we need more than an abstract or `purely computational' existence of such a map. 
We need a concrete construction that helps analyze whether special elements in the algebra like $\hat{Q} := L_{SS} + L^{\top}_{SS} + (Q_0)_{SS} - 2\Id$ are sums of squares.   
\end{remark}

\begin{definition}[Components of $\rho$]\label{def:Mi}
For $\alpha\in\SALD$, we use $\rho_i(\alpha)$ to denote the $e_i(\rho(\alpha))e_i$ component in the direct sum \eqref{eq:directsum}, and we express the direct-sum decomposition of $\rho$ as 
\begin{equation}\label{eq:rho_i}
    \rho=\bigoplus_{i=0}^{\dsos}\rho_i.
\end{equation}
In other words, $\rho_i(\alpha)(a,b)$ is entry $\rho(\alpha)(a,b)$ in block $(a+i,b+i)$. We will view each $\rho_i(\alpha)$ as a $(D-i+1)$-by-$(D-i+1)$ matrix. Similarly, we view the rescaling matrix $C$ as a $(D-i+1)\times(D-i+1)$ matrix with entry $a!$ at position $(a,a)$ ($0\leq a\leq D-i$) when we multiply it with matrices in $\im(\rho_i)$.
\end{definition}

We record some more properties of the representation $\rho$. 

\begin{proposition}\label{prop:rhoproperties}
Suppose $\alpha\in\SALD$ has expression $\alpha=\sum_{S(a,b;u)} c_\alpha(a,b;u)S(a,b;u)$.
\begin{enumerate}
    \item (Expression of $\rho_i$)
    \label{item:rhoialpha}
    For any $i\in[0,\dsos]$, if we view $\rho_i(\alpha)$ as a $(D-i+1)\times(D-i+1)$ matrix on $\supp_i$ \eqref{eq:suppi}, then 
    \begin{equation}\label{eq:rhoialpha}
        \rho_i(\alpha)(a,b):=\sum_{u=0}^{\min\{a,b\}}\binom{a}{u}\binom{b}{u}\frac{u!}{\sqrt{a!b!}} c_\alpha(a-u,b-u;u+i).
    \end{equation}

    \item \label{item:consistent}
    We say that $\rho(\alpha)$ satisfies the \emph{leading principal submatrix condition (lpc)} if 
    its top-left element is 1 and for all $i$, $\rho_{i+1}(\alpha)$ is a leading principal submatrix of
    $\rho_i(\alpha)$. 
    Then $\rho(\alpha)$ satisfies lpc if and only if $\alpha$ is consistent (\Cref{def:consistent}).

    \item (Inverse of $\rho$)\label{item:inverse}
    Suppose $\alpha\in\SALD$ is consistent. 
    For each $t\in[0,\dsos]$, if we denote 
    \[v_t(i):=\rho_0(\alpha)(i+t,i),\ \  w_t(i):=c_{\alpha}(i+t,i;0)\] 
    which are vectors in $\R^{D-t+1}$, then $v_t=H_tw_t$ where $H_t$ is $[0,\dsos-t]\times[0,\dsos-t]$ lower-triangular:  
    \begin{equation}\label{eq:Ht}
    H_t(i,j)=\begin{cases*}
    0, & if $i<j$;\\
    \frac{\sqrt{i!(i+t)!}}{(i-j)!j!(j+t)!} & otherwise.
    \end{cases*}
    \end{equation}
    Inversely, $w_t=H_t^{-1}v_t$, where $\abs{H_t^{-1}(i,j)}\leq 2^{i-j}$.
\end{enumerate}
\end{proposition}
\begin{proof}
By definition $\rho(-)=\sqrt{C}^{-1}\rhopre(-)\sqrt{C}$, 
where for every $k_1,k_2\in[0,D]$ and $i$ in an appropriate range, we have: 
\begin{align*}
\rhopre(\alpha)\upharpoonright_{B(k_1,k_2)}(i,i+k_2-k_1){=}&\sum_{u=0}^{\min\{k_1,k_2\}}c_\alpha(k_1-u,k_2-u;u)\cdot\rhopre(k_1,k_2;u)\upharpoonright_{B(k_1,k_2)}(i,i+k_2-k_1)\\
\stackrel{\eqref{eq:rhopre}}{=}&\sum_{u=k_1-i}^{\min\{k_1,k_2\}}c_\alpha(k_1-u,k_2-u;u)\cdot\binom{i}{k_1-u}/(k_2-u)!.
\end{align*} 
So by the definition of $C$ \eqref{eq:rescaleC},
\begin{equation}\label{eq:rhocalc}
\rho(\alpha)\upharpoonright_{B(k_1,k_2)}(i,i+k_2-k_1)=
\sum_{u=k_1-i}^{\min\{k_1,k_2\}}c_\alpha(k_1-u,k_2-u;u)\cdot \binom{i}{k_1-u}\frac{\sqrt{(i+k_2-k_1)!}}{\sqrt{i!}(k_2-u)!}.
\end{equation} 
Now $\rho_i(\alpha)(a,b)$ is by definition the $(a,b)$th entry of block $B(a+i,b+i)$ in $\rho(\alpha)$, so 
\begin{equation}\label{eq:rhoialpha,2}
\begin{aligned}
\rho_i(\alpha)(a,b)\stackrel{\eqref{eq:rhocalc}}{=}&\sum_{u=i}^{\min\{a+i,b+i\}}c_\alpha(a+i-u,b+i-u;u)\cdot \binom{a}{a+i-u}\frac{\sqrt{b!}}{\sqrt{a!}(b+i-u)!}\\
\stackrel{(u':=u-i)}{=}&\sum_{u'=0}^{\min\{a,b\}}c_\alpha(a-u',b-u';u'+i)\binom{a}{u'}\binom{b}{u'}\frac{u'!}{\sqrt{a!b!}}=\eqref{eq:rhoialpha}.
\end{aligned}
\end{equation}
This proves \Cref{item:rhoialpha}.

For \Cref{item:consistent}, assume $\alpha$ is consistent. Then the coefficient of $S(0,0;0)$ in $\alpha$ is 1, so by \eqref{eq:rho}, the $(0,0)$th element in the $(0,0)$th block of $\rho(\alpha)$ is 1. Also, since $c_\alpha(a,b;x)$ depends only on $(a,b)$,  each $\rho_i(\alpha)$ in \eqref{eq:rhoialpha} is a leading principal submatrix of $\rho_0(\alpha)$. This shows that $\rho(\alpha)$ satisfies lpc. 

For the reverse direction, suppose $\rho(\alpha)$ satisfies lpc. Then the top-left entry is 1, so by \eqref{eq:rho}, the coefficient of $S(0,0;0)$ in $\alpha$ has to be 1. Next, for the entries $(a,0)$ and $(0,b)$ where $0\leq a,b\leq D$, by comparing $\rho_0(\alpha)$ and $\rho_i(\alpha)$ on these entries for all possible $i$, we see that for indices in $\{(a,b;x)\mid \text{one of $a$, $b$ is 0}\}$, $c_\alpha(a,b;x)$ depends only on $(a,b)$. This allows us to use induction to show the following property $P_t$: 
\[(P_{t}):\ \text{restricted to index set $\{(a,b;x)\mid\text{$a\leq t$ or $b\leq t$}\}$, $c_\alpha(a,b;x)$ depends only on $(a,b)$.}\] 
We've shown $P_0$ holds. Assuming $P_{t-1}$, using the lpc on $\rho(\alpha)$ we can get that on $(a,b;x)$, where $a=t$ or $b=t$, $c_\alpha(-)$ is independent of $x$. Thus, $P_t$ holds. This means $P_t$ holds for all $t\in[0,D]$, i.e., $c_\alpha(a,b;x)$ depends only on $(a,b)$. 
Together, we get that $\alpha$ is consistent. 

For \Cref{item:inverse}, the identity $v_t=H_tw_t$ follows by expanding \Cref{item:rhoialpha}, where we substitute all $c_{\alpha}(i-u,i+t-u;u)$ by $c_\alpha(i-u,i+t-u;0)$ (we can do so because $\alpha$ is consistent) and use the running variable $u':=i-u$ in the sum. We next estimate the entries in $H_t^{-1}$. There is a general formula for the inverse of a lower-triangular invertible matrix $H$ of dimension $d$:\footnote{To see this, write $H=(\Id+H') H_{\diag}$ where $H'$ is nilpotent, then $H^{-1}=H_{\diag}^{-1}(\Id + H' + \ldots + (H')^{d-1})$.}
\begin{equation}\label{eq:Hinverse}
    H^{-1}(a,b)=\Sum_{k=0}^{d-1}{(-1)^k \over H(a,a)}\cdot \Sum_{\substack{(i_1,\dots,i_{k+1}):\\a=i_1>\ldots>i_{k+1}=b}} \Prod_{l=1}^k\frac{H(i_l,i_{l+1})}{H(i_{l+1},i_{l+1})}
\end{equation}
where we let the empty product (for $k=0$) be 1 if $a=b$ and 0 otherwise. Applying this formula to $H_t$, since $\frac{1}{H_t(a,a)}\Prod_{l=1}^k\frac{H_t(i_l,i_{l+1})}{H_t(i_{l+1},i_{l+1})}=\frac{1}{(a-i_2)!(i_2-i_3)!\ldots(i_k-b)!\sqrt{b!(b+t)!}}\leq {1\over b!}$ by \eqref{eq:Ht}, 
\[\abs{H_t^{-1}(a,b)}\leq {1\over b!}\sum_{k=0}^{D-t+1}\#\big(\text{strictly decreasing sequence $(i_1=a,\dots, i_{k+1}=b)$}\big)= {1\over b!}\binom{a-b-1}{k-1}\leq 2^{a-b}.\]
Here, we let $\binom{n}{-1}$ be $1$ if $n=-1$ and $0$ otherwise.
\end{proof}

\subsection{PSDness of $\SS{(Q_0+L+L^{\top}-2\Id)}$ Under $\rho$}\label{sec:rhoQhatPSD}
In this subsection, we show that $\rho\left(\SS{(Q_0+L+L^{\top}-2\Id)}\right)$ is PSD. To simplify notation, in the rest of this section we denote 
\[l(t):=\E_A[h_t(x)],\quad \forall t\in\N.\]

\begin{definition}[$\hat{Q}$]\label{def:Qhat}
We denote $\SS{(Q_0 + L + L^{\top} - 2\Id)}$ by $\hat{Q}$, i.e., 
\begin{equation}\label{eq:QhatIJ}
\hat{Q}=\Sum_{(k_1,k_2)\in[0,D]\times[0,D]}\Sum_{u=0}^{\min\{k_1,k_2\}} {l(k_1+k_2-2u)\cdot S(k_1-u,k_2-u; u)}.
\end{equation}
\end{definition} 
Recall that in the notation of \Cref{prop:rhoproperties},  
\begin{equation}\label{eq:Qdecomp}
\rho(\hat{Q}) = \rho_0(\hat{Q}) + \dots + \rho_D(\hat{Q}).
\end{equation} 
By \Cref{item:rhoialpha} of \Cref{prop:rhoproperties} applied to $\hat{Q}$, we immediately get the following. 
\begin{proposition}\label{prop:Qhatdecomp}
For each $i=0,\dots,D$, $\rho_i(\hat{Q})$ is a symmetric matrix supported on $\supp_i$ \eqref{eq:suppi}, each $\rho_{i+1}(\hat{Q})$ is a leading principal submatrix of $\rho_i(\hat{Q})$, and 
\begin{equation}
\label{eq:rho0Qhat,abstract}
\rho_0(\hat{Q})(a,b)=\sum_{u=0}^{\min\{a,b\}}\binom{a}{u}\binom{b}{u} \frac{u!}{\sqrt{a!b!}}\cdot l(a+b-2u).
\end{equation}
where we recall that $l(n)=\E\limits_A[h_n(x)]$.
\end{proposition}

In particular, to find a sum-of-squares expression of $\rho(\hat{Q})$, we only need to do so for $\rho_0(\hat{Q})$. 

\begin{lemma}\label{lem:rho0PSD}
$\rho_0(\hat{Q})$ in \eqref{eq:rho0Qhat} is a sum-of-squares---more precisely, an expectation of squares. 
\end{lemma}
\begin{proof}
Recall that $l(t):=\E_A[h_t(x)]$. By \eqref{eq:rho0Qhat},
\begin{equation}\label{eq:rho0Qhatx}
\rho_0(\hat{Q})=\E_{x\sim A}[\rho_0(\hat{Q})(x)]
\end{equation}
where $\rho_0(\hat{Q})(x)$ is the matrix
\begin{equation}\label{eq:rho0x}
\rho_0(\hat{Q})(x)(i,j):=\sum_{u=0}^{\min\{i,j\}}\binom{i}{u}\binom{j}{u} \frac{u!}{\sqrt{i!j!}}\cdot h_{i+j-2u}(x).
\end{equation} 
The product formula of the physicists' Hermite polynomials $\{H_i(-)\}$ says (see e.g. \cite{carlitz1957formula})
\begin{equation}\label{eq:Hermiteproduct}
    H_i(x)H_j(x)=\sum_{u=0}^{\min\{i,j\}}\binom{i}{u}\binom{j}{u}u!2^uH_{i+j-2u}(x),
\end{equation}
and there is a relation $H_n(x)=h_n(\sqrt{2}x)2^{n/2}$ between $H_n$ and $h_n$, which together imply that
\[\rho_0(\hat{Q})(x)(i,j)=\frac{h_i(x)}{\sqrt{i!}}\frac{h_j(x)}{\sqrt{j!}}.\] 
Viewing $\rho_0(\hat{Q})$ as having row and column indices $[0,D]$, we have for any $x\in\R$, 
\begin{equation}
\label{eq:rho0Qhatvv}
\rho_0(\hat{Q})(x)=v_0(x)v_0(x)^{\top}
\end{equation}
where 
\begin{equation}\label{eq:v0x}
    v_0(x)=(\frac{h_0(x)}{\sqrt{0!}},\dots,\frac{h_D(x)}{\sqrt{D!}})^{\top}.
\end{equation} 
So $\rho_0(\hat{Q})=\E_A[v_0v_0^{\top}]$ which is an expectation of a PSD matrix.
\end{proof}

\subsection{Positive-Definiteness of $\SS{Q}$ as an Element in the Simple Spider Algebra}\label{subsec:PDQhat}

In this subsection, we prove the positive-definiteness of $\SS{Q}$ in the algebra $\SALD$. 
More specifically, we show that $\SS{Q}$ consists of only good simple spiders, and it is a sum-of-squares of good simple spiders under the $\SALDprod$-product. This is \Cref{lem:QSSPSD} below. 

To provide motivation, note that $\rho(\hat{Q})$ is PSD and we have the equation $L_{SS} \SALDprod Q_{SS} \SALDprod L^{\top}_{SS} = \hat{Q}$.  Given this, it is natural to attempt to transfer the PSDness of $\hat{Q}$ to $\SS{Q}$ by multiplying $\SS{L}^{-1}$, $(\SS{L}^{-1})^\top$ from both sides. A subtlety is that we need to ensure $\SS{Q}$ contains only good simple spiders for the approximation algebra to work well (cf. \Cref{rmk:approx}), and to control the coefficients in its square root (important for \Cref{sec:psdness-qSSD}). For this reason, we prove the following quantitative result.

\begin{lemma}[Positive-definiteness of $\SS{Q}$ as an element in $\SALD$]\label{lem:QSSPSD} 
Suppose $X\in\SALD$ satisfies: 
\begin{equation}\label{eq:X_defining_equation}
\SS{L}\star X\star\SS{L}^{\top}=\hat{Q}.
\end{equation}
Then:
\begin{enumerate}
\item\label{item:Xbasic} $X$ is consistent and good;
\item\label{item:XPSD} $\rho(X)\succeq (6\cu)^{-2\dsos}\cl^{-\dsos} \Id$, and $X = \alpha_0\star\alpha_0^{\top}$ for some $\alpha_0\in\SALD$ that is left, good and consistent;

\item\label{item:alpha0bounds} For all $0\leq j\leq i\leq\dsos$, we have bounds
\begin{align}
\abs{c_{\alpha_0}(i,j;0)}&\leq 2^{i(i+4)}\cdot {\cu}^{2i^2 + i},
\label{eq:alpha0upperbound}\\
\abs{c_{\alpha_0^{\sinv}}(i,j;0)}&\leq \left(8\cu^2\sqrt{\cl}\right)^{i}. \label{eq:alpha0sinvupperbound}
\end{align}
\end{enumerate}
\end{lemma}

\begin{proof} 
The equation that $X$ satisfies can be rewritten as 
\begin{equation}\label{eq:QSSabstract}
X=\SS{L}^{\sinv}\star\hat{Q}\star\left(\SS{L}^{\sinv}\right)^{\top}\end{equation}
where $(-)^{\sinv}$ means the inverse in algebra $\SALD$. We analyze the RHS of \eqref{eq:QSSabstract} in the following steps, where we recall that $\rho_i(\hat{Q})$ denotes a component of $\rho(\hat{Q})$ as in \Cref{lem:rho0PSD}. 

{\bf 1. Proof of \Cref{item:Xbasic}.} 
Recall that $A$ matches the first $k-1$ moments of $\cN(0,1)$, so 
\begin{equation}
\label{eq:rho0Qhat}
\rho_0(\hat{Q})=
\begin{pmatrix}
\Id & B^{\top}\\
B & F
\end{pmatrix}
\end{equation}
for some real matrices $B$ and $F$. Here, the ``$\Id$'' is supported on $[0,\floor{\frac{k-1}{2}}]\times[0,\floor{\frac{k-1}{2}}]$. As for $\rho(\SS{L})$, it is lower-triangular, its off-diagonal part equals the lower half of $\rho(\hat{Q})$ and moreover, all its diagonal entries are all 1 by the computation \eqref{eq:rhoialpha} where we recall that the only simple spiders in $\SS{L}$ having equally many left and right legs are the trivial shapes $\{S(0,0;u)\mid 0\leq u\leq \dsos\}$. Therefore, 
\begin{equation}\label{eq:rho0LSS}
\rho_0(\SS{L})=
\begin{pmatrix}
\Id &  0\\
B & \Id+F_{\text{low}}
\end{pmatrix}\ \text{where $F_{\text{low}}$ means the strict lower-triangular part of $F$.}
\end{equation}
We then have
\begin{align}
\rho(X)&=\rho(\SS{L})^{-1}\rho(\hat{Q})\left(\rho(\SS{L})^{-1}\right)^{\top}
\label{eq:rhoQSSsimple}\\
&=\sum_{i=0}^\dsos 
\underbrace{\rho_i(\SS{L})^{-1}\cdot \rho_i(\hat{Q})\cdot \left(\rho_i(\SS{L})^{-1}\right)^{\top}}_{=\rho_i(X)}
\label{eq:rhoQSS}
\end{align}
where 
\begin{align}
\rho_0(\SS{L})^{-1}&=
\begin{pmatrix}
\Id &  0\\
-(\Id+F_{\text{low}})^{-1}B & (\Id+F_{\text{low}})^{-1}
\end{pmatrix}, \label{eq:rho0LSSinv}\\
\rho_0(X)&=
\begin{pmatrix}
\Id &  0\\
0 & (\Id+F_{\text{low}})^{-1}\cdot\left(F-BB^{\top}\right)\cdot(\Id+F_{\text{low}}^{\top})^{-1}
\end{pmatrix}.
\label{eq:rho0X}
\end{align}

To see that $X$ is consistent and good: since $\hat{Q}$ and $\SS{L}$ are consistent, by item 2 of \Cref{prop:rhoproperties}, both $\rho(\hat{Q})$ and  $\rho(\SS{L})$ satisfy the leading principal submatrix condition (lpc). Then $\rho(X)$ satisfies it too by the following easy fact applied to \eqref{eq:rhoQSS} twice, so $X$ is consistent.
\begin{fact}\label{fct:consistent}
    If $A$, $B$ are square matrices such that $A$ is lower-triangular or $B$ is upper-triangular, 
    then the leading principal submatrices of $AB$ are products of the leading principal ones of $A$ and $B$.
\end{fact}

By \eqref{eq:rho0X} and consistency, $\rho(X)$ is supported on  the images of good simple spiders, so $X$ is good. 

{\bf 2. Proof of \Cref{item:XPSD}.} 
We start by lower bounding the eigenvalues of the symmetric matrix in \eqref{eq:rho0X}, 
\begin{equation}\label{eq:rho0X,2}
(\Id+F_{\text{low}})^{-1}\cdot\left(F-BB^{\top}\right)\cdot(\Id+F_{\text{low}}^{\top})^{-1}.
\end{equation} 
Here, $(\Id+F_{\text{low}})$ is a $(\dsos+1)$-by-$(\dsos+1)$ submatrix of $\rho(\SS{L})$, so by the definition of $\rho$ \eqref{eq:rhoialpha}, its entries can be upper bounded in absolute value by $(3\cu)^\dsos$. Thus, its Frobenius norm is at most $(3\cu)^\dsos \cdot (\dsos + 1) \leq (6\cu)^\dsos$. We now use the following fact. 
\begin{fact}
For any invertible real matrix $Y$, $\sigma_{\min}(Y) \geq (\norm{Y^{-1}}_{Fr})^{-1}$, where $\norm{\cdot}_{Fr}$ 
is the Frobenius norm and $\sigma_{\min}(Y)$ is the smallest singular value of $Y$.
\end{fact}
\begin{proof}[Proof of Fact]
Letting $\sigma_1,\ldots,\sigma_{n'}$ be the singular values of $Y$, the singular values of $Y^{-1}$ are $\sigma_1^{-1},\ldots,\sigma_{n'}^{-1}$ so $\norm{Y^{-1}}_{Fr}^2 = \sum_{i = 1}^{n'}{\sigma_{i}^{-2}} \geq \sigma_{\min}(Y)^{-2}$ and thus $\sigma_{min}(Y) \geq \frac{1}{\norm{Y^{-1}}_{Fr}}$.
\end{proof}
Using this fact, we get $\sigma_{\min}((\Id+F_{\text{low}})^{-1}) \geq \norm{(\Id+F_{\text{low}})}_{Fr}^{-1} \geq  (6\cu)^{-\dsos}$. As for the $(F-BB^{\top})$ factor in \eqref{eq:rho0X,2}, we bound its smallest singular value by the following simple claim.

\begin{claim}\label{claim:elem}
For any real symmetric matrix 
$Y=\begin{pmatrix}
\Id & U^{\top}\\
U & V
\end{pmatrix}$, $\sigma_{\min}(V-UU^{\top})\geq\sigma_{\min}(Y)$.
\end{claim}
\begin{proof}
Let $Z:=V-UU^{\top}$ and let $z$ be a unit vector such that $\left\vert z^{\top}Zz\right\vert = \sigma_{\min}(Z)$. Taking $y=-U^{\top}z$, $\sigma_{\min}(Z)=\left\vert z^{\top}Zz\right\vert=
\left\vert(y^{\top},z^{\top})Y\begin{pmatrix} y\\z\end{pmatrix}\right\vert\geq\sigma_{\min}(Y)\cdot\norm{\begin{pmatrix} y\\z\end{pmatrix}}\geq\sigma_{\min}(Y)\left\Vert z \right\Vert_2=\sigma_{\min}(Y)$.
\end{proof}
\noindent 
Thus we have:  
\begin{equation}\label{eq:sigmaminrho0X}
\sigma_{\min}\Big((\Id+F_{\text{low}})^{-1}\left(F-BB^{\top}\right)(\Id+F_{\text{low}}^{\top})^{-1}\Big)
\geq {\sigma_{\min}(F-BB^{\top})\over 
{(6\cu)^{2\dsos}}
} 
\stackrel{\text{\Cref{claim:elem}}}{\geq}
\frac{\sigma_{\min}(\rho_0(\hat{Q}))}
{(6\cu)^{2\dsos}}.
\end{equation}
Continuing, by \eqref{eq:rho0Qhatvv} and \eqref{eq:v0x}, $v^{\top}\rho_0(\hat{Q})v = 
\E_A\ [p_v(x)^2]$ where $p_v(x):=\Sum_{i=0}^\dsos v_i{h_i(x)\over \sqrt{i!}}$ has $l_2$-norm $\normt{v}$ under $\cN(0,1)$, so $v^{\top}\rho_0(\hat{Q})v \geq \cl^{-\dsos}\normt{v}^2$ by the definition of $\cl$.   
This implies that 
\begin{equation}\label{eq:sigmaminrho0Qhat}
\sigma_{\min}(\rho_0(\hat{Q})) \geq \cl^{-\dsos}.
\end{equation}
So by \eqref{eq:rho0X}, \eqref{eq:sigmaminrho0X}, the consistency of $X$, and the fact that $\rho_0(X)$ is PSD (this follows from the PSDness of $\rho_0(\hat{Q})$ (\Cref{lem:rho0PSD}), \eqref{eq:rho0Qhat}, and \eqref{eq:rho0X}), we have:
\begin{equation}\label{eq:underrhokey}
\rho(X) \succeq \frac{\Id}{(6\cu)^{2\dsos}\cl^{\dsos}},
\ \text{in particular,}\ \rho(X) = WW^\top\text{ for some $W\in\im(\rho)$.}
\end{equation}
\noindent Applying $\rho^{-1}(-)$ to \eqref{eq:underrhokey}, we get $X = \alpha_0\star\alpha_0^{\top}\ \text{for $\alpha_0:=\rho^{-1}(W)\in \SALD$}$.

To show that we can choose $\alpha_0$ to be left, good and consistent, it suffices to construct a $W$ satisfying \eqref{eq:underrhokey} that is lower-triangular, supported on the image of good simple spiders, and which satisfies the leading principal submatrix condition (lpc) (\Cref{prop:rhoproperties}, \Cref{item:consistent}). Since $X$ is good, $\rho(X)$ is supported on the images of good simple spiders, i.e., the entry positions are in  $\left([0,D]\backslash[0,\floor{\frac{k-1}{2}}]\right)\times\left([0,D]\backslash[0,\floor{\frac{k-1}{2}}]\right)$ and the diagonal, so we can let $W$ be supported on them too. For every PSD matrix, the Cholesky decomposition gives a factorization of the form $ZZ^{\top}$ where $Z$ is lower-triangular so we can also assume that the $\rho_0$ component of $W$ is lower-triangular and we can choose its top-left entry to be $1$ (as opposed to $-1$) in this decomposition. Using \Cref{fct:consistent}, we can then take the other components of $W$ to be the leading principal submatrices of the $\rho_0$ component of $W$ and this will satisfy $WW^{\top} = \rho(X)$. The matrix $W$ now satisfies all the three conditions.

{\bf 3. Proof of \Cref{item:alpha0bounds}, Inequality \eqref{eq:alpha0upperbound}.}
We first bound the magnitudes of entries in $\rho_0(\hat{Q})$, $\rho_0(\SS{L})$, $\rho_0(\SS{L})^{-1}$ by
\begin{align}
\left\vert{\vphantom{\rho_0(\hat{Q})} 
\rho_0(\SS{L}) (i,j)}\right\vert,\ \left\vert\rho_0(\hat{Q}) (i,j)\right\vert 
& \leq 2^{i}\cdot \cu^{i+j},
\label{eq:rho0Qhatentrybound}\\
\abs{\rho_0(\SS{L})^{-1}(i,j)}&\leq 2^{i(i+1)}\cdot 
\cu^{2i^2},
\ \ \forall j\leq i\leq\dsos.
\label{eq:rho0LSSinventrybound}
\end{align}

To see \eqref{eq:rho0Qhatentrybound}, observe that the coefficients in $\hat{Q}$ satisfy $\left\vert c_{\hat{Q}}(a,b;u)\right\vert=\left\vert\E\limits_A[h_{a+b}(x)]\right\vert\leq \cu^{a+b}$, so by \eqref{eq:rhoialpha}, $\abs{\rho_0(\hat{Q}) (i,j)} \leq 
\sum_{u=0}^{j}{\binom{i}{u}\cu^{i+j-2u}} \leq 2^{i}\cu^{i+j}$. The same is true for $\abs{\rho_0(\SS{L}) (i,j)}$.

To see \eqref{eq:rho0LSSinventrybound}, observe that since $\rho_0(\SS{L})$ is lower-triangular, all of its diagonal entries are $1$ \eqref{eq:rho0LSS}, and $\abs{\rho_0(\SS{L})(i,j)}\leq 2^i\cu^{i+j}$, applying \eqref{eq:Hinverse} to $\rho_0(\SS{L})$ gives the bound
\[
\abs{\rho_0(\SS{L})^{-1}(i,j)} \leq 2^{i-j}\left(2^i\cu^{2i}\right)^{i-j} \leq 2^{i(i+1)}\cu^{2i^2}.
\]

Next, we prove \eqref{eq:alpha0upperbound} using \eqref{eq:rho0Qhatentrybound} and \eqref{eq:rho0LSSinventrybound}. Since $\rho_0(X) = \rho_0(\SS{L})^{-1} \cdot \rho_0(\hat{Q}) \cdot (\rho_0(\SS{L})^{-1})^\top$, it holds that 
\begin{equation}\label{eq:rho0Xijbound}
\abs{\rho_0(X)(i,j)}
\stackrel{\eqref{eq:rho0Qhatentrybound}\text{ and }\eqref{eq:rho0LSSinventrybound}}{\leq}
\sum_{\substack{k_1,k_2:\ k_1\leq i,\ k_2\leq j}} {2^{i(i+1)+i+j(j+1)}{\cu}^{2i^2 + 2i + 2j^2}}
\leq 
\underbrace{
2^{2i^2+5i}{\cu}^{4i^2+2i}
}_{:=r(i,j)}.
\end{equation}
In other words, on the support of $\im(\rho_0)$, $\abs{(WW^\top)(i,j)} \leq r(i,j)$ for all $i, j$. This implies that $\norm{(WW^\top)\upharpoonright_{[0,i]\times[0,i]}} \leq (i+1)r(i,i)$ so $\norm{W\upharpoonright_{[0,i]\times[0,i]}} \leq \sqrt{(i+1)r(i,i)} \leq 2^{i(i+3)}{\cu}^{2i^2 + i}$ by the lower-triangularity of $W$. It follows that entry-wise, $\abs{W(i,j)}\leq 2^{i(i+3)}{\cu}^{2i^2 + i}$. We now use the matrix $W$ to analyze $\alpha_0$: for each $t\leq \dsos$ we apply the matrix $H_t^{-1}$ in \Cref{item:inverse} of \Cref{prop:rhoproperties} to the vector $v_t(i):= W(i+t,i)$, $i\in [0, D-t]$, where again we focus only on the part of $W$ supported on $\im(\rho_0)$. As a result,    
\begin{equation}\label{eq:alpha0upperbound,2}
\abs{c_{\alpha_0}(i+t,i;0)} \leq 
\sum_{j=0}^i 2^{i-j} 2^{(i+t)(i+t+3)}\cdot {\cu}^{2(i+t)^2+(i+t)} \leq 2^{(i+t)(i+t+4)}\cdot {\cu}^{2(i+t)^2+(i+t)}.
\end{equation}
This proves the inequality \eqref{eq:alpha0upperbound}. 

{\bf 4. Proof of \Cref{item:alpha0bounds}, Inequality \eqref{eq:alpha0sinvupperbound}.} 
For $0\leq {t'}\leq\dsos$, denote by $W_{t'}$ the part of $W$ on the support of $\im(\rho_{D-t'})$. Then $W_{t'}=W_0\upharpoonright_{[0,t']\times[0,t']}$ since $W$ satisfies lpc. Also, $W_{t'}W_{t'}^\top=\rho_0(X)\upharpoonright_{[0,{\dsos-t'}]\times[0,{\dsos-t'}]}$ since $\rho_0(X)=W_0W_0^\top$ and $W_0$ is lower-triangular. Using this and \eqref{eq:rhoQSS}, we get:
\begin{equation}\label{eq:Wt',1}
    W_{t'} W_{t'}^\top = \left( \rho_0(\SS{L})^{-1} \cdot 
 \rho_0(\hat{Q}) \cdot (\rho_0(\SS{L})^{-1})^\top \right) \upharpoonright_{[0,{t'}]\times[0,{t'}]}
\end{equation}
Since $\rho_0(\SS{L})^{-1}$ is lower-triangular and $(\rho_0(\SS{L})^{-1})^\top$ is upper-triangular, the operation of taking leading principal submatrices on the RHS of \eqref{eq:Wt',1} can be applied to each term in the product. This gives: 
\begin{equation}\label{eq:Wt',2}
    W_{t'} W_{t'}^\top = \rho_{\dsos-{t'}}(\SS{L})^{-1} \cdot \rho_{\dsos-{t'}}(\hat{Q})\cdot (\rho_{\dsos-{t'}}(\SS{L})^{-1})^\top,
\end{equation}
where we have used that $\rho(\SS{L})^{-1}$ and $\rho(\hat{Q})$ both satisfy lpc. Taking the matrix inverse, we get:
\begin{equation}\label{eq:Wt'inv}
    (W_{t'}^\top)^{-1} (W_{t'}^{-1}) = \rho_{\dsos-{t'}}(\SS{L})^\top \cdot \rho_{\dsos-{t'}}(\hat{Q}) \cdot \rho_{\dsos-{t'}}(\SS{L}).
\end{equation}
We now upper bound the norms of $\rho_{\dsos-t'}(\SS{L})$ and $\rho_{\dsos-{t'}}(\hat{Q})$. They are the $[0,t']$-by-$[0,t']$ submatrices of $\rho_0(\SS{L})$ and $\rho_0(\hat{Q})$, respectively. The intuition for the upper bounds is to let the SoS degree be $t'$ and then use the corresponding version of \eqref{eq:rho0Qhatentrybound}, \eqref{eq:sigmaminrho0Qhat}. Formally, since $\abs{\rho_0(\SS{L}) (i,j)} \leq 2^{i}\cu^{i+j}$ for all $0 \leq j \leq i \leq \dsos$ by \eqref{eq:rho0Qhatentrybound}, we have $\Norm{\rho_{\dsos-t'}(\SS{L})}\leq (t'+1)2^{t'}\cu^{2t'}\leq (2\cu)^{2t'}$. 
As for $\rho_{\dsos-{t'}}(\hat{Q})$, by setting $D\leftarrow t'$ in the derivation of \eqref{eq:sigmaminrho0Qhat}, we get $\sigma_{\min}(\rho(\hat{Q}\upharpoonright_{[0,t']\times[0,t']}))\geq \cl^{-t'}$ and so $\Norm{\rho_{\dsos-{t'}}(\hat{Q})}\leq \cl^{t'}$. 

Combining these bounds and \eqref{eq:Wt'inv}, we get $\norm{W_0^{-1}\upharpoonright_{[0,t']\times[0,t']}} \leq (4\cu^{2}\sqrt{\cl})^{t'}$, which implies that for all $0 \leq j \leq i \leq \dsos$, $\abs{W_0^{-1}(i,j)} \leq (4\cu^{2}\sqrt{\cl})^{i}$. 

By \Cref{item:inverse} of \Cref{prop:rhoproperties}, if we take $w_t$ to be the vector with coordinates $w_t(i) = {c_{\alpha_0^{\sinv}}(i+t,i;0)}$, then $w_t = H_t^{-1}v_t$ where $v_t$ is the vector with coordinates $v_t(i) = W^{-1}(i+t,i)$, $H_t^{-1}$ is lower triangular, and $\abs{H_t^{-1}(i,j)} \leq 2^{i-j}$. Thus,
\[
\abs{c_{\alpha_0^{\sinv}}(i+t,i;0)} = \abs{w_t(i)} = \abs{\sum_{j=0}^{i}{H_t^{-1}(i,j)W^{-1}(j+t,j)}} \leq \sum_{j=0}^{i}{2^{i-j}(4\cu^{2}\sqrt{\cl})^{i+t}} \leq (8\cu^{2}\sqrt{\cl})^{i+t}.
\]
The inequality \eqref{eq:alpha0sinvupperbound} follows. 
\end{proof}
        \section{PSDness of $[Q]_{\wellbehaved}$ via Parallel Multiplication}
\label{sec:psdness-qSSD} 

In this section, we show that $\SSD{Q}$---recall that it denotes $[Q]_{\wellbehaved}$---is positive-definite. The main result of this section is the following.

\begin{restatable}[PSDness of $\SSD{Q}$]{lemma}{PDQSSD}\label{lem:PDQSSD}
Assume $\truncation\geq\max\{100\dsos, 20\dsos^2, 2k\log n\}$ and 
\begin{equation}\label{eq:condition_QSSD}
(4\cu)^{4\dsos^2} \cl^{2\dsos} \truncation^{32\Cuniv\dsos} < n^{\frac{\eps}{30}}.
\end{equation}
Then 
\begin{equation}\label{eq:QSSD_final}
[Q]_{\wellbehaved} \succeq n^{-\frac{\eps}{30}} \Id.
\end{equation}
\end{restatable}

The intuition is that roughly speaking, the multiplication of two good simple spider disjoint unions (SSD) amounts to component-wise simple spider multiplications. A formal statement and proof require caution and we need the following definition. 

Recall that by \Cref{def:disj}, an SSD $\alpha$ can be viewed as a multi-set of simple spiders, where different ways of splitting the vertices in $U_{\alpha} \cap V_{\alpha}$ result in the same SSD. 

\begin{definition}[$\dsos$-combination]\label{def:Dcombination}
For a consistent simple spider linear combination $\alpha$, we define its {\bf $\dsos$-combination}, denoted as $[\alpha]^{\dsos}$, to be a linear combination of SSDs with the following coefficients: if an SSD $\mset{S_1,\dots,S_t}$ has both left and right index size no more than ${\dsos}$, then the coefficient of this scaled SSD shape is $\Prod_{i=1}^t c_\alpha(S_i)$, otherwise the coefficient is 0. Here, $c_\alpha(\cdot)$ means the coefficient of the scaled shapes in $\alpha$.
\end{definition}
It is clear from the definition that if $\alpha\in\SALD$ is consistent, then $[\alpha^{\top}]^{\dsos}=\left([\alpha]^{\dsos}\right)^{\top}$.

Our plan is as follows. We show that $\SSD{Q} = [\SS{Q}]^{\dsos}$, and that if $\SS{Q}\approx AA^{\top}$ then $[\SSD{Q}]^{\dsos}\approx [A]^{\dsos}\left([A]^{\dsos}\right)^{\top}$. We then show that $[A]^{\dsos}\left([A]^{\dsos}\right)^{\top}$ is sufficiently positive-definite using \Cref{lem:gammanorm}. \Cref{lem:PDQSSD} then follows.

We now establish some properties of the $\dsos$-combination operator. We will then use them to show that $\SSD{Q} =  [\SS{Q}]^{\dsos}$ and is positive-definite.

{\bf Notation.} 
We denote the left- and right-index size of a simple spider $S(a,b;u)$ by $l(S(a,b;u)):=a+u$, $r(S(a,b;u)):=b+u$, respectively, and the left- and right-index size of an SSD $x=\mset{S_1,\dots,S_t}$ by $l(x):=\sum_{i=1}^t l(S_i)$, $r(x):=\sum_{i=1}^t r(S_i)$, respectively.

We extend the $|\cdot|_\infty$ norm on $\SALD$ to SSD linear combinations in the natural way. 
Namely, if $Y=\sum_{i}c_i S_i$ is a formal linear combination of scaled SSD shapes, we let $|Y|_\infty:=\max_i{|c_i|}$.

\begin{observation}\label{obs:leftright}
For any $\alpha\in\SALD$, $\alpha$ is left (or right) iff $\rho(\alpha)$ is lower- (or upper-) triangular.
\end{observation}
\begin{proof}
The definition of ${\rho}$ \eqref{def:rho} implies that $\alpha$ is left if and only if $\rho(\alpha)$ is block-lower-triangular. Since all matrices in $\im(\rho)$ are diagonal on the diagonal blocks, this holds if and only if $\rho(\alpha)$ is lower-triangular. The case for $\alpha$ being right is similar.
\end{proof}

\begin{proposition}\label{prop:consistent}
Suppose $\alpha,\beta\in\SALD$ are consistent and either $\alpha$ is left or $\beta$ is right. Then $\alpha\star\beta$ is consistent.
\end{proposition}

\begin{proof}
We assume $\alpha$ is left; the other case is similar. 
One way to see the proposition is by using the representation $\rho$. Namely, \Cref{item:consistent} of \Cref{prop:rhoproperties} says $\alpha$ is consistent if and only if $\{\rho_i(\alpha)\mid i=0,\dots,{\dsos}\}$ satisfies the leading principal submatrix condition (lpc). Recall $\rho(\alpha)\rho(\beta)=\Sum_{i=0}^{\dsos} \rho_i(\alpha) \rho_i(\beta)$. If in addition $\alpha$ is left, then every $\rho_i(\alpha)$ is lower-triangular, so by \Cref{fct:consistent} we have that $\{\rho_i(\alpha)\rho_i(\beta)\}$ satisfies lpc too. 
\end{proof}

The next lemma shows that the $[\cdot]^{\dsos}$ operator interacts nicely with $\wbp$ and $\star$.

\begin{lemma}[Commutative diagram]\label{lem:disj}
For all consistent $\alpha,\beta \in \SALD$ where either $\alpha$ is left or $\beta$ is right,
\begin{equation}\label{eq:disjD}
[\alpha]^{\dsos} \wbp [\beta]^{\dsos} = [\alpha \star \beta]^{\dsos}.
\end{equation}
\end{lemma}
\begin{proof}
Let $\gamma:=\alpha\star\beta$. By \Cref{prop:consistent}, $\gamma$ is also consistent so $[\gamma]^{\dsos}$ is defined. We will use the following claim:
\begin{claim}\label{claim:z}
Suppose $z=\mset{S_1,\dots,S_t}$ is a simple spider disjoint union shape where $l(z),r(z)\leq {\dsos}$. Then the coefficient of the scaled graph matrix $z$ in $[\alpha]^{\dsos} \wbp [\beta]^{\dsos}$ is equal to the product of $S_i$'s coefficient in $\alpha\star\beta$ (over $i=1,\dots,t$). 
\end{claim}
If we have this claim, then we can prove \Cref{lem:disj} as follows. The product of any two graph matrices can be expressed as a sum of graph matrices, i.e., the coefficients of every ribbon realizing the same shape is the same, so we only need to analyze one ribbon for each shape. For scaled SSD shapes, \Cref{claim:z} says their coefficients in $[\gamma]^{\dsos}$ are the same as in $[\alpha]^{\dsos}\star[\beta]^{\dsos}$. \Cref{lem:disj} follows.
\end{proof}
\begin{proof}[Proof of \Cref{claim:z}]
Fix a scaled ribbon $R$ that realizes $z$. Then $R$ has $t$ disjoint components. The well-behaved product configurations for $[\alpha]^{\dsos} \wbp [\beta]^{\dsos}$ that result in $R$ can be characterized by all sets $A=\{(R_0,R'_0), \ (R_1,R_1'),\dots,\ (R_t,R_t')\}$ that satisfy the following: 
\begin{enumerate}
    \item $R_0 = R'_{0}$ is the ribbon such that $V(R_0) = U_{R_0} = V_{R_0} = U_{R} \cap V_{R}$ and $E(R_0) = \emptyset$.
    \item For all $i \in [t]$, $(R_i,R'_i)$ is a composable pair of scaled simple spider ribbons such that 
    \begin{enumerate}
        \item Either $R_i$ and $R'_i$ have the same circle vertex or exactly one of $R_i$ and $R'_i$ is trivial. 
        \item $U_{R_i} \cap V_{R_i} \cap U_{R_i} \cap V_{R_i} = \emptyset$.
        \item If $\alpha$ is left, then $|V_{R_i}| = |U_{R'_i}| \leq |U_{R_i}|$. If $\beta$ is right, then $|V_{R_i}| = |U_{R'_i}| \leq |V_{R'_i}|$.
    \end{enumerate}
    \item There is no ``cross-pair'' vertex intersections, i.e., $\big(V(R_i)\cup V(R_i')\big)\cap\big(V(R_j)\cup V(R_j')\big)=\emptyset$ if $i\neq j$. 
\end{enumerate}
Observe that the contribution of each pair $(R_i, R'_i)$ to the $\wbp$-product is equal to its contribution to the $\star$-product. In particular, the difference between $\star$ and $\wbp$ mentioned in \Cref{rmk:approx} will not cause any issues. This holds by the definitions of $\star$ and $\wbp$ when both $R_i$ and $R'_i$ have circle vertices. In the special case where $R_i$ or $R'_i$ has a trivial shape, it follows from the fact that all trivial shapes have coefficient 1 in $\alpha,\beta$, due to their consistency. 

If we count all sets of the form $\{(R_1,R_1'),\dots,(R_t,R_t')\}$ satisfying the above conditions with coefficient $\prod_{i=1}^t c_\alpha(R_i)c_\beta(R_i')$, then we get the coefficient of $R$ in $[\gamma]^{\dsos}$. This follows by simply expanding the definition of $[\gamma]^{\dsos}$.
\end{proof}

\begin{lemma}\label{cor:QSSD} 
Assuming $\truncation\geq 20\dsos^2$, $Q_{SSD} = [Q_{SS}]^{D}$.
\end{lemma}

\begin{proof}
Since $L_{SSD} = [L_{SS}]^{D}$, $L^{\top}_{SSD} = [L^{\top}_{SS}]^{D}$, and $L_{SS} \SALDprod Q_{SS} \SALDprod L^{\top}_{SS} = (L + Q_0 + L^{\top} - 2\Id)_{SS}$ (\Cref{lem:LQLSSbehavior}), by applying Lemma \ref{lem:disj} we have:
\begin{equation}\label{eq:QSS^D_satisfies_equation}
L_{SSD} \wbp [Q_{SS}]^{D} \wbp L^{\top}_{SSD} = [(L + Q_0 + L^{\top} - 2\Id)_{SS}]^{D}.
\end{equation}
From the definition of $L,Q_0$ and $M$, $[(L + Q_0 + L^{\top} - 2\Id)_{SS}]^{D}=\SSD{M}$. Putting these pieces together, we have that $L_{SSD} \wbp [Q_{SS}]^{D} \wbp L^{\top}_{SSD}=\SSD{M}$.

Now consider the following linear equation in $X$ in the $\R$-algebra of simple spider disjoint unions with multiplication $\wbp$:
\begin{equation}\label{eq:eq_for_QSSD}
    L_{SSD} \wbp X \wbp L^{\top}_{SSD} = \SSD{M}.
\end{equation} 
By the paragraph above, $[\SS{Q}]^\dsos$ satisfies this equation. By \Cref{lem:Q_wellbehaved_characterization}, $\SSD{Q}$ also satisfies it when $\truncation \geq 20\dsos^2$. However, the kernel of this linear form is zero. To see this, assume $Y\neq 0$ and consider the smallest nonzero SSD shape $\alpha$ in $Y$ (under any total order on SSD shapes that respects the number of edges). The only term in $L_{SSD} \wbp Y \wbp L^{\top}_{SSD}$ that results in $\alpha$ is $\Id \wbp \alpha \wbp \Id$, so the coefficient of $\alpha$ in the result is the same as in $Y$, and $Y$ cannot be in the kernel. Thus, equation \eqref{eq:eq_for_QSSD} uniquely determines $X$ so $\SSD{Q}=[\SS{Q}]^\dsos$. 
\end{proof}

Recall that by Lemma \ref{lem:QSSPSD}, $Q_{SS} = \alpha_0\star\alpha_0^{\top}$ for some $\alpha_0\in\SALD$ that is left, good and consistent. By Lemma \ref{lem:disj} and Lemma \ref{cor:QSSD}, we have that 
\begin{equation}\label{eq:sqrtQSSD}
Q_{SSD} = [Q_{SS}]^{D} = [\alpha_0]^{\dsos} \wbp ([\alpha_0]^{\dsos})^{\top}.
\end{equation}
Using this, we can write 
\begin{equation}\label{eq:QSSD_and_alpha0}
\begin{aligned}
Q_{SSD} &= [\alpha_0]^{\dsos} \cdot ([\alpha_0]^{\dsos})^{\top} + \left([\alpha_0]^{\dsos} \wbp ([\alpha_0]^{\dsos})^{\top} - [\alpha_0]^{\dsos} \cdot ([\alpha_0]^{\dsos})^{\top}\right).
\end{aligned}
\end{equation}
To prove that the expression \eqref{eq:QSSD_and_alpha0} is positive-definite, we show: 
\begin{enumerate}
    \item \label{item:QSSD_and_alpha0,1} 
    $[\alpha_0]^{\dsos}$ is not too close to being singular; and 
    \item \label{item:QSSD_and_alpha0,2}
    $\norm{([\alpha_0]^{\dsos} \wbp ([\alpha_0]^{\dsos})^{\top} - [\alpha_0]^{\dsos} \cdot ([\alpha_0]^{\dsos})^{\top}}$ is small.
\end{enumerate}
We prove item \ref{item:QSSD_and_alpha0,2} in Section \Cref{sec:error_analysis} (\Cref{lem:alphazeroerror}). Here, we prove item \ref{item:QSSD_and_alpha0,1} by the following lemma where we recall that $c_\alpha(i,j;0)$ denotes the coefficient of $S(i,j;0)$ in $\alpha\in\SALD$.

\begin{lemma}[Preserving well-conditionedness]\label{lem:gammanorm}
Suppose $\gamma\in\SALD$ satisfies two conditions: 
\begin{enumerate}[label=(\roman*)]
    \item\label{item:lgc}$\gamma$ is left, good, and consistent. 
    \item\label{item:sigmamingamma} $\sigma_{\min}\left(\rho(\gamma)\right)>0$.
\end{enumerate} 
Then $\gamma^{\sinv}$ exists, and it is left, good, and consistent. 
If additionally there are positive numbers 
\begin{equation}\label{eq:Cup}
A_0,\ldots,A_{\dsos}, \ 
B_0,\ldots,B_{\dsos}
\end{equation}
such that $A_0=B_0=1$, $A_i\geq \max\limits_{j\leq i}\big\{\abs{c_\gamma(i,j;0)}\}$, $B_i\geq \max\limits_{j\leq i}\big\{\abs{c_{\gamma^{\sinv}}(i,j;0)}\}$, and 
\begin{equation}\label{eq:ABcondition}
A_i A_j\leq A_{i+j},\ \ 
B_i B_j\leq B_{i+j}\ \ \text{for all $i,j$ where $i+j\leq \dsos$,}
\end{equation} 
then assuming $\truncation\geq\max\{100\dsos,\ 2k\log n\}$, 
the following hold for $[\gamma]^\dsos$, $[\gamma^{\sinv}]^\dsos$ as real matrices:
\begin{enumerate}
    \item\label{item:Did} 
    $[\gamma]^\dsos\cdot[\gamma^{\sinv}]^\dsos=\Id + E$ where $\norm{E} \leq A_{\dsos}B_{\dsos} \truncation^{20 \Cuniv \dsos} n^{-\frac{\eps}{12}}$. 
    \item\label{item:Dnorm} 
    $\sigma_{\min}([\gamma]^\dsos) \geq \left(1- A_{\dsos}B_{\dsos} \truncation^{20 \Cuniv \dsos} n^{-\frac{\eps}{12}} \right) B_\dsos^{-1} \truncation^{-32 \Cuniv \dsos}$.
\end{enumerate}
\end{lemma}

\begin{proof}
We denote the direct-sum decomposition of $\gamma$ by $\gamma=\Sum_{i=0}^\dsos\gamma_i$, that is, $\gamma_i:=\rho^{-1}\left(\rho_i(\gamma)\right)$. 
    
Note that condition \ref{item:sigmamingamma} implies $\rho_0(\gamma)$ is invertible. Since $\gamma$ is consistent, each $\rho_i(\gamma)$ is a leading principal submatrix of $\rho_0(\gamma)$ and so is also invertible, hence $\rho(\gamma)^{-1}$ exists, and so does $\gamma^{\sinv}$ by applying $\rho^{-1}$. The lower-triangularity of $\rho(\gamma)^{-1}$ holds since it is the inverse of a lower-triangular matrix. Consequently, $\gamma^{\sinv}$ is left. Moreover, $\gamma^{\sinv}$ is consistent and good by a similar inspection of $\rho(\gamma)$ as in the proof for $\alpha_0$ in \Cref{lem:QSSPSD}, \Cref{item:XPSD}. 
    
Below, we prove items \ref{item:Did} and \ref{item:Dnorm} of \Cref{lem:gammanorm} in three steps.

{\bf 1. We upper bound the coefficient magnitudes in $[\gamma]^\dsos$ and $[\gamma^{\sinv}]^\dsos$.}
    
\noindent For a simple spider $S(a,b;u)$, define its heavy-side weight to be $\hw(S):=\max\{a+u,b+u\}$. For an SSD shape $\alpha$ in $[\gamma^{\sinv}]^\dsos$, the number $l$ of its simple spider components is at most $\dsos$, and the sum of the heavy-side weights of its components $\sum_{i=1}^l\hw_i$ is at most $\dsos$. So by the monotonicity of these numbers, we have $|c_{[\gamma]^\dsos}(\alpha)| \leq \prod_{i=1}^l A_{\hw_i}$. By property \eqref{eq:ABcondition} of $\{A_i\}$, the RHS as a function on $(\hw_1,\ldots,\hw_l)$ subject to $\sum_{i=1}^l \hw_i\leq\dsos$ is maximized at $(0,0,\ldots,\dsos)$, thus 
\begin{equation}\label{eq:gammaDcombbound}
\linfty{[\gamma]^\dsos} \leq A_{\dsos}.
\end{equation}
The same argument works for $c_{[\gamma^{\sinv}]^\dsos}$ using numbers $\{B_i\}$, giving the bound
\begin{equation}\label{eq:gammasinvDcombbound}
\linfty{[\gamma^{\sinv}]^\dsos}\leq B_{\dsos}.
\end{equation}
    
{\bf 2. Proof of \Cref{item:Did}.} 
Since $[\gamma]^\dsos$ and $[\gamma^{\sinv}]^\dsos$ are linear combinations of disjoint unions of good simple spiders, we have 
\begin{equation}
\begin{aligned}
\norm{[\gamma]^\dsos \cdot [\gamma^{\sinv}]^\dsos - \Id}
&\stackrel{\Cref{lem:disj}}{=}\norm{[\gamma]^\dsos \cdot [\gamma^{\sinv}]^\dsos - [\gamma]^\dsos \wbp [\gamma^{\sinv}]^\dsos}\\
&\stackrel{\Cref{lem:alphazeroerror}}{\leq}\linfty{[\gamma]^\dsos}\cdot \linfty{[\gamma^{\sinv}]^{\dsos}}\cdot 
\truncation^{20 \Cuniv \dsos} n^{-\frac{\eps}{12}}.
\end{aligned}
\end{equation}
Plugging \eqref{eq:gammaDcombbound} and \eqref{eq:gammasinvDcombbound} into the above, we get $\norm{[\gamma]^\dsos \cdot [\gamma^{\sinv}]^\dsos - \Id} \leq A_D B_D  \truncation^{20 \Cuniv \dsos} n^{-\frac{\eps}{12}}$. \Cref{item:Did} follows.  

{\bf 3. Proof of \Cref{item:Dnorm}.} We first show a claim:
\begin{claim}\label{claim:numberofSSD}
There are at most $5^{2\dsos}$ many non-equivalent SSDs having left and right index size $\leq \dsos$. 
\end{claim}
\begin{proof}
If $\dsos \leq 3$, it is not hard to check this statement directly. If $\dsos \geq 4$, We can bound the number of possible SSDs as follows. Each SSD can be represented by a set of nontrivial simple spiders of the form $S(i,j;0)$ where $i+j > 0$ together with a trivial simple spider $S(0,0;u)$ for some $u \geq 0$. We first specify the number of non-trivial simple spiders in the SSD. There are at most $D+1$ choices for this as there are at most $D$ non-trivial simple spiders in the SSD. For each non-trivial simple spider, we go through the square vertices one by one and specify the following data.
\begin{enumerate}
    \item Is the square vertex a left leg or a right leg?
    \item Is the square vertex the last square vertex of the current simple spider?
\end{enumerate}
After specifying all of the non-trivial simple spiders, if there are fewer than $2\dsos$ square vertices then we specify whether there is another square vertex or not. If so, we increase $u$ by $1$ and then repeat this process if we still have fewer than $2\dsos$ square vertices. If not, we stop.

Since there are at most $2\dsos$ square vertices and there are at most $4$ possibilities for each square vertex, the total number of possibilities is at most $(\dsos+1)4^{2\dsos} \leq 5^{2\dsos}$.
\end{proof}
    
We now upper bound $\norm{[\gamma^{\sinv}]^\dsos}$. Since each good SSD in $[\gamma^{\sinv}]^\dsos$ has total size at most $6\dsos$, the norm of each (scaled) good SSD $\alpha$ is at most $((2\dsos)^2\cdot 2k\log n)^{6\Cuniv\dsos}$ by \Cref{thm:norm_control} and the fact that $\lambda_\alpha\Anorm(\alpha)=1$. Combining this, \Cref{claim:numberofSSD}, and the coefficient bound \eqref{eq:gammasinvDcombbound}, we have:
\begin{equation}\label{eq:gammasinvDcombnorm}
\norm{[\gamma^{\sinv}]^\dsos}\leq B_\dsos 5^{2\dsos} (8k\dsos^2\log n)^{6\Cuniv \dsos} \leq B_D\truncation^{32\Cuniv\dsos}
\end{equation}
where the last step used $\Cuniv\geq 1$ and $\truncation\geq\{2k\log n,\ 100\dsos\}$. 
    
Finally, recall that $[\gamma]^{\dsos}=(\Id+E)\left([\gamma^{\sinv}]^{\dsos}\right)^{-1}$ by \Cref{item:Did}, so by taking $\sigma_{\min}([\gamma]^\dsos)$ and applying the basic properties of $\sigma_{\min}$ that $\sigma_{\min}(X^{-1})=1/\norm{X}$, $\sigma_{\min}(XY)\geq\sigma_{\min}(X)\sigma_{\min}(Y)$ and $\sigma_{\min}(X+Y)\geq \sigma_{\min}(X)-\norm{Y}$, we have:
\begin{align}
\sigma_{\min}\left([\gamma]^{\dsos}\right)&\geq (1-\norm{E})\cdot\sigma_{\min}\left(\left([\gamma^{\sinv}]^{\dsos}\right)^{-1}\right) \nonumber\\
&=(1-\norm{E})\cdot\norm{[\gamma^{\sinv}]^\dsos}^{-1}.\label{eq:sigmamingammaDcomb}
\end{align}
By plugging into \eqref{eq:sigmamingammaDcomb} the bounds on $\norm{E}$, $\norm{[\gamma^{\sinv}]^\dsos}$ from \Cref{item:Did} and \eqref{eq:gammasinvDcombnorm}, we get 
\[
\sigma([\gamma]^{\dsos})\geq \left(1- A_{\dsos}B_{\dsos} \truncation^{20 \Cuniv \dsos} n^{-\frac{\eps}{12}} \right) B_\dsos^{-1} \truncation^{-32 \Cuniv \dsos}. \qedhere
\]
\end{proof}

We now prove \Cref{lem:PDQSSD} using \Cref{lem:gammanorm}. 

\PDQSSD*

\begin{proof}
Recall that $[Q]_{\wellbehaved} = [Q_{SS}]^{D} = [\alpha_0]^{\dsos} \wbp ([\alpha_0]^{\dsos})^{\top}$, where $\alpha_0\in\SALD$ is as in \eqref{eq:sqrtQSSD}. 
We want to apply \Cref{lem:gammanorm} with $\gamma \leftarrow \alpha_0$ and 
\[
A_i\leftarrow 2^{i(i+4)}\cdot \cu^{2i^2+2i},\ \ 
B_i\leftarrow 8^i\cdot \cu^{2i}\cl^{i/2}.
\]
For this, we note that $A_0=B_0=1$, condition \eqref{eq:ABcondition} holds, and $A_i\geq \max\limits_{j\leq i}\big\{\abs{c_\gamma(i,j;0)}\}$, $B_i\geq \max\limits_{j\leq i}\big\{\abs{c_{\gamma^{\sinv}}(i,j;0)}\}$ by \eqref{eq:alpha0upperbound} and \eqref{eq:alpha0sinvupperbound}. Thus, \Cref{lem:gammanorm} applies, whose \Cref{item:Dnorm} yields:
\begin{align}
[\alpha_0]^{\dsos} \cdot ([\alpha_0]^{\dsos})^{\top}
&\succ \left(1-(4\cu)^{4\dsos^2} \cl^{\dsos} \truncation^{20 \Cuniv \dsos} n^{-\frac{\eps}{12}}\right) (3\cu\cl)^{-2\dsos}\truncation^{-32\Cuniv\dsos}.\nonumber
\end{align}
We can simplify this bound using the assumption $(4\cu)^{4\dsos^2}\cl^\dsos\truncation^{20\Cuniv\dsos}<n^{\frac{\eps}{30}}$ \eqref{eq:condition_QSSD}. As a result, 
\begin{equation}\label{eq:alpha0_PD}
    [\alpha_0]^{\dsos} \cdot ([\alpha_0]^{\dsos})^{\top} \succ 
    0.99\cdot (3\cu\cl)^{-2\dsos} \truncation^{-32\Cuniv\dsos}.
\end{equation}
On the other hand, $[\alpha_0]^\dsos$ is good and $\linfty{[\alpha_0]^\dsos}\leq 2^{\dsos^2+4\dsos}\cu^{\dsos^2+2\dsos}$  (\Cref{lem:QSSPSD}), so by \Cref{lem:alphazeroerror},
\begin{align}
\Norm{[\alpha_0]^\dsos\cdot ([\alpha_0]^\dsos)^\top - [\alpha_0]^\dsos\wbp([\alpha_0]^\dsos)} 
&\leq (2\cu)^{2\dsos^2+4\dsos}\truncation^{20 \Cuniv \dsos} n^{-\frac{\eps}{12}}\nonumber\\
&\leq n^{\frac{\eps}{30} - \frac{\eps}{12}}
=n^{-\frac{\eps}{20}} \label{eq:alpha0D_wbp_approx}
\end{align}
where the second inequality used assumption \eqref{eq:condition_QSSD}. 

By \eqref{eq:alpha0_PD} and \eqref{eq:alpha0D_wbp_approx}, $[Q]_{\wellbehaved} = [\alpha_0]^\dsos\wbp([\alpha_0]^\dsos) \succ \left( 0.9(3\cu\cl)^{-2\dsos} \truncation^{-32\Cuniv\dsos} -  n^{-\frac{\eps}{20}}\right)\Id$. Again, we can simplify this using $n^{-\frac{\eps}{20}}< 0.4\cdot (3\cu\cl)^{-2\dsos} \truncation^{-32\Cuniv\dsos}$ by \eqref{eq:condition_QSSD}. As a result, 
\[
[Q]_{\wellbehaved}\succ \frac{1}{2}(3\cu\cl)^{-2\dsos} \truncation^{-32\Cuniv\dsos} \Id \succeq (5\cu\cl)^{-2\dsos} \truncation^{-32\Cuniv\dsos} \Id.
\]
Finally, given \eqref{eq:condition_QSSD}, we have $(5\cu\cl)^{-2\dsos} \truncation^{-32\Cuniv\dsos} \geq n^{-\frac{\eps}{30}}$, so \Cref{lem:PDQSSD} follows.
\end{proof}

	\section{Error Analysis}\label{sec:error_analysis}
In this section, we upper bound the terms that we have omitted as ``errors'' in the matrix multiplications so far. This analysis does not depend on the results in \Cref{sec:psdness-qSS} or \Cref{sec:psdness-qSSD}. The main conclusions of this section are that $Q$ is well-approximated by $\SSD{Q}$ up to an additive $n^{-\Omega(\eps)}$ error (\Cref{thm:erroranalysis}) and that the other error terms are small (\Cref{lem:Mcan}, \Cref{lem:LQLT_truncation}, and \Cref{lem:alphazeroerror}). We prove these statements using \Cref{thm:erroranalysis}, which is the main technical result of this section. 

We will need to show in various situations that the product of two or more scaled shapes has a small norm. Here, the key idea is to keep track of the exponent over $n$ from the product of the coefficients and from the norm bound in \Cref{thm:norm_control}. The full analysis is involved (see \Cref{sec:analyzemiddle} and \Cref{sec:intersection}). To carry out this analysis, we first set up some definitions and properties.

\subsection{Properties of Square Separators}\label{sec:sepprop}
We recall some familiar properties of vertex separators and vertex disjoint paths. Unless specified otherwise, all shapes have square-vertex indices and edges go between a circle and a square.

\begin{lemma}[Menger's theorem for bipartite graphs]\label{lem:menger}
For any shape $\alpha$, the maximum number of square-disjoint paths is equal to the size of the minimum square separator between $U_{\alpha}$ and $V_{\alpha}$. 
\end{lemma}
\begin{proof}
We will apply Menger's theorem to a shape $\beta$ which is created by replacing each circle vertex $x$ in $\alpha$ with a clique of \emph{new} vertices $x_1,...,x_{|V(\alpha)|+1}$ and replacing each edge to $x$ with $|V(\alpha)|+1$ edges, one for each $x_i$. 
The indices of $\beta$ are $U_\alpha$ and $V_\alpha$.
We observe the following two facts.
\begin{enumerate}
    \item If $P$ is a set of square-disjoint paths from $U_\alpha$ to $V_\alpha$ in $\alpha$ with maximum size and $Q$ is a set of vertex disjoint paths from $U_\alpha$ to $V_\alpha$ in $\beta$ with maximum size then $|P|=|Q|$. To see this, note that $|P|\leq |Q|$ since from each path $p\in P$, we can get a path in $\beta$ by following $p$ while using a different new vertex each time when needed; this is possible because $|P|\leq|V(\alpha)|$ and any path uses a vertex no more than once. The resulting paths are vertex-disjoint. Conversely, a vertex disjoint path set in $\beta$ is automatically a square-disjoint path set in $\alpha$ by collapsing the cliques to circle vertices, so $|Q|\leq |P|$. Thus, $|P|=|Q|$.
    \item If $S$ is a minimum square separator of $\alpha$ and $T$ is a minimum vertex separator for $\beta$ then $|S|=|T|$. To see this, note that $S$ is also a vertex separator for $\beta$, so $|T| \leq |S| \leq |V(\alpha)|$. Conversely, we claim that $T$ contains no \emph{new} vertex. To see this, assume $T$ contains some new vertex $z_i$. Since the clique of $z$ has size greater than $|V(\alpha)|$ and thus greater than $|T|$, there is some $z_j \notin T$ so $T\backslash\{z_i\}$ is already a vertex separator for $\beta$ since we can replace $z_i$ with $z_j$ on any path connecting $U_\alpha,V_\alpha$. This contradicts the minimality of $T$. Thus, the claim holds. In particular, $T$ is identified with a square-vertex set in $\alpha$. It follows that $T$ is a square separator in $\alpha$, so $|T|\geq |S|$. Putting everything together, $|S|=|T|$.
\end{enumerate}

Now Menger's theorem says that $|Q|=|T|$ in $\beta$ (where we view all vertices as one sort). 
Thus, $|P|=|Q|=|T|=|S|$ and the proposition follows.
\end{proof}
\begin{definition}
Given a shape $\alpha$ and vertex separators $S,T$, we say that $S$ is to the left of $T$ if $S$ separates $U_\alpha$ and $T$ and we say that $S$ is to the right of $T$ if $S$ separates $V_\alpha$ and $T$.
\end{definition}
\begin{lemma}[Existence of leftmost and rightmost minimum vertex separators]\label{lem:leftmostseparator}
For all shapes $\alpha$, there exist unique leftmost and rightmost minimum weight separators of $\alpha$. Similarly, for all shapes $\alpha$, there exist unique leftmost and rightmost minimum square separators of $\alpha$.
\end{lemma}
\begin{proof}
We prove this statement for minimum weight separators; the proof for minimum square separators is similar. 
Given minimum weight separators $S$ and $T$, let $L$ be the set of vertices $v$ such that $v \in S \cup T$ and there is a path from $U_{\alpha}$ to $v$ which does not pass through any vertex in $(S \cup T) \setminus \{v\}$. Similarly, let $R$ be the set of vertices $v$ such that $v \in S \cup T$ and there is a path from $U_{\alpha}$ to $v$ which does not pass through any vertex in $(S \cup T) \setminus \{v\}$. Note that both $L$ and $R$ are vertex separators of $\alpha$.

We now make the following observations:
\begin{enumerate}
    \item $w(L) + w(R) - w(L \cap R) = w(L \cup R) \leq w(S \cup T) = w(S) + w(T) - w(S \cap T)$.
    \item $L \cap R \subseteq S \cap T$. To see this, observe that if $v \in L \cap R$ then there is a walk from $U_{\alpha}$ to $V_{\alpha}$ which passes through $v$ and does not pass through any other vertices of $S \cup T$. Since $S$ and $T$ are vertex separators between $U_{\alpha}$ and $V_{\alpha}$, we must have that $v \in S \cap T$.
    \item Since $S$ and $T$ have minimum weight, $w(L) \geq w(S) = w(T)$ and $w(R) \geq w(S) = w(T)$.
\end{enumerate}
Together, these observations imply that $L \cap R = S \cap T$ and $w(L) = w(R) = w(S) = w(T)$.
\end{proof}
\begin{corollary}[Expandability of vertex sets reachable by square-disjoint paths]\label{cor:findingpaths}
Suppose $\alpha$ is a shape with indices $(U_\alpha,V_\alpha)$. If $X,Y \subseteq V(\alpha)$ are square vertex sets such that $X \subseteq Y$ and 
\begin{enumerate}
    \item There exist $|X|$ square-disjoint paths from $U_{\alpha}$ to $X$;
    \item There exist $|U_{\alpha}|$ square-disjoint paths from $U_{\alpha}$ to $Y$.
\end{enumerate}
Then there is a subset $X' \subseteq Y$ such that $X \subseteq X'$, $|X'| = |U_{\alpha}|$ and there are $|U_{\alpha}|$ square-disjoint paths from $U_{\alpha}$ to $X'$.
\end{corollary}
\begin{proof}

We build up $X'$ by starting from $X$ and adding one vertex at a time:
\begin{lemma}
For all $j \leq |U_{\alpha}| - |X|$, there is a set of vertices $X'_j$ such that $X \subseteq X'_j \subseteq Y$, $|X'_j| = |X| + j$ and there are $|X'_j| = |X| + j$ square-disjoint paths from $U_{\alpha}$ to $X'_j$.
\end{lemma}
\begin{proof}
We can prove this by induction. For the base case $j = 0$, we take $X'_0 = X$. For the inductive step,  assume that $j \in [0,|U_{\alpha}| - |X| - 1]$ and we have a set of vertices $X'_j$ which satisfies the inductive assumption. Let $S$ be the leftmost minimum square vertex separator between $U_{\alpha}$ and $X'_j$.

Observe that for each vertex $v \in Y \setminus X'_j$, one of the following two cases must hold:
\begin{enumerate}
    \item $S$ is a square vertex separator between $U_{\alpha}$ and $X'_j \cup \{v\}$.
    \item There are $|X'_j| + 1$ square-disjoint paths between $U_{\alpha}$ and $X'_j \cup \{v\}$.
\end{enumerate}
To see this, assume that there are only $|X'_j|$ vertex disjoint paths between $U_{\alpha}$ and $X'_j \cup \{v\}$. If so, then by \Cref{lem:menger} there is a minimum square vertex separator $S'$ between $U_{\alpha}$ and $X'_j \cup \{v\}$ of size $|X'_j|$. Since $S'$ is also a square vertex separator between $U_{\alpha}$ and $X'_j$ and $S$ is the leftmost minimum square vertex separator between $U_{\alpha}$ and $X'_j$, $S$ must be a square vertex separator between $U_{\alpha}$ and $S'$ which implies that $S$ is a square vertex separator between $U_{\alpha}$ and $X'_j \cup \{v\}$.

We now observe that the first case cannot hold for all vertices $v \in Y \setminus X'_j$ as otherwise $S$ would be a square vertex separator between $U_{\alpha}$ and $Y$ but this is impossible as $|S| \leq |X'_j| < |U_{\alpha}|$. Thus, the second case must hold for some vertex $v \in X \setminus X'_j$ so we can take $X'_{j+1} = X'_j \cup \{v\}$.
\end{proof}
Corollary \ref{cor:findingpaths} follows from the lemma by taking $X' = X'_{|U_{\alpha}| - |X|}$.
\end{proof}
\begin{corollary}\label{cor:gammapaths}
Suppose $\gamma$ is a left shape and $V \subseteq V(\gamma)$ is a vertex set such that $U_{\gamma}$ is the leftmost minimum square vertex separator between $U_{\gamma}$ and $V$. Then we can find square-disjoint paths $P_1,\ldots,P_{|U_{\gamma}|}$ such that:
\begin{enumerate}
    \item For all $j \in [|V_{\gamma}|]$, $P_j$ starts at a square vertex in $U_{\gamma}$ and ends at a square vertex in $V_{\gamma}$. No other vertices in $P_j$ are in $U_{\gamma} \cup V_{\gamma}$.
    \item For all $j \in [|V_{\gamma}|+1,|U_{\gamma}|]$, $P_j$ starts at a square vertex in $U_{\gamma}$ and ends at a vertex $v_j \in V$. Moreover, $v_j$ is the only vertex on $P_j$ which is in $V$ and no vertices in $P_j$ except the starting vertex are in $U_{\gamma}$.
\end{enumerate}
\end{corollary}
\begin{proof}
Let $Y$ be the set of square vertices which are either in $V$ or are adjacent to a vertex in $V$. We apply Corollary \ref{cor:findingpaths} with $X = V_{\gamma}$ and $Y$.

For each of the square-disjoint paths $P$ given by Corollary \ref{cor:findingpaths}, we take the segment of $P$ starting from the last vertex in $U$. For the paths ending at a vertex in $V_{\gamma}$, we take the entire path segment. For the other paths, we stop at the next square vertex $v$ which is in $V$ or adjacent to a vertex in $V$. If $v \notin V$ then we extend this path by an edge from $v$ to a vertex in $V$.
\end{proof}

\subsection{Strategy for Analyzing Error Terms}
Recall that each shape $\alpha$ has a scaling coefficient $\lambda_{\alpha} = {n^{-{w(E(\alpha))\over2}}}$ (\Cref{def:coeff}). Our strategy for analyzing error terms is to show that for each term $\tau'$ (which could either be a proper middle shape $\tau$ or a term $\resultP$ which results from an intersection configuration $\mathcal{P}$), we can redistribute the factors in $\lambda_{\tau'}$ times its approximate norm bound (see \Cref{def:Anorm} below) so that the factors are split between the vertices of $\tau'$ and the factor for each vertex $v \in V(\tau')$ is at most $n^{-\slack(v)}$ for some carefully chosen slack function $\slack(v)$. We will choose $\slack(v)$ so that it is always non-negative and whenever $\tau'$ is an error term, there is at least one vertex $v$ such that $\slack(v)$ is $-\Omega(\eps)$.

For our analysis, we need a few definitions.
\begin{definition}[$A$-norm]\label{def:Anorm}
For each shape $\alpha$, we define its \emph{approximate norm upper bound} to be $\Anorm(\alpha):=n^{\frac{w(V(\alpha))+w(I_s(\alpha))-w(S_{\alpha})}{2}}$, where $I_s(\alpha)$ is the set of isolated vertices in $\alpha$ and $S_{\alpha}$ is a minimum weight separator of $\alpha$.
\end{definition}
In order to redistribute the factors in $\lambda_{\tau'}\Anorm(\tau')$, it is very useful to consider whether the vertices in $\tau'$ are connected to $U_{\tau'}$ or $V_{\tau'}$. 
\begin{definition}[Left-connected and right-connected vertices]\label{def:connected-one-shape}
Given a (possibly improper) shape $\tau'$ and a minimum weight vertex separator $S$ for $\tau'$, we say that a vertex $v \in V(\tau')$ is left-connected if there is a path $P$ from $U_{\tau'}$ to $v$ which does not contain any vertices in $S \setminus \{v\}$. Similarly, we say that a vertex $v \in V(\tau')$ is right-connected if there is a path $P$ from $v$ to $V_{\tau'}$ which does not contain any vertices in $S \setminus \{v\}$.
\end{definition}
In our analysis, we will show that for each shape $\alpha$, we can automatically assign some slack to the square vertices of $\alpha$ which are not in $U_{\alpha} \cup V_{\alpha}$ or have a larger degree than expected. 
\begin{definition}[Default slack of a shape]\label{def:slack-one-shape}
Given a shape $\alpha$, we define the default slack of a square vertex $v$ to be:
\begin{enumerate}
    \item If $v \in U_{\alpha} \cap V_{\alpha}$ then $\defaultslack_{\alpha}(v) := \frac{\deg(v)}{4}$.
    \item If $v \in U_{\alpha} \setminus V_{\alpha}$ or $v \in V_{\alpha} \setminus U_{\alpha}$ then $\defaultslack_{\alpha}(v) := \frac{\deg(v) - 1}{4}$. 
    \item If $v \in V(\alpha) \setminus (U_{\alpha} \cup V_{\alpha})$ then $\defaultslack_{\alpha}(v) := \frac{(\deg(v) - 2)}{4} + \frac{{\eps}\deg(v)}{8}$.
\end{enumerate}
We define $\defaultslack(\alpha) = \sum_{v \in V_{\square}(\alpha)}{\defaultslack_{\alpha}(v)}$. 
We define the default slack for circle vertices to be $0$.
\end{definition}
\subsection{Analyzing Middle Shapes}\label{sec:analyzemiddle}
We first show how to use this strategy to analyze proper middle shapes $\tau$. 

\begin{definition}[Extra slack of proper middle shapes]
Given a proper middle shape $\tau$ and a minimum weight vertex separator $S$ for $\tau$, we define the extra slack of a circle vertex $v \in V(\tau)$ with respect to $S$ to be $\extraslack_{\tau,S}(v) = \frac{k\eps}{8}$ if $v \notin S$ and $0$ if $v \in S$. For square vertices $v \in V(\alpha)$, we take $\extraslack_{\tau,S}(v) = \frac{\eps}{4}$ if $v \in S \setminus (U_{\tau} \cap V_{\tau})$ and $\extraslack_{\tau,S}(v) = 0$ otherwise. We define $\extraslack(\tau,S)$ to be 
\[
\extraslack(\tau,S) = \sum_{v \in V(\tau)}{\extraslack_{\tau,S}(v)} = \frac{k\eps}{8}|V_{\bigcirc}(\alpha) \setminus S| + \frac{\eps}{4}|(V_{\square}(\tau) \cap S) \setminus (U_{\tau} \cap V_{\tau})|.
\]
We define $\extraslack(\tau)$ to be the maximum of $\extraslack(\tau,S)$ over all minimum weight vertex separators $S$ of $\tau$.
\end{definition}
\begin{definition}[Slack of proper middle shapes]
Given a proper middle shape $\tau$ and a minimum weight vertex separator $S$ for $\tau$, we define $\slack_{\tau,S}(v) = \defaultslack(v) + \extraslack_{\tau,S}(v)$.

We define the total slack for $\tau$ to be $\slack(\tau) = \defaultslack(\tau) + \extraslack(\tau)$. Equivalently, $\slack(\tau)$ is the maximum of $\sum_{v \in V(\tau)}{slack_{\tau,S}(v)}$ over all minimum weight vertex separators $S$ of $\tau$.
\end{definition}

\begin{lemma}[Error analysis of proper middle shapes]
\label{lem:error_analysis_mid}
For any proper middle shape $\tau$ in which all circle vertices have degree at least $k$ and all square vertices have total degree at least 2, $\lambda_{\tau}\Anorm(\tau) \leq n^{-\slack(\tau)}$.
\end{lemma}
\begin{proof}
Let $S$ be a minimum weight vertex separator of $\tau$ which maximizes $\extraslack(\tau,S)$. If there are edges with a label $l > 1$, for the purposes of this argument we view each such edge $e$ as $l$ copies of $e$.

We can interpret $\lambda_{\tau}$ as giving a factor of $n^{-\frac{1}{2}}$ to each edge. Similarly, we can interpret $\Anorm(\tau)$ as giving a factor of $n^{\target(v)}$ to each vertex $v$ where $\target(v)$ is defined below. Our goal is to redistribute the factor for each edge to its two endpoints so that the total factor given to each vertex $v$ is at most $n^{-\slack_{\tau,S}(v)}$.
\begin{definition}[Target factors for a proper middle shape]\label{def:targetfactors}
Given a square vertex $v \in V(\tau)$, we define $\target(v) = \frac{1}{2}$ if $v \notin S$ and $\target(v) = 0$ if $v \in S$. Similarly, given a circle vertex $w \in V(\tau)$, we define $\target(w) = \frac{1}{2}\log_n(m) \leq \frac{(1-\eps)k}{4}$ if $w \notin S$ and $\target(w) = 0$ if $w \in S$.
\end{definition}

Since $\tau$ is a proper middle shape, there are $|U_{\tau}|$ square-vertex-disjoint paths (possibly of length $0$) from $U_{\tau}$ to $V_{\tau}$ and each vertex in $U_{\tau} \cup V_{\tau}$ is an endpoint for exactly one of these paths. For each edge $e \in E(\tau)$, we distribute the $n^{-\frac{1}{2}}$ factor on $e$ as follows. 

If $e = \{u,v\}$ is on one of the paths from $U_{\tau}$ to $V_{\tau}$ and goes from $u$ to $v$ on this path,
\begin{enumerate}
\item If $u$ and $v$ are left-connected and $u \notin S$, we have $e$ point from $v$ to $u$. Note that this is the opposite direction as the path containing $e$.
\item If $u$ and $v$ are right-connected and $v \notin S$, we have $e$ point from $u$ to $v$. Note that these two cases cannot happen simultaneously as this would imply that there is a path from $U_{\alpha}$ to $V_{\alpha}$ which does not go through any vertices of $S$.
\item If neither of the above cases holds then we do not assign a direction to $e$.
\end{enumerate}
If $e$ is not on one of the paths from $U_{\tau}$ to $V_{\tau}$ then we do not assign a direction to $e$.
{\color{black}
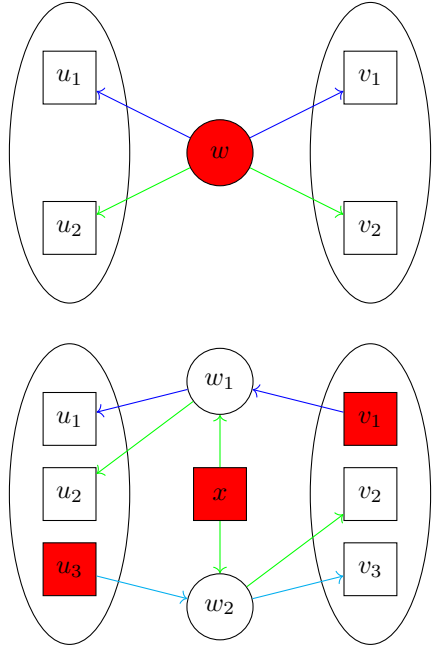
\begin{figure}[!hbt]
    \centering

    \begin{tikzpicture}[
      mycircle/.style={
         circle,
         draw=black,
         fill=white,
         fill opacity = 1,
         text opacity=1,
         inner sep=0pt,
         minimum size=25pt,
         font=\small},
      mysquare/.style={
         rectangle,
         draw=black,
         fill=white,
         fill opacity = 1,
         text opacity=1,
         inner sep=0pt,
         minimum height=20pt, 
         minimum width=20pt,
         font=\small},
      myarrow/.style={-Stealth},
      node distance=0.6cm and 1.2cm
      ]
      \draw (-2,0) ellipse (.8cm and 2cm);
      \draw (2,0) ellipse (.8cm and 2cm);
      \node[mysquare]  at (-2, 1) (u1) {$u_1$};
      \node[mysquare]  at (-2, -1) (u2) {$u_2$};
      \node[mycircle, fill = red]  at (0, 0) (w) {$w$};
      \node[mysquare]  at (2, 1) (v1) {$v_1$};
      \node[mysquare]  at (2, -1) (v2) {$v_2$};
      \draw[->, color = blue] (w) -- (u1);
      \draw[->, color = green] (w) -- (u2);
      \draw[->, color = blue] (w) -- (v1);
      \draw[->, color = green] (w) -- (v2);
    \end{tikzpicture}

\vspace{14pt}

    \begin{tikzpicture}[
      mycircle/.style={
         circle,
         draw=black,
         fill=white,
         fill opacity = 1,
         text opacity=1,
         inner sep=0pt,
         minimum size=25pt,
         font=\small},
      mysquare/.style={
         rectangle,
         draw=black,
         fill=white,
         fill opacity = 1,
         text opacity=1,
         inner sep=0pt,
         minimum height=20pt, 
         minimum width=20pt,
         font=\small},
      myarrow/.style={-Stealth},
      node distance=0.6cm and 1.2cm
      ]
      \draw (-2,0) ellipse (.8cm and 2cm);
      \draw (2,0) ellipse (.8cm and 2cm);
      \node[mysquare]  at (-2, 1) (u1) {$u_1$};
      \node[mysquare]  at (-2, 0) (u2) {$u_2$};
      \node[mysquare, fill = red]  at (-2, -1) (u3) {$u_3$};
      \node[mycircle]  at (0, 1.5) (w1) {$w_1$};
      \node[mysquare, fill=red]  at (0, 0) (x) {$x$};
      \node[mycircle]  at (0, -1.5) (w2) {$w_2$};
      \node[mysquare, fill = red]  at (2, 1) (v1) {$v_1$};
      \node[mysquare]  at (2, 0) (v2) {$v_2$};
      \node[mysquare]  at (2, -1) (v3) {$v_3$};
      \draw[->, color = blue] (w1) -- (u1);
      \draw[->, color = green] (w1) -- (u2);
      \draw[->, color = cyan] (u3) -- (w2);
      \draw[->, color = green] (x) -- (w1);
      \draw[->, color = green] (x) -- (w2);
      \draw[->, color = blue] (v1) -- (w1);
      \draw[->, color = green] (w2) -- (v2);
      \draw[->, color = cyan] (w2) -- (v3);
      \end{tikzpicture}
    \caption{Edge factor assignments in middle shapes. The top shape is a dominant shape, while the bottom one is not. Vertices in red form a minimum weight separator.}
    \label{fig:middle_assignment}
\end{figure}
}

Roughly speaking, we assign the $n^{-\frac{1}{2}}$ factor for each edge to the endpoint which it points to if it has a direction and we split the factor evenly between the endpoints if it does not have a direction. However, we modify this idea slightly in order to ensure each vertex has the desired slack.
\begin{definition}\label{def:edgefactorredistribution}
For each edge, we redistribute its factor of $n^{-\frac{1}{2}}$ between its endpoints as follows.
\begin{enumerate}
    \item For each edge pointing from a circle vertex to a square vertex, we assign a factor of $n^{-\frac{1}{2}}$ to the square vertex and a factor of $1$ to the circle vertex. 
    \item For each edge pointing from a square vertex to a circle vertex, we assign a factor of $n^{-\frac{\eps}{4}}$ to the square vertex and a factor of $n^{-\frac{1}{2} + \frac{\eps}{4}}$ to the circle vertex. 
    \item For edges which don't point in either direction, we assign a factor of $n^{-\frac{1}{4} - \frac{\eps}{8}}$ to the square endpoint of $e$ and a factor of $n^{-\frac{1}{4} + \frac{\eps}{8}}$ to the circle endpoint.
\end{enumerate}
For each vertex $v \in V(\tau)$, we define $\assigned(v)$ so that the product of the edge factors assigned to $v$ is $n^{-\assigned(v)}$.
\end{definition}
\begin{lemma}\label{lem:propoermiddleshapeslackanalysis}
For all vertices $v \in V(\tau)$, $\assigned(v) - \target(v) \geq \slack_{\tau,S}(v)$.
\end{lemma}
\begin{proof}
The key observation is that whenever a path passes through a vertex $v$ which is not in $S$, either one of the two edges on the path incident to $v$ points to $v$ or neither edge is given a direction.
\begin{proposition}\label{prop:edgedirections}
If $v \notin S$ and $\{u,v\}$ and $\{v,w\}$ are two edges on one of the paths from $U_{\tau}$ to $V_{\tau}$ then one of the following cases holds:
\begin{enumerate}
    \item $\{u,v\}$ points from $u$ to $v$.
    \item $\{v,w\}$ points from $w$ to $v$.
    \item Neither $\{u,v\}$ nor $\{v,w\}$ is given a direction.
\end{enumerate}
\end{proposition}
\begin{proof}
If $v$ is left-connected, then $w$ is left-connected as well so the edge $\{v,w\}$ points from $w$ to $v$. Similarly, if $v$ is right-connected, then $u$ is right-connected so the edge $\{u,v\}$ points from $u$ to $v$. If $v$ is neither left-connected nor right-connected, then neither $\{u,v\}$ nor $\{v,w\}$ is given a direction.
\end{proof}
We now make the following observations:
\begin{enumerate}
    \item If $v$ is a circle vertex and $v \notin S$, $\assigned(v) \geq \deg(v)(\frac{1}{4} - \frac{\eps}{8})$. To see this, observe that each edge which points away from $v$ contributes $0$ to $\assigned(v)$, each edge which points towards $v$ contributes $\frac{1}{2} + \frac{\eps}{4}$ to $\assigned(v)$, and each edge incident to $v$ which does not have a direction contributes $\frac{1}{4} + \frac{\eps}{8}$ to $\assigned(v)$. By \Cref{prop:edgedirections}, for each edge which points away from $v$, there must be another edge which points towards $v$. Thus, on average, each edge incident to $v$ contributes at least $\frac{1}{4} + \frac{\eps}{8}$ to $\assigned(v)$ which implies that $\assigned(v) \geq \deg(v)(\frac{1}{4} - \frac{\eps}{8})$. Since $\target(v) = \frac{(1-\eps)k}{4}$ and $\deg(v) \geq k$, $\assigned(v) - \target(v) \geq\slack_{\tau,S}(v) = \frac{k\eps}{8}$.
    \item If $v$ is a circle vertex and $v \in S$ then $\assigned(v) \geq 0$, $\target(v) = 0$, and $\slack_{\tau,P}(v) = 0$ so $\assigned(v) - \target(v) \geq\slack_{\tau,S}(v)$.
    \item If $v$ is a square vertex and $v \notin S$ then $\target(v) = \frac{1}{2}$. We observe that if there is a path $P$ passing through $v$, the two edges incident to $v$ contribute $\frac{1}{2} + \frac{\eps}{4}$ to $\assigned(v)$. If $v$ is an endpoint of this path, the edge incident to $v$ will point to $v$ and thus contribute $\frac{1}{2}$ to $\assigned(v)$. Each other edge incident to $v$ contributes $\frac{1}{4} + \frac{\eps}{8}$ to $\assigned(v)$. We have the following cases:
    \begin{enumerate}
        \item If $v \in (U_{\tau} \setminus V_{\tau}) \cup (V_{\tau} \setminus U_{\tau})$ then $\assigned(v) \geq \frac{1}{2} + (\frac{1}{4} + \frac{\eps}{8})(\deg(v) - 1)$ so 
        \[
        \assigned(v) - \target(v) \geq (\frac{1}{4} + \frac{\eps}{8})(\deg(v) - 1) \geq \frac{\deg(v) - 1}{4}
        \]
        \item If $v \notin U_{\tau} \cup V_{\tau}$ then $\assigned(v) \geq \deg(v)(\frac{1}{4} + \frac{\eps}{8})$ so 
        \[
        \assigned(v) - \target(v) \geq \frac{\deg(v) - 2}{4} + \frac{{\eps}\deg(v)}{8}.
        \]
    \end{enumerate}
    \item If $v \in S$ then $\target(v) = 0$. If there is a path passing through $v$, the two edges of this path which are incident to $v$ contribute at least $\frac{\eps}{4} + \frac{\eps}{4}$ to $\assigned(v)$. If $v$ is an endpoint of a path then the edge of this path incident to $v$ contributes at least $\frac{\eps}{4}$ to $\assigned(v)$. Other edges incident to $v$ contribute $\frac{1}{4} + \frac{\eps}{8}$ to $\assigned(v)$. We have the following cases:
    \begin{enumerate}
        \item If $v \in U_{\tau} \cap V_{\tau}$ then $\assigned(v) \geq (\frac{1}{4} + \frac{\eps}{8})\deg(v)$ so $\assigned(v) - \target(v) \geq \frac{\deg(v)}{4}$.
        \item If $v \in (U_{\tau} \setminus V_{\tau}) \cup (V_{\tau} \setminus U_{\tau})$, then $\assigned(v) \geq \frac{\eps}{4} + (\frac{1}{4} + \frac{\eps}{8})(\deg(v) - 1)$ so 
        \[
        \assigned(v) - \target(v) \geq \frac{\eps}{4} + (\frac{1}{4} + \frac{\eps}{8})(\deg(v) - 1) \geq \frac{\deg(v) - 1}{4} + \frac{\eps}{4}.
        \]
        \item If $v \notin U_{\tau} \cup V_{\tau}$, then $\assigned(v) \geq \frac{\eps}{2} + (\deg(v) - 2)(\frac{1}{4} + \frac{\eps}{8})$ so 
        \[
        \assigned(v) - \target(v) \geq \frac{\deg(v) - 2}{4} + \frac{{\eps}\deg(v)}{8} + \frac{\eps}{4}.
        \]
    \end{enumerate}
\end{enumerate}
\Cref{lem:propoermiddleshapeslackanalysis} is proved.
\end{proof}
Putting everything together,
\[
\lambda_{\tau}\Anorm(\tau) = \prod_{v \in V(\tau)}{n^{\target(v) - \assigned(v)}} \leq n^{-\sum_{v \in V(\tau)}{\slack_{\tau,S}(v)}} = n^{-\slack(\tau)}.
\]
\Cref{lem:error_analysis_mid} is proved.
\end{proof}
\subsection{Analyzing Intersection Configurations}\label{sec:intersection}
In this subsection, we study the quantitative properties of intersection configurations. As described in \Cref{sec:productconfiganalysis}, our techniques can also be used to analyze product configurations $\mathcal{P} \in \mathcal{P}_{\alpha_1,\alpha_2}$ (see \Cref{def:SSD-product-config}) where both $\alpha_1$ and $\alpha_2$ are good. 

We start by recasting the definition of intersections (\Cref{def:intersconfig}), with a slightly expanded scope. We introduce some additional terminology (such as vertex equivalence relations) which will be convenient for our later discussions. One modification from the original \Cref{def:intersconfig} is that we will assign the label for each edge specified by $\mathcal{P}$ post-linearization to be the minimal possible value. This specification will be adequate for the purpose of upper bounding the norm of the resulting term.

\begin{definition}[Intersection configuration]\label{def:configuration}
An {\it intersection configuration} $\mathcal{P}$ is a sequence of shapes $(\gamma_{j},\ldots,\gamma_1,\tau,{\gamma'}^{\top}_1,\ldots,{\gamma'}^{\top}_j)$ where $j\geq 1$, together with a vertex {\it equivalence relation} $\equiv$ such that:
\begin{enumerate}
    \item\label{Condition:config-degree]} 
    In every shape in the sequence, each square vertex has total degree at least $2$, and each circle vertex has degree at least $k$.
    \item\label{Condition:config-middle}
    $\tau$ is a proper middle shape in $Q_0$. 
    \item\label{Condition:config-left-right} 
    For each $i\in[1,j]$, $\gamma_i$ is a left shape in $L$ and ${\gamma'}^{\top}_i$ is a right shape in $L^{\top}$. Moreover, for each $i \in [j]$, at least one of $\gamma_i$ and ${\gamma'}^{\top}_i$ is non-trivial.
    \item\label{Condition:config-intersection}
    $U_{\gamma_{i}}=V_{\gamma_{i+1}}$ and $V_{{\gamma'}^{\top}_{i}}=U_{{\gamma'}^{\top}_{i+1}}$ for all $i\in[0,j-1]$ where we take $\gamma_0:={\gamma'}^{\top}_0:=\tau$. There are no other vertex identifications between the shapes.
    \item\label{Condition:config-equiv}
    The equivalence relation $\equiv$ on $\bigcup\limits_{i=1}^j \left(V(\gamma_i)\cup V({\gamma_i'}^{\top})\right)\cup V(\tau)$ satisfies the following conditions. 

    \begin{enumerate}[leftmargin=0pt]
        \item If $v \equiv w$ then either $v$ and $w$ are both circle vertices or $v$ and $w$ are both square vertices.
        \item For each $\alpha \in \mathcal{P}$, the restriction of $\equiv$ to $\alpha$ is trivial.
        \item $\equiv$ satisfies the following \emph{intersection-separation condition}. For each $i \in [j]$, if we take $Int(\gamma_i)$ to be the set of vertices in $V(\gamma_i) \setminus V_{\gamma_i}$ which are equivalent to a vertex in $\left(\cup_{i' \in [i-1]}{V(\gamma_{i'})}\right) \cup V(\tau) \cup \left(\cup_{i'' \in [i]}{V({\gamma_{i''}}^{\top})}\right)$ then $U_{\gamma_i}$ is a minimum square separator between $U_{\gamma_i}$ and $Int(\gamma_i) \cup V_{\gamma_i}$. Similarly, if we take $Int({\gamma'_{i}}^{\top})$ to be the set of vertices in $V({\gamma'_{i}}^{\top}) \setminus U_{{\gamma'_{i}}^{\top}}$ which are equivalent to a vertex in $\left(\cup_{i' \in [i]}{V(\gamma_{i})}\right) \cup V(\tau) \cup \left(\cup_{i'' \in [i-1]}{V({\gamma'_{i''}}^{\top})}\right)$ then $V_{{\gamma'_{i}}^{\top}}$ is a minimum square separator between $U_{{\gamma'_{i}}^{\top}} \cup Int({\gamma'_{i}}^{\top})$ and $V_{{\gamma'_{i}}^{\top}}$.
    \end{enumerate} 
\end{enumerate}
We call $j$ the length of $\mathcal{P}$. We call $\gamma_j,\dots,\tau,\dots,\gamma'^{\top}_j$ the shapes in $\mathcal{P}$, which are distinguished from each other. The vertex set of $\mathcal{P}$ is $V(\mathcal{P}):=\bigcup\limits_{i=1}^j \left(V(\gamma_i)\cup V({\gamma_i'}^{\top})\right)\cup V(\tau)$, the edge set of $\mathcal{P}$ is $E(\mathcal{P}):=\bigsqcup\limits_{i=1}^j \left(E(\gamma_i)\sqcup E({\gamma_i'}^{\top})\right)\sqcup E(\tau)$ where $\sqcup$ means disjoint union and each edge carries the label from the corresponding shape. We define the total size of $\mathcal{P}$ to be $\totalsize(\mathcal{P}):=|V(\mathcal{P})|+w(E(\mathcal{P}))$.

Finally, we define the \emph{result of $\mathcal{P}$} to be the following shape $\resultP$. First we take the quotient of the union of shapes in $\mathcal{P}$ by $\equiv$, which can have multi-edges. Then we replace all edges between each pair of vertex classes with a single edge, where if the edge labels are $i_1,...,i_t$, then the label of this edge is the minimal $i$ such that $h_i$ appears in the linear expansion of $\Prod_{l=1}^t h_{i_l}$ in the basis $\{h_l\mid l\in \N\}$. Note that $\resultP$ may have isolated vertices but no multi-edges.
\end{definition}

Note that vertices in each shape in $\mathcal{P}$ always stay distinct under $\equiv$. Also, for any fixed equivalence class, there are at most $2j+1$ many vertices in that class.

{\bf Notation.} 
We will call an intersection configuration just a configuration. Given a configuration, we use $[v]$ to denote the equivalence class of $v$ under $\equiv$, which is also viewed as a vertex in $\resultP$. We let $[U]:=\cup_{v\in U}[v]$ if $U\subseteq V(\mathcal{P})$. We use symbols like $v,w,v'$ to denote vertices in $V(\mathcal{P})$ and use square-bracketed symbols like $[v],[w],[v']$ to denote vertex classes, or vertices in $\resultP$. We write $v\in[w]$ to mean that vertex $v\in V(\mathcal{P})$ is in class $[w]$ under $\equiv$.

For each configuration $\mathcal{P}$, we fix an arbitrary minimum weight separator $S$ in the result $\resultP$. For $v\in V(\mathcal{P})$, we let $\deg_{\mathcal{P}}(v)$ be the sum of the degrees of $v$ over all shapes in $\mathcal{P}$ that contains $v$. For a class $[v]$, $\deg_{\mathcal{P}}([v]):=\Sum_{v'\in [v]}\deg_{\mathcal{P}}(v')$.

As in \Cref{def:configcoeff}, we let $\intset$ denote the set of all distinct intersection configurations on $\gamma_j,\ldots,\gamma_1,\tau,{\gamma'}^{\top}_1,\ldots,{\gamma'}^{\top}_j$. We also recall the definitions of the scaling coefficient $\lambda_{\mathcal{P}}$ and the Hermite coefficient $\eta_{\mathcal{P}}$. In particular, recall that $\lambda_P:=\lambda_{\gamma_{j}}\cdots\lambda_{\tau}\cdots\lambda_{{\gamma'}^{\top}_j}=n^{-\sum_{e\in E(\mathcal{P})}l(e)/2}$.

\begin{definition}[Side size of $\mathcal{P}$] 
The side size of a configuration $\mathcal{P}=\gamma_{j},\ldots,\gamma_1,\tau,{\gamma'}^{\top}_1,\ldots,{\gamma'}^{\top}_j$ is $\dside:=\max\{|U_{\gamma_j}|,\dots,|U_\tau|,\dots,|U_{\gamma'^{\top}_j}|,|V_{\gamma'^{\top}_j}|\}$.  
\end{definition}

\begin{remark}
To visualize a configuration $\mathcal{P}$, we can draw the sequence of shapes $(\gamma_j,\dots,\tau,\dots,\gamma'^{\top}_j)$ from left to right, one stacked after another, identifying $V_{\gamma_{i}}=U_{\gamma_{i+1}}$, $V_{{\gamma'}^{\top}_{i}}=U_{{\gamma'}^{\top}_{i+1}}$ for all $i\in[0,j-1]$ where $\gamma_0:={\gamma'}^{\top}_0:=\tau$. 
We then draw dotted lines between different vertices identified under $\equiv$. See \Cref{fig:config} for an example.
\end{remark}

{\color{black}
\begin{figure}
    \centering
    \begin{tikzpicture}[
      mycircle/.style={
         circle,
         draw=black,
         fill=white,
         fill opacity = 1,
         text opacity=1,
         inner sep=0pt,
         minimum size=25pt,
         font=\small},
      mysquare/.style={
         rectangle,
         draw=black,
         fill=white,
         fill opacity = 1,
         text opacity=1,
         inner sep=0pt,
         minimum height=20pt, 
         minimum width=20pt,
         font=\small},
      myarrow/.style={-Stealth},
      node distance=0.6cm and 1.2cm
      ]
      \draw (-7,0) ellipse (.8cm and 2cm);
      \draw (7,0) ellipse (.8cm and 2cm);
      \draw (-1,0) ellipse (.8cm and 1.5cm);
      \draw (1,0) ellipse (.8cm and 1.5cm);
      \node[mysquare]  at (-7, 1) (u1) {$u_1$};
      \node[mysquare]  at (-7, 0) (u2) {$u_2$};
      \node[mysquare]  at (-7, -1) (u3) {$u_3$};
      \node[mycircle]  at (-5.5, 0) (w1) {$w_1$};
      \node[mysquare]  at (-4, 1) (x1) {$x_1$};
      \node[mysquare]  at (-4, 0) (x2) {$x_2$};
      \node[mysquare]  at (-4, -1) (x3) {$x_3$};
      \node[mycircle]  at (-2.5, 0) (w2) {$w_2$};
      \node[mysquare]  at (-1, 0) (v) {$v$};
      \node at (0,0) {$\times$};
      \node[mysquare]  at (1, 0) (up) {$v$};
      \node[mycircle]  at (2.5, 0) (w3) {$w_3$};
      \node[mysquare]  at (4, 1) (y1) {$y_1$};
      \node[mysquare]  at (4, 0) (y2) {$y_2$};
      \node[mysquare]  at (4, -1) (y3) {$y_3$};
      \node[mycircle]  at (5.5, 0) (w4) {$w_4$};
      \node[mysquare]  at (7, 1) (vp1) {$v'_1$};
      \node[mysquare]  at (7, 0) (vp2) {$v'_2$};
      \node[mysquare]  at (7, -1) (vp4) {$v'_3$};
      \draw[-] (u1) -- (w1);
      \draw[-] (u2) -- (w1);
      \draw[-] (u3) -- (w1);
      \draw[-] (w1) -- (x1);
      \draw[-] (w1) -- (x2);
      \draw[-] (w1) -- (x3);
      \draw[-] (x1) -- (w2);
      \draw[-] (x2) -- (w2);
      \draw[-] (x3) -- (w2);
      \draw[-] (w2) -- (v);
      \draw[-] (up) -- (w3);
      \draw[-] (w3) -- (y1);
      \draw[-] (w3) -- (y2);
      \draw[-] (w3) -- (y3);
      \draw[-] (y1) -- (w4);
      \draw[-] (y2) -- (w4);
      \draw[-] (y3) -- (w4);
      \draw[-] (w4) -- (vp1);
      \draw[-] (w4) -- (vp2);
      \draw[-] (w4) -- (vp3);
      \draw[color=orange, dotted, line width=0.4mm] (x1) .. controls (-1,2) and (1,2) .. (y1);
      \draw[color=orange, dotted, line width=0.4mm] (x2) .. controls (-1,1.5) and (1,1.5) .. (y2);
      \draw[color=orange, dotted, line width=0.4mm] (x3) .. controls (-1,-2) and (1,-2) .. (y3);
      \draw[color=orange, dotted, line width=0.4mm] (w1) .. controls (-4.5,3) and (4.5,3) .. (w4);
      \draw[color=orange, dotted, line width=0.4mm] (w2) .. controls (-0.5,1) and (0.5,1) .. (w3);
      \end{tikzpicture}
    \caption{An intersection configuration $\mathcal{P}$ of length 1. The trivial shape in the middle is omitted.}
    \label{fig:config}
\end{figure}
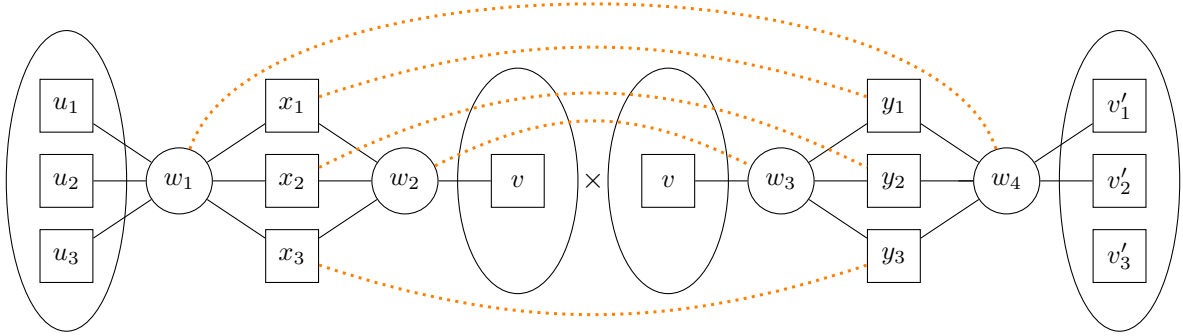
}

We extend the notion of a vertex being left-connected (see \Cref{def:connected-one-shape}) to configurations.

\begin{definition}\label{def:left-connected-P}
       We call a vertex $v \in V(\mathcal{P})$ left-connected if there is a path in $\resultP$ from $[U_{\gamma_j}]$ to $[v]$ which does not intersect any vertex of $S \setminus [v]$. Similarly, we call $v$ right-connected if there is a path from $[v]$ to $[V_{\gamma'^{\top}_j}]$ in $\resultP$ which does not intersect any vertex of $S \setminus [v]$.
\end{definition}

\subsubsection{Slack for intersection configurations}
For this section, we assume that we have an intersection configuration $\mathcal{P}$ and a minimum weight vertex separator $S$ for $\resultP$.bFor our analysis, it is very useful to consider the first and last times that each vertex appears non-trivially.
\begin{definition}[Non-trivial appearances]\label{def:nontrivial}
Given $\alpha, \alpha' \in \mathcal{P}$ and a vertex $v$ such that $v \in V(\alpha)$, we say that $v$ appears in $\alpha'$ if there is a vertex $v' \in V(\alpha')$ in the same equivalence class as $v$. If $v' \notin U_{\alpha'} \cap V_{\alpha'}$ then we say $v$ appears non-trivially in $\alpha'$. If $v' \in U_{\alpha'} \cap V_{\alpha'}$ then we say that $v$ appears trivially in $\alpha'$.

Similarly, we say a vertex $[v] \in V(\resultP)$ appears non-trivially in $\alpha$ if there is a vertex $v' \in V(\alpha)$ such that $v' \in [v]$. If $v' \notin U_{\alpha'} \cap V_{\alpha'}$ then we say $[v]$ appears non-trivially in $\alpha'$. If $v' \in U_{\alpha'} \cap V_{\alpha'}$ then we say that $[v]$ appears trivially in $\alpha'$.
\end{definition}
Note that if $v \in U_{\alpha} \cap V_{\alpha}$ then we say that $v$ appears trivially in $\alpha$ even if $v$ is incident to edges in $E(\alpha)$.
\begin{fact}
If $v$ is a circle vertex and $v$ appears in $\alpha$ then $v$ appears non-trivially in $\alpha$.
\end{fact}
\begin{fact}
If $v \in U_{\alpha} \setminus V_{\alpha}$ then either $v$ appears non-trivially in an earlier shape or there is a vertex $v'$ in the same equivalence class as $v$ such that $v' \in U_{\gamma_j}$. Similarly, if $v \in V_{\alpha} \setminus U_{\alpha}$ then either $v$ appears non-trivially in a later shape or there is a vertex $v'$ in the same equivalence class as $v$ such that $v' \in V_{{\gamma'}^{\top}_{j}}$.
\end{fact}
We now define the slack for the vertices in $\gamma_{j},\ldots,\gamma_1,\tau,{\gamma'}^{\top}_1,\ldots,{\gamma'}^{\top}_j$.
\begin{definition}[Default slack]
We define $\defaultslack(\resultP) = \sum_{\alpha \in \mathcal{P}}{\sum_{v \in V(\alpha)}{\defaultslack_{\alpha}(v)}}$. 
\end{definition}
\begin{definition}[Extra slack]
\label{def:intersectionvertexextraslack}
For each $\alpha \in \mathcal{P}$, we define the extra slack for vertices $v \in V(\alpha)$ as follows.
\begin{enumerate}
    \item\label{item:circle-extra-alpha-S} If $v$ is a circle vertex, we take $\extraslack_{\alpha,S}(v) = 0$ if $[v] \in S$ or $v$ appears non-trivially in both an earlier and a later shape. Otherwise, we take $\extraslack_{\alpha,S}(v) = \frac{k{\eps}}{16}$.
    \item \label{item:square-extra-alpha-S} If $v$ is a square vertex, we take $\extraslack_{\alpha,S}(v) = \frac{\eps}{8}$ if $v \notin U_{\alpha} \cap V_{\alpha}$ and at least one of the following holds:
    \begin{enumerate}
        \item \label{item:square-extra-alpha-S-1} $[v] \in S$ or $v$ appears non-trivially in both an earlier and a later shape.
        \item \label{item:square-extra-alpha-S-2} $v \in U_{\alpha} \setminus V_{\alpha}$ and $v$ appears in a later shape (note that at least one such appearance must be non-trivial).
        \item \label{item:square-extra-alpha-S-3} $v \in V_{\alpha} \setminus U_{\alpha}$ and $v$ appears in an earlier shape (note that at least one such appearance must be non-trivial).
        \item \label{item:square-extra-alpha-S-4} $v \in U_{\alpha} \setminus V_{\alpha}$, $v$ appears non-trivially in an earlier shape, and $[v]$ is not isolated in $\resultP$.
        \item \label{item:square-extra-alpha-S-5} $v \in V_{\alpha} \setminus U_{\alpha}$, $v$ appears non-trivially in a later shape, and $[v]$ is not isolated in $\resultP$.
    \end{enumerate}
Otherwise, we take $\extraslack_{\alpha,S}(v) = 0$.
\end{enumerate}
We define $\extraslack_{S}(\resultP) := \sum_{\alpha \in \mathcal{P}}{\sum_{v \in V(\alpha)}{\extraslack_{\alpha,S}(v)}}$. Then $\extraslack(\resultP)$ is defined to be the maximum of $\extraslack_{S}(\resultP)$ over all minimum weight vertex separators $S$ of $\resultP$.
\end{definition}
\begin{remark}
Note that conditions \ref{item:circle-extra-alpha-S} and \eqref{item:square-extra-alpha-S-1} are similar to the extra slack for proper middle shapes except that the condition ``$v \in S$'' is replaced by the condition that $[v] \in S$ or $v$ appears non-trivially in both an earlier and a later shape and we divide the extra slack by a factor of $2$ in order to avoid double counting. This is necessary as there are a few cases where a vertex has slack in two different shapes $\alpha, \alpha' \in \mathcal{P}$ but only obtains extra edge factor(s) from the edges in one of these shapes.

We include conditions \eqref{item:square-extra-alpha-S-2}, \eqref{item:square-extra-alpha-S-3}, \eqref{item:square-extra-alpha-S-4}, and \eqref{item:square-extra-alpha-S-5} in order to show that the only configurations with zero slack are well-behaved intersection configurations.
\end{remark}
\begin{definition}
For each $\alpha \in \mathcal{P}$ and $v \in V(\alpha)$, we define $\slack_{\alpha,S}(v) = \defaultslack_{\alpha}(v) + \extraslack_{\alpha,S}(v)$.
\end{definition}
\begin{definition}[Slack of $\resultP$]\label{eq:slack-resultP}
We define $\slack(\resultP) := \defaultslack(\resultP) + \extraslack(\resultP)$. Note that $\slack(\resultP) = \max\limits_{\text{Minimum weight separators $S$ of $\resultP$}}{\{\Sum\limits_{\alpha \in \mathcal{P}}{\Sum\limits_{v \in V(\alpha)}{\slack_{\alpha,S}(v)}}\}}$.
\end{definition}
\begin{lemma}[Key lemma for error analysis]\label{lem:error_key}
For all intersection configurations $\mathcal{P}$, 
\[
\lambda_{\mathcal{P}}\Anorm(\resultP) \leq n^{-\slack(\resultP)}.
\]
\end{lemma}
\begin{proof}
We prove this lemma using a similar analysis as we used to prove Lemma \ref{lem:error_analysis_mid}. In particular, we show that we can assign directions to the edges of $\gamma_j,\ldots,\gamma_1,\tau,\gamma'^{\top}_1,\ldots, \gamma'^{\top}_j$ so that if we use the redistribution scheme of Definition \ref{def:edgefactorredistribution}, the total factor for each vertex $v$ is at most $n^{-\sum_{\alpha \in \mathcal{P}}(\defaultslack_{\alpha}(v)+\extraslack_{\alpha,S}(v))}$.
In order to assign directions to the edges of $\gamma_j,\ldots,\gamma_1,\tau,\gamma'^{\top}_1,\ldots, \gamma'^{\top}_j$, we need a few more definitions and observations.

In this proof, we fix an arbitrary minimum weight vertex separator $S$ for $\resultP$. We fix sets of paths for $\gamma_{j},\ldots,\gamma_1$ and ${\gamma'}^{\top}_1,\ldots,{\gamma'}^{\top}_j$ in $\mathcal{P}$ given by Corollary \ref{cor:gammapaths} as well as $|U_{\tau}|$ many square-disjoint paths from $U_{\tau}$ to $V_{\tau}$. We orient all these paths from left to right. We define {\it anchor vertices} for these paths as follows.
\begin{definition}[Anchor vertices]\label{def:anchor}
For each chosen path $P$ in $\gamma_i$, we define the \emph{left anchor vertex} $u_P$ for $P$ to be the \emph{last} vertex on the path that appears in $U_{\gamma_i}$ or in an earlier shape. We then define the \emph{right anchor vertex} $v_P$ for $P$ to be the \emph{first} vertex equal to or after $u_P$ that appears in $V_{\gamma_i}$ or a later shape.

Similarly, for each path $P$ in $\tau$, we define the left anchor vertex $u_P$ for $P$ to be the last vertex on the path that appears in $U_{\tau}$ or in an earlier shape. We then define the right anchor vertex $v_P$ for $P$ to be the first vertex equal to or after $u_P$ that appears in $V_{\tau}$ or a later shape.

For each path $P$ in ${\gamma'}^{\top}_i$, we define the right anchor vertex $v_P$ for $P$ to be the \emph{first} vertex on the path (going from left to right) that appears in $V_{{\gamma'}^{\top}_i}$ or in a later shape. We then define the left anchor vertex $u_P$ for $P$ to be the \emph{last} vertex equal to or before $v_P$ that appears in $U_{{\gamma'}^{\top}_i}$ or an earlier shape.
\end{definition}

We now assign directions to the edges of $\gamma_j,\ldots,\gamma_1,\tau,\gamma'^{\top}_1,\ldots, \gamma'^{\top}_j$. For each shape $\alpha \in \resultP$ and each chosen path $P$ in $\alpha$ (going from left to right):
\begin{enumerate}
    \item For each edge $e = (u,v)$ to the left of $u_P$, have $e$ point to $u$.
    \item For each edge $e = (u,v)$ to the right of $v_P$, have $e$ point to $v$.
    \item For each edge $e = (u,v)$ between $u_P$ and $v_P$, if one of its endpoints is in $S$, we have $e$ point to the other endpoint (if both endpoints are in $S$, the choice is arbitrary). Otherwise, we have $e$ point to $u$ if both $u$ and $v$ are left-connected in $\resultP$ and we have $e$ point to $v$ if neither $u$ nor $v$ is left-connected in $\resultP$.
\end{enumerate}
Notice that the last item is well-defined. This is because each edge $e = (u,v)$ on $P$ between $u_P$ and $v_P$ gives rise to a single edge in $\resultP$ (in particular the edge won't disappear). Thus, if neither $u$ nor $v$ is in $S$ then either both $u$ and $v$ are left-connected or neither $u$ nor $v$ are left-connected.

Edges which are not on one of the chosen paths are not given a direction.

{\color{black}
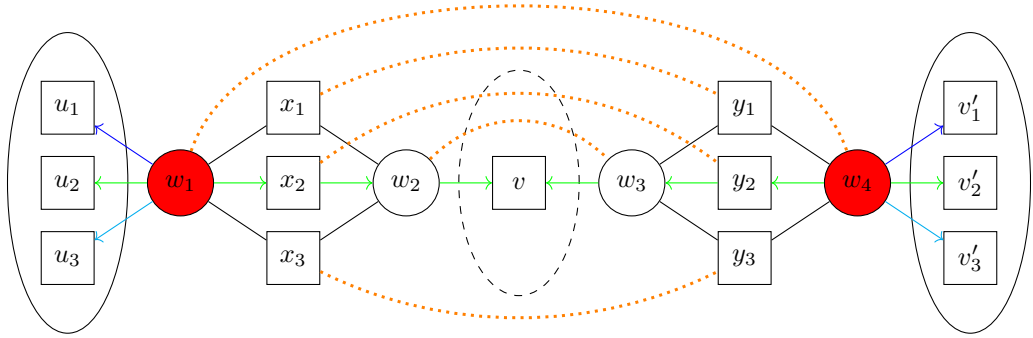
\begin{figure}
    \centering
    \begin{tikzpicture}[
      mycircle/.style={
         circle,
         draw=black,
         fill=white,
         fill opacity = 1,
         text opacity=1,
         inner sep=0pt,
         minimum size=25pt,
         font=\small},
      mysquare/.style={
         rectangle,
         draw=black,
         fill=white,
         fill opacity = 1,
         text opacity=1,
         inner sep=0pt,
         minimum height=20pt, 
         minimum width=20pt,
         font=\small},
      myarrow/.style={-Stealth},
      node distance=0.6cm and 1.2cm
      ]
      \draw (-6,0) ellipse (.8cm and 2cm);
      \draw (6,0) ellipse (.8cm and 2cm);
      \draw[dashed] (0,0) ellipse (.8cm and 1.5cm);
      \node[mysquare]  at (-6, 1) (u1) {$u_1$};
      \node[mysquare]  at (-6, 0) (u2) {$u_2$};
      \node[mysquare]  at (-6, -1) (u3) {$u_3$};
      \node[mycircle, fill = red]  at (-4.5, 0) (w1) {$w_1$};
      \node[mysquare]  at (-3, 1) (x1) {$x_1$};
      \node[mysquare]  at (-3, 0) (x2) {$x_2$};
      \node[mysquare]  at (-3, -1) (x3) {$x_3$};
      \node[mycircle]  at (-1.5, 0) (w2) {$w_2$};
      \node[mysquare]  at (0, 0) (v) {$v$};
      \node[mycircle]  at (1.5, 0) (w3) {$w_3$};
      \node[mysquare]  at (3, 1) (y1) {$y_1$};
      \node[mysquare]  at (3, 0) (y2) {$y_2$};
      \node[mysquare]  at (3, -1) (y3) {$y_3$};
      \node[mycircle, fill = red]  at (4.5, 0) (w4) {$w_4$};
      \node[mysquare]  at (6, 1) (vp1) {$v'_1$};
      \node[mysquare]  at (6, 0) (vp2) {$v'_2$};
      \node[mysquare]  at (6, -1) (vp3) {$v'_3$};
      \draw[->, color = blue] (w1) -- (u1);
      \draw[->, color = green] (w1) -- (u2);
      \draw[->, color = cyan] (w1) -- (u3);
      \draw[-] (w1) -- (x1);
      \draw[->, color = green] (w1) -- (x2);
      \draw[-] (w1) -- (x3);
      \draw[-] (x1) -- (w2);
      \draw[->, color = green] (x2) -- (w2);
      \draw[-] (x3) -- (w2);
      \draw[->, color = green] (w2) -- (v);
      \draw[->, color = green] (w3) -- (v);
      \draw[-] (y1) -- (w3);
      \draw[->, color = green] (y2) -- (w3);
      \draw[-] (y3) -- (w3);
      \draw[-] (w4) -- (y1);
      \draw[->, color = green] (w4) -- (y2);
      \draw[-] (w4) -- (y3);
      \draw[->, color = blue] (w4) -- (vp1);
      \draw[->, color = green] (w4) -- (vp2);
      \draw[->, color = cyan] (w4) -- (vp3);
      \draw[color=orange, dotted, line width=0.4mm] (x1) .. controls (-1,2) and (1,2) .. (y1);
      \draw[color=orange, dotted, line width=0.4mm] (x2) .. controls (-1,1.5) and (1,1.5) .. (y2);
      \draw[color=orange, dotted, line width=0.4mm] (x3) .. controls (-1,-2) and (1,-2) .. (y3);
      \draw[color=orange, dotted, line width=0.4mm] (w1) .. controls (-3.5,3) and (3.5,3) .. (w4);
      \draw[color=orange, dotted, line width=0.4mm] (w2) .. controls (-0.5,1) and (0.5,1) .. (w3);
      \end{tikzpicture}
    \caption{This figure shows the directions of the edges for the configuration $\mathcal{P}$ in \Cref{fig:config}. The vertices in the minimum weight vertex separator $S$ of $\resultP$ are shown in red.}
    \label{fig:config_assignment_before}
\end{figure}
}
There are two key features of this scheme.
\begin{enumerate}
    \item For each shape $\alpha$, all of the vertices $v \in V(\alpha)$ except for the anchor vertices obtain the needed number of edge factors from the edges in $E(\alpha)$. 
    \item For each path $P$ such that the left anchor vertex $u_P$ is not equal to the right anchor vertex $v_P$, neither $u_P$ nor $v_P$ is isolated in $\resultP$. In fact, none of the vertices between $u_P$ and $v_P$ are isolated in $\resultP$ but we only need this fact for $u_P$ and $v_P$.
\end{enumerate}

\begin{definition}[Guarded]
We call a vertex $v \in V(\alpha)$ guarded if it appears between $u_P$ and $v_P$, including $u_P$ and $v_P$, on some chosen path $P$ where $u_P \neq v_P$.  
\end{definition}

\begin{observation}
If $v$ is guarded, then $[v]$ is not isolated in $\resultP$.
\end{observation}

\begin{definition}[Satisfied]\label{def:satisfied}
For each shape $\alpha$ in $\mathcal{P}$ and each vertex $v\in V(\alpha)$ such that $[v]\notin S$, we say that $v$ is satisfied in $\alpha$ if one of the following conditions is satisfied:
\begin{enumerate}
    \item $v \notin U_{\alpha} \cup V_{\alpha}$ and there are at least as many edges pointing towards $v$ as there are edges pointing away from $v$. Note if we split the factor for each edge using the scheme in Definition \ref{def:edgefactorredistribution}, $v$ is guaranteed to obtain roughly $\frac{\deg_\alpha(v)}{2}$ edge factors from the edges in $E(\alpha)$ as edges which do not have a direction split their factor roughly evenly between their endpoints.
    \item $v \in (U_{\alpha} \setminus V_{\alpha}) \cup (V_{\alpha} \setminus U_{\alpha})$ and there are more edges pointing towards $v$ than there are edges pointing away from $v$. $v$. Note that if we split the factor for each edge using the scheme in Definition \ref{def:edgefactorredistribution}, $v$ is guaranteed to obtain roughly $\frac{\deg_\alpha(v) + 1}{2}$ edge factors from the edges in $E(\alpha)$.
\end{enumerate}
Here, $\deg_\alpha(v)$ means the degree of $v$ in $\alpha$.
\end{definition}

We always assume $k\geq 2$ so the degree of every circle vertex in every shape is at least $2$. 
\begin{lemma}\label{lem:satisfied}
Let $\alpha$ be a shape in $\mathcal{P}$ and $v \in V(\alpha) \setminus (U_{\alpha} \cap V_{\alpha})$. If $[v]\notin S$ and $v$ is not satisfied in $\alpha$, then one of the following cases holds. 
\begin{enumerate}
    \item\label{item:notleftconnected} $v$ is the left anchor vertex for some path and is not left-connected.
    \item\label{item:leftconnected} $v$ is the right anchor vertex for some path and is left-connected.
    \item\label{item:both} $v$ is both the left and the right anchor vertex for some path.
\end{enumerate}
\end{lemma}
\begin{proof}
We have the following cases:
\begin{enumerate}
    \item If $v \in U_\alpha \cap V_\alpha$ then one of the paths $P$ for $\alpha$ will be the degenerate path $P = \{v\}$ and we will have that $v = u_P = v_P$.
    \item If $v \in U_{\alpha} \setminus V_{\alpha}$ then there will be a path $P$ starting from $v$. Unless $v$ is the left anchor vertex of $P$ and is not left-connected, the first edge of $P$ will point to $v$ so $v$ will be satisfied.
    \item If $v \in V_{\alpha} \setminus U_{\alpha}$ then there will be a path $P$ ending at $v$. Unless $v$ is the right anchor vertex of $P$ and is left-connected, the last edge of $P$ will point to $v$ so $v$ will be satisfied.
    \item If $v \in V(\alpha) \setminus (U_{\alpha} \cup V_{\alpha})$ then for each path $P$ containing $v$, letting $u$ be the vertex before $v$ on $P$ and letting $w$ be the vertex after $v$ on $P$, we observe that
    \begin{enumerate}
        \item The edge $(u,v)$ will point to $v$ if $v$ comes after $v_P$ or $v$ comes after $u_P$ and is not left-connected.
        \item The edge $(v,w)$ will point to $v$ if $v$ comes before $u_P$ or $v$ comes before $v_P$ and is left-connected. 
    \end{enumerate}
    This implies that if neither of the edges $(u,v)$ and $(v,w)$ points to $v$ then either $v = u_P$ and $v$ is not left-connected, $v = v_P$ and $v$ is left-connected, or $v = u_P = v_P$. Thus, if none of the cases of Lemma \ref{lem:satisfied} hold for $v$ then there will be at least as many edges pointing to $v$ as there are edges pointing away from $v$ so $v$ will be satisfied.
\end{enumerate}
\Cref{lem:satisfied} is proved.
\end{proof}

We now give the slack analysis for \Cref{lem:error_key}. 
For this analysis, we need to divide the factor for each vertex $v \in V(\resultP)$ among all of the different shapes $\alpha \in \mathcal{P}$ such that $v$ appears in $\alpha$. We do this as follows.

\begin{definition}[Target factors for an intersection configuration]\label{def:intersectiontargetfactors}
For each $\alpha \in \mathcal{P}$ and $v \in V(\alpha)$,
\begin{enumerate}
    \item If $[v]\in S$ then we set $\target_{\alpha,S}(v) = 0$ for all $\alpha \in \mathcal{P}$ such that $v \in V(\alpha)$.
    \item If $[v]\notin S$ and $v$ only appears non-trivially in one shape $\alpha \in \mathcal{P}$ then we set $\target_{\alpha,S}(v) = \frac{1}{2}$ and $\target_{\alpha,S}(v) = \frac{\log_n(m)}{2} \leq \frac{(1-\eps)k}{4}$ if $v$ is a circle vertex. Note that in this case, $v$ cannot be isolated in $\resultP$.
    \item If $[v]\notin S$ and $v$ appears non-trivially in at least two different shapes in $\mathcal{P}$, let $\alpha$ and $\alpha'$ be the first and last shapes that $v$ appears in non-trivially. We have the following cases:
    \begin{enumerate}
        \item If $v$ is the left anchor vertex of a chosen path $P$ for $\alpha$ and $v$ is not left-connected then we set $\target_{\alpha,S}(v) = 0$ and we set $\target_{\alpha',S}(v) = \frac{1}{2}$ if $v$ is a square vertex and $\target_{\alpha',S}(v) = \frac{\log_n(m)}{2} \leq \frac{(1-\eps)k}{4}$ if $v$ is a circle vertex.
        \item Similarly, if $v$ is the right anchor vertex of a chosen path $P$ for $\alpha'$ and $v$ is left-connected then we set $\target_{\alpha',S}(v) = 0$ and we set $\target_{\alpha,S}(v) = \frac{1}{2}$ if $v$ is a square vertex and $\target_{\alpha,S}(v) = \frac{\log_n(m)}{2} \leq \frac{(1-\eps)k}{4}$ if $v$ is a circle vertex.
        \item If neither of the above cases holds then we set $\target_{\alpha,S}(v) = \target_{\alpha',S}(v) = \frac{1}{4}$ if $v$ is a square vertex which is not isolated in $\resultP$, we set $\target_{\alpha,S}(v) = \target_{\alpha',S}(v) = \frac{1}{2}$ if $v$ is a square vertex which is isolated in $\resultP$, we set $\target_{\alpha,S}(v) = \target_{\alpha',S}(v) = \frac{1}{4}\log_n(m) \leq \frac{(1-\eps)k}{8}$ if $v$ is a circle vertex which is not isolated in $\resultP$, and we set $\target_{\alpha,S}(v) = \target_{\alpha',S}(v) = \frac{1}{2}\log_n(m) \leq \frac{(1-\eps)k}{4}$ if $v$ is a circle vertex which is isolated in $\resultP$.
    \end{enumerate}
    In all of these cases, we set $\target_{\alpha'',S}(v) = 0$ for all $\alpha'' \in \mathcal{P}$ such that $\alpha''$ is not $\alpha$ or $\alpha'$ and $v \in V(\alpha'')$.
\end{enumerate}
\end{definition}
The following proposition follows from a direct inspection of \Cref{def:intersectiontargetfactors}.
\begin{proposition}
For all $[v] \in V(\resultP)$, letting $\target_{S}([v]) := \sum_{\alpha \in \mathcal{P}}{\sum_{v \in V(\alpha): v \in [v]}{\target_{\alpha,S}(v)}}$, we have that $\target_{S}([v]) = 0$ if $[v] \in S$, $\target_{S}([v]) = \frac{1}{2}$ if $[v] \notin S$ and $[v]$ is a square vertex which is not isolated in $\resultP$, $\target_{S}([v]) = 1$ if $[v] \notin S$ and $[v]$ is a square vertex which is isolated in $\resultP$, $\target_{S}([v]) = \frac{\log_n(m)}{2}$ if $[v] \notin S$ and $[v]$ is a circle vertex which is not isolated in $\resultP$, and $\target_{S}([v]) = \log_n(m)$ if $[v] \notin S$ and $[v]$ is a circle vertex which is not isolated in $\resultP$. 
\end{proposition}
We now consider the edge factors assigned to each vertex.
\begin{definition}[Assigned factors]\label{def:assigned_alpha_S_v}
For each vertex $\alpha \in \mathcal{P}$ and each $v \in V(\alpha)$, we define $\assigned_{\alpha,S}(v)$ so that when we split the factor for each edge $e$ based on its direction (if any) using the scheme in Definition \ref{def:edgefactorredistribution}, the product of the edge factors assigned to $v$ from the edges in $E(\alpha)$ is $n^{-\assigned_{\alpha,S}(v)}$. 

\end{definition}
\begin{lemma}\label{lem:assigned[v]}
For each vertex $[v] \in V(\resultP)$, we have $\assigned_{S}([v]) \geq \target_{S}([v]) + \slack_{S}([v])$, where $\assigned_{S}([v]):=\Sum_{\alpha\in\mathcal{P}}\Sum_{v\in[v]}\assigned_{\alpha,S}(v)$ and $\slack([v]):=\Sum_{\alpha\in\mathcal{P}}\Sum_{v\in[v]}\slack_{\alpha,S}(v)$.
\end{lemma}
\begin{proof}
We prove this lemma using the same ideas we used to prove Lemma \ref{lem:propoermiddleshapeslackanalysis}. The main modification is that for each shape $\alpha \in \mathcal{P}$, we treat the set of vertices $v \in V(\alpha)$ such that $\target_{\alpha,S}(v) = 0$ as in the ``separator'' $S_{\alpha}$ for $\alpha$, defined as below. 
\begin{definition}[$S_\alpha$]\label{def:Salpha}
For each $\alpha \in \mathcal{P}$, we define $S_{\alpha}$ to be the set of vertices $v \in V(\alpha)$ such that at least one of the following holds:
\begin{enumerate}
    \item $[v] \in S$ or $v$ appears non-trivally in both an earlier and a later shape in $\mathcal{P}$.
    \item $v$ is the left anchor vertex of some chosen path $P$ and $v$ is not left-connected.
    \item $v$ is the right anchor vertex of some chosen path $P$ and $v$ is left connected.
\end{enumerate}
\end{definition}
\begin{remark}
While every chosen path for $\alpha$ will have at least one vertex in $S_{\alpha}$, $S_{\alpha}$ may not be a vertex separator of $\alpha$.
\end{remark}
We now make the following observations, which can be shown similarly to the proof of Lemma \ref{lem:propoermiddleshapeslackanalysis}: 
For each $\alpha \in \mathcal{P}$ and each $v \in V(\alpha)$, $\assigned_{\alpha,S}(v) - \target_{\alpha,S}(v) \geq \defaultslack_{\alpha}(v)$. Moreover,
\begin{enumerate}
    \item If $v$ is a circle vertex and $v \notin S_{\alpha}$ then $\assigned_{\alpha,S}(v) - \target_{\alpha,S}(v) \geq \defaultslack_{\alpha}(v) + \frac{k\eps}{8}$.
    \item If $v$ is a square vertex and $v \in S_{\alpha}$ then $\assigned_{\alpha,S}(v) - \target_{\alpha,S}(v) \geq \defaultslack_{\alpha}(v) + \frac{\eps}{4}$.
\end{enumerate}
If $v$ is a circle vertex, the only ways that we can have $\assigned_{\alpha,S}(v) - \target_{\alpha,S}(v) < \defaultslack_{\alpha}(v) + \extraslack_{\alpha,S}(v)$ are as follows:
\begin{enumerate}
    \item $v$ is a left anchor vertex for some chosen path $P$ for $\alpha$, $v$ is not left-connected, and $v$ does not appear non-trivially in an earlier shape. In this case, letting $\alpha'$ be the last shape where $v$ appears non-trivially, $v \notin S_{\alpha'}$ so $\assigned_{\alpha',S}(v) - \target_{\alpha',S}(v) \geq \defaultslack_{\alpha'}(v) + \frac{k\eps}{8}$.
    \item $v$ is a right anchor vertex for some chosen path $P$ for $\alpha$, $v$ is left-connected, and $v$ does not appear non-trivially in a later shape. In this case, letting $\alpha'$ be the first shape where $v$ appears non-trivially, $v \notin S_{\alpha'}$ so $\assigned_{\alpha',S}(v) - \target_{\alpha',S}(v) \geq \defaultslack_{\alpha'}(v) + \frac{k\eps}{8}$.
\end{enumerate}
If $v$ is a square vertex, the only ways that we can have $\assigned_{\alpha,S}(v) - \target_{\alpha,S}(v) < \defaultslack_{\alpha}(v) + \extraslack_{\alpha,S}(v)$ are as follows:
\begin{enumerate}
    \item $v \in V_{\alpha} \setminus U_{\alpha}$, $[v]$ is not isolated in $\resultP$, $\alpha$ is the first shape in which $v$ appears non-trivially, and $\target_{\alpha,S}(v) = \frac{1}{2}$. In this case, letting $\alpha'$ be the last shape in which $v$ appears non-trivially, we must have that $v \in U_{\alpha'} \setminus V_{\alpha'}$, $v$ is the left anchor vertex of a chosen path $P$ for $\alpha'$, and $v$ is not left-connected as otherwise we would have had that $\target_{\alpha,S}(v) = \frac{1}{4}$. This implies that $v \in S_{\alpha'}$ so $\assigned_{\alpha',S}(v) - \target_{\alpha',S}(v) \geq \defaultslack_{\alpha'}(v) + \frac{\eps}{4}$.
    \item $v \in U_{\alpha} \setminus V_{\alpha}$, $[v]$ is not isolated in $\resultP$, $\alpha$ is the last shape in which $v$ appears non-trivially, and $\target_{\alpha,S}(v) = \frac{1}{2}$. In this case, letting $\alpha'$ be the first shape in which $v$ appears non-trivially, we must have that $v \in V_{\alpha'} \setminus U_{\alpha'}$, $v$ is the right anchor vertex of a chosen path $P$ for $\alpha'$, and $v$ is left-connected as otherwise we would have had that $\target_{\alpha,S}(v) = \frac{1}{4}$. This implies that $v \in S_{\alpha'}$ so $\assigned_{\alpha',S}(v) - \target_{\alpha',S}(v) \geq \defaultslack_{\alpha'}(v) + \frac{\eps}{4}$. 
\end{enumerate}
Summing up over shapes $\alpha\in\mathcal{P}$ and vertices in the class $[v]$, we get the conclusion of \Cref{lem:assigned[v]}.
\end{proof}
Putting everything together, we get the following which completes the proof of \Cref{lem:error_key}.
\[
\lambda_{\mathcal{P}}\Anorm(\resultP) = \prod_{[v] \in V(\resultP)}{n^{\target([v]) - \assigned([v])}} \leq n^{-\sum_{[v] \in V(\resultP)}{\slack([v])}} = n^{-\slack(\resultP)}. \qedhere
\]
\end{proof}

\subsubsection{Slack for configurations which are not well-behaved}\label{subsec:characterization}
We now lower bound the slack for configurations which are not well-behaved.
\begin{proposition}\label{prop:zeroslackvertices}
If $\alpha \in \mathcal{P}$, $v \in V(\alpha)$, and $\slack_{\alpha,S}(v) = 0$, then one of the following must hold:
\begin{enumerate}
    \item $v$ is a circle vertex and either $[v] \in S$ or $v$ appears non-trivially in both an earlier and a later shape.
    \item $v$ is a square vertex, $v \in U_{\alpha} \cap V_{\alpha}$, and $\deg_{\alpha}(v) = 0$.
    \item $v$ is a square vertex, $v \in U_{\alpha} \setminus V_{\alpha}$, $\deg_{\alpha}(v) = 1$, $[v] \notin S$, $v$ does not appear in any later shape, and either $[v] \in U_{\resultP}$ or $[v]$ is isolated in $\resultP$.
    \item $v$ is a square vertex, $v \in V_{\alpha} \setminus U_{\alpha}$, $\deg_{\alpha}(v) = 1$, $[v] \notin S$, $v$ does not appear in any earlier shape, and either $[v] \in V_{\resultP}$ or $[v]$ is isolated in $\resultP$.
\end{enumerate}
If none of these conditions hold, then $\slack_{\alpha,S}(v) \geq \frac{\eps}{8}$.
\end{proposition}
\begin{proof}
We make the following observations. Given a shape $\alpha \in \mathcal{P}$ and a vertex $v \in V(\alpha)$,
\begin{enumerate}
    \item If $v$ is a circle vertex then by statement 1 of Definition \ref{def:intersectionvertexextraslack}, we only have that $\slack_{\alpha,S}(v) = 0$ if $[v] \in S$ or $v$ appears non-trivially in both an earlier and a later shape. Otherwise, $\extraslack_{\alpha,S}(v) = \frac{k\eps}{16} \geq \frac{\eps}{8}$.
    \item If $v$ is a square vertex and $v \notin U_{\alpha} \cup V_{\alpha}$ then $\defaultslack_{\alpha}(v) = \frac{(\deg_{\alpha}(v) - 2)}{4} + \frac{{\eps}\deg_{\alpha}(v)}{8} \geq \frac{\eps}{4}$.
    \item If $v$ is a square vertex and $v \in U_{\alpha} \cap V_{\alpha}$ then $\defaultslack_{\alpha}(v) = \frac{\deg_{\alpha}(v)}{4}$ so we can only have $\slack_{\alpha,S}(v) = 0$ if $\deg_{\alpha}(v) = 0$. If $\deg_{\alpha}(v) > 0$ then $\slack_{\alpha,S}(v) \geq \frac{1}{4} > \frac{\eps}{8}$.
    \item If $v \in U_{\alpha} \setminus V_{\alpha}$ or $v \in V_{\alpha} \setminus U_{\alpha}$ then $\defaultslack_{\alpha}(v) := \frac{\deg_{\alpha}(v) - 1}{4}$ so we can only have $\slack_{\alpha,S}(v) = 0$ if $\deg_{\alpha}(v) = 1$. If $\deg_{\alpha}(v) > 1$ then $\slack_{\alpha,S}(v) \geq \frac{\eps}{8}$.
    \item If $v \in U_{\alpha} \setminus V_{\alpha}$ then by statements 2a, 2b, and 2d of Definition \ref{def:intersectionvertexextraslack}, if $[v] \in S$, $v$ appears in a later shape, or $v$ appears non-trivially in an earlier shape and $[v]$ is not isolated in $\resultP$ then $\slack_{\alpha,S}(v) \geq \frac{\eps}{8}$.
    \item If $v \in V_{\alpha} \setminus U_{\alpha}$ then by statements 2a, 2c, and 2e of Definition \ref{def:intersectionvertexextraslack}, if $[v] \in S$, $v$ appears in an earlier shape, or $v$ appears non-trivially in a later shape and $[v]$ is not isolated in $\resultP$ then $\slack_{\alpha,S}(v) \geq \frac{\eps}{8}$.
\end{enumerate}
Putting these observations together, the only cases where $\slack_{\alpha,S}(v) = 0$ are the cases stated in Proposition \ref{prop:zeroslackvertices}.
\end{proof}
\begin{corollary}\label{cor:zeroslack_means_wellbehaved}
If the intersection configuration $\mathcal{P}$ is not well-behaved, then $\slack(\resultP) \geq \frac{\eps}{8}$. 
\end{corollary}
\begin{proof}
Recall that $\mathcal{P}$ is well-behaved if the following conditions are satisfied:
\begin{enumerate}
    \item\label{Condition:well-behaved-1} $\gamma_j,\ldots,\gamma_1,\tau,{\gamma'}^{\top}_1,\ldots,{\gamma'}^{\top}_j$ are all disjoint unions of simple spiders.
    \item\label{Condition:well-behaved-2} $\mathcal{P}$ has no non-trivial intersections between square vertices.
    \item\label{Condition:well-behaved-3} Whenever there is a square vertex $v$ which is not in $U_{\resultP} \cup V_{\resultP}$, there is an intersection between the two neighbors of $v$ which results in the two edges incident to $v$ combining into a double edge which vanishes.
\end{enumerate}
We show that if any of these conditions are not satisfied, then $\slack(\resultP) \geq \frac{\eps}{8}$:
\begin{enumerate}
    \item If there is a shape $\alpha \in \mathcal{P}$ which is not a disjoint union of simple spiders then there is a square vertex $v \in V(\alpha)$ such that $\defaultslack_{\alpha}(v) \geq \frac{\eps}{8}$.
    \item If all shapes $\alpha \in \mathcal{P}$ are disjoint unions of simple spiders and there is a non-trivial intersection between square vertices then there must be a shape $\alpha \in \mathcal{P}$ and a square vertex $v \in V(\alpha)$ such that $v \in U_{\alpha} \setminus V_{\alpha}$ and $v$ appears in a later shape. By statement 2b of Definition \ref{def:intersectionvertexextraslack}, $\extraslack_{\alpha,S}(v) = \frac{\eps}{8}$.
    \item If the first two conditions are satisfied then each square vertex $v \in V(\mathcal{P})$ such that $v \notin U_{\resultP} \cup V_{\resultP}$, letting $\alpha$ and $\alpha'$ be the first and last shapes that $v$ appears in, $v$ will be adjacent to a circle vertex $u \in V(\alpha)$ and a circle vertex $u' \in V(\alpha')$. For all shapes $\alpha'' \in \mathcal{P}$ between $\alpha$ and $\alpha'$, $v \in U_{\alpha''} \cap V_{\alpha''}$ and $v$ is not incident to any edges of $\alpha''$. If there is no intersection between $u$ and $u'$ (i.e., $u$ and $u'$ are not equivalent) then $[v]$ is not isolated in $\resultP$ so by statements 2d and 2e of Definition \ref{def:intersectionvertexextraslack}, $\extraslack_{\alpha',S}(v) = \extraslack_{\alpha,S}(v) = \frac{\eps}{8}$. \qedhere
\end{enumerate}
\end{proof}

\begin{lemma}\label{lem:non-well-behaved}
If $\mathcal{P}$ is an intersection configuration which is not well-behaved such that $|U_{\resultP}| \leq t$ and $|V_{\resultP}| \leq t$ where $t\geq 1$, then $\slack(\resultP) \geq \max{\{\frac{\eps}{8},\frac{\eps}{12}(\totalsize(\mathcal{P}) - 12{t}^2)\}}$.
\end{lemma}
\begin{proof}
Since $\mathcal{P}$ is not well-behaved, $\slack(\resultP) \geq \frac{\eps}{8}$ by \Cref{cor:zeroslack_means_wellbehaved}. To show that $\slack(\resultP) \geq \frac{\eps}{12}(\totalsize(\mathcal{P}) - 12{t}^2)$, we observe: 
\begin{enumerate}
    \item $\Sum_{\alpha \in \mathcal{P}}{\Sum_{v \in V(\alpha)}{\defaultslack_{\alpha}(v)}} \geq \Sum_{\alpha \in \mathcal{P}}{\Sum_{v \in V(\alpha)}{\frac{\eps}{8}(\deg_{\alpha}(v) - 1_{\alpha \in (U_{\alpha} \setminus V_{\alpha}) \cup (V_{\alpha} \setminus U_{\alpha})})}}$.
    \item $\Sum_{\alpha \in \mathcal{P}}{\Sum_{v \in V(\alpha)}{1_{\alpha \in (U_{\alpha} \setminus V_{\alpha}) \cup (V_{\alpha} \setminus U_{\alpha})}}} \leq 2t(2t+1)$. 
    \item $\totalsize(\mathcal{P}) \leq \Sum_{\alpha \in \mathcal{P}} (|V(\alpha)|+w(E(\alpha)))\leq \Sum_{\alpha \in \mathcal{P}} \left(|U_\alpha\cap V_\alpha|+\frac{3}{2}{\Sum_{v \in V(\alpha)}{\deg_{\alpha}(v)}}\right)$, where in the second inequality we used $\Sum_{v\in V(\alpha)} \deg_\alpha(v) = 2w(E(\alpha))$ and $w(E(\alpha))\geq \frac{1}{2}|V(\alpha)\backslash(U_\alpha \cup V_\alpha)|$. Thus, $\Sum_{\alpha \in \mathcal{P}}{\Sum_{v \in V(\alpha)}{\deg_{\alpha}(v)}} \geq \frac{2}{3}\totalsize(\mathcal{P}) - \frac{2}{3}t(2t+1)\geq \frac{2}{3}\totalsize(\mathcal{P}) - 2t^2$.
\end{enumerate}
Putting these observations together, 
\[
\slack(\resultP) \geq \Sum_{\alpha \in \mathcal{P}}{\Sum_{v \in V(\alpha)}{\defaultslack_{\alpha}(v)}} \geq 
\frac{\eps}{8}(\frac{2}{3}\totalsize(\mathcal{P}) - 8t^2) = \frac{\eps}{12}(\totalsize(\mathcal{P}) - 12t^2). \qedhere
\]
\end{proof}

\subsection{Main Error Analysis Results for Intersection and Product Configurations}
\label{sec:boundingterms}
Combining \Cref{lem:error_key} and \Cref{lem:non-well-behaved}, we get the following main result of \Cref{sec:error_analysis}:

\begin{theorem}[Main result for error analysis for intersection configurations]\label{thm:erroranalysis}
For any intersection configuration $\mathcal{P}$, we have the following dichotomy.
\begin{enumerate}
    \item \label{item:dichotomy_SSD} If $\mathcal{P}$ is a well-behaved configuration, $\lambda_{\mathcal{P}}\Anorm(\resultP)=1$;  
    \item \label{item:dichotomy_nonSSD} If $\mathcal{P}$ is not a well-behaved configuration, 
    \[
    \lambda_{\mathcal{P}}\Anorm(\resultP)\leq n^{-\max\big\{\frac{\eps}{8},\frac{\eps}{12}\left(\totalsize(\mathcal{P}) - 12\dside^2\right)\big\}}
    \] 
    where $\dside:=\max\{\left\vert U_{\resultP} \right\vert, \left\vert V_{\resultP}\right\vert\}$.
\end{enumerate}
\end{theorem}

\begin{proof}
\Cref{item:dichotomy_nonSSD} follows from \Cref{lem:error_key} and \Cref{lem:non-well-behaved}. 
We show \Cref{item:dichotomy_SSD} below. 

Assume $\mathcal{P}$ is a well-behaved configuration. 
For convenience, we denote it by $\mathcal{P}=(\gamma_{-j},\ldots,\gamma_j)$ where $j\geq 1$. 
We first prove two properties about $\resultP$: 
\begin{enumerate}[label={(\arabic*)}]
    \item \label{item:well-behaved-result-1}$\resultP$ is an SSD shape plus isolated square vertices.
    \item \label{item:well-behaved-result-2} The (unique) minimum weight separator of $\resultP$ consists of $U_{\resultP}\cap V_{\resultP}$ plus all circle vertices. 
\end{enumerate}
Property \ref{item:well-behaved-result-1} follows directly from conditions \ref{Condition:well-behaved-1}, \ref{Condition:well-behaved-2} and \ref{Condition:well-behaved-3} of well-behaved configurations. For property \ref{item:well-behaved-result-2}, having shown that $\resultP$ is an SSD shape plus isolated squares, it suffices to show that every circle vertex $[v]\in V(\resultP)$ has left- and right- degrees both at least $\ceil{\frac{k}{2}}$ in $\resultP$, i.e., $[v]$ is good. For this, consider the first shape $\gamma_{j'}\in\mathcal{P}$ in which some circle vertex $v_1$ of class $[v]$ appears. We claim that $j'\leq 0$. For otherwise, $\gamma_{j'}$ is a (proper) right SSD shape so the right degree of $v_1$ is positive, then by the intersection-separating condition \ref{Condition:config-intersection}(c) and condition \ref{Condition:well-behaved-2} of well-behaved configurations, the circle vertex $v_1$ must be equivalent to a $v'_1$ in one of $\gamma_{-(j'-1)},\ldots,\gamma_{j'-1}$, which is a contradiction to the assumption on $j'$. Therefore, the left degree of $[v]$ in $\resultP$ is at least $\ceil{\frac{k}{2}}$. The right degree of $[v]$ in $\resultP$ can be lower bounded by $\ceil{\frac{k}{2}}$ similarly by considering the last shape in $\mathcal{P}$ in which some circle vertex of class $[v]$ appears. Property \ref{item:well-behaved-result-2} follows.

From properties \ref{item:well-behaved-result-1} and \ref{item:well-behaved-result-2}, we have $\log_n(\lambda_{\mathcal{P}}\cdot A_{\resultP}) = -\frac{w(E(\mathcal{P}))}{2} + |\Iso(\resultP)|+\frac{|U_{\gamma_j}\backslash V_{\gamma'^\top_j}|+|V_{\gamma'^\top_j}\backslash U_{\gamma_j}|}{2}$. By conditions \ref{Condition:well-behaved-1} and \ref{Condition:well-behaved-2} of well-behaved configurations, all vertices in $(U_{\gamma_j}\backslash V_{\gamma'^\top_j})\cup (V_{\gamma'^\top_j}\backslash U_{\gamma_j})$ have degree 1 in $\mathcal{P}$, and all vertices in $U_{\gamma_j}\cap V_{\gamma'^\top_j}$ has degree 0 in $\mathcal{P}$. By condition \ref{Condition:well-behaved-3} of well-behaved configurations, all other square vertices in $\mathcal{P}$, i.e., those in $V_\square(\mathcal{P})\backslash(U_{\gamma_j}\cup V_{\gamma'^\top_j})$, are singleton classes, have degree $2$ in $\mathcal{P}$, and 1-1 correspond to the isolated vertices $\Iso(\resultP)$. Thus, $w(E(\mathcal{P}))=\sum_{v\in V_{\square(\mathcal{P})}} \deg_{\mathcal{P}}(v) = 2|\Iso(\resultP)| + |U_{\gamma_j}\backslash V_{\gamma'^\top_j}|+|V_{\gamma'^\top_j}\backslash U_{\gamma_j}|$. 
This means $\log_n(\lambda_{\mathcal{P}}\cdot A_{\resultP})=0$.
\end{proof}

In the applications of \Cref{thm:erroranalysis} in \Cref{sec:psdness-qSS} and \Cref{sec:psdness-qSSD}, the length $j$ of the configuration is at most $2\dsos$, and the side size $\dside$ is at most $\dsos$.

\subsubsection{Analyzing Product Configurations}\label{sec:productconfiganalysis}
We now describe how to modify our techniques to analyze product configurations and give analogous results for product configurations.

We can think of product configurations $\mathcal{P} \in \mathcal{P}_{\alpha_1,\alpha_2}$ as intersection configurations of length $1$ where $\tau$ is the identity, $\gamma_1 = \alpha_1$, and ${\gamma'_1}^T = \alpha_2$. Keeping the same definitions for default slack and extra slack, we have 
\begin{lemma}[Key lemma for error analysis for product configurations]\label{lem:producterror_key}
For all SSD $\alpha_1$ and $\alpha_2$ which are good, for all product configurations $\mathcal{P} \in \mathcal{P}_{\alpha_1,\alpha_2}$, $slack(\alpha_{\mathcal{P}}) \geq 0$ and 
\[
\lambda_{\mathcal{P}}\Anorm(\alpha_{\mathcal{P}}) \leq n^{-\slack(\alpha_{\mathcal{P}})}.
\]
\end{lemma}
\begin{proof}[Proof sketch]
This lemma can be shown using the same techniques used to prove \Cref{lem:error_key}, though a few modifications are needed.
\begin{enumerate}
    \item Different choice of paths: Instead of considering paths from $U_{\gamma_1}$ to the union of the set of intersected vertices and $V_{\gamma_1}$, we do the following. Since $\alpha_1$ is good and is an SSD, for each circle vertex $v \in V_{\bigcirc}(\alpha_1)$, we can take $\lceil{\frac{k}{2}}\rceil$ square-vertex disjoint paths from $U_{\alpha_1}$ to $V_{\alpha_1}$ passing through $v$.
    \item Modified scheme for assigning directions: For each edge which is not on a path, instead of not giving it a direction, we have it point to its square endpoint.

    As noted in the next point, we can afford to do this because each circle vertex in a shape only needs $\frac{k}{2}$ edge factors and the $\lceil{\frac{k}{2}}\rceil$ paths through the vertex are sufficient to guarantee this.
    \item Modified definition of being satisfied: Instead of having circle vertices be satisfied if the number of edges pointing towards the vertex is greater than or equal to the number of edges pointing away from the vertex, we say that circle vertices are satisfied in a shape if at least $\lceil{\frac{k}{2}}\rceil$ edges in the shape point to the vertex.

    With this modification, \Cref{lem:satisfied} holds and can be proved using the same logic as before.
\end{enumerate}
\end{proof}
Our main error analysis result for product configurations is as follows.

\begin{theorem}\label{thm:producterroranalysis}
For all SSD $\alpha_1,\alpha_2$ such that $\alpha_1$ and $\alpha_2$ are good, $|U_{\alpha_1}| \leq \dsos$, $|V_{\alpha_1}| = |U_{\alpha_2}| \leq \dsos$, and $|V_{\alpha_2}| \leq \dsos$, for all product configurations $\mathcal{P} \in \mathcal{P}_{\alpha_1,\alpha_2}$, we have the following dichotomy.
\begin{enumerate}
    \item \label{item:productconfig_dichotomy_SSD} If $\mathcal{P}$ is a well-behaved configuration, $\lambda_{\mathcal{P}}\Anorm(\alpha_{\mathcal{P}})=1$;  
    \item \label{item:productconfig_dichotomy_nonSSD} If $\mathcal{P}$ is not a well-behaved configuration then $\lambda_{\mathcal{P}}\Anorm(\alpha_{\mathcal{P}}) \leq n^{-\frac{\eps}{12}}$.
\end{enumerate}
\end{theorem}
\begin{proof}[Proof sketch]
This can be proved by the same logic we used to prove \Cref{prop:zeroslackvertices}, \Cref{cor:zeroslack_means_wellbehaved}, and \Cref{thm:producterroranalysis}.
\end{proof}

\subsection{Useful Bounds}\label{sec:usefulbounds}
Before giving our estimates for the error terms, we show a few useful bounds. 
\subsubsection{Bound on the contribution from a given shape}
\begin{lemma}\label{lem:generalalphanormbound}
Assume $\truncation\geq 2k\log n$. For all shapes $\alpha$ such that $\totalsize(\alpha) \leq 3\truncation$,
\begin{equation}\label{eq:generalalphanormbound}
\norm{\eta_{\alpha} \lambda_{\alpha} M_\alpha} \leq \cu^{\totalsize(\alpha)}(3\truncation)^{3\Cuniv\totalsize(\alpha)}n^{2\dsos - \frac{\eps}{8}\totalsize(\alpha)}.
\end{equation}
\end{lemma}

\begin{proof}
By \Cref{thm:norm_control}, 
\begin{equation*}
\norm{\eta_\alpha\lambda_\alpha M_\alpha}\leq (2k\cdot \totalsize(\alpha)^2\log n)^{\Cuniv\totalsize(\alpha)} \cdot {(\cu)}^{\totalsize(\alpha)}\cdot n^{-\left(w(E(\alpha)) - w(V(\alpha)) + w(S)\right)/2}
\end{equation*}
where $S$ is a minimum weight separator of $\alpha$ and $C$ is an absolute constant. Since $\totalsize(\alpha) \leq 3\truncation$ and $\truncation \geq 2k\log n$, $(2k\cdot\totalsize(\alpha)^2\log n)^{\Cuniv\totalsize(\alpha)} \leq (3\truncation)^{3\Cuniv\totalsize(\alpha)}$.

We estimate the exponent over $n$ using the following edge weight assignment. Let each edge in $\alpha$ have a factor of $1$ which we split between its endpoints as follows (an edge with label $t>1$ is viewed as $t$ parallel edges for this argument). For each edge $e \in E(\alpha)$, we assign $(1-\eps/2)/2$ to the circle and $(1+\eps/2)/2$ to the square. With this assignment, each circle vertex receives a factor $(1-\eps/2)/2$ times its degree, and each square vertex receives a factor $(1+\eps/2)/2$ times its degree. We now make the following observations:
\begin{enumerate}
    \item Since each circle vertex has degree at least $k$ and each square vertex outside of $U_{\alpha} \cup V_{\alpha}$ has degree at least $2$, each vertex not in $U_\alpha \cup V_\alpha$ receives a factor of at least its weight plus $\eps/4$ times its degree.
    \item Each vertex $v \in U_{\alpha} \cup V_{\alpha}$ receives a factor of $\frac{1}{2} + \frac{\eps}{4}$ times its degree.
\end{enumerate}

Using these observations, 
\begin{equation}\label{eq:wE-wV}
w(E(\alpha)) - w(V(\alpha)) \geq \frac{\eps}{4}\sum_{v \in V(\alpha)}{\deg_{\alpha}(v)} - |U_{\alpha} \cup V_{\alpha}| \geq \frac{\eps}{4}\sum_{v \in V(\alpha)}{\deg_{\alpha}(v)} - 2\dsos.
\end{equation}
To lower bound $w(E(\alpha)) - w(V(\alpha))$ in terms of $\totalsize(\alpha)$, we make the following further observations:

\begin{enumerate}
    \item $\deg(v)\geq 2$ for each $v \in V(\alpha) \setminus (U_{\alpha} \cup V_{\alpha})$, so $\frac{1}{2}\Sum_{v \in V(\alpha)}{\deg_{\alpha}(v)} \geq |V(\alpha)| - |U_\alpha\cup V_\alpha| \geq |V(\alpha)| - 2\dsos$.
    \item $\sum_{v \in V(\alpha)}{\deg_{\alpha}(v)} = 2w(E(\alpha))$.
\end{enumerate}
Putting these observations together, 
\[
\sum_{v \in V(\alpha)}{\deg_{\alpha}(v)} = \frac{1}{2}\sum_{v \in V(\alpha)}{\deg_{\alpha}(v)} + \frac{1}{2}\sum_{v \in V(\alpha)}{\deg_{\alpha}(v)} \geq w(E(\alpha)) + |V(\alpha)| - 2D = \totalsize(\alpha) - 2D.
\] 
By this and \eqref{eq:wE-wV},
\[
w(E(\alpha)) - w(V(\alpha)) + w(S) \geq \frac{\eps}{4}\totalsize(\alpha) - 4\dsos
\]
Thus, for each shape $\alpha$ such that $\totalsize(\alpha) \leq 3\truncation$,
\begin{align*}
\norm{\eta_\alpha\lambda_\alpha M_\alpha} &\leq \cu^{\totalsize(\alpha)}(3\truncation)^{3\Cuniv\totalsize(\alpha)}n^{-\left(w(E(\alpha)) - w(V(\alpha)) + w(S)\right)/2} \\
&\leq \cu^{\totalsize(\alpha)}(3\truncation)^{3\Cuniv\totalsize(\alpha)}n^{2\dsos - \frac{\eps}{8}\totalsize(\alpha)}. \qedhere
\end{align*}
\end{proof}

\subsubsection{Bounds on the number of shapes and configurations}

\begin{lemma}[Number of configurations]\label{lem:configcount}
For all $a,b, \dsos, \truncation \in \mathbb{N} \cup \{0\}$, there are at most 
\[
(48{\dsos}^2\truncation)^{a}(36{\dsos}\truncation^2)
\]
configurations $\mathcal{P}$ such that $|U_{\resultP}| \leq \dsos$, $|V_{\resultP}| \leq \dsos$, $|V(\mathcal{P})|=a$, $w(E(\mathcal{P}))=b$, and $\totalsize(\mathcal{P}) \leq 3\truncation$.
\end{lemma}

Note that \Cref{lem:configcount} applies to both intersection configurations and product configurations.

\begin{proof}
We first observe that there are at most $2\dsos-1$ rounds of intersections for $\mathcal{P}$. To see this, observe that if $\tau_{P_1}$ is the result after the first round of intersection, $|U_{\tau_{P_1}}| \geq 1$ and $|V_{\tau_{P_1}}| \geq 1$ so $|U_{\tau_{P_1}}| + |V_{\tau_{P_1}}| \geq 2$. Each additional round of intersections increases the size of $U \cup V$ by at least one so there can be at most $2\dsos-2$ additional rounds of intersection. 

We can specify $\mathcal{P}$ by describing the vertices and edges of $\mathcal{P}$ and how they are connected. For each vertex $v \in V(\mathcal{P})$, we specify the following data:
\begin{enumerate}
    \item Is $v$ is a circle vertex or a square vertex? There are $2$ choices for this.
    \item What are the first and last shapes that $v$ appears in (before taking intersections into account)? There are at most $\binom{4\dsos - 1}{2} + 4\dsos -1 \leq 8\dsos^2$ choices for this.
    \item Is there an intersection between $v$ and a previous vertex? If so, what is the most recent vertex we've seen which has an intersection with $v$? There are at most $3\truncation$ choices for this.
\end{enumerate}
For each edge $e \in E(\mathcal{P})$, we specify the following data:
\begin{enumerate}
    \item Which shape $e$ appears in.
    \item The two endpoints of $e$. There are at most $\binom{3\truncation}{2}$ choices for this.
    \item Whether $e$ remains after the edge $e$ and edges parallel to it are linearized (if the result is an edge with label $l$ then we choose $l$ of these edges to remain).
\end{enumerate}
Thus, the total number of choices for each $e \in E(\mathcal{P})$ is at most $36\dsos\truncation^2$. Note that for this argument, we treat an edge with label $l$ as $l$ separate edges.

Putting these bounds together, the total number of choices is at most $(48{\dsos}^2\truncation)^{a}(36{\dsos}\truncation^2)$.
\end{proof}

\begin{remark}
Note that for each $v \notin U_{\gamma_j} \cap V_{{\gamma'_j}^{\top}}$, the parity labels for $v$ (if any) are determined by the parity of $\deg_{\alpha}(v)$ for the shapes $\alpha$ such that $v \in V(\alpha)$ (before intersections are taken into account). If $v \in U_{\gamma_j} \cap V_{{\gamma'_j}^{\top}}$ then there are two choices for the parity labels of $v$. This factor can be absorbed into specifying the first and last shapes $v$ appears in. If $j = 0$ then we take $\gamma_j = {\gamma'_j}^{\top} = \tau$.
\end{remark}
\begin{corollary}\label{cor:simplerconfigbound}
Assume $\truncation \geq 100\dsos$. For all $s \in [3\truncation]$, there are at most $(\truncation)^{3s}$ configurations $\mathcal{P}$ such that $\totalsize(\mathcal{P}) = s$. As a special case, there are at most $(\truncation)^{3s}$ proper shapes $\alpha$ such that $\totalsize(\alpha) = s$.
\end{corollary}
\begin{proof}
Since $\totalsize(\mathcal{P}) = |V(\mathcal{P})| + w(E(\mathcal{P}))$ and $\truncation \geq 100\dsos
$, by \Cref{lem:configcount}, there are at most 
\[
\sum_{a = 0}^{s}{(48{\dsos}^2\truncation)^{a}(36{\dsos}\truncation^2)^{s-a}} = \left(48{\dsos}^2\truncation + 36{\dsos}\truncation^2\right)^{s} \leq (\truncation)^{3s}
\]
configurations with total size $s$.
\end{proof}

\subsubsection{Bounds on $\norm{L_{\leq l}}$ and $\norm{L - L_{\leq l}}$}
\begin{corollary}\label{cor:roughLnormbound}
Assume $\truncation\geq 2k\log n$ and $\cu (3\truncation)^{6\Cuniv} \leq \frac{1}{2} n^{\frac{\eps}{8}}$. Then for all $l \leq \truncation$, 
\[
\norm{L_{\leq l}} \leq n^{2\dsos}.
\] 
Note that this bound holds for $L$ since $L = L_{\leq \truncation}$.
\end{corollary}
\begin{proof}
Recall that $L_{\leq l} = \sum_{\sigma: \sigma \text{ is a left shape}, \totalsize(\sigma) \leq l}{\eta_{\sigma}\lambda_{\sigma} M_{\sigma}}$. By \Cref{cor:simplerconfigbound}, for all $s \leq 3\truncation$, there are at most $(\truncation)^{3s}$ many proper shapes of total size $s$. Using this, \Cref{lem:generalalphanormbound} and the assumptions on parameters, we have:
\[
\norm{L_{\leq l}} \leq 1 + \sum_{s=1}^{\truncation}{(\truncation)^{3s}} {\cu^{s}} (3\truncation)^{3\Cuniv s} n^{2\dsos - \frac{\eps}{8}s} \leq n^{2\dsos}. \qedhere
\] 
\end{proof}
\begin{corollary}\label{cor:roughlargeLnormbound}
Assume $\truncation\geq \max\{\frac{100}{\eps}\dsos,\ 2k\log n\}$ and $\cu (3\truncation)^{6\Cuniv} \leq n^{\frac{\eps}{20}}$.Then for all $l \in [3\truncation]$, 
\[
\norm{L - L_{\leq l}} \leq n^{2\dsos - \frac{\eps}{16}l}.
\]
\end{corollary}
\begin{proof}
Recall that $L - L_{\leq l} = \sum_{\sigma: \sigma \text{ is a left shape, } l < \totalsize(\sigma) \leq \truncation}{\eta_{\sigma}\lambda_{\sigma} M_{\sigma}}$. By \Cref{cor:simplerconfigbound}, for all $s \leq 3\truncation$, there are at most $(\truncation)^{3s}$ many proper shapes of total size $s$. 
Using this, 
\Cref{lem:generalalphanormbound}, the fact that $\Cuniv\geq 1$ and the assumptions on parameters, we have that
\[
\norm{L - L_{\leq l}} \leq \sum_{s=l+1}^{\truncation}
{(\truncation)^{3s}} \cu^{s} (3\truncation)^{3\Cuniv s} n^{2\dsos - \frac{\eps}{8}s} \leq n^{2\dsos - \frac{\eps}{16}l}.\qedhere
\]
\end{proof}

\subsubsection{Bounds on the contributions from intersection and product configurations}
\begin{lemma}[Combinatorial factor $N(\mathcal{P})$ in  configurations]\label{lem:NP}
For every intersection or product configuration $\mathcal{P}$,
\begin{equation}\label{eq:bound_NP}
N(\mathcal{P})\leq (2\totalsize(\mathcal{P}))^{2\totalsize(\mathcal{P})}.
\end{equation}
\end{lemma}
\begin{proof}
Fix a ribbon $R$ whose shape is the shape resulting from $\mathcal{P}$ ($\tau_{\mathcal{P}}$ if $\mathcal{P}$ is an intersection configuration and $\alpha_{\mathcal{P}}$ if $\mathcal{P}$ is a product configuration). To bound $N(\mathcal{P})$, we only need to specify:
\begin{itemize}
    \item A set of isolated vertices in $V(\mathcal{P})$. This costs at most $|V(\mathcal{P})|$ many bits. 
    \item For each non-isolated vertex $v\in V(\mathcal{P})$, which vertex in $V(R)$ it goes to. This costs at most $|V(\mathcal{P})| \cdot \log(|V(R)|)$ many bits. Notice that this induces a complete specification for each edge $e\in E(\mathcal{P})$ as of where it goes in $E(R)$.
    \item The coefficient of $R$ in the linearization process of the improper shape resulting from $\mathcal{P}$. By \Cref{prop:expansion}, this can be bounded by $(w\left(E(\mathcal{P})\right) \cdot |V(\mathcal{P})|)^{w\left(E(\mathcal{P})\right)}$ in absolute value, where the bound on the exponent is from the fact that there are at most that many steps of expanding multi-edges. 
\end{itemize}
Together, we have $|N(\mathcal{P})|\leq 2^{|V(\mathcal{P})|} |V(R)|^{V(\mathcal{P})} \left( w\left( E(\mathcal{P})\right) \cdot |V(\mathcal{P})|\right)^{w\left(E(\mathcal{P})\right)} \leq (2\totalsize(\mathcal{P}))^{2\totalsize(\mathcal{P})}$.
\end{proof}

\begin{lemma}\label{lem:intersectionconfigurationcoefficientbound}
Assume $\truncation\geq 2k\log n$. For every intersection configuration $\mathcal{P}$ such that $\totalsize(\mathcal{P}) \leq 3\truncation$, 
\begin{equation}\label{eq:intersectionconfigurationcoefficientbound}
\abs{\eta_{\mathcal{P}}} N(\mathcal{P})
\frac{\norm{M_{\tau_{\mathcal{P}}}}
}{\Anorm(\resultP)} \leq \cu^{\totalsize(\mathcal{P})}(6\truncation)^{5\Cuniv\totalsize(\mathcal{P})}.
\end{equation}
\end{lemma}
\begin{proof}
By definition \ref{def:configcoeff},
\begin{equation}\label{eq:eta_tau_P}
\eta_{\mathcal{P}} = \left(\prod_{i=1}^{j}{\eta_{\gamma_i}}\right)\eta_{\tau}\left(\prod_{i=1}^{j}{\eta_{{\gamma'}^{\top}_i}}\right)\prod_{\text{multi-edges } e \in E(\tau_{\mathcal{P}})}{(\text{coefficient of } l_e \text{ when linearizing } e)}.
\end{equation}
We now make the following observations:
\begin{enumerate}
    \item We have that $\left(\prod_{i=1}^{j}{\eta_{\gamma_i}}\right)\eta_{\tau}\left(\prod_{i=1}^{j}{\eta_{{\gamma'}^{\top}_i}}\right) \leq \cu^{\totalsize(\mathcal{P})}$. 
    \item To handle the factor $\prod_{\text{multi-edges } e \in E(\tau_{\mathcal{P}})}{(\text{coefficient of } l_e \text{ when linearizing } e)}$ in \eqref{eq:eta_tau_P}, we use \Cref{cor:expansionbounds}, which in particular implies that this factor can be absorbed into the norm bound of \Cref{thm:norm_control}. Following similar reasoning as in \Cref{lem:generalalphanormbound},
    \begin{align*}
    \frac{\norm{M_{\tau_{\mathcal{P}}}}}{\Anorm(\resultP)}\prod_{\text{multi-edges } e \in E(\tau_{\mathcal{P}})}{\abs{(\text{coefficient of } l_e \text{ when linearizing } e)}} &\leq (2k\cdot\totalsize(\alpha)^2\log n)^{\Cuniv\totalsize(\alpha)}\\ 
    &\leq (3\truncation)^{3\Cuniv\totalsize(\mathcal{P})}.
    \end{align*}
    \item By \Cref{lem:NP}, $N(\mathcal{P}) \leq (2\totalsize(\mathcal{P}))^{2\totalsize(\mathcal{P})}$.
\end{enumerate}
Multiplying the three bounds, we get \eqref{eq:intersectionconfigurationcoefficientbound}. 
\end{proof}

We have a similar bound for product configurations.

\begin{lemma}\label{lem:productconfigurationnormadjustment}
Assume $\truncation\geq 2k\log n$. For every product configuration $\mathcal{P}$ such that $\totalsize(\mathcal{P}) \leq 2\truncation$, 
\begin{equation}\label{eq:productconfigurationnormadjustment}
\abs{\eta_{\mathcal{P}}} N(\mathcal{P})
\frac{\norm{M_{\alpha_{\mathcal{P}}}}}{\Anorm(\resultP)} \leq (4\truncation)^{5\Cuniv\totalsize(\mathcal{P})}.
\end{equation}
\end{lemma}
\begin{proof}
Following the same reasoning we used to prove \Cref{lem:intersectionconfigurationcoefficientbound}, we have: 
\begin{equation}
\begin{aligned}
\abs{\eta_{\mathcal{P}}}
\frac{\norm{M_{\alpha_{\mathcal{P}}}}}{\Anorm(\resultP)} &= \frac{\norm{M_{\alpha_{\mathcal{P}}}}}{\Anorm(\resultP)}\prod_{\text{multi-edges } e \in E(\alpha_{\mathcal{P}})}{\abs{(\text{coefficient of } l_e \text{ when linearizing } e)}} \\
&\leq (3\truncation)^{3\Cuniv\totalsize(\mathcal{P})}  
\end{aligned}
\end{equation}
and
\begin{equation}
N(\mathcal{P}) \leq (2\totalsize(\mathcal{P}))^{2\totalsize(\mathcal{P})}.
\end{equation}
Multiplying these two bounds, we get \eqref{eq:productconfigurationnormadjustment}.
\end{proof}

\subsection{Bounding Error Terms}\label{sec:error_terms}

We now bound all the error terms. 

\begin{lemma}[Truncation error in canonical product]\label{lem:Mcan}
Assume $\truncation\geq \max\{\frac{100}{\eps}\dsos,\ 2k\log n\}$ and $\cu (3\truncation)^{6\Cuniv} \leq \frac{1}{3} n^{\frac{\eps}{8}}$. Then
\begin{equation}\label{eq:err_in_can}
\norm{M - [L,Q_0,L^{\top}]_{\can}}\leq {n}^{-\frac{\eps}{50}\truncation}.
\end{equation}
\end{lemma}

\begin{proof} 
By definition, $M$ is $[L,Q_0,L^{\top}]_{\can}$ minus the terms $\eta_{\alpha}\lambda_{\alpha}M_{\alpha}$ where $\totalsize(\alpha) > \truncation$. Note that $\totalsize(\alpha) \leq 3\truncation$ for all such terms. By \Cref{lem:generalalphanormbound}, for each such $\alpha$,
\[
\norm{\eta_{\alpha} \lambda_{\alpha} M_\alpha} \leq \cu^{\totalsize(\alpha)}(3\truncation)^{3\Cuniv\totalsize(\alpha)}n^{2\dsos - \frac{\eps}{8}\totalsize(\alpha)}.
\] 
By \Cref{cor:simplerconfigbound}, for all $s \leq 3\truncation$, there are at most $(\truncation)^{3s}$ many proper shapes having total size $s$. Combining these bounds and using the assumptions on the parameters, similarly to the proof of \Cref{cor:roughlargeLnormbound}, we have: 
\[
\norm{M - [L,Q_0,L^{\top}]_{\can}} \leq \sum_{s = \truncation + 1}^{3\truncation}{\cu^{s}(3\truncation)^{3\Cuniv s + 3s}n^{2\dsos - \frac{\eps}{8}s}} \leq n^{-\frac{\eps}{50}\truncation}.  
\qedhere
\]
\end{proof}

\begin{lemma}[Truncation error in factorization]\label{lem:LQLT_truncation}
Assume $\truncation \geq \max\{50{\dsos}^2, \frac{500}{\eps}\dsos, 2k\log n\}$ and $\cu (6\truncation)^{8\Cuniv} \leq \frac{1}{2} n^{\frac{\eps}{24}}$. Then
\begin{equation}\label{eq:lem:LQLT_truncation}
\norm{[L,Q_0,L^{\top}]_{\can} - LQL^{\top}} \leq 2n^{-\frac{\eps}{100}\truncation}.
\end{equation}
\end{lemma}

\begin{proof}
Recall that
\begin{align*}
&[L,Q_0,L^{\top}]_{\can} - LQL^{\top} \\
&= \sum_{j=1}^{2D}{(-1)^{j}
\sum_{\mathcal{P} \in \intsetoflength{j}:\textrm{ $l_{\mathcal{P} } \geq |U_{\resultP}|$, $r_{\mathcal{P} } \geq |V_{\resultP}|$}}
{\eta_{\mathcal{P}}{\lambda_{\mathcal{P}}}N(\mathcal{P})\left(L_{\leq l_{\mathcal{P} }}M_{\resultP}L^{\top}_{\leq r_{\mathcal{P} }} - {L}M_{\resultP}L^{\top}\right)}}
\end{align*} 
where 
\[
L_{\leq l_{\mathcal{P} }}M_{\resultP}L^{\top}_{\leq r_{\mathcal{P} }} - {L}M_{\resultP}L^{\top} = (L_{\leq l_{\mathcal{P} }} - L)M_{\resultP}L^{\top}_{\leq r_{\mathcal{P} }} + LM_{\resultP}(L^{\top}_{\leq r_{\mathcal{P} }} - L^{\top}).
\]
For these terms, we have that
\begin{enumerate}
    \item By \Cref{lem:intersectionconfigurationcoefficientbound} and \Cref{thm:erroranalysis}, noting that if $\totalsize(\mathcal{P}) \geq 12\dsos^2$ then $\mathcal{P}$ cannot be well-behaved,
    \begin{equation}\label{eq:eta-lambda-N-M}
    \begin{aligned}
    \norm{\eta_{\mathcal{P}}{\lambda_{\mathcal{P}}}N(\mathcal{P})M_{\resultP}} &= \left(|\eta_{\mathcal{P}}|N(\mathcal{P})\frac{\norm{M_{\resultP}}}{\Anorm(\resultP)}\right)\lambda_{\mathcal{P}}\Anorm(\resultP) \\
    &\leq \cu^{\totalsize(\mathcal{P})}(6\truncation)^{5\Cuniv\totalsize(\mathcal{P})}n^{\frac{\eps}{12}\cdot\min\{0,\ -\left(\totalsize(\mathcal{P})-12\dsos^2\right)\}}.
    \end{aligned}
    \end{equation}
    Moreover, by \Cref{lem:configcount}, for all $s \in [3\truncation]$, there are at most $(\truncation)^{3s}$ intersection configurations $\mathcal{P}$ of total size $s$.
    \item By \Cref{cor:roughLnormbound}, $\norm{L} \leq n^{2\dsos}$ and $\norm{L^{\top}_{\leq r_{\mathcal{P} }}} \leq n^{2\dsos}$.
    \item By \Cref{cor:roughlargeLnormbound}, $\norm{L_{\leq l_{\mathcal{P} }} - L} \leq n^{2\dsos - \frac{\eps}{16}l_{\mathcal{P}}} \leq n^{2\dsos - \frac{\eps}{24}l_{\mathcal{P}}}$ and $\norm{L^{\top}_{\leq r_{\mathcal{P} }} - L^{\top}} \leq n^{2\dsos - \frac{\eps}{16}r_{\mathcal{P}}} \leq n^{2\dsos - \frac{\eps}{24}r_{\mathcal{P}}}$.

    \item By the definitions of $l_{\mathcal{P}}$ and $r_{\mathcal{P}}$, $l_{\mathcal{P}} + \totalsize(\mathcal{P}) \geq \truncation$ and $r_{\mathcal{P}} + \totalsize(\mathcal{P}) \geq \truncation$.
\end{enumerate}
Combining these bounds, 
\begin{align*}
[L,Q_0,L^{\top}]_{\can} - LQL^{\top} 
&\leq 2n^{4\dsos}\sum_{s=1}^{3\truncation}{(\truncation)^{3s}}\cu^{s}(6\truncation)^{5{\Cuniv}s}n^{-\frac{\eps}{24}(\truncation - s) - \frac{\eps}{12}(s- 12\dsos^2)} \\
&\leq 2n^{4\dsos + {\eps}{\dsos}^2 - \frac{\eps}{24}\truncation} \leq 2n^{-\frac{\eps}{100}\truncation}.\qedhere
\end{align*}
\end{proof}

We now analyze the errors in our approximation of $Q$. For this, the following bound is useful.

\begin{lemma}\label{lemma:reductioneffect}
Assume $8\truncation^2\leq n$. Given $\gamma_j,\ldots,\gamma_1,\tau,{\gamma'}^{\top}_1,\ldots,{\gamma'}^{\top}_j$ and a well-behaved $\mathcal{P} \in \intset$ such that $\totalsize(\mathcal{P}) \leq 3\truncation$, 
\[
\norm{{\lambda_{\mathcal{P}}}M_{\tau_{\mathcal{P}}} - \lambda_{\mathcal{P},\reduced}\frac{M_{\tau_{\mathcal{P},\reduced}}}{|\Iso(\tau_{\mathcal{P}})|!}} \leq \frac{10\truncation^2}{n}\norm{{\lambda_{\mathcal{P}}}M_{\tau_{\mathcal{P}}}},
\]
where $\Iso(\tau_{\mathcal{P}})$ is the number of isolated vertices in $\tau_{\mathcal{P}}$. Similarly, for all good SSDs $\alpha_1$ and $\alpha_2$ which having both left and right indices at most $\dsos$, and all product configurations $\mathcal{P} \in \mathcal{P}_{\alpha_1,\alpha_2}$,
\[
\norm{{\lambda_{\mathcal{P}}}M_{\alpha_{\mathcal{P}}} - \lambda_{\mathcal{P},\reduced}\frac{M_{\alpha_{\mathcal{P},\reduced}}}{|\Iso(\alpha_{\mathcal{P}})|!}} \leq \frac{10\truncation^2}{n}\norm{{\lambda_{\mathcal{P}}}M_{\alpha_{\mathcal{P}}}}.
\]
\end{lemma}
\begin{proof}
We only prove the first statement as the second statement can be proved in the same way. By \Cref{prop:reduction}, given $\gamma_j,\ldots,\gamma_1,\tau,{\gamma'}^{\top}_1,\ldots,{\gamma'}^{\top}_j$ and a well-behaved $\mathcal{P} \in \intset$, letting $t := |\Iso(\tau_{\mathcal{P}})|$, 
\[
{\lambda_{\mathcal{P}}}M_{\tau_{\mathcal{P}}} = {\lambda_{\mathcal{P},\reduced}}\left(\frac{\Prod_{i=0}^{t-1}{(n - |V_{\square}(\tau_{\mathcal{P},\reduced})| - i)}}{n^{t}}\right)\frac{M_{\tau_{\mathcal{P},\reduced}}}{t!}
\]
which implies 
\[
\lambda_{\mathcal{P},\reduced}\frac{M_{\tau_{\mathcal{P},\reduced}}}{|\Iso(\tau_{\mathcal{P}})|!} - {\lambda_{\mathcal{P}}}M_{\tau_{\mathcal{P}}} = \left(\frac{n^{t}}{\Prod_{i=0}^{t-1}{(n - |V_{\square}(\tau_{\mathcal{P},\reduced})| - i)}} - 1\right){\lambda_{\mathcal{P}}}M_{\tau_{\mathcal{P}}}.
\]
Since $|V_{\square}(\tau_{\mathcal{P},\reduced})| + t \leq \totalsize(\mathcal{P}) \leq 3\truncation$, it holds $\Prod_{i=0}^{t-1}(\frac{n-|V_{\square}(\tau_{\mathcal{P},\reduced})| - i}{n})\geq 1-\frac{4\truncation^2}{n}$ and so
\[
\frac{n^{t}}{\Prod_{i=0}^{t-1}{n - |V_{\square}(\tau_{\mathcal{P},\reduced})| - i}} - 1 
\leq (1-\frac{4\truncation^2}{n})^{-1} -1 \leq \frac{10\truncation^2}{n}.\qedhere
\]
\end{proof}

\begin{lemma}[Approximation of $Q$ by well-behaved part]\label{lem:error_QSSD}
Assume $\truncation \geq \max\{100\dsos,\ 2k\log n\}$ and  
\begin{equation}\label{eq:Qwellbehaved_condition}
\cu^{18\dsos^2} (10\truncation)^{200 \Cuniv \dsos^2} < n^{\frac{\eps}{30}}.
\end{equation}
Then 
\begin{equation}\label{eq:err_in_wbp}
\norm{Q - [Q]_{\wellbehaved}} < n^{-\frac{\eps}{20}}.
\end{equation}
\end{lemma}
\begin{proof}
Recall that 
\begin{align}\label{eq:Q_and_QSSD}
Q &= Q_0 + \sum_{j=1}^{2D}(-1)^{j}
\sum_{\mathcal{P} \in \intsetoflength{j}:\text{ $l_{\mathcal{P} } \geq |U_{\resultP}|$, $r_{\mathcal{P} } \geq |V_{\resultP}|$}}{\eta_{\mathcal{P}}{\lambda_{\mathcal{P}}}M_{\resultP}}\\
[Q]_{\wellbehaved} &= Q_0 + \sum_{j=1}^{2D}(-1)^{j}\sum_{\substack{\mathcal{P} \in \intsetoflength{j}:\\\text{$\mathcal{P}$ is well-behaved, $l_{\mathcal{P} } \geq |U_{\resultP}|$, $r_{\mathcal{P} } \geq |V_{\resultP}|$}}}{\eta_{\mathcal{P}}{\lambda_{\mathcal{P},\reduced}}N(\mathcal{P})\frac{M_{\tau_{\mathcal{P},\reduced}}}{|\Iso(\resultP)|!}}
\end{align}
Then $\norm{Q-[Q]_{\wellbehaved}}\leq A + B$ where 
\begin{align}
A&:=\Big\Vert{\Sum_{j=1}^{2D}(-1)^{j}\Sum_{\substack{\mathcal{P} \in \intsetoflength{j}:\\ \text{$\mathcal{P}$ is not well-behaved,}\\ l_{\mathcal{P}} \geq |U_{\resultP}|,\ r_{\mathcal{P}} \geq |V_{\resultP}|}}{\eta_{\mathcal{P}}{\lambda_{\mathcal{P}}}N(\mathcal{P})M_{\resultP}}}\Big\Vert,\label{eq:wellbehaved,A}\\ 
B&:=\Big\Vert{\Sum_{j=1}^{2D}(-1)^{j}\Sum_{\substack{\mathcal{P} \in \intsetoflength{j}:\\ \text{$\mathcal{P}$ is well-behaved,}\\ l_{\mathcal{P} } \geq |U_{\resultP}|,\ r_{\mathcal{P} } \geq |V_{\resultP}|}}{\eta_{\mathcal{P}}N(\mathcal{P})\left({\lambda_{\mathcal{P}}}M_{\resultP} - {\lambda_{\mathcal{P},\reduced}}\frac{M_{\tau_{\mathcal{P},\reduced}}}{|\Iso(\resultP)|!}\right)}}\Big\Vert.
\label{eq:wellbehaved,B}
\end{align}
Below, we show that $A < \frac{1}{2} n^{-\frac{\eps}{20}}$ and $B < \frac{1}{2} n^{-\frac{\eps}{20}}$, and the lemma follows by adding the two.

To upper bound $A$, we observe that for each $\mathcal{P}$ that appears in $A$, denoting $s=\totalsize(\mathcal{P})$: 
\begin{enumerate}
    \item $s\leq (2\dsos+1)\truncation \leq 3\dsos\truncation$.
    \item $\abs{\eta_{\mathcal{P}}} N(\mathcal{P})\frac{\norm{M_{\tau_{\mathcal{P}}}}}{\Anorm(\resultP)} \leq {\cu^s} (6\truncation)^{5\Cuniv s}$ by \Cref{lem:intersectionconfigurationcoefficientbound}.
    \item $\Vert\lambda_{\mathcal P}\Anorm(\resultP)\Vert\leq n^{-\max\left\{\frac{\eps}{8},\ \frac{\eps}{12}(s-\dside^2)\right\}}$ by \Cref{thm:erroranalysis} as $\mathcal P$ is not well-behaved. 
\end{enumerate}
Let $s_0:=12\dsos^2$. Using the above observations and the fact that there are at most $s^{3s}$ intersection configurations with total size $s$ (\Cref{cor:simplerconfigbound}), we have:
\begin{equation}\label{eq:bound_A,sum}
\begin{aligned}
A&\leq\Sum_{\substack{\text{$\mathcal{P}$ appearing in \eqref{eq:wellbehaved,A}},\\ 
\totalsize(\mathcal{P})\leq s_0}}\norm{\eta_{\mathcal{P}} N(\mathcal{P}) \lambda_{\mathcal{P}} M_{\resultP}} + \Sum_{\substack{\text{$\mathcal{P}$ appearing in \eqref{eq:wellbehaved,A},}\\ \totalsize(\mathcal{P}) > s_0}} \norm{\eta_{\mathcal{P}} N(\mathcal{P}) \lambda_{\mathcal{P}} M_{\resultP}}\\
& \leq \Sum_{s=1}^{s_0} s^{3s} \cu^{s} (6\truncation)^{5\Cuniv s} n^{-\frac{\eps}{8}} + \Sum_{s=s_0+1}^{3\dsos\truncation} s^{3s} {\cu^s} (6\truncation)^{5\Cuniv s} n^{-\eps\cdot\frac{s-s_0}{12}}\\
&\leq 2{s_0}^{3s_0} \cu^{s_0} (6\truncation)^{5\Cuniv s_0} n^{-\frac{\eps}{8}} + 2\left((3{\dsos}\truncation)^{3} {\cu} (6\truncation)^{5\Cuniv} \right)^{s_0 + 1}n^{-\frac{\eps}{12}}\\
&<\frac{1}{2}n^{-\frac{\eps}{20}}
\end{aligned}
\end{equation}
where in the last step we used $\dsos<\frac{\truncation}{100}$ and $\cu^{18\dsos^2} (10\truncation)^{200 \Cuniv \dsos^2} < n^{\frac{\eps}{30}}$ by \eqref{eq:Qwellbehaved_condition}. 

To upper bound $B$, given $\mathcal{P}$ in \eqref{eq:wellbehaved,B}, denote again $s:=\totalsize(\mathcal{P})$. By \Cref{lemma:reductioneffect} and the three observations above (where we note that $\Anorm(\mathcal{P})=1$ by \Cref{thm:erroranalysis}), 
\begin{equation}\label{eq:bound_wellbehaved,individual}
\Norm{{\eta_{\mathcal{P}}N(\mathcal{P})\left({\lambda_{\mathcal{P}}}M_{\resultP} - {\lambda_{\mathcal{P},\reduced}}\frac{M_{\tau_{\mathcal{P},\reduced}}}{|\Iso(\resultP)|!}\right)}}\leq \frac{10\truncation^2}{n} \cu^s (6\truncation)^{5\Cuniv s}.
\end{equation}
Note that in every well-behaved intersection configuration $\mathcal{P}$ appearing in $B$, each SSD shape is good and so has total size at most $6\dsos$, so $\totalsize(\mathcal P)\leq(2\dsos+1)\cdot 6\dsos\leq 18\dsos^2$. Summing over all $\mathcal{P}$ in \eqref{eq:wellbehaved,B} and using \Cref{cor:simplerconfigbound}, we have:
\begin{equation}\label{eq:bound_B,sum}
B \leq \Big(\Sum_{s=1}^{18\dsos^2} s^{3s}\Big) \cu^{18\dsos^2} (10\truncation)^{92 \Cuniv \dsos^2}n^{-1} < \cu^{18\dsos^2}(10\truncation)^{200 \Cuniv \dsos^2}n^{-1} < \frac{1}{2}n^{-\frac{\eps}{20}}
\end{equation}
where we used $\Cuniv\geq 1$ and $\truncation\geq 100\dsos$ for the second to last step and \eqref{eq:Qwellbehaved_condition} for the last step.
\end{proof}

\begin{lemma}[Approximation of good SSD products]\label{lem:alphazeroerror}
Assume $\truncation\geq\max\{100\dsos,\ 2k\log n\}$. If $\alpha,\beta$ are linear combinations of (scaled) good SSD shapes having both left and right indices of size at most $\dsos$, and $\linfty{\alpha}\leq A$ and  $\linfty{\beta}\leq B$ for some $A,B>0$, then
\begin{equation}\label{eq:alphazeroerror,SSD}
\norm{\alpha \cdot \beta - \alpha\wbp\beta} \leq AB\truncation^{20 \Cuniv \dsos} n^{-\frac{\eps}{12}}. 
\end{equation}
Similarly, if $\alpha,\beta\in\SALD$ are good and $\linfty{\alpha}\leq A$ and  $\linfty{\beta}\leq B$, then
\begin{equation}\label{eq:alphazeroerror,SS}
\norm{\SS{(\alpha \cdot \beta)} - \alpha\star\beta} \leq AB\truncation^{20 \Cuniv \dsos} n^{-\frac{\eps}{12}}. 
\end{equation}
\end{lemma}
\begin{proof}
We only prove \eqref{eq:alphazeroerror,SSD} as the proof of \eqref{eq:alphazeroerror,SS} is similar. Denote by $c_{\alpha}(\cdot)$, $c_\beta(\cdot)$ the coefficient of SSD shapes in $\alpha$, $\beta$ respectively. Then  
\begin{equation}\label{eq:approx_good_SSD}
\begin{aligned}
\alpha \cdot \beta - \alpha\wbp\beta &= \Sum_{\substack{\gamma_1,\gamma_2:\ \\ \gamma_1\in\alpha,\ \gamma_2\in\beta}} c_{\alpha}(\gamma_1)\cdot c_{\beta}(\gamma_2)\Sum_{\substack{\mathcal P\in\mathcal{P}_{\gamma_1,\gamma_2}:\ \\\text{$\mathcal P$ is not well-behaved}}} \eta_{\mathcal{P}}N(\mathcal{P})\lambda_{\mathcal{P}}M_{\resultP} \\
&+ \Sum_{\substack{\gamma_1,\gamma_2:\ \\ \gamma_1\in\alpha,\ \gamma_2\in\beta}} c_{\alpha}(\gamma_1)\cdot c_{\beta}(\gamma_2)\Sum_{\substack{\mathcal P\in\mathcal{P}_{\gamma_1,\gamma_2}:\ \\\text{$\mathcal P$ is well-behaved}}} {\left(\eta_{\mathcal{P}}N(\mathcal{P})\left({\lambda_{\mathcal{P}}}M_{\resultP} - {\lambda_{\mathcal{P},\reduced}}\frac{M_{\tau_{\mathcal{P},\reduced}}}{|\Iso(\resultP)|!}\right)\right)}.
\end{aligned}
\end{equation}
Since all $\gamma_1,\gamma_2$ are good, by \Cref{thm:producterroranalysis}, \Cref{lem:productconfigurationnormadjustment}, and \Cref{lemma:reductioneffect} we have that 
\begin{enumerate}
    \item For each $\mathcal{P}$ in the above sum which is not well-behaved, $\totalsize(\mathcal{P})\leq 2\truncation$ and 
    \[
    \norm{\eta_{\mathcal{P}}N(\mathcal{P})\lambda_{\mathcal{P}}M_{\resultP}} \leq (4\truncation)^{5\Cuniv\totalsize(\mathcal{P})} n^{-\frac{\eps}{12}}.
    \]
    \item For each $\mathcal{P}$ in the above sum which is well-behaved, $\totalsize(\mathcal{P})\leq 2\truncation$ and 
    \[
    \norm{\eta_{\mathcal{P}}N(\mathcal{P})\left({\lambda_{\mathcal{P}}}M_{\resultP} - {\lambda_{\mathcal{P},\reduced}}\frac{M_{\tau_{\mathcal{P},\reduced}}}{|\Iso(\resultP)|!}\right)} \leq \frac{10{\truncation}^2}{n}(4\truncation)^{5\Cuniv\totalsize(\mathcal{P})}.
    \]
\end{enumerate}
The number of pairs $(\gamma_1,\gamma_2)$ is at most $5^{4\dsos}$ as there can be at most $5^{2\dsos}$ many SSD shapes in $\alpha,\beta$ (\Cref{claim:numberofSSD}). By \Cref{cor:simplerconfigbound}, for each pair $(\gamma_1,\gamma_2)$, The size of the set $\mathcal{P}_{\gamma_1,\gamma_2}$ is at most $(2\dsos)^{6\dsos}$. Plugging these bounds and $|c_\alpha(\gamma_1)\cdot c_{\beta}(\gamma_2)|\leq AB$ into \eqref{eq:approx_good_SSD}, we have that 
\[
\norm{\alpha \cdot \beta - \alpha\wbp\beta} \leq 
2AB (5^{4\dsos}) (2\dsos)^{6\dsos} (4\truncation)^{10 \Cuniv \dsos} n^{-\frac{\eps}{12}} < AB\truncation^{20\Cuniv \dsos} n^{-\frac{\eps}{12}},
\]
where we used that $\Cuniv\geq 1$ and $\truncation\geq 100\dsos$. 
\end{proof}
        \section{Putting Things Together: Proof of Theorem~\ref{thm:main-formal}}
\label{sec:puttogether}

The last ingredient for proving \Cref{thm:main-formal} is that $LL^{\top}$ is not too close to being singular. 

\begin{lemma}\label{lem:LLT}
Suppose $\truncation\geq \log n$ and $C_U(3\dsos)^{6\Cuniv}<n^{\frac{\eps}{30}}$. Then 
\begin{equation}\label{eq:Lnorm}
LL^\top \succ (9\truncation^2n)^{-\dsos}\Id.
\end{equation}
\end{lemma}
\begin{proof}
Define a scaled matrix $L':=\diag\left(n^{-|I|/2}\right)_{I\in\binom{[n]}{\leq D}}\cdot L \cdot \diag\left(n^{|I|/2}\right)_{I\in\binom{[n]}{\leq D}}$. By definition, 
$L'=
\begin{pmatrix}
L'_{0,0}  &          &    &   \\
L'_{1,0}  & L'_{1,1} &    & 0 \\
\vdots    &  \ddots  &    &   \\
          &  \dots   &    & L'_{D,D}
\end{pmatrix}$ 
which is a block-lower-triangular matrix where $L'_{i,j}$ is supported on $\binom{[n]}{i}\times\binom{[n]}{j}$. Since $L'$ contains only left shapes, $L_{i,i}=L'_{i,i}=\Id_{\binom{[n]}{i}}$ for all $i$. 

We first upper bound $\norm{L_{i,j}}$ and $\norm{L'_{i,j}}$ for all $0\leq j\leq i\leq \dsos$. For each shape $\alpha$ in $L_{i,j}$ of total size $s$, since $\alpha$ is left and proper, we can use the same edge-factor assignment as for middle shapes (\Cref{def:edgefactorredistribution}) and \Cref{thm:norm_control} to get: 
\begin{equation}\label{eq:left_shape_bound}
    \norm{\eta_{\alpha} \lambda_{\alpha} M_\alpha}\leq \cu^{s'}(3\truncation)^{\Cuniv\cdot s'} n^{\frac{(i-j)}{2}-\frac{\eps}{8}\cdot w(E(\alpha))},
\end{equation}
where we have used $\truncation\geq\log n$ and have taken $s':=\totalsize(\alpha)-|\{v \in U_\alpha\cap V_\alpha: \deg(v) = 0\}|$, which we call the {\it essential size} of $\alpha$. Since $\alpha$ has no degree $0$ vertices outside of $U_\alpha \cap V_\alpha$, it holds that $s'  \leq w(E(\alpha)) + 2w(E(\alpha)) = 3w(E(\alpha))$. This and \eqref{eq:left_shape_bound} imply that $\norm{\eta_{\alpha} \lambda_{\alpha} M_\alpha}\leq \cu^{s'} (3\truncation)^{3\Cuniv s'} n^{\frac{(i-j)}{2}-\frac{\eps}{24}s'}$. As there are at most $(s')^{3s'}$ shapes in $L$ having essential size $s'$ and total size $s$ (by the same proof of \Cref{lem:configcount}), we have: 
\begin{equation}\label{eq:bound_Lij}
\norm{L_{i,j}}\leq \Sum_{s=1}^{\truncation}\Sum_{s'=0}^s s^{3s'}\cu^{s'} (3\truncation)^{3\Cuniv s'} n^{\frac{i-j}{2}-\frac{\eps}{24}s'}<n^{\frac{i-j}{2}}\Sum_{s=1}^{\truncation} 2 = 2\truncation n^{\frac{i-j}{2}}, 
\end{equation}
where we used $\Cuniv\geq 1$ and $C_U(3\dsos)^{6\Cuniv}<n^{\frac{\eps}{30}}$. 
Recall that $L'_{i,j}= n^{-\frac{i-j}{2}}L_{i,j}$, so by \eqref{eq:bound_Lij}, 
\begin{equation}\label{eq:boudnd_L'ij}
    \norm{L'_{i,j}}\leq 2\truncation.
\end{equation}

Next, we upper bound $\norm{(L')^{-1}}$ and $\norm{L^{-1}}$. 
Denote $A_i:= L'\upharpoonright_{\text{blocks $\{(a,b)\mid a,b\leq i\}$}}$. 
Then $\norm{A_0^{-1}}=1$, $A_{i+1}\!=\!
\begin{pmatrix}
    A_i         &   0    \\
    L'_{i+1,i}  &   \Id  
\end{pmatrix}$,   
$A^{-1}_{i+1}\!=\!
\begin{pmatrix}
    A^{-1}_i             &   0    \\
    -L'_{i+1,i}A_i^{-1}  &   \Id  
\end{pmatrix}$, 
and by \eqref{eq:boudnd_L'ij} we have
\begin{equation}\label{eq:L'inv}
\norm{A^{-1}_{i+1}}\!\leq\! \norm{A^{-1}_i}\cdot (1+\norm{L'_{i+1,i}})+1\!< 2\truncation\norm{A^{-1}_i}+1.
\end{equation}
By \eqref{eq:L'inv} and induction, we have the bound $\norm{(L')^{-1}}=\norm{A_D^{-1}} < (3\truncation)^{\dsos}$. We now recall that $L^{-1}=\diag\left(n^{\frac{|I|}{2}}\right)_{I\in\binom{[n]}{\leq D}} (L')^{-1} \diag\left(n^{-\frac{|I|}{2}}\right)_{I\in\binom{[n]}{\leq D}}$, so $\norm{L^{-1}} \leq n^{\frac{\dsos}{2}}\cdot \norm{(L')^{-1}}\cdot 1 < n^{\frac{\dsos}{2}}(3\truncation)^{\dsos}$. This means $\sigma_{\min}(L)>n^{-\frac{\dsos}{2}}(3\truncation)^{-\dsos}$ and thus $LL^{\top} \succ (9\truncation^2n)^{-\dsos}\Id$.
\end{proof}

We can now prove \Cref{thm:main-formal}. For convenience, we restate \Cref{thm:main-formal} below. 

\mainformal*

\begin{proof} 
For any $\delta>0$, if $n$ is large enough then with probability $> 1-\delta$, the norm bounds in \Cref{thm:norm_control}, \Cref{cor:expansionbounds} hold for all relevant shapes. Below we assume this is the case. 

By assumption \eqref{eq:main_condition}, we have:
\begin{itemize}
    \item $M=LQL^{\top}+E$ where $\Vert E\Vert\leq 3n^{-\frac{\eps}{100}\truncation}$ by \Cref{lem:Mcan} and \Cref{lem:LQLT_truncation}. Since $\truncation\geq \frac{500}{\eps}\dsos$, it follows that $\norm{E} \leq 3n^{-5\dsos}$.

    \item $Q \succ \left( n^{-\frac{\eps}{30}} - n^{-\frac{\eps}{20}}\right)\Id \succ \frac{1}{2}n^{-\frac{\eps}{30}} \Id$ by \Cref{lem:PDQSSD} and \Cref{lem:error_QSSD}. Here, we note that the conditions required in these two lemmas follow from condition \eqref{eq:main_condition}.

    \item $LQL^{\top} \succ (9\truncation^2 n)^{-\dsos} \left(\frac{1}{2}n^{-\frac{\eps}{30}}\right)\Id \succ n^{-1.1\dsos}\Id$ by the item above and \Cref{lem:LLT}, where for the last step we use $9\truncation^2<n^{\frac{1}{300}}$ from \eqref{eq:main_condition}.

\end{itemize}
Therefore, $M \succeq LQL^{\top} - \norm{E}\Id \succ \left(n^{-1.1\dsos} - 3n^{-5\dsos}\right)\Id \succ n^{-2\dsos}\Id$ for all large $n$.
\end{proof}

        \section{Applications to Learning Mixture Models and Robust Statistics}
\label{sec:applications}

In this section, we show that \Cref{thm:main-formal} implies strong SoS lower bounds for a range of problems in learning theory and algorithmic robust statistics. The high-level idea is as follows. We will construct specific distributions $A$ in \Cref{def:ngca_distinguishing} such that the resulting NGCA instance is a valid instance of the problem being considered, and then we apply the lower bound \Cref{thm:main-formal}. 

For \Cref{thm:main-formal} to apply, we need to check that $A$ matches $(k-1)$ moments with that of a Gaussian (where $k$ depends on the application) and to bound the quantities $\cu$ and $\cl$. It will be simpler to bound the measures $U_A(t)$ and $L_A(t)$ for all $t$, where we recall that 
\begin{align}
U_A(t)&:=\max\limits_{i\leq t}\left\vert\E\limits_A\ [h_i(x)]\right\vert \label{cond:A<'},\\
L_A(t)&:=\min\limits_{\substack{p(x):\ \deg(p)\leq t\\
\text{and } \int_{\cN(0,1)}p^2(x)=1}}\E\limits_A\ [p^2(x)]. \label{cond:A>'}
\end{align}
\noindent For this reason, we will use the following corollary of \Cref{thm:main-formal}, which follows by a simple inspection of the definitions of $U_A(t),L_A(t)$ and $\cu,\cl$ (\Cref{def:cucl}).

\begin{corollary}[SoS lower bounds in terms of $U_A(t),L_A(t)$]\label{cor:main_UALA} 
For any $\delta\in(0,1)$, if $n$ is sufficiently large then the following holds. 
    
Suppose $\eps\in(0,1)$, $k\geq 2$, and $A$ is a 1-dimensional distribution (where $\eps$, $k$, and $A$ may all depend on $n$) such that: 
\begin{align}
&\text{$A$ matches the first $k-1$ moments with $\cN(0,1)$.}\label{eq:main_moment_match_app}\\ 
&U_A(t)\leq \poly(\log n, k, t)^t\text{ and }L_A(t)\geq \poly(\log n, k, t)^{-t}\text{ for all }t\geq 1.\label{eq:main_condition_app} 
\end{align}
Set $\dsos=o(\sqrt{\frac{\eps\log n}{\log\log n}})$ and let $\truncation$ be a suitably large polynomial in $k$ and $\log n$ such that $\truncation\geq\max\{50{\dsos}^2, \frac{500}{\eps}\dsos, 2k\log n\}$. Then, if we draw $m < n^{(1-\eps)k/2}$ many i.i.d. samples from $\cN(0,\Id_n)$, with probability greater than $1-\delta$, the degree-$\dsos$ pseudo-calibration moment matrix \eqref{eq:ME} with truncation threshold $\truncation$ is positive-definite. 
\end{corollary}

In \Cref{sec:applications_toolkit}, we state and prove lemmas helpful for bounding $U_A(t)$ and $L_A(t)$ for various distributions $A$. Since $U_A(0)=L_A(0)=1$, we will focus on $t\geq 1$. 
In \Cref{sec:concrete_apps}, we apply \Cref{cor:main_UALA} to different problems in learning theory and algorithmic robust statistics. 

In this section, we denote $\hn_i(x) := h_i(x)/\sqrt{i!}$ where $h_i(x)$ is the $i$th  Probabilist's Hermite polynomial. We call a function $f(x)$ {\bf unit-norm} if $\E_{x \sim \cN(0, 1)}[f^2(x)] = 1$. 

\subsection{Technical Toolkit} \label{sec:applications_toolkit}

In this section, we state some technical lemmas required to prove bounds on $U_A(t)$ and $L_A(t)$.

The first lemma is a generic lower bound on the $l_2$-norm of polynomials under $A$.

\begin{lemma}[Generic lower bound]
\label{fact:unit_norm_lower_bound_gen}
Let $A$ be a distribution with pdf $q(x)$, and let $g(x)$ denote the pdf of $\cN(0, 1)$. Let $t$ be any natural number. Suppose $\R$ can be partitioned into two measurable sets $I_1$ and $I_2$ such that: (1) $\int_{I_1} p^2(x) g(x) dx \leq 1-\delta$ for every degree $\leq t$ unit-norm polynomial $p$, and (2) $\frac{q(x)}{g(x)} \geq \eps$ for all $x\in I_2$.  
Then 
\[
\min\limits_{\substack{p(x):\ \deg(p)\leq t\\ \Vert p(x)\Vert_{\mathcal{N}(0,1),l_2}=1}}\ \E_{x\sim A}[p^2(x)] \geq \delta \eps.\]
\end{lemma}
\begin{proof}
The lemma follows from a simple calculation: 
\begin{align}
\int_{\R} p^2(x) q(x) ~dx
&=\int_{\R} p^2(x) g(x)\frac{q(x)}{g(x)} ~dx\\
&\geq\int_{I_2} p^2(x) g(x)\frac{q(x)}{g(x)}~dx\\
&\geq\eps \int_{I_2} p^2(x) g(x) ~dx\\
&=\eps \Paren{1-\int_{I_1} p^2(x) \exp(-x^2/2) ~dx}\\
&\geq\eps \delta. \qedhere
\end{align}
\end{proof}

The following corollary of \Cref{fact:unit_norm_lower_bound_gen}, focuses on the lower bound in terms of the likelihood ratio in a large interval around 0. 

\begin{restatable}[Lower bound by density ratio around $0$]{corollary}{aroundzero}\label{cor:aroundzero}
Let $A$ be a distribution with pdf $q(x)$ and $g(x)$ be the pdf of $\cN(0, 1)$. Then for any $t\geq 1$, 
\begin{equation}\label{eq:Laroundzero}
L_A(t) \geq \min_{|x| \leq 10\sqrt{t \log (t+1)}} \left \{ \frac{q(x)}{2g(x)} \right \}.
\end{equation}
\end{restatable}

We also need the following generic upper bound on $U_A(t)$, which we will later specialize to specific distributions $A$. 

\begin{restatable}[Generic upper bound on $U_A$]{corollary}{genericUA}\label{cor:generic_U_A}
Let $A$ be a distribution whose pdf outside the region $[-B,B]$ is upper bounded by $\exp(-(|x|-C)^2/(2\sigma^2))$, where $B\geq C\geq 0$ and $B>1$. Then:   
\[
U_A(t) \leq (8t)^{t}\cdot(\sigma C^{t} + 2B^{2t} + \sigma^{t+1}t^{t/2}) \text{ for all } t \geq 1.
\]
\end{restatable}

The final corollary provide bounds on $U_A$ and $L_A$ for $A$ which a mixture of a small number of Gaussians with bounded means and variances.

\begin{restatable}[Gaussian mixture bounds]{corollary}{gaussianlb}\label{cor:gaussian_lb} 
Suppose $A = \alpha \cN(\mu, \sigma^2) + (1-\alpha) E$, where $\alpha\in (0,1)$ and $E$ is a distribution whose pdf outside $[-B, B]$ is upper bounded by twice that of $\cN(\mu', (\sigma')^2)$, where $1\leq B$ and $|\mu'|\leq B$. Then for any $t\geq 1$,
\begin{align} 
U_A(t) &\leq O\Big(t\cdot(|\mu|+|\mu'|+1)\cdot(\sigma+\sigma'+1)\cdot B\Big)^{2t},\\
L_A(t) &\geq\ \frac{\alpha}{2\sigma}
\exp\left( - \ \frac{O(t\log (t+1))+\mu^2}{\sigma^2}\right).
\end{align}
Here, the constants in the $O$ notation are independent of $t,\alpha,\mu,\mu',\sigma,\sigma',B$.
\end{restatable}

Proofs of the above corollaries are deferred to \Cref{app:applications_toolkit}. 

\subsection{Applications to Other Learning Tasks}
\label{sec:concrete_apps}

In this section, we will use \Cref{cor:main_UALA} to show that distinguishing versions of many fundamental problems in learning theory and robust statistics are hard for SoS programs. For each problem we consider, there are known SQ lower bounds (see, e.g.,~\cite{DK23-book})  and low-degree polynomial testing lower bounds (via the near-equivalence between the two models \cite{brennan2020statistical}). Our SoS lower bounds strengthen these prior results.

It is worth noting that with the exception of a single application --- learning $k$-GMMs, where reduction-based hardness was recently established (assuming subexponential hardness of LWE) ---  
the prior SQ and low-degree lower bounds were the only known evidence of hardness.  

See \Cref{tab:sample_complexity} for a summary of the problems we examine and the guarantees we obtain. We prove lower bounds for the hypothesis testing version of the problems, which are known to be efficiently reducible to the corresponding learning problem. 

Specifically, our SoS lower bounds for robust statistics tasks apply under the ``TV-corruption'' model or the ``Huber contamination'' model. In the TV-corruption model, the adversary may draw samples from any other distribution $D'$ of its choice, as long as the distribution satisfies  $d_{TV}(D, D') \leq \tau$, for some $\tau>0$ which is the proportion of contamination. 

\begin{definition}[TV-Corruption Model]\label{def:tv_corruption}
Let $\mathcal{D}$ be a set of
distributions. We define $\cB_{TV}(\tau, \cN(0, \Id_n), \cD)$ to be the following hypothesis testing problem: Given $m$ i.i.d.\ samples $ \{ x_1, \dots, x_m  \} \subseteq \R^n$ drawn from one of the following two distributions, the goal is to determine which one: (a) $\cN(0, \Id_n)$; and (b) $D'$ such that $d_{TV}(D', D) \leq \tau$ for $D$ drawn uniformly at random from $\cD$.
\end{definition}

In the Huber contamination model, the adversary cannot delete samples but may add samples of their choice. In the definition below and throughout the remaining section, we use a convex combination of distributions $c_1 D_1+\ldots + c_r D_r$ to denote the distribution where we first draw $x\in\{1,\ldots,r\}$ with probability $c_i$ for $x=i$, and then draw a sample according to $D_x$.

\begin{definition}[Huber Contamination Model]\label{def:huber_contamination}
Let $\mathcal{D}$ be a family of distributions. We define $\cB_{\huber}(\tau, \cN(0, \Id_n), \cD)$ to be the following hypothesis testing problem: Given $m$ i.i.d.\ samples $ \{ x_1, \dots, x_m  \} \subseteq \R^n$ drawn from one of the following two distributions, determine which one: (a) $\cN(0, \Id_n)$; (b) $D'$, which is $(1-\tau) D + \tau B$, where $D$ is drawn uniformly at random from $\cD$ and $B$ is an arbitrary distribution possibly dependent on $D$. 
\end{definition}

We now demonstrate SoS lower bounds for the problems considered in \Cref{tab:sample_complexity} using \Cref{cor:main_UALA}. Our list is by no means exhaustive, and there are likely several other problems for which similar lower bounds may be obtained. 

\subsubsection{Robust Mean Estimation of Bounded Covariance Gausians} 
The problem of Gaussian Robust Mean Estimation with bounded covariance is the problem of robust mean estimation, where samples are drawn from $\cN(\mu, \Sigma)$ for some $\Sigma \preceq \Id$, and up to a $\tau$ fraction of these samples may be arbitrarily corrupted. 

In this context, it is information-theoretically possible to estimate $\mu$ to an $\ell_2$-error of $O(\tau)$ using $O(n/\tau^2)$ samples; see, e.g.,~\cite{Huber09}. Lower bounds against SQ algorithms suggest that achieving an error $o(\tau^{1/2})$ in polynomial time requires $\Omega(n^2)$ samples. Such an SQ lower bound is implicit in~\cite{DKS19} and can be directly deduced from \cite{diakonikolas2022robust}. 

We show a degree $\widetilde{\Omega}_{\eps}(\sqrt{\log n})$ 
SoS lower bound against a hypothesis testing problem known to be efficiently reducible to Gaussian RME with bounded covariance, defined below: 

\begin{problem}[Hypothesis-Testing-RME with Bounded Covariance]
\label{prob:RME_bdd_cov_learn}
Hypothesis-Testing-RME with bounded covariance is the problem $\cB_\huber(\tau, \cN(0, \Id_n), \cD)$ (\Cref{def:huber_contamination}), where every $D \in \cD$ has the form $\cN(\mu_D, \Sigma_D)$ such that $\Sigma_D \preceq \Id_n$ and $\norm{\mu_D}\geq \Omega(\sqrt \tau)$ for some constant independent of $\tau,n$. 
\end{problem}
To prove the lower bound for \Cref{prob:RME_bdd_cov_learn}, we will apply \Cref{cor:main_UALA} to an instance of \Cref{def:ngca_distinguishing} where the distribution in the hidden direction is the following $A_{\text{RME}}$. 
\begin{lemma}[Lemma F.2. from \cite{DKS19}]
\label{lem:F2}
For any constant $0 < \tau < 1/2$ and $\mu = \sqrt{\tau/{10001}}$, there is a distribution $A_{\text{RME}} = (1 - \tau) \cN(\mu, 2/3) + \tau \tau_1 \cN(\mu_1, \sigma_1) + \tau (1 - \tau_1) \cN(\mu_2, \sigma_2)$ which matches its first 3 moments with $\cN(0, 1)$, where $|\mu_1|, |\mu_2| < \frac{2}{\sqrt{\tau}}$, $0.9 < \sigma_1, \sigma_2 < 1.1$ and $\tau_1=\Theta(1)$. 
\end{lemma} 

Setting $A=A_{\text{RME}}$ in \Cref{def:ngca_distinguishing}, we get an instance of \Cref{prob:RME_bdd_cov_learn} where the set of distributions is $\cD:=\{\cN(0, \Id_{n-1})_{v^\perp} \times A_v \mid v\in\{\pm1/\sqrt{n}\}^n\}$. Note that  
\begin{align*} 
\label{eq:planted_mean_dist}
\cN(0, \Id_{n-1})_{v^\perp} \times A_v 
&= (1-\tau) \cN(0, \Id_{n-1})_{v^\perp} \times \cN(\mu, 2/3)_v \\
& \qquad + \tau (1-\tau_1) \cN(0, \Id_{n-1})_{v^\perp} \times \cN(\mu_1, \sigma_1)_{v}+ \tau \tau_1 \cN(0, \Id_{n-1})_{v^\perp} \times \cN(\mu_2, \sigma_2)_{v} \\
&= (1-\tau) \cN(\mu v, \Id_{n} - \frac{1}{3} vv^T) \\
&\qquad + \tau \cdot \left( (1-\tau_1) \cN(0, \Id_{n-1})_{v^\perp} \times \cN(\mu_1, \sigma_1)_{v} + \tau_1 \cN(0, \Id_{n-1})_{v^\perp} \times \cN(\mu_2, \sigma_2)_{v} \right).
\end{align*} 
In other words, each distribution in $\cD$ has the form $(1-\tau)\cdot\cN(\sqrt{\frac{\tau}{10001}}~v, \Id_{n}\! - \! \frac{1}{3} vv^T) + \tau\cdot B$ for some $B$, thus satisfying the requirement in the model $\cB_\huber(\tau, \cN(0, \Id_n), \cD)$.

\begin{corollary}[SoS Lower Bound for Bounded Covariance Gaussian Robust Mean Estimation]
\label{thm:RME_bdd_cov}
Fix $\tau\in(0,\frac{1}{2})$. For any $\eps=\eps(n)\in (0,1)$, as the dimension $n$ increases, the SoS program described in \Cref{sec:prob_statement} with degree $o(\sqrt{\frac{\eps\log n}{\log\log n}})$ cannot solve \Cref{def:ngca_distinguishing} given fewer than $n^{2(1-\eps)}$ samples, where $A: = A_{\text{RME}}$ with $k:=4$.

In particular, if the samples are drawn from $\cN(0, \Id)$, the program cannot rule out the existence of a hidden direction $v$ along which the input distribution is $A_{\text{RME}}$.
\end{corollary}

\begin{proof}
Given that \Cref{lem:F2} shows that $A$ matches the first 3 moments of $\cN(0, 1)$, our lemma directly follows by applying \Cref{cor:main_UALA} after verifying that the hypotheses hold. Since $A$ consists of a three-component Gaussian mixture with bounded means and variances, satisfying the prerequisites for \Cref{cor:gaussian_lb}, $U_A$ and $L_A$ are appropriately bounded, and the conclusion of \Cref{cor:main_UALA} follows. 
\end{proof}

\begin{remark}
A natural modification to the SoS program would be to add indicator variables for whether samples are corrupted. Such indicator variables have been used in several SoS algorithms for robust estimation (see, for e.g. \cite{kothari2017better,KarKK19,diakonikolas2022robust}). While we expect that adding these indicator variables does not give SoS much additional power in solving this problem, this does not technically follow from our lower bound as these indicator variables are not low-degree polynomials in the existing solution variables and input variables (though they are often approximated by such low-degree polynomials).
\end{remark}

\subsubsection{Robust Mean Estimation for Identity Covariance Gaussians}

For the problem of Gaussian robust mean estimation, even if we further assume that the covariance of the clean Gaussian distribution is the identity, SQ lower bounds suggest that it is not possible to efficiently achieve the information-theoretic optimal error $O(\tau)$ without using significantly more samples \cite{diakonikolas2017statistical}. We show lower bounds for \Cref{prob:RME_id}, a hypothesis testing problem known to be efficiently reducible to the problem of Gaussian robust mean estimation with identity covariance (see, e.g., \cite{diakonikolas2017statistical} and Lemma 8.5 of \cite{DK23-book}). 

\begin{problem}[Hypothesis-Testing-RME with Identity Covariance]
\label{prob:RME_id}
Let $\tau > 0$ and $B = O(\log^{1/2}(1/\tau))$ be a parameter. Hypothesis-Testing-RME with identity covariance $\cB_{TV}(\tau, \cN(0, \Id_n), \cD)$ (\Cref{def:tv_corruption}), where every $D \in \cD$ is of the form $\cN(\mu_D, \Id_n)$ and $\norm{\mu_D}\geq \Omega(\tau \log(1/\tau)^{1/2}) / B^2)$. 
\end{problem}

To prove a lower bound against \Cref{prob:RME_id}, we will apply \Cref{cor:main_UALA} to the instance of \Cref{def:ngca_distinguishing} where the distribution $A$ is the following $A_{\text{RME}-\Id}$. 

\begin{lemma}[$A_{\text{RME}-\Id}$ in Proof of Proposition 5.2 in \cite{diakonikolas2017statistical}]
\label{lem:poly_perturb_A}
Let $\tau > 0$ be smaller than some absolute constant,  let $B: = \sqrt{\log(1/\tau)} - \tau > 1$, $\xi = c \tau \log^{1/2}(1/\tau)/B^2$ for a large enough constant $c>0$ and let $g(x)$ be the density of $\cN(0, 1)$. Then the following holds. Define $A_{\text{RME}-\Id}$ to have density function
\[A(x) := g(x-\xi) + q(x)\mathbf{1}(x \in [-B, B])\]
where $q(x)$ is the degree-$k$ polynomial uniquely determined by $\int_{-B}^B q(x) ~dx =0$ and $\int_{-B}^B q(x)x^i~dx = \int_{-B}^B (g(x) - g(x-\xi))x^i~dx$ for 
$1\leq i\leq\floor{\sqrt{B}}$. Then:
\begin{enumerate}
    \item $A$ matches $\cN(0,1)$ on the first $k$ moments. 
    \item $TV(A, \cN(\xi, 1)) \leq O(\xi ~k^2/\log(1/\xi))$.
    \item \label{cond:g-q} 
    For all $x \in [-B, B]$, it holds that $|q(x)|\leq \tau B^{-3/4}$ and $g(x-\xi) + q(x) \geq \sqrt{\frac{\xi}{2\pi}} - 3\xi\sqrt{\log \frac{1}{\xi}}$. 
\end{enumerate}
\end{lemma}

Setting $A := A_{\text{RME}-\Id}$ in \Cref{def:ngca_distinguishing}, we get an instance of \Cref{prob:RME_id} with the set of distributions $\cD := \{ \cN(\xi v, \Id_n) \mid v \in \{ \pm1 \}^n \}$, which satisfies the total variance condition in \Cref{prob:RME_id} as shown in \cite{diakonikolas2017statistical}. We have the following bounds on $U_A(\cdot)$ and $L_A(\cdot)$, whose proof is deferred to \Cref{app:ulbounds}. 

\begin{restatable}[Bounds on $U_{A_{\text{RME}-\Id}}$ and $L_{A_{\text{RME}-\Id}}$]{lemma}{polyperturb}\label{cor:poly_perturb_lb}
For any $0<\xi<\frac{1}{26}$, 
\[
U_{A_{\text{RME}-\Id}}(t) \leq (32t\log(1/\tau))^t\ \ \text{and}\ \ L_{A_{\text{RME}-\Id}}(t) \geq \xi \cdot\exp\left(-10\xi\sqrt{t\log(t+1)}\right),\ \ \forall t\geq 1.
\] 
\end{restatable}
Since $\tau$ and $\xi$ are both constants, this yields the following lower bound.

\begin{corollary}[Gaussian Robust Mean Estimation with Identity Covariance]
Fix $\tau\in(0,\frac{1}{26})$. For any $\eps=\eps(n)\in (0,1)$, as the dimension $n$ increases, the SoS program described in \Cref{sec:prob_statement} with degree $o(\sqrt{\frac{\eps\log n}{\log\log n}})$ cannot solve \Cref{def:ngca_distinguishing} given fewer than $n^{\frac{k+1}{2}(1-\eps)}$ samples, where $A: = A_{\text{RME}-\Id}$ with $k: = \floor{\sqrt {\log (1/\tau)}}$.

In particular, if the samples are drawn from $\cN(0, \Id)$, then the SoS program cannot rule out the existence of a hidden direction $v$ along which the input distribution is $A_{\text{RME}-\Id}$.
\end{corollary}

\begin{proof}
The conditions for \Cref{cor:main_UALA} to apply are verified by  \Cref{lem:poly_perturb_A} and \Cref{cor:poly_perturb_lb}. The conclusion follows. 
\end{proof}

\subsubsection{Mean Estimation with Bounded $t$-Moments}

If we relax the assumption in \Cref{prob:RME_bdd_cov_learn} to include subgaussian distributions rather than strictly Gaussian ones then SQ lower bounds suggest that it is not possible to efficiently recover the mean to an error of less than $O(\tau^{1-1/t})$ where $t$ is the number of moments that are bounded \cite{diakonikolas2022robust}. This can be shown via a lower bound for \Cref{prob:mean-bdd-mom} defined below, which is known to be efficiently reducible to the problem of robust mean estimation for subgaussian distributions \cite{diakonikolas2022robust}.

\begin{problem}[Hypothesis-Testing-RME-Bounded-$t$-Moments]
\label{prob:mean-bdd-mom}
Let $t$ be a positive integer and $\tau\in(0, 1)$. Hypothesis-Testing-RME-Bounded-$t$-Moments is the problem $\cB_{TV}(\tau, \cN(0, \Id_n), \cD)$ (\Cref{def:tv_corruption}) where each $D \in \cD$ satisfies the following: (1) the mean vector $\mu$ satisfies $\| \mu \| \geq \Omega(\frac{1}{t} \tau^{1-1/t})$; (2) $D$ has subgaussian tails, i.e., for all $v \in \R^n$ and $1 \leq i \leq t$, $\E_{x \sim D}[\abs{v^T (x - \mu)}^i]^{1/i} \leq O(\sqrt{i})$. Here, the constants in the $O$ and $\Omega$ notations are independent of $\tau, t, n$. 
\end{problem}

\Cref{prob:mean-bdd-mom} is known to be efficiently reducible to the problem of robustly estimating the mean of a distribution with bounded moments (see Section 6 and the discussion therin from \cite{diakonikolas2022robust}). 
To prove an SoS lower bound against \Cref{prob:mean-bdd-mom}, we will apply \Cref{cor:main_UALA} to the instance of \Cref{def:ngca_distinguishing} with $A:=A_{\text{RME-$t$-Mom}}$ defined as below. 

\begin{lemma}[Lemma 6.14 from \cite{diakonikolas2022robust}]
\label{lem:mean-bdd-def}
Assume $k \in \mathbb{Z}^+$ and $(c(k-1))^{-(k-1)} \leq \tau\leq\frac{1}{2}$ for some positive constant $c$ independent of $k$. 
Then there exists a distribution $A$ over $\mathbb{R}$ that satisfies the following.
\begin{enumerate}
    \item  $A = (1 - \tau)Q_1 + \tau Q_2$ where $Q_1(x) = g(x - \delta) + \frac{1}{1 - \tau} p(x) \mathbf{1}_{[-1,1]}(x)$, $Q_2(x) = g(x-\delta')$, and $p(\cdot)$ is a degree $k-1$ polynomial (chosen below). 
    \item $\delta: = \frac{1}{2000(k-1)}\tau^{1-\frac{1}{k-1}}$,  $\delta' := -\frac{(1 - \tau)}{\tau}\delta$. 
    \item $p(\cdot)$ is a univariate polynomial satisfying 
    \begin{enumerate}
        \item $\int_{-1}^{1} p(x) \, dx = 0,$
        \item $\int_{-1}^{1} p(x)x \, dx = 0,$
        \item $\max\limits_{x \in [-1,1]} |p(x)| \leq 0.1$.
    \end{enumerate}
    \item $A$ matches the first $k-1$ moments with $N(0, 1)$.
    \item For all $i \geq 1$, $\E\limits_{x \sim Q_1}[\abs{v^T (x - \mu)}^i]^{1/i} \leq O(\sqrt{i})$.
\end{enumerate}
We denote this distribution by $A_{\text{RME-$(k-1)$-Mom}}$. 
\end{lemma}

Setting $A = A_{\text{RME-$t$-Mom}}$ in \Cref{def:ngca_distinguishing} gives us an instance of \Cref{prob:mean-bdd-mom} with the set of distributions $\cD := \{ \cN({\frac{1}{2000t}\tau^{1-1/t}}\cdot v, \Id)\mid v \in \{ \pm 1/\sqrt n\}^n \}$, which satisfies the conditions required. We have the following bounds on $U_A(\cdot)$ and $L_A(\cdot)$ for $A = A_{\text{RME-$t$-Mom}}$, whose proof is deferred to \Cref{app:ulbounds}. 

\begin{restatable}[Bounds on $U_{A_{\text{RME-$t$-Mom}}}$ and $L_{A_{\text{RME-$t$-Mom}}}$]{lemma}{ulboundsmom}\label{lem:bounds_A_RME_t_mom}
For any $i\geq 1$, 
\[
U_{A_{\text{RME-$t$-Mom}}}(i)\leq O(i)^i,\ \ L_{A_{\text{RME-$t$-Mom}}}(i) \geq \exp\left(-O(\sqrt{i\log(i+1)})\right).
\] 
\end{restatable}

By \Cref{lem:mean-bdd-def}, $A_\text{RME-$t$-Mom}$ matches the first $k-1$ moments with that of $\cN(0, 1)$. Additionally, \Cref{lem:bounds_A_RME_t_mom} verifies the conditions for \Cref{cor:main_UALA} to apply. By \Cref{cor:main_UALA}, we get the following.

\begin{restatable}[SoS Lower Bound for Mean Estimation with Bounded $t$-Moments]{corollary}{meanbddmom}\label{cor:mean-bdd-mom}
Fix $\tau\in(0,\frac{1}{2})$ and a positive integer $k\geq 2$. For any $\eps=\eps(n)\in (0,1)$, as the dimension $n$ increases, the SoS program described in \Cref{sec:prob_statement} with degree $o(\sqrt{\frac{\eps\log n}{\log\log n}})$ cannot solve \Cref{def:ngca_distinguishing} given fewer than $n^{k(1-\eps)/2}$ samples, where $A:= A_\text{RME-$(k-1)$-Mom}$. 

In particular, if the samples are drawn from $\cN(0, \Id)$, the program cannot rule out the existence of a hidden direction $v$ along which the input distribution is $A_\text{RME-$t$-Mom}$.
\end{restatable}

\subsubsection{Gaussian List-decodable Mean Estimation}

In the problem of Gaussian list-decodable mean estimation (LDME), the algorithm is given a set of samples a majority of which consist of arbitrary (and possibly adversarial) outliers, while a $\tau<\frac{1}{2}$ fraction are drawn from $\cN(\mu, \Id_n)$ with an unknown $\mu$. The goal is to return a set of $O(\frac{1}{\tau})$ candidates at least one of which is close to the true mean, $\mu$. 

Information-theoretically, under mild conditions, it is possible to achieve this goal within error $O(\log\frac{1}{\tau})$ using $O_\tau(n)$ many samples (see, e.g., Section 5.2 of \cite{DK23-book}). Algorithmically, however, \cite{DKS18-list} proved SQ lower bounds which suggest that it is impossible to efficiently achieve a smaller error than $O(\tau^{-1/k})$ using fewer than $n^{\Omega(k)}$ samples (see \cite{DK23-book} for a different exposition). 

We prove an $n^{k/2}$ sample lower bound for SoS (\Cref{cor:list_ME_SoS_lb}), 
for the natural hypothesis testing version of the problem as below.

\begin{problem}[Hypothesis-Testing-LDME]
\label{prob:LDME}
Given $\tau\in(0,\frac{1}{2})$ and positive integer $k\geq 2$, the hypothesis-testing-LDME is the problem $\cB_\huber(1-\tau, \cN(0, \Id_n), \cD)$ (\Cref{def:huber_contamination}), where every $D \in \cD$ has the form $\cN(\mu_D, \Id_n)$ for some $\mu_D\in\R^n$ whose $l_2$-norm is at least $\Omega(\tau^{-1/k})$. Here, the constant in the $\Omega$ notation is independent of $\tau,k,n$.
\end{problem}

\Cref{prob:LDME} is known to be efficiently reducible to the problem of Gaussian LDME~\cite{DKS18-list}. To prove an SoS lower bound against \Cref{prob:LDME}, we will apply \Cref{cor:main_UALA} to the instance of \Cref{def:ngca_distinguishing} with $A:=A_{\text{LDME}}$, defined below. 

\begin{lemma}[\cite{DKS18-list}; see also Lemma 8.21 in \cite{DK23-book}]
\label{lem:LDME_A}
For each $k \in \Z_+$, there exists a univariate distribution $A_{\text{LDME}} = \tau \cN(\mu, 1) + (1-\tau) E$ for some distribution $E$ and $\mu = 10 c_k \tau^{-1/k}$ where $c_k$ depends only on $k$, such that $A_{\text{LDME}}$ matches the first $k$ moments with $\cN(0, 1)$. Moreover, the pdf of $E$ is upper bounded by two times the pdf of $\cN(0, 1)$ pointwise. 
\end{lemma}

Setting $A = A_{\text{LDME}}$ in \Cref{def:ngca_distinguishing} gives an instance of \Cref{prob:LDME} with $\cD := \{ \cN(\mu v, \Id_n) \mid v \in \{\pm1/\sqrt n \}^n \}$, which satisfies the conditions on $\cD$ required by the problem definition. We now state and prove our lower bound. 
\begin{corollary}[SoS Lower Bound for Gaussian List-decodable Mean Estimation]
\label{cor:list_ME_SoS_lb}
Fix $\tau\in(0,\frac{1}{2})$ and positive integer $k\geq 2$. For any $\eps=\eps(n)\in (0,1)$, as the dimension $n$ increases, the SoS program described in \Cref{sec:prob_statement} with degree $o(\sqrt{\frac{\eps \log n}{\log\log n}})$ cannot solve \Cref{def:ngca_distinguishing} given fewer than $n^{\frac{k+1}{2}(1-\eps)}$ samples, where $A: = A_{\text{LDME}}$. 

In particular, if the samples are drawn from $\cN(0, \Id)$, the program cannot cannot rule out the existence of a hidden direction $v$ along which the input distribution is $A_{\text{LDME}}$.
\end{corollary}

\begin{proof}
The conditions for \Cref{cor:main_UALA} to apply are verified by \Cref{lem:LDME_A} and \Cref{cor:gaussian_lb}, where we set $\mu = 10 c_k \tau^{-1/k}$, $\mu' = 0$, $\sigma = \sigma' = 1$ and recall that $k$ is a fixed constant beforehand. The conclusion follows. 
\end{proof}

\subsubsection{Gaussian Robust Covariance Estimation in Spectral Norm}

We now consider the problem of robustly estimating the covariance (RCE) of a Gaussian up to a constant multiplicative factor in spectral norm, given that a $\tau$ fraction of the samples may be arbitrarily corrupted.  
It is information-theoretically possible to solve this problem of  with $O_\tau(n)$ samples. The SQ lower bound shown in \cite{diakonikolas2017statistical} suggests that any algorithm using $o(n^2)$ samples requires super-polynomial runtime. 

\begin{problem}[Hypothesis-Testing-RCE in Spectral Norm]
\label{prob:RCE_spectral_learn}
Let $0 < \tau < \frac{1}{2}$. 
Hypothesis-Testing-RCE is the problem
$\cB_{\huber}(\tau, \cN(0, \Id_n), \cD)$ (\Cref{def:huber_contamination}), 
where each $D \in \cD$ has the form $\cN(0, \Sigma_D)$  =where either $\Sigma_D \prec \frac{1}{2}\Id_n$ or $\Sigma_D \succ 2 \Id_n$. 
\end{problem}

\Cref{prob:RCE_spectral_learn} is known to be efficiently reducible to the problem of Gaussian RCE in spectral norm~\cite{diakonikolas2017statistical}. 
To prove a SoS lower bound against \Cref{prob:RCE_spectral_learn}, we will apply \Cref{cor:main_UALA} to the instance of \Cref{def:ngca_distinguishing} with the distribution in the hidden direction set to $A_{\text{COV}}$, defined below. Given $c\in(0, \frac{1}{6})$ and $n>1$, let $\tau:= \frac{c}{\log n}$, and $A_{\text{COV}}$ is the following three-component mixture. 
\[ 
A_{\text{COV}} := (1-\tau)\cdot \cN(0, 1-\frac{4/5}{1-\tau}) + \frac{\tau}{2}\cdot \cN(\sqrt{4/(5\tau)}, 1) + \frac{\tau}{2}\cdot \cN(-\sqrt{4/(5\tau)}, 1). 
\] 
 
Setting $A = A_{\text{COV}}$ in \Cref{def:ngca_distinguishing} provides an instance of \Cref{prob:RCE_spectral_learn} with the set of distributions $\cD := \{ \cN(0, \Id_n - \frac{4/5}{1-\tau}~vv^T) \mid v \in \{\pm1/\sqrt n \}^n \}$. Here, note that $\cN(0, \Id_n - \frac{4/5}{1-\tau}~vv^T)$ is derived from the first component of $A_{\text{COV}}$, and the condition on $\cD$ is satisfied. We now state and prove our lower bound. 

\begin{corollary}[SoS Lower Bound for Robust Covariance Estimation, Multiplicative]\label{lem:RCE_spectral_SoS_lb}
Fix $c\in(0,\frac{1}{6})$. For any $\eps=\eps(n)\in (0,1)$, as the dimension $n$ increases, the SoS program described in \Cref{sec:prob_statement} with degree $o(\sqrt{\frac{\eps\log n}{\log\log n}})$ cannot solve \Cref{def:ngca_distinguishing} given fewer than $n^{2(1-\eps)}$ samples, where $A: = A_{\text{COV}}$ with $k:=4$ and $\tau:=\frac{c}{\log n}$.

In particular, if the samples are drawn from $\cN(0, \Id)$, the program cannot show that there is no hidden direction $v$ such that the input has distribution $A_{\text{COV}}$ in this direction.
\end{corollary}
\begin{proof}
By a direct calculation, we see that $A_{\text{COV}}$ matches the first 3 moments with $\cN(0, 1)$ (cf. Theorem 6.1 in  \cite{diakonikolas2017statistical}). By applying \Cref{cor:gaussian_lb} to $A_{\text{COV}}$ where we set $\alpha:=1-\tau\geq \frac{1}{2}$, $\mu:=0$, $\sigma:=1$, $\mu':=\sqrt{4/(5\tau)}=O(\sqrt{\frac{\log n}{c}})$, $\sigma':=1$, $B:=1$, we get
\[
U_{A_{\text{COV}}}(t)\leq O\left(\frac{\log n}{c}\right)^{t}\ \ \text{and}\ \ L_A(t)\geq (t+1)^{-O(t)},\ \ \forall t\geq 1.
\] 
Thus, $A_{\text{COV}}$ satisfies the conditions in \Cref{cor:main_UALA}, from which the conclusion follows. 
\end{proof}

\subsubsection{Gaussian Robust Covariance Estimation to Small Additive Error }

Here, we want to estimate the Gaussian covariance up to an additive error, given the guarantee that the covariance matrix $\Sigma$ is bounded by $\Id \preceq \Sigma \preceq 2\Id$. This is information theoretically possible 
up to an error of $O(\tau)$ using $O_\tau(n)$ samples. However, it is unclear how to achieve this with an efficient algorithm. SQ lower bounds in \cite{diakonikolas2017statistical} suggest that polynomial time algorithms cannot recover the unknown covariance to an error smaller than $o(\tau \log\frac{1}{\tau})$ without using $n^{\Omega\left((\log\frac{1}{\tau})^c\right)}$ samples, for some positive constant $c$. 

We demonstrate an SoS lower bound against the following hypothesis testing version of the problem of Gaussian RCE, where the goal is to recover the mean up to an error of $o(\tau \log\frac{1}{\tau})$. 

\begin{problem}[Hypothesis-Testing-RCE-Additive $\tau \log\frac{1}{\tau}$]
\label{prob:RCE2}
Given $\tau \in(0, \frac{1}{2})$, Hypothesis-Testing-RCE-Additive is the problem $\cB_{TV}(\tau, \cN(0, \Id_n), \cD)$ (\Cref{def:tv_corruption}), where every $D \in \cD$ has the form $\cN(0, \Sigma)$, $\| \Sigma - \Id \|_2 \geq\Omega(\tau \log \frac{1}{\tau})$, and the constant in the $\Omega$ notation is independent of $\tau,n$.
\end{problem}

\Cref{prob:RCE2} is known to be efficiently reducible to the problem of Gaussian RCE in spectral norm~\cite{diakonikolas2017statistical}. To prove an SoS lower bound against \Cref{prob:RCE2}, we will apply \Cref{cor:main_UALA} to the instance of \Cref{def:ngca_distinguishing} with the distribution in the hidden direction set to $A_{\text{COV-close}}$, defined in \Cref{lem:def_A_cov_close}. 

\begin{lemma}[$A_{\text{COV-close}}$ from Proposition 5.13. in \cite{diakonikolas2017statistical}]
\label{lem:def_A_cov_close}
Suppose $0 < \delta < \frac{1}{3}$, $1\ll k \ll (\log\frac{1}{\delta})^{\frac{1}{4}}$ and $k\ll B \ll (\log \frac{1}{\delta})^{1/2}$, where $Y \ll X$ means that $X>cY$ for some absolute positive constant $c$. Let $g(x)$ denote the density of $\cN(0, 1)$. Define $$A_{\text{COV-close}} := \frac{g(x/(1-\delta))} {1-\delta} - q(x) \mathbf{1}(x \in [-B, B]),$$ where $q(x)$ is the unique degree-$k$ polynomial for $A_{\text{COV-close}}$ to match the first $k$ moments of $\cN(0, 1)$. Then, $A_{\text{COV-close}}$ satisfies the following conditions:
\begin{enumerate}
    \item[(i)] $A_{\text{COV-close}}$ and $N(0, 1)$ agree on the first $k$ moments.
    \item[(ii)] $d_{\text{TV}}\Big(A_{\text{COV-close}},\ N\left(0,(1 - \delta)^2\right)\Big) \leq O(\delta k^4/\log(1/\delta)) \leq O(\delta)$.
    \item[(iii)] \label{cond:ACOV,q} 
    $|q(x)| \leq 10\delta k^4 B^{-3}$ in the interval $[-B, B]$. 
\end{enumerate}
\end{lemma}
Setting $A = A_{\text{COV-close}}$ in \Cref{def:ngca_distinguishing} gives us an instance of \Cref{prob:RCE2} with the set of distributions $\cD := \{ \cN(0, \Id_n - (1-(1-\delta)^2) vv^T) \mid v \in \{ \pm 1/\sqrt n \}^n \}$, which satisfies the conditions in \Cref{prob:RCE2} with respect to the parameter $\tau$ where $\delta=2\tau(\log\frac{1}{\tau})/k^4$. We now state bounds on $U_A(\cdot)$ and $L_A(\cdot)$, whose proof is deferred to \Cref{app:ulbounds}. 

\begin{restatable}[Bounds on $U_{A_{\text{COV-close}}}$ and $L_{A_{\text{COV-close}}}$]{lemma}{acovclose}\label{lem:bounds_A_cov_close}
Assume $0<\delta<\frac{1}{4}$ and let the absolute constants be sufficiently large in the construction of $A_{\text{COV-close}}$ in \Cref{lem:def_A_cov_close}. Then:
\[
U_{A_{\text{COV-close}}}(t)\leq O(t^2\log\frac{1}{\delta})^{t}
\ \ \text{and}\ \ 
L_{A_{\text{COV-close}}}(t) \geq 
\exp\left(- O( \delta t\log (t+1) + \delta\log\frac{1}{\delta} )\right),
\ \ \forall t\geq 1.
\]
\end{restatable}

An application of \Cref{cor:main_UALA} now gives us the following.

\begin{corollary}[SoS Lower Bound for Covariance Estimation, Additive]
Fix $\tau\in(0,\frac{1}{20})$ and a positive integer $k\geq 2$. For any $\eps=\eps(n)\in (0,1)$, as the dimension $n$ increases, the SoS program described in \Cref{sec:prob_statement} with degree $o\left(\sqrt{\frac{\eps\log n}{\log\log n}}\right)$ cannot solve \Cref{def:ngca_distinguishing} given fewer than $n^{\frac{k+1}{2}(1-\eps)}$ samples, where $A:= A_{\text{COV-close}}$ with $\delta:=2\tau(\log\frac{1}{\tau})/k^4$.

In particular, if the samples are drawn from $\cN(0, \Id)$, the program cannot rule out the existence of a hidden direction $v$ along which the input distribution is $A_{\text{COV-close}}$.  
\end{corollary}

\subsubsection{Learning Mixtures of $k$ Gaussians}

\label{subsubsec:gaussian}
We now discuss the problem of learning a mixture of $k$ Gaussians ($k$-GMM): given samples drawn from an unknown mixture of $k$ Gaussians, the goal is to recover a mixture that is close in total variation to the original. This problem can be information-theoretically solved in $\poly(k, n)$ samples. The SQ lower bound in \cite{diakonikolas2017statistical} suggests that any algorithm that takes less than $n^{O(k)}$ samples runs in time $2^{\Omega(n^c)}$ for some constant $c$. Building on the SQ-hard instances, \cite{gupte2022continuous} (see also \cite{bruna2021continuous}) gave a reduction-based cryptographic lower bound assuming sub-exponential hardness of LWE, which shows that if $k=\log(n)$ and the number of samples is $\poly(n)$ then the runtime has to be quasi-polynomial in $n$. 

Here we prove an unconditional information-computation tradeoff for SoS algorithms, which is  somewhat stronger compared to the cryptographic lower bound mentioned above. Specifically, we show an SoS lower bound against a hypothesis testing version of the problem of $k$-GMM, defined in \Cref{prob:gmm_learning}. In particular, setting $k = \log(n)$ or even $k=\omega(1)$, our lower bound suggests that both the runtime and the sample complexity are quasi-polynomial in $n$.

\begin{problem}[Hypothesis-Testing-$k$-GMM]\label{prob:gmm_learning}
Let $0<\gamma <1$. Hypothesis-Testing-$k$-GMM is the problem $\cB_{\huber}(0, \cN(0, \Id_n), \cD)$ (\Cref{def:huber_contamination}), where every $D \in \cD$ is a mixture of $k$ Gaussians such that each pair of the Gaussians are $1-\gamma$ apart in total variation and $d_{TV}(D, \cN(0, \Id)) \geq \frac{1}{2}$. 
\end{problem}

\Cref{prob:gmm_learning} is known to be efficiently reducible to the problem of learning a $k$-GMM~\cite{diakonikolas2017statistical}. To prove an SoS lower bound against \Cref{prob:gmm_learning}, we will apply \Cref{cor:main_UALA} to the instance of \Cref{def:ngca_distinguishing} with the distribution in the hidden direction set to the ``parallel pancakes'' distribution $A_{\text{GMM}}$, introduced in~\cite{diakonikolas2017statistical}, 
defined below:
\[ 
A_{\text{GMM}} := \sum_{i=1}^k w_i \cN(\sqrt{1-\delta} ~\mu_i, \delta).
\]
Here, $\mu_i$s and $w_i$s are from the Gaussian quadrature, i.e.,  $\mu_i = \sqrt{2} x_i$ for $x_1<\cdots<x_k$ zeros of the $k$th Probabilist's Hermite polynomial and $w_i = \frac{k!}{k^2 He_{i-1}(x_i)^2}$, and $\delta = \Theta((k^2 \log^2(k + 1/\gamma))^{-1})$.

Setting $A = A_{\text{GMM}}$ in \Cref{def:ngca_distinguishing} gives an instance of \Cref{prob:gmm_learning} with the set of distributions $\cD := \{\sum_{i=1}^k w_i \cN(v \mu_i, \Id_{n} - (1-\delta)vv^T) \mid v \in \{\pm 1/\sqrt n \}^n \}$,  which satisfies the conditions on $\cD$ required by the problem definition (see \cite{diakonikolas2017statistical}). We now state and prove our lower bound. 

\begin{corollary}[SoS Lower Bound for learning GMMs]
\label{lem:gmm_SoS_lb}
Fix $\gamma>0$ and a positive integer $k\geq 2$. For any $\eps=\eps(n)\in (0,1)$, as the dimension $n$ increases, the SoS program described in \Cref{sec:prob_statement} with degree $o(\sqrt{\frac{\eps\log n}{\log\log n}})$ cannot solve \Cref{def:ngca_distinguishing} given fewer than $n^{k(1-\eps)}$ samples, where $A: = A_{\text{GMM}}$. 

In particular, if the samples are drawn from $\cN(0, \Id)$, the program cannot rule out the existence of a hidden direction $v$ along which the input distribution is $A_{\text{GMM}}$.
\end{corollary}
\begin{proof}
\cite[Proposition 4.2]{diakonikolas2017statistical} 
shows that $A$ matches the first $2k-1$ moments of $\cN(0, 1)$ and $|\mu_i| \leq \sqrt{k}$. What remains to be shown is that $U_A$ and $L_A$ are appropriately bounded for these distributions. Since $\delta < 1$ and $\gamma$ are constants, \Cref{cor:gaussian_lb} verifies the conditions for \Cref{cor:main_UALA}, an application of which gives the conclusion. 
\end{proof}

\subsubsection{Learning Mixtures of Two Separated Gaussians With Common Covariance}
Here we study the task of learning $2$-GMMs with an unknown common covariance whose components are separated. This is a classical special case of the GMM learning problem: a simple (computationally efficient) spectral algorithm succeeds using $O(n^2)$ samples. On the other hand, $O(n)$ samples suffice information-theoretically. \cite{Davis2021ClusteringAM} gave evidence that this gap is inherent by establishing an $\Omega(n^2)$ sample lower bound for low-degree tests. Moreover, by leveraging the SoS lower bound technology of~\cite{ghosh2020sum}, \cite{Davis2021ClusteringAM} also showed an SoS lower bound of $\Omega(n^{3/2})$ on the sample size.

It turns out that the hard family of instances of \cite{Davis2021ClusteringAM} is an instance of NGCA where the univariate distribution $A$ is discrete. While such instances (with discrete $A$) can be solved efficiently via lattice-basis reduction with $O(n)$ samples~\cite{DK22LLL,ZSWB22LLL}, one can appropriately add ``noise'' to the instance so that the lower bound still applies while LLL-type algorithms fail. We construct such a modified instance and establish an SoS lower bound against a hypothesis testing version of the problem (defined in \Cref{prob:2gaussians}), suggesting more than $n^{2(1-\eps)}$ samples are required. 

\begin{problem}[Hypothesis-Testing-2-GMM (Common Covariance)]\label{prob:2gaussians}
Let $\delta > 0$. Hypothesis-Testing-2-GMM with a common covariance is the problem $\cB_{\huber}(0, \cN(0, \Id_n), \cD)$ where every $D \in \cD$ is of the form $\frac 1 2 \cN(-\mu_D, \Sigma_D) + \frac 1 2 \cN(\mu_D, \Sigma_D)$, where $0 \preceq \Sigma_D$ and $\mu_D^T \Sigma_D^{-1} \mu_D > 1/\delta^2$. 
\end{problem}

\Cref{prob:2gaussians} is known to be efficiently reducible to the problem of estimating a $2$-GMM with unknown common covariance (see Lemma 8.5 of \cite{DK23-book}). To prove an SoS lower bound against \Cref{prob:2gaussians}, we will apply \Cref{cor:main_UALA} to the instance of \Cref{def:ngca_distinguishing} with the distribution in the hidden direction set to $A_{\text{MIX}} = \frac{1}{2} \mathcal{N}(-\mu, \delta^2\mu^2) + \frac{1}{2} \mathcal{N}(\mu, \delta^2\mu^2)$, where $\delta,\mu > 0$ and $\delta\mu\neq 1$. 

This gives us an instance of \Cref{prob:2gaussians} with 
\[
\cD := \{ (1/2) \cN(-\mu v, \Id_n + (\delta^2 \mu^2 -1) vv^T) + (1/2) \cN(\mu v, \Id_n + (\delta^2 \mu^2 -1) vv^T) \mid v \in \{\pm 1 / \sqrt n \}^n \} \;.
\]
This satisfies the conditions on $\cD$ required by the problem definition, since 
(1) $0 \preceq \Id_n + (\delta^2 \mu^2 -1) vv^T$ and (2) $\mu^2 (v^T (\Id_n + (\delta^2 \mu^2 -1) vv^T)^{-1} v) = \mu^2 \|v\|^2 ( 1 + \delta^2\mu^2 -1 )^{-1} = 1/\delta^2$. 

We now state and prove our lower bound. 

\begin{corollary}[SoS Lower Bound for \Cref{prob:2gaussians}]
Fix $\delta \in (0, 1)$ and let $\mu = 1/\sqrt{1+\delta^2}$. For any $\eps=\eps(n)\in (0,1)$, as the dimension $n$ increases, the SoS program described in \Cref{sec:prob_statement} with degree $o(\sqrt{\frac{\eps\log n}{\log\log n}})$ cannot solve \Cref{def:ngca_distinguishing} given fewer than $n^{2(1-\eps)}$ samples, where $A:= A_{\text{MIX}}$. 

In particular, if the samples are drawn from $\cN(0, \Id)$, the program cannot rule out the existence of a hidden direction $v$ along which the input distribution is $A_{\text{MIX}}$.
\end{corollary}
\begin{proof}
$A_{\text{MIX}}$ matches the first 3 moments of $\cN(0, 1)$, which can be verified by noting that the variance is $1$ and the odd moments are $0$. Also, $A_{\text{MIX}}$ is a mixture of $2$ Gaussians with means and variances bounded by a constant so \Cref{cor:gaussian_lb} applies.  This verifies the conditions for \Cref{cor:main_UALA} to apply, and the conclusion follows. 
\end{proof}

\begin{remark}
A natural question is whether our SoS lower bound for approximate parallel pancakes can be used to derive an SoS lower bound for the Sherrington-Kirkpatrick problem. We give a sketch of how this can be done in \Cref{app:sos-sherrington-kirkpatrick}. 
\end{remark}
	
	\bibliographystyle{alphaabbrv}
	\bibliography{allrefs}
	
\newpage

 \appendix

\section*{Appendix}

\section{Pseudo-Calibration Calculation}\label{app:pseud-calib}
In this section we compute the pseudo-calibration when the planted vector $v$ is drawn uniformly at random from $\{\pm {1\over\sqrt{n}}\}^n$. We let the planted distribution $\pl$ be as follows: first choose $v\sim \{\pm{1\over\sqrt{n}}\}^n$ uniformly at random, then choose i.i.d. samples $x_1,\ldots, x_m$ 
\[
x_i=\big((x_i)_{v^\perp},\ (x_i)_v)\sim N(0,\Id_{n-1}\big)_{v^\perp}\times A_v
\]
where $A$ is the one-dimensional distribution of interest in direction $v$. 

For $a=(a_1,\dots,a_m)\in(\N^n)^m$ and $I\in \N^n$, recall that the pseudo-calibration is then given by
\begin{equation}\label{eq:calib2}
\forall I\in\binom{[n]}{\leq \dsos}\quad \pE(\vsos^I):=\Sum_{a\in(\N^n)^m:\ \totalsize^I(a)\leq \truncation}\ \E_{(x,v)\sim \Dpl}\left[ v^I {h_{a}\over\sqrt{a!}} \right]\cdot {h_{a}\over\sqrt{a!}},
\end{equation}
where $\totalsize^I(a):=\normo{a}+\left\vert I\cup \{i\in[n]:\ (\exists u\in[m])a_u(i)>0\}\right\vert + {(1-\eps)k\over 2}\cdot\left\vert\{u\in[m]:\ (\exists i\in[n]) a_u(i)>0\}\right\vert$. Note that $\pE(\vsos^I)$ is a function on the inputs $(x_1,...,x_m) \in (\R^n)^m$. 

\begin{lemma}\label{lem:calib}
	For any $I \subseteq [n]$, the pseudo-calibration value is:
	\begin{equation}\label{appeq:calib-calc}
		\begin{aligned}
			\pE(\vsos^I)=\Sum_{\substack{a\in(\N^n)^m:\ \totalsize^I(a)\leq 
                \truncation,\vspace{2pt}\\ 
                \text{and }(\forall i\in[n])\ I(i)+\sum_{u}a_u(i)\text{ is even}}}
                n^{-{|I|+|a|\over2}}\Prod_{u=1}^m \E_A [h_{|a_u|}] {h_{a_u}\over a_u!}
		\end{aligned}
	\end{equation}
\end{lemma}
\begin{proof}
	We calculate 
	$\mathop{E}_{(x,v)\sim\Dpl} [v^I h_{a}] = \E_{v} \left[v^I \cdot \E_{y\sim (N_{v^\perp}\times A_v)^m} [h_{a}(y)]\right]$ in \eqref{eq:calib2} fixing $a\in(\N^n)^m$. By independence, $\E_{y\sim (N_{v^\perp}\times A_v)^m}[h_{a}(y)]=\Prod_{u=1}^m\E_{y\sim N_{v^\perp}\times A_v}[h_{a_u}(y)]$. 
	
	We claim that $\E_{y\sim N_{v^\perp}\times A_v}[h_{\beta}(y)]=\E_{y\sim A}[h_{|\beta|}(y)]\cdot v^\beta$. To see this, consider the generating function for $n$-dimensional Hermite polynomials: $\Sum_{\beta\in\N^n}h_{\beta}(y){t^\beta\over\beta!}=\exp\left(\langle y,t\rangle-{\Vert t\Vert^2\over2}\right)$, $\forall y,t\in\R^n$. Fixing a $v\in\R^n$, write $y=y_{v^\perp}+y_v$, $t=t_{v^\perp}+t_v$, then
	\begin{equation}
		\begin{aligned}
			\Sum_{\beta\in\N^n}h_{\beta}(y){t^\beta\over\beta!}=\exp\left(\langle y_{v^\perp},t_{v^\perp}\rangle-{\Vert t_{v^\perp}\Vert^2\over2}\right)\cdot\exp\left(\langle y_{v},t_{v}\rangle-{\Vert t_{v}\Vert^2\over2}\right).
		\end{aligned}
	\end{equation}
	Taking an expectation with respect to $y$ drawn from  $N_{v^\perp}\times A_v$, the left hand side above equals $\Sum_{\beta\in\N^n}\left(\E_{y\sim N_{v^\perp}\times A_v}[h_{\beta}(y)]\right){t^\beta\over\beta!}$, and the right hand side is given by, 
	\[
    \underbrace{\E_{y_{v^\perp}\sim N_{v^\perp}}\left[\exp\left(\langle y_{v^\perp},t_{v^\perp}\rangle-{\Vert t_{v^\perp}\Vert^2\over2}\right)\right]}_{\text{constant 1}}
	\E_{y_v\sim A}\left[{\exp\left(\langle y_v,t_v\rangle-{\Vert t_v\Vert^2\over2}\right)}\right]=\Sum_{k\geq 0}\E_{y\sim A}[h_{k}(y_v)]{t_v^k\over k!} \;, 
    \] 
    where $y_v=\langle v,y\rangle$, $t_v=\langle v,t\rangle$. Taking ${\partial\over\partial t^\beta}\mid_{t=0}$ on both sides gives the claim. 
	
	The above claim implies  
	\begin{equation}
	\mathop{E}_{(x,v)\sim\Dpl} [v^I h_{a}]=\E_v\left[v^{I+\Sum_{u}a_u}\right]\Prod_{u=1}^m \E_{A}[h_{|a_u|}]=1_{\substack{\forall i\in[n]\left(\Sum_{u}a_u(i)+I(i)\text{ even}\right)}}\cdot n^{-{|I|+|a|\over2}}\Prod_{u=1}^m \E_{A}[h_{|a_u|}].
	\end{equation}
	Plugging this to \eqref{eq:calib2} we get \eqref{appeq:calib-calc}, where recall that $a!$ means $\Prod_{u=1}^ma_u!$.
\end{proof}
    \section{Omitted Proofs from \Cref{sec:psdness-qSS}}
\label{app:psdness-qSS}
We prove the map $\rhopre$ constructed in \Cref{def:rhopre} is a representation of $\SALD$; in fact, it is an isomorphism.
\begin{lemma}\label{lem:iso_app}
	For any fixed $D$, $\rhopre:\ \SALD\to\rhopre(\SALD)$ is an $\R$-algebra isomorphism.
\end{lemma}
\begin{proof}
We show that $\rhopre$ is a homomorphism and is injective. Then by the dimension counting under \Cref{def:linfty}, $\dim_\R(\SALD)=\Sum_{i=1}^{D+1} i^2$ which is the same as $\dim_{\R}\left(M_\R(1)\oplus\dots\oplus M_\R(D)\right)$, showing that $\rhopre$ is an isomorphism.\smallskip
	
\paragraph{$\rhopre$ is a homomorphism.} Fix any $S(k_1-u,k_2-u;u)$ and $S(k_2-v,k_3-v;v)$, we need to show: 
	\begin{equation}\label{eq:hom}
		\rhopre\left(S(k_1-u,k_2-u;u)\SALDprod S(k_2-v,k_3-v;v)\right)=\reppre{k_1}{k_2}{u}\cdot\reppre{k_2}{k_3}{v}.
	\end{equation}
	Here, recall $(k_1,k_2,k_3,u,v)$ satisfies the conditions
	\begin{equation}\label{eq:rhocond}
		0\leq k_1,k_2,k_3\leq D,\ u\leq\min\{k_1,k_2\},\ v\leq\min\{k_2,k_3\}.
	\end{equation}
	
	By the definition of $\SALDprod$, $S(k_1-u,k_2-u;u)\SALDprod S(k_2-v,k_3-v;v)$ is a sum of simple spiders each having left-, right-set size $k_1$, $k_3$, so the LHS of \eqref{eq:hom} is supported on block $B(k_1,k_3)$ by definition of $\rhopre$. The RHS of \eqref{eq:hom} is a product of two matrices supported on $B(k_1,k_2)$ and $B(k_2,k_3)$, so it is also supported on block $B(k_1,k_3)$. Moreover both sides are diagonal from bottom-right up on this block, so we can use $LHS(a)$ [$RHS(a)$] to denote the $(a,a+k_3-k_1)$-entry in this block for the LHS [RHS] of \eqref{eq:hom}, where the parameter $a$ ranges in $[0,k_1]$.
	
	For the LHS of \eqref{eq:hom}, from the definition of $\star$-product \eqref{eq:SALDprod} and linearity of $\rhopre$,
	\begin{align*}
		&\rhopre\left(S(k_1-u,k_2-u;u)\SALDprod S(k_2-v,k_3-v;v)\right)\\
		=&\sum_{i=\max\{0,u+v-k_2\}}^{\min\{u,v\}}\binom{k_1-i}{k_1-u}\binom{k_3-i}{k_3-v}/(k_2+i-u-v)!\cdot \reppre{k_1}{k_3}{i}
	\end{align*}
	so 
	\begin{equation}\label{eq:rhoLHS1}
		LHS(a)=\sum_{i=\max\{0,u+v-k_2\}}^{\min\{u,v\}}\binom{k_1-i}{k_1-u}\binom{k_3-i}{k_3-v}/(k_2+i-u-v)!\cdot \reppre{k_1}{k_3}{i}.
	\end{equation}
	To expand this expression, note that if $a<k_1-u$ then $a<k_1-i$ for all $i\leq u$, meaning that row $a$ is 0 for all $\rep{k_1}{k_3}{i}$ in \eqref{eq:rhoLHS1}. If $a\geq k_1-u$, then for row $a$ to be nonzero in $\rep{k_1}{k_3}{i}$ we need $a\geq k_1-i$, or equivalently, $i\geq k_1-a$ ($\geq 0$). So we have 
	\begin{align*}
		LHS(a)=\sum_{i=\max\{k_1-a,u+v-k_2\}}^{\min\{u,v\}}\binom{k_1-i}{k_1-u}\binom{k_3-i}{k_3-v}/(k_2+i-u-v)!\cdot \binom{a}{k_1-i}/(k_3-i)!
	\end{align*}
	\begin{align}\label{eq:rhoLHS}
		={a!\over(k_1-u)!(k_3-v)!}\sum_{i=\max\{k_1-a,u+v-k_2\}}^{\min\{u,v\}}{1\over(u-i)!(v-i)!(i+k_2-u-v)!(a+i-k_1)!}.
	\end{align}
	
	Now we look at the RHS of \eqref{eq:hom}, $\reppre{k_1}{k_2}{u}\cdot\reppre{k_2}{k_3}{v}$. Again, if $a<k_1-u$ then row $a$ is zero in $\reppre{k_1}{k_2}{u}$ and hence zero in this product. Otherwise, $a\geq k_1-u$, in which case by the definition of $\rhopre$ \eqref{eq:rho} we have: 
	\begin{equation}\label{eq:rhoRHS}
		RHS(a)={\binom{a}{k_1-u}\over(k_2-u)!}\cdot{\binom{a+k_2-k_1}{k_2-v}\over(k_3-v)!}.
	\end{equation}
	
	Our task is to show that \eqref{eq:rhoLHS} and \eqref{eq:rhoRHS} are equal. First, let's look at the range of the index in summation \eqref{eq:rhoLHS}. If $\min\{u,v\}<\max\{k_1-a,u+v-k_2\}$ then \eqref{eq:rhoLHS} is 0, while at the same time we have $v<k_1-a$ (this is because $u,v\geq u+v-k_2$ always holds by \eqref{eq:rhocond} and we already assumed $a\geq k_1-u$), which implies $\binom{a+k_2-k_1}{k_2-v}=0$ and so \eqref{eq:rhoRHS} is also 0. Therefore, we only need to prove \eqref{eq:rhoLHS} equals \eqref{eq:rhoRHS} in the case where $u,v,k_1,k_2$ additionally satisfies:
	\begin{equation}
		\min\{u,v\}\geq \max\{k_1-a,u+v-k_2\}.
	\end{equation}
	Expanding the binomials in \eqref{eq:rhoRHS} and canceling out terms with \eqref{eq:rhoLHS}, we're tasked to show 
	\begin{align*}
		&\sum_{i=\max\{k_1-a,u+v-k_2\}}^{\min\{u,v\}}{1\over(u-i)!(v-i)!(i+k_2-u-v)!(a+i-k_1)!}=\\
            &\hspace{9cm} {(a+k_2-k_1)!\over(a+u-k_1)!(k_2-v)!(a+v-k_1)!(k_2-u)!}.
	\end{align*}
	Rearranging terms, this is equivalent to showing
	\begin{align}\label{eq:rhofinal}
		\sum_{i=\max\{k_1-a,u+v-k_2\}}^{\min\{u,v\}}\binom{k_2-v}{u-i}\binom{a+v-k_1}{i+k_2-u-v}=&\binom{a+k_2-k_1}{k_2-u}.
	\end{align}
	To see \eqref{eq:rhofinal}, consider choosing a size $(k_2-u)$ subset from a size $(a+k_2-k_1)$ set $Z$. There are $\binom{a+k_2-k_1}{k_2-u}$ many ways. Counting it differently, we can first fix a partition $Z=Z_1\coprod Z_2$ with $|Z_1|=k_2-v$, $|Z_2|=a+v-k_1$ and then choose subsets $Y_\theta\subseteq Z_\theta$ ($\theta=1,2$) such that $|Y_1|+|Y_2|=k_2-u$. For $Y_1$ and $Y_2$ to exist, $|Y_1|$ must be at least $\max\{(k_2-u)-|T_2|,0\}=\max\{(k_1-a)+k_2-u-v,0\}$ and at most $\min\{|T_1|,k_2-u\}=\min\{k_2-v,k_2-u\}$. So if we let $i:=|U_1|-(k_2-u-v)$, then $i$ ranges from $\max\{k_1-a,u+v-k_2\}$ to $\min\{u,v\}$, giving the counting $\sum_{i=\max\{k_1-a,u+v-k_2\}}^{\min\{u,v\}}\binom{k_2-v}{u-i}\binom{a+v-k_1}{i+k_2-u-v}$. Comparing the two counts gives \eqref{eq:rhofinal}, and thus \eqref{eq:hom}.
	
\paragraph{Proof that $\rhopre$ is injective.} Order all simple spiders $\{S(i,j;u)\}$ by the lexicographical order on tuples $\{(i+u,j+u,i)\}$. Given a nontrivial linear combination $\sum_{i}c_i\cdot S_i$ where $S_i$s are distinct simple spiders and all appearing $c_i$s are nonzero, consider the minimum simple spider in it, which we assume is $S_1$ and has form $S(a,b;u)$. Then the $(a,b)$th entry in block $B(a+u,b+u)$ of $\rhopre(x)$ comes only from $c_1\rhopre(S_1)$, so it is nonzero. This shows that $\rhopre$ is injective.
\end{proof}

    \section{Omitted Proofs from \Cref{sec:applications}}
\label{app:applications}

\subsection{Technical Toolkit} 
\label{app:applications_toolkit}

In this section, we establish some technical lemmas required to prove bounds on $U_A(t)$ and $L_A(t)$. We use the convention that $0!=0^0=1$. Also, recall that $U_A(0)=L_A(0)=1$ from the definition.

We will need the following upper bound on the normal expectation of $x^{2t}$ on a neighborhood around $\infty$. 

\begin{fact}[See e.g. Lemma A.4 from \cite{KarKK19}] 
For any $L\geq 0$ and positive integer $t$, 
\label{fact:tail_bound}
\begin{align*}
\int_{\R \setminus [-L, L]} x^{2t} \exp\left(-x^2/2\right) dx \leq 2 \exp\left(-L^2/2\right)\left( L^{4t} + (16t)^t \right).
\end{align*}
\end{fact}

We have a pointwise upper bound on $p(x)^2$ and $\hn_i(x)$ as follows.

\begin{lemma}
\label{fact:poly_bounds}
For any $t\geq 0$, we have:  
\begin{enumerate}
    \item \label{item:hnipointwise} For all $i \leq t$ and all $x$, $\hn_i(x)^2 \leq 2t! \max\{1, x^{2t}\}$. 
    \item \label{item:ppointwise} If $p(x)$ is a degree $t$ unit-norm polynomial, then $p^2(x)\leq 2(t+1)! \max\{1,x^{2t}\}$ for all $x$. 
\end{enumerate}
\end{lemma}
\begin{proof}
Recall that $\hn_i(x) = \sqrt{i!}\Sum_{j=0}^{\lfloor i/2 \rfloor} \frac{(-1)^j}{2^j j! (i-2j)!} x^{i-2j}$, so we have:
\begin{align}
\hn_i(x)^2 &= i!\left(\Sum_{j=0}^{\lfloor i/2 \rfloor} \frac{(-1)^j}{2^j j! (i-2j)!} x^{i-2j}\right)^2\\
&\leq i! \left( \Sum_{j=0}^{\lfloor i/2 \rfloor}(2^{-j})^2\cdot x^{2(i-2j)} \right) \left( \Sum_{j=0}^{\lfloor i/2 \rfloor} \frac{1}{j!^2 (i-2j)!^2}\right)\label{eq:C-S}\\
&\leq i! \Sum_{j=0}^{\lfloor i/2 \rfloor}2^{-2j}\max\{1,x^{2i}\}\\
&\leq i!\cdot 2\max\{1,x^{2i}\}\\
&\leq 2t!\max\{1,x^{2t}\}
\end{align}
where \eqref{eq:C-S} is again by Cauchy-Schwarz and the fact that $\frac{1}{j!(i-2j)!}\leq \lfloor 2/i \rfloor$. \Cref{item:hnipointwise} follows. 

For \Cref{item:ppointwise}, $p(x)=\Sum_{i=0}^t c_i \hn_i(x)$ where $\Sum_i c_i^2 = 1$, as $\{\hn_i\}$ is a $l_2$-orthonormal basis. By Cauchy-Schwarz, $p^2(x)= \left(\Sum_{i} c_i \hn_i(x)\right)^2\leq (\Sum_i c_i^2)(\Sum_{i} {\hn_i}(x)^2)\leq (t+1)~\max\limits_{0\leq i\leq t} \hn_i(x)^2$. The conclusion follows from \Cref{item:hnipointwise}.
\end{proof}

Below, we use \Cref{fact:tail_bound}, \Cref{fact:poly_bounds} and \Cref{fact:unit_norm_lower_bound_gen} to prove bounds on $L_A(t)$ and $U_A(t)$. The first one is a lower bound on $L_A(t)$ in terms of the likelihood ratio on a large interval around 0.

\aroundzero*

\begin{proof}
By \Cref{fact:poly_bounds} and \Cref{fact:tail_bound}, for any $L\geq 0$ we have  
\begin{align*}
\int_{\R\backslash[-L,L]} p^2(x) g(x) dx &\leq 2(t+1)!~\int_{\R\backslash[-L,L]} \max \{1, x^{2t} \} g(x) dx\\
&\leq 8 t^t~\exp(-L^2/2) (L^{4t} + (16t)^t),
\end{align*}
where the last inequality used $(t+1)! \leq 2t^t$ for all $t\geq 1$. Then we can apply \Cref{fact:unit_norm_lower_bound_gen} with $I_1:=\{x\in\R\mid |x|>L\}$ and $I_2:=\R\backslash I_1$, to obtain that for any $L\geq 0$, 
\begin{align}\label{eq:around0}
\E_{x\sim A}[p^2(x)] &\geq \min_{|x| \leq L}\left\{\frac{q(x)}{g(x)}\right\}
\Big(1 - 8\exp\left(-L^2/2\right)\cdot\left(L^{4t}t^t +16^t t^{2t}\right)\Big).
\end{align}
Now set 
\begin{equation}\label{eq:set-L}
L:=10\sqrt{t\ln (t+1)}.
\end{equation}
Since $L>\sqrt{t}$ and $L>1$, we have $L^{4t}t^t + 16^t t^{2t} \leq L^{6t} + 16^tL^{4t} \leq 17^t L^{6t}$, so in \eqref{eq:around0}, 
\begin{equation}\label{eq:L-estimate}
\exp\left(-L^2/2\right)\cdot\left(L^{4t}t^t +16^t t^{2t}\right)\leq \exp\big(\! - \! L^2/2 + (6\ln L + \ln 17)t\big) < \exp(-4),
\end{equation}
where the last inequality follows by plugging in $L=10\sqrt{t\ln(t+1)}$. 
Since $\exp(-4)<1/16$, plug the bound \eqref{eq:L-estimate} into \eqref{eq:around0} and we get the conclusion.
\end{proof}

The next one is an upper bound on $U_A(t)$. 

\genericUA*
\begin{proof}
Recall that $U_A(t):=\max\limits_{i\leq t}\left\vert\E\limits_A\ [h_i(x)]\right\vert$. An application of Cauchy-Schwarz implies 
\[
|\E_{A}[h_i(x)]| \leq \E_{A}[h_i(x)^2]^{1/2}.
\]
Using $\sqrt{a + b} \leq \sqrt a + \sqrt b$ for all $a, b \in \R_+$, we can continue the estimate by 
\[
|\E_{A}[h_i(x)]| \leq \E_{A}[h_i(x)^2 \mathbf{1}(x \in [-B, B])]^{1/2} +\E_{A}[h_i(x)^2 \mathbf{1}(x \notin [-B, B])]^{1/2}.
\]
Applying \Cref{fact:poly_bounds} to $h_i(x)=\sqrt{i!}\hn_i$, we further have 
\begin{align} 
|\E_{A}[h_i(x)]| &\leq 2t^t~\left( \E_{A}[\max \{1, x^{2t} \} \mathbf{1}(x \in [-B, B])]^{1/2} + \E_{A}[\max \{1, x^{2t} \}  \mathbf{1}(x \notin [-B, B])]^{1/2} \right)\nonumber\\
&\leq 2t^t~\left( B^t +  \E_{A}[x^{2t}\mathbf{1}(x \notin [-B, B])]^{1/2} \right)
\label{eq:UAt_split}
\end{align}
where the last step used $B>1$. We now focus on $\E_{A}[x^{2t} \mathbf{1}(x \notin [-B, B])]^{1/2}$:
\begin{align*}
&\E_{A}[x^{2t} \mathbf{1}(x \notin [-B, B])]& \\
\leq & 2 \int_{x \geq B} x^{2t} \exp(-(x-C)^2/2\sigma^2)~dx
&\tag*{(Hypothesis)}\\
= &2 \sigma \int_{\sigma x+C \geq B} (\sigma x+C)^{2t} \exp(-x^2/2)~dx
&\tag*{(Change of variables)}\\
\leq  
& \sigma \cdot (2C)^{2t} + \sigma \cdot (2\sigma)^{2t}\int_{x \geq (B-C)/\sigma} x^{2t} \exp(-x^2/2)~dx 
&\tag*{($ (a+b)^{2t} \leq 2^{2t-1}(a^{2t} + b^{2t})$)}\\
\leq & \sigma \cdot (2C)^{2t} + \sigma \cdot (2\sigma)^{2t} \exp(-(B-C)^2/2\sigma^2) \left((\frac{B-C}{\sigma})^{4t} + (16t)^t\right)
&\tag*{(\Cref{fact:tail_bound})}\\
\leq & 16^t~( \sigma C^{2t} + (B-C)^{4t} / \sigma^{2t-1} + \sigma^{2t+1} t^{t})
&\tag*{($\exp(-|x|)\leq 1$)}\\
\leq & 16^{t}~(\sigma C^{2t} + B^{4t} + \sigma^{2t+1}t^{t}).
&\tag*{($\sigma>1$, $B\geq C > 0$)}
\end{align*}
Taking term-wise square roots, then plugging it into \eqref{eq:UAt_split}, we get
\[ 
U_A(t) \leq (8t)^t \left( \sigma C^t + 2 B^{2t} + \sigma^{t+1} t^{t/2}\right).\qedhere
\]
\end{proof}

Finally, we record the following bounds on $U_A(t)$ and $L_A(t)$, for distributions $A$ that is typically a mixture of a small number of Gaussians with bounded means and variances.

\gaussianlb*

\begin{proof}
    For $U_A(t)$, by the definition of $A$ we have $U_A(t)\leq U_{\cN(\mu,\sigma^2)}(t) + U_{E}(t)$. To each term, we apply \Cref{cor:generic_U_A}: 
    for $\cN(\mu,\sigma)$, we let $C\leftarrow |\mu|$ and $B\leftarrow \max\{1,|\mu|\}$; 
    for $E$, we let $C\leftarrow |\mu'|$, $B\leftarrow B$, and note that its pdf on $\R\backslash[-B,B]$ incurs an additional factor of $2$. As a result, 
    \begin{align*}
    U_A(t) &\leq (8t)^t \cdot \left(\left(\sigma |\mu|^t + 2 \max\{1,|\mu|\}^{2t} + \sigma^{t+1}t^{t/2}\right) + 2\cdot \left(\sigma' |\mu'|^t + 2 B^{2t} + \sigma^{t+1}t^{t/2}\right)\right)\\
    &\leq O\Big(t\cdot(|\mu|+|\mu'|+1)\cdot(\sigma+\sigma'+1)\cdot B\Big)^{2t},\ \ \forall t\geq 1.
    \end{align*}

    For $L_A(t)$, from the assumption we have that $\E\limits_{x\sim A} [p^2(x)]\geq \alpha \E\limits_{x\sim\cN(\mu,\sigma^2)} [p^2(x)]$, so we only need to lower bound the latter. Denote by $q(x)$ the pdf of $\cN(\mu,\sigma^2)$, then for any $L\geq 0$, 
\begin{equation}\label{eq:q-over-g}
\begin{aligned}
\min_{|x| \leq L} \frac{q(x)}{g(x)} &= \frac{1}{\sigma} \exp\left(-(L+|\mu|)^2/2\sigma^2 + L^2/2\right) \\
&\geq \frac{1}{\sigma} \exp\left(-(L^2+\mu^2)/\sigma^2\right).
\end{aligned}
\end{equation}
We choose $L: = 10\sqrt{t\ln (t+1)}$ in \eqref{eq:q-over-g} and then use \Cref{cor:aroundzero}, which gives:  
\[
L_A(t)\geq \frac{\alpha}{2\sigma}\cdot 
\exp\left(-\ \frac{O(t\ln (t+1))+\mu^2}{\sigma^2}\right),\ \ \forall t\geq 1. \qedhere
\]
\end{proof}

\subsection{Bounds on $U_A$ and $L_A$ for Specific Distributions}
\label{app:ulbounds} 
In this section, we prove bounds on $U_A$ and $L_A$ for the specific choices of $A$ from \Cref{sec:applications}. We start with $A_{\text{RME}}$.

\polyperturb*

\begin{proof}
For the upper bound on $U_{A_{\text{RME}-\Id}}(t)$, we apply \Cref{cor:generic_U_A} with $C: = \tau$, $B: = \sqrt{\ln(1/\tau)} - \tau$ and $\sigma: = 1$, which gives that 
\[
U_{A_{\text{RME}-\Id}}(t) \leq (8t)^t(B^t + 2B^{2t} + t^{t/2}) \leq (8t^t)(4B^2t)^t = (32B^2t)^t < (32t\ln\frac{1}{\tau})^t,\ \ \forall t\geq 1.
\] 

For the lower bound on $L_{A_{\text{RME}-\Id}}(t)$ where $t\geq 1$, we apply \Cref{cor:aroundzero} with $L: = 10\sqrt{t\ln (t + 1)}$. This gives $L_{A_{\text{RME}-\Id}}(t) \geq \min\limits_{|x| \leq L} \frac{A(x)}{2g(x)}$. 
We further lower bound $\frac{A(x)}{g(x)}$ for $x\in[-L,L]$ as follows. 
\begin{enumerate}
    \item If $x \notin [-B, B]$, then $\frac{A(x)}{g(x)} 
    = \frac{g(x-\xi)}{g(x)}=\exp(\frac{x^2}{2} - \frac{(x-\xi)^2}{2}) 
    = \exp( x\xi - \frac{\xi^2}{2})$. Since $x\in [-L,L]$, this expression is at least $\exp(- L\xi - \frac{\xi^2}{2}) \geq 
    \exp\left( -10\tau\sqrt{t\ln(t+1)} -1 \right)$. 
    \item If $x \in [-B, B]$, then $\frac{A(x)}{g(x)}
    = \frac{g(x-\xi) + q(x)}{g(x)}
    \geq \sqrt{\frac{\xi}{2\pi}} - 3\xi \sqrt{\ln\frac{1}{\xi}}$ by property \ref{cond:g-q} of $q(x)$ in \Cref{lem:poly_perturb_A}. 
    Here, $\sqrt{\xi/(2\pi)}\geq 2\xi$ if $\xi\leq \frac{1}{8\pi}\approx \frac{1}{25.13}$, so $\xi < \frac{1}{26} \Rightarrow \frac{A(x)}{g(x)}\geq\xi$ for all $x\in[-B,B]$.
\end{enumerate} 
In both cases, if $\tau < \xi <\frac{1}{26}$ then $\min\limits_{|x|\leq L} \frac{A(x)}{2g(x)} \geq \tau\cdot\exp\left(-10\xi\sqrt{t\ln(t+1)}\right)$. The conclusion follows.
\end{proof}

\ulboundsmom*
\begin{proof}
Let $g(x)$ be the pdf of $\cN(0,1)$. Recall that $A(x)=(1-\tau) g(x-\delta)+\tau g(x-\delta') + p(x)\pmb{1}_{[-1,1]}(x)$, where $\tau = O(k-1)^{-(k-1)} \leq \frac{1}{2}$, $\delta=\frac{1}{2000(k-1)}\tau^{1-\frac{1}{k-1}}$ and $\delta'=-\frac{1-\tau}{\tau}\delta = O(1)$. (Although it will be irrelevant for the proof, we also recall that $k$ is a constant fixed beforehand.) 

For $U_{A_\text{RME-$t$-Mom}}$, we apply \Cref{cor:generic_U_A} with $B=C=\max\{1,\delta'\}=O(1)$, $\sigma= 1$. As a result, $U_{A_\text{RME-$t$-Mom}}(i)=O(i)^i$.

For $L_{A_\text{RME-$t$-Mom}}$, by \Cref{cor:aroundzero} we have $L_{A_\text{RME-$t$-Mom}} \geq \min\limits_{|x| \leq L}\frac{A(x)}{2g(x)}$ for $L:=10\sqrt{i\ln(i+1)}\geq 1$. We now lower bound $\frac{A(x)}{g(x)}$ in the range $[-L,L]$: 
\begin{enumerate}
    \item If $x \notin [-1, 1]$, then $\frac{A(x)}{g(x)}=\frac{(1-\tau)g(x-\delta) + \tau g(x-\delta')}{g(x)} = (1-\tau)\exp(x\delta - \frac{\delta^2}{2}) + \tau \exp(x\delta' - \frac{\delta'^2}{2})$. Since $x\in[-L,L]$, this is lower bounded by $\exp(-O(L))$. 
    \item If $x \in [-1, 1]$, then $\frac{A(x)}{g(x)} = \frac{(1-\tau)g(x-\delta) + \tau g(x-\delta') + p(x)}{g(x)} 
    \geq (1-\tau)\exp(x\delta-\frac{\delta^2}{2}) - \frac{0.1}{\sqrt{2\pi}}e$. Here, $(1-\tau)\exp(x\delta-\frac{\delta^2}{2})\geq \frac{1}{2}\exp(-2\delta) \geq \exp(-0.001)/2$, so $\frac{A(x)}{g(x)} > \frac{1}{3}$. 
\end{enumerate}
Together, we have that $L_{A_\text{RME-$t$-Mom}}(i) \geq \exp\left(-O(\sqrt{i\ln(i+1)})\right)$.
\end{proof}

\acovclose*

\begin{proof}
By definition, outside the interval $[-B, B]$ it holds $A_{\text{COV-close}}(x)\leq O(1) \exp(-\frac{x^2}{2(1-\delta)^2})$. So by an application of \Cref{cor:generic_U_A} and using $1\ll B \ll (\log\frac{1}{\delta})^{1/2}$, we have 
\[
U_A(t) \leq O\left((8Bt)^{2t}\right)\leq O(t^2\ln\frac{1}{\delta})^{t},\ \ \forall t\geq 1.
\]

For the lower bound on $L_A(t)$, \Cref{cor:aroundzero} implies that $L_A(t) \geq \frac{1}{2}\min_{|x| \leq L} \frac{A(x)}{g(x)}$, where $L:= 10\sqrt{t \ln (t + 1)}$. We now lower bound $\frac{A(x)}{g(x)}$ for $x\in[-L,L]$. 
\begin{enumerate}
    \item If $x \notin [-B, B]$, then the ratio is $\frac{g(x/(1-\delta)}{(1-\delta)g(x)} \geq \exp(\frac {x^2}{2} 
    (1-\frac{1}{(1-\delta)^2})) = \exp( - O(\delta x^2) )$. In the range $[-L,L]$ this is lower bounded by $\exp\left(-O\big(\delta t\ln(t+1)\big)\right)$.
    \item If $x \in [-B, B]$, then the ratio is $\frac{g(x/(1-\delta))/(1-\delta)-q(x)}{g(x)}$. Recall that $|q(x)|\leq 10\delta k^4 B^{-3}$ for $x\in[-B,B]$ and $\delta<1/4$. 
    By choosing in $k\ll B \ll (\log \frac{1}{\delta})^{1/2}$ the absolute constants to be sufficiently large, we have that for $x\in[-B,B]$, $\frac{g(x/(1-\delta))}{(1-\delta)} \geq 
    \frac{1}{(1-\delta)\sqrt{2\pi}}\exp(\frac{\log\delta}{2}) > 20\delta k^4 B^{-3}$. Consequently, $\frac{A(x)}{g(x)}\geq\frac{1}{2} \frac{g(x/(1-\delta)}{(1-\delta)g(x)}$ for $x\in[-B,B]$, and this is further lower bounded by $\exp(-O(\delta B^2)) = \exp(-O(\delta\ln\frac{1}{\delta}))$ by the same calculus as in the case above.
\end{enumerate}
The lower bound on $L_A(t)$ then follows by taking the minimum of the two, divided by two. 
\end{proof}

\begin{section}{Sketch for deriving an SoS lower bound for Sherrington-Kirkpatrick from approximate parallel pancakes}
\label{app:sos-sherrington-kirkpatrick}
A natural question is whether an SoS lower bound for approximate parallel pancakes can be used to derive an SoS lower bound for Sherrington-Kirkpatrick. There are two reasons why this is desirable. The main reason is that even though there is an SoS lower bound, the exact parallel pancakes problem is actually easy. A second reason is that we only have an SoS lower bound for the exact parallel pancakes when $m \leq n^{\frac{3}{2} - \epsilon}$ whereas our lower bound for approximate parallel pancakes applies when $m \leq n^{2 - \epsilon}$.

While intuitively, the answer is yes, obtaining such a derivation is subtle. In this appendix, we sketch an approach for such a derivation. For this approach, we use the same setup as \cite{ghosh2020sum} which in turn was based on the reduction from the Boolean vector in a random subspace problem to Sherrington-Kirkpatrick by \cite{MRX2020Lifting}. The setup is as follows.
\begin{enumerate}
    \item Observe that the span of the top $d$ eigenvectors of $M$ is a random $d$-dimensional subspace of $\mathbb{R}^n$. This means that we can represent the span of the top $d$ eigenvectors of $M_{SK}$ as the rows of an $d \times n$ matrix $U$ whose entries are independently drawn from $N(0,1)$. For each $i \in [n]$, we take $u_i \in \mathbb{R}^n$ to be the vector with coordinates $(u_i)_j = U_{ij}$.
    \item We take the problem variables to be the coefficients $v_1,\ldots,v_d$ such that $x = \sum_{i=1}^{d}{{v_i}u_i}$. Equivalently, $x_j = \sum_{i=1}^{d}{{v_i}{U_{ij}}}$.
    \item We observe that if $x \in \mathbb{R}^n$ is in the span of the top $d$ eigenvectors of $M$ then $x^{T}M_{SK}x \geq \lambda_{d}||x||^2$ where $\lambda_d$ is the $d$th largest eigenvalue of $M_{SK}$. Moreover, SoS can prove this fact.
\end{enumerate}
To relate this to NGCA, we observe that if we take $e'_j \in \mathbb{R}^d$ to be the $j$th column of $U$ then $x_j = v \cdot e'_j$.

When $d \geq n^{\frac{2}{3} + \epsilon}$ where $\epsilon > 0$ and $n$ is sufficiently large, the SoS lower bound for parallel affine planes \cite{ghosh2020sum} gave pseudo-expectation values for $v$ satisfying the equations that for all $i \in [d]$, $v_i^2 = \frac{1}{d}$ and for all $j \in [n]$, $x_j^2 = (v \cdot e'_j)^2 = 1$. Thus, this gave pseudo-expectation values for $x$ satisfying the equation $x_j^2 = 1$ for all $j \in [n]$ such that $\tilde{E}[{x^T}Mx] \geq {\lambda_d}n \geq (2 - \delta(d,n))n^{\frac{3}{2}}$ where $\delta(d,n)$ is $o(1)$ (as a function of $n$) as long as $d$ is $o(n)$.

When $d \geq n^{\frac{1}{2} + \epsilon}$ where $\epsilon > 0$ and $n$ is sufficiently large, our SoS lower bound for NGCA gives pseudo-expectation values $\tilde{E}$ for $v$ satisfying the equations that for all $i \in [d]$, $v_i^2 = \frac{1}{d}$ and approximately satisfying the equations that for all $j \in [n]$, $x_j^2 = (v \cdot e'_j)^2 = 1$. In particular, if we take $A$ to be a mixture of two Gaussian distributions with standard deviation $\delta$ and means $-1$ and $1$ respectively then $\tilde{E}[(x_j^2 - 1)^{2k}]$ will be $O(\delta^{2k})$ where the constant is a function of $k$.

One way to handle this is as follows:
\begin{enumerate}
    \item Construct variables $x'_j$ which are polynomials in $v$ and the entries of $U$ such that ${x'}_j^2$ is even closer to $1$ than $x_j^2$. For example, if we take $x'_j = \frac{3x_j - x_j^3}{2}$ then if $x_j = \pm{1} + \Delta$ then 
    \[
    x'_j = \frac{3x_j - x_j^3}{2} = \frac{\pm{3} + 3\Delta \mp{1} - 3\Delta \mp{3}\Delta^2 - \Delta^3}{2} = \pm{1} - \frac{\pm{3}\Delta^2 + \Delta^3}{2}
    \]
    \item Observe that $||x' - x||^2$ is small so $|\tilde{E}[{x'}^T{M_{SK}}x'] - \tilde{E}[{x}^T{M_{SK}}x]|$ is small and thus $\tilde{E}[{x'}^T{M_{SK}}x'] \geq (2 - \delta'(d,n))n^{\frac{3}{2}}$ where $\delta'(d,n)$ is $o(1)$ (as a function of $n$) as long as $d$ is $o(n)$ and $\delta$ is $o(1)$.
    \item Letting $M$ be the moment matrix with the variables $\{x'_j: j \in [n]\}$, $M \succeq 0$ as these variables are polynomials of the solution variables $\{v_i: i \in [d]\}$ and the entries of $U$. That said, we need to divide the SoS degree by the degree of these polynomials. 
    \item To satisfy the constraints that ${x'_j}^2 = 1$ exactly, we need to adjust the moment matrix $M$. However, since we are only guaranteed that $M$ is PSD rather than positive definite, there is a danger that this adjustment will break the PSDness of $M$. To handle this, we can modify $M$ to obtain a moment matrix $M'$ which is positive definite rather than just positive semidefinite. Intuitively, we obtain this modification by randomizing each coordinate $x'_j$ with probability $\epsilon$. Note that this reduces the objective function by a factor of roughly $(1 - \epsilon)^2$

    After obtaining the modified moment matrix $M'$, we take our final moment matrix $M''$ so that $M''_{AB} = M'_{(A \setminus (A \cap B))(B \setminus (A \cap B))}$ so that the equations ${x'_j}^2 = 1$ are satisfied.

    We describe and analyze this adjustment in more detail in Section \ref{sec:momentmatrixadjustment} below.
\end{enumerate}
\begin{remark}
If $\delta \leq n^{-C\dsos}$ for a sufficiently large constant $C$ then we can skip the first step. That said, the first step allows us to take $\delta = \frac{1}{n^{c}}$ for any constant $c > 0$. 
\end{remark}
\subsection{Adjusting the moment matrix}\label{sec:momentmatrixadjustment}
\begin{definition}
Given $A,B \subseteq [n]$, the disjoint union $A \Delta B$ of $A$ and $B$ is $A \Delta B = (A \setminus B) \cup (B \setminus A)$.
\end{definition}
\begin{definition}
Given $S \subseteq [n]$ and $k \geq |S|$, we define $\Id_{S,k}$ to be the diagonal matrix such that $(\Id_{S,k})_{AA} = 1$ if $S \subseteq A$ 
and $|A| = k$ and $0$ otherwise. We define $\Id_k = \Id_{\emptyset,k}$.
\end{definition}
\begin{lemma}\label{lem:randomnessgivesslack}
If $M$ is a matrix indexed by sets $A,B \subseteq [n]$ of size at most $\dsos$ such that $M \succeq 0$, the diagonal entries of $M$ are at least $\frac{1}{2}$, and all entries of $M$ have magnitude at most $2$ then if we take $M'$ to be the matrix with entries $M'_{AB} = (1 - \epsilon)^{|A \Delta B|}M_{AB}$ for some $\epsilon > 0$ (which may depend on $\dsos$ and $n$), $M' \succeq \frac{\epsilon^{\dsos}}{8}\sum_{k=0}^{\dsos}{\frac{1}{(32{\dsos}n)^{\dsos - k}}\Id_k}$.
\end{lemma}
\begin{proof}
Let $\alpha = 1 - \epsilon$. Given $S \subseteq [n]$, let $M_S$ be the matrix with entries $(M_S)_{AB} = {\alpha}^{|A \setminus S| + |B \setminus S|}M_{AB}$ if $S \subseteq A \cap B$ and $(M_S)_{AB} = 0$ otherwise. Observe that $M' = \sum_{S \subseteq [n]: |S| \leq \dsos}{(1 - \alpha^2)^{|S|}M_S}$. To see this, observe that for all $S' \subseteq [n]$, $1 = (\alpha^2 + (1 - \alpha^2))^{|S'|} = \sum_{S \subseteq S'}{{\alpha}^{2(|S'| - |S|)}(1 - \alpha^2)^{|S|}}$. Thus, for all $A,B \subseteq [n]$ of size at most $\dsos$, 
\begin{align*}
\sum_{S \subseteq [n]: |S| \leq \dsos}{(1 - \alpha^2)^{|S|}(M_S)_{AB}} &= \sum_{S \subseteq A \cap B}{{\alpha}^{|A \setminus S| + |B \setminus S|}(1 - \alpha^2)^{|S|}M_{AB}} \\
&= \alpha^{|A \Delta B|}\sum_{S \subseteq A \cap B}{{\alpha}^{2(|A \cap B| - |S|)}(1 - \alpha^2)^{|S|}M_{AB}} \\
&= \alpha^{|A \Delta B|}M_{AB} = M'_{AB}
\end{align*}
We need to extract a small multiple of the identity from $\sum_{S \subseteq [n]: |S| \leq \dsos}{(1 - \alpha^2)^{|S|}M_S}$. To do this, we use the following lemma.
\begin{lemma}\label{lem:rescalingtrick}
For all $S \subseteq [n]$ such that $|S| \leq \dsos$, $M_S  \succeq \frac{1}{4}\Id_{S,|S|} - \sum_{j = |S|+1}^{\dsos}{(4n)^{j}\Id_{S,|S| + j}}$.
\end{lemma}
\begin{proof}
We use the following trick. Let $M'_{S}$ be the matrix with entries $(M'_S)_{AB} = \frac{1}{2}(M_{S})_{AB}$ if $A = B = S$, $(M'_S)_{AB} = (M_{S})_{AB}$ if $S \subseteq A \cap B$ and either $A = S$ or $B = S$ but not both, $(M'_S)_{AB} = 2(M_{S})_{AB}$ if $S \subseteq A \cap B$, $A \neq S$, and $B \neq S$, and $(M'_S)_{AB} = 0$ otherwise.

Since $M_S \succeq 0$ and $M'_{S}$ is obtained from $M_S$ by multiplying the rows and columns with index $S$ by $\frac{1}{\sqrt{2}}$ and multiplying the rows and columns with index $A$ where $S \subsetneq A$ by $\sqrt{2}$, $M'_S \succeq 0$. Since $(M_{S})_{SS} \geq \frac{1}{2}$ and the entries of $M_S$ all have magnitude at most $2$, it is not hard to show that $M'_S - M_S - \frac{1}{4}\Id_{S,|S|} \preceq \sum_{j=1}^{\dsos - |S|}{(4n)^{j}\Id_{S,|S| + j}}$. Rearranging and using the fact that $M'_S \succeq 0$, $M_S  \succeq \frac{1}{4}\Id_{S,|S|} - \sum_{j = |S|+1}^{\dsos}{(4n)^{j}\Id_{S,|S| + j}}$, as needed.
\end{proof}
We now observe that 
\begin{align*}
\sum_{S \subseteq [n]: |S| \leq \dsos}{(1 - \alpha^2)^{|S|}M_S} &\succeq \epsilon^{\dsos}\sum_{S \subseteq [n]: |S| \leq \dsos}{\frac{1}{(32{\dsos}n)^{\dsos - |S|}}M_{S}}\\
&\succeq \epsilon^{\dsos}\sum_{S \subseteq [n]: |S| \leq D}{\frac{1}{(32{\dsos}n)^{\dsos - |S|}}\left(\frac{1}{4} - \sum_{S' \subsetneq S}{\left(\frac{4n}{32{\dsos}n}\right)^{|S| - |S'|}}\right)\Id_{S,|S|}}\\
&\succeq \epsilon^{\dsos}\sum_{S \subseteq [n]: |S| \leq D}{\frac{1}{8(32{\dsos}n)^{\dsos - |S|}}\Id_{S,|S|}} = \frac{\epsilon^{\dsos}}{8}\sum_{k=0}^{\dsos}{\frac{1}{(32{\dsos}n)^{\dsos - k}}\Id_k}. \qedhere
\end{align*}
\end{proof}
\begin{corollary}
Under the same setup as Lemma \ref{lem:randomnessgivesslack}, if $M''$ is a matrix indexed by subsets $A,B \subseteq [n]$ of size at most $\dsos$ such that for all $A,B \subseteq [n]$, $|M''_{AB} - M'_{AB}| \leq \frac{1}{16}\left(\frac{\epsilon}{32{\dsos}n}\right)^{D}$ then $M'' \succeq 0$.
\end{corollary}
\end{section}
	
\end{document}